\documentclass[twoside,11pt]{article}

\usepackage{blindtext}
\usepackage{todonotes}
\usepackage[preprint]{jmlr2e}

\usepackage{mdframed}
\usepackage{breakurl}
\usepackage{booktabs} 
\usepackage{threeparttable} 
\usepackage{mathtools}
\usepackage[english]{babel}
\usepackage{xparse}
\usepackage{bbm}
\usepackage[nameinlink,noabbrev]{cleveref}
\usepackage{amsfonts}
\usepackage{verbatim} 

\DeclareMathSymbol{\shortminus}{\mathbin}{AMSa}{"39}

\DeclareMathOperator\loc{loc}
\DeclareMathOperator\Aff{Aff}

\DeclareMathOperator\sgn{sgn}

\DeclareMathOperator\LU{LU}
\DeclareMathOperator\Id{Id}

\DeclareMathOperator\GL{GL}
\DeclareMathOperator\FNN{FNN}
\DeclareMathOperator\op{op}

\DeclareMathOperator\NF{NF}

\DeclareMathOperator\ReLU{ReLU}

\DeclareMathOperator\diag{diag}
\DeclareMathOperator\swish{swish}
\DeclareMathOperator\interior{interior}

\DeclareMathOperator\mon{mon}
\DeclareMathOperator\AEN{AE}

\DeclareMathOperator\ABS{ABS}

\DeclareMathOperator\STEP{STEP}

\DeclareMathOperator\PLCSM{PLCSM}
\DeclareMathOperator\PLC{PLC}
\DeclareMathOperator\FLOOR{FLOOR}

\usepackage{tikz}
\usepackage{adjustbox}
\usetikzlibrary{shapes.geometric, arrows, fit}
\usetikzlibrary{automata,positioning}
\usetikzlibrary{shapes.multipart}

\renewenvironment{proof}[1][Proof]{\par\noindent{\bf #1. }\ignorespaces}{\hfill$\blacksquare$\par}

\usepackage{ulem}

\crefalias{lemma}{theorem}
\crefalias{proposition}{theorem}
\crefalias{corollary}{theorem}
\crefalias{definition}{theorem}
\crefalias{example}{theorem}
\crefalias{remark}{theorem}

\crefname{lemma}{lemma}{lemmas}
\Crefname{lemma}{Lemma}{Lemmas}

\crefname{proposition}{proposition}{propositions}
\Crefname{proposition}{Proposition}{Propositions}

\crefname{corollary}{corollary}{corollaries}
\Crefname{corollary}{Corollary}{Corollaries}

\crefname{definition}{definition}{definitions}
\Crefname{definition}{Definition}{Definitions}

\crefname{example}{example}{examples}
\Crefname{example}{Example}{Examples}

\crefname{remark}{remark}{remarks}
\Crefname{remark}{Remark}{Remarks}

\usepackage{lastpage}
\jmlrheading{25}{2025}{1-\pageref{LastPage}}{3/25}{}{}{Dennis Rochau and Robin Chan and Hanno Gottschalk}

\ShortHeadings{}{On minimal widths for universal approximation with neural networks}
\firstpageno{1}

\begin{document}
\thispagestyle{empty}

\title{New advances in universal approximation with neural networks of minimal width}

\author{\name Dennis Rochau \email rochau@math.tu-berlin.de \\
       \addr Institute of Mathematics\\
       Technische Universität Berlin\\
       10629 Berlin, Straße des 17. Juni 136, Germany
       \AND
       \name Robin Chan \email chan@math.tu-berlin.de \\
       \addr Institute of Mathematics\\
       Technische Universität Berlin\\
       10629 Berlin, Straße des 17. Juni 136, Germany
       \AND
       \name Hanno Gottschalk \email gottschalk@math.tu-berlin.de \\
       \addr Institute of Mathematics\\
       Technische Universität Berlin\\
       10629 Berlin, Straße des 17. Juni 136, Germany}

\editor{}

\maketitle

\begin{abstract}
     We prove several universal approximation results at minimal or near-minimal width for approximation of $L^p(\mathbb{R}^{d_x}, \mathbb{R}^{d_y})$ and $C^0(\mathbb{R}^{d_x}, \mathbb{R}^{d_y})$ on compact sets. Our approach uses a unified coding scheme that yields explicit constructions relying only on standard analytic tools. We show that feedforward neural networks with two leaky ReLU activations $\sigma_\alpha$, $\sigma_{-\alpha}$ achieve the optimal width $\max\{d_x, d_y\}$ for $L^p$ approximation, while a single leaky ReLU $\sigma_\alpha$ achieves width $\max\{2, d_x, d_y\}$, providing an alternative proof of the results of \cite{cai2023achieve}. By generalizing to stepped leaky ReLU activations, we extend these results to uniform approximation of continuous functions while identifying sets of activation functions compatible with gradient-based training. Since our constructions pass through an intermediate dimension of one, they imply that autoencoders with a one-dimensional feature space are universal approximators. We further show that squashable activations combined with FLOOR achieve width $\max\{3, d_x, d_y\}$ for uniform approximation. We also establish a lower bound of $\max\{d_x, d_y\} + 1$ for networks when all activations are continuous and monotone and $d_y \leq 2d_x$. Moreover, we extend our results to invertible LU-decomposable networks, proving distributional universal approximation for LU-Net normalizing flows and providing a constructive proof of the classical theorem of Brenier and Gangbo on $L^p$ approximation by diffeomorphisms.
\end{abstract}

\begin{keywords}
  universal approximation, minimal width, distributional universal approximation, normalizing flows, LU-Net, leaky ReLU, squashable activations
\end{keywords}



\section{Introduction}

Feedforward neural networks (FNNs) are compositions of affine transformations and activation functions applied component-wise to their inputs. Their structure allows them to transform inputs of any dimension into outputs of any desired dimension. Given the success of FNNs across diverse applications involving the approximation of multidimensional functions, a natural question is whether this success reflects strong theoretical approximation capabilities.

Indeed, this turns out to be true: a rich body of mathematical results, known as \textit{universal approximation} (UA) theorems, demonstrates that FNNs with suitable activation functions can approximate key function classes used in practice, specifically continuous and \(L^p\) functions. Mathematically, universal approximation results characterize the closure of a function class corresponding to a neural network architecture with respect to a chosen norm. This is typically either an \(L^p\) norm with \(p\in[1,\infty)\) for \(L^p\) integrable functions, or the supremum norm for continuous functions. Usually, the function classes are restricted to compact subsets of the input space. For instance, if a neural network architecture is an \(L^p\) \textit{universal approximator}, then for any compact set \(\mathcal{K}\), the \(L^p\)-closure of the functions realizable by the architecture contains all \(L^p\) functions on \(\mathcal{K}\). Similarly, we call a function class a \textit{uniform universal approximator} if its closure with respect to the supremum norm contains all continuous functions. If the approximation holds on the entire input space rather than just compact sets, it is referred to as \textit{global universal approximation}.

\medskip
The classical results are the universal approximation theorems of \citet{Cybenko}, \citet{HornikEtAl89}, \citet{LESHNO1993861}, and \citet{Pinkus_1999}. These analyze \textit{shallow} FNNs, which have either a single hidden layer or a uniformly bounded number of hidden layers, but can have arbitrary width. \citet{HornikEtAl89} proved that shallow FNNs with non-decreasing sigmoidal functions universally approximate continuous functions. Independently, \citet{Cybenko} showed that shallow FNNs with continuous sigmoidal activations universally approximate continuous functions and \(L^1\) integrable functions on compact sets. \citet{LESHNO1993861} proved that shallow FNNs with a continuous activation function are universal approximators of continuous functions if and only if the activation is non-polynomial. \citet{Pinkus_1999} further investigated the degree of approximation and the width needed for UA at a given accuracy.

With the rise of deep neural networks, which have achieved strong results in various applications, universal approximation for \textit{narrow} FNNs has become an active research topic. Narrow FNNs have fixed width but arbitrary depth: they can have any number of layers, but the number of neurons per layer is bounded. Here, the \textit{width} of an FNN refers to the maximum number of neurons across all layers. For given activation functions, target function class, and norm, the \textit{minimal width} \(w_{\min}\) is the smallest width such that the corresponding FNNs can universally approximate the target class. By Lemma~1 of \cite{cai2023achieve}, the minimal width is lower bounded by \(\max\{d_x, d_y\}\) for any activation set, where \(d_x\) and \(d_y\) denote the input and output dimensions; we call this the \textit{optimal minimal width} \(w^*_{\min}\). Additionally, for a fixed FNN, we define the \textit{minimal interior dimension} \(d_{\min}\) as the minimum number of neurons across all hidden layers. An FNN with \(d_{\min}\) smaller than both the input and output dimensions can be viewed as an autoencoder: the network encodes inputs into a lower-dimensional representation until reaching the layer with \(d_{\min}\) neurons, then decodes back to the output dimension.

Modern universal approximation results for narrow FNNs not only establish the existence of finite width bounds for universal approximation, but also provide concrete bounds on \(w_{\min}\). To obtain an upper bound, approximations of arbitrary precision must be constructed for every function in the target class. In contrast, proving a lower bound only requires exhibiting a single function that the architecture cannot approximate below that width. Consequently, proving upper bounds for the minimal width is generally more challenging than proving lower bounds.

\medskip
An important connection exists between the minimal width and invertible network architectures. Invertible neural networks define bijective mappings, requiring square weight matrices and thus fixed-width layers throughout the network, including input and output layers. This structure enables generative learning via maximum likelihood training, as implemented in normalizing flows \citep{dinh2017density, kingma2018glow, Kobyzev_2021}. Consequently, universal approximation results for narrow networks, when the bounds are sufficiently tight, can also establish the universal approximation properties of normalizing flows.

A suitable architecture for this purpose is LU-Net \citep{chan2023lunet}, an invertible neural network designed to resemble a "vanilla" FNN as closely as possible. LU-Net directly inherits the universal approximation properties established for narrow networks while retaining the generative capabilities of normalizing flows. The expressiveness of many normalizing flow architectures has been studied extensively, including volume-preserving and coupling-based flows \citep{dinh2017density, teshima2020couplingbased, draxler2025universalityvolumepreservingcouplingbasednormalizing} and those based on neural ordinary differential equations \citep{chen2018neural, teshima2020universalapproximationpropertyneural, zhang2020approximationcapabilitiesneuralodes}; see also \citet{li2020deeplearningdynamicalsystems, ruiz2023neural, marzouk2024distribution}.

\begin{table}[t]
  \centering
  \begin{threeparttable}
  \scalebox{.92}{
  \begin{minipage}{\textwidth}
  \begin{tabular}{llll}
    \toprule
    Reference     & Function class     & Activation   & \(w_{\min}\) \& \(d_{\min}\) \\ 
    \midrule
    \cite{park2020minimum} & \(L^p(\mathbb{R}^{d_x},\mathbb{R}^{d_y})\) & ReLU & \(w_{\min}=\max\{d_x+1,d_y\}\), \(d_{\min}=1\)\\
    & \(L^p(\mathcal{K},\mathbb{R}^{d_y})\) & cont.\,non-poly. & \(w_{\min}\leq \max\{d_x+2,d_y+1\}\), \(d_{\min}=1\) \\
    \midrule
    \cite{cai2023achieve} & \(L^p(\mathcal{K},\mathbb{R}^{d_y})\) & Arbitrary &  \(w_{\min}^{*}:=\max\{d_x,d_y\}\leq w_{\min}\)  \\
    & \(L^p(\mathcal{K},\mathbb{R}^{d_y})\) & LReLU+ABS & \(w_{\min}=w_{\min}^{*}=\max\{d_x,d_y\}\) \\
    & \(L^p(\mathcal{K},\mathbb{R}^{d_y})\) & LReLU & \(w_{\min}=\max\{2,d_x,d_y\}\) \\
    \midrule
    \cite{kim2024minimumwidthuniversalapproximation} & \(L^p(\mathcal{K},\mathbb{R}^{d_y})\) & ReLU-Like$^{\diamond}$ & \(w_{\min}= \max\{2,d_x,d_y\}\), \(d_{\min}=1\) \\
    \midrule
    \cite{shin2025minimumwidthuniversalapproximation} & \(L^p(\mathcal{K},\mathbb{R}^{d_y})\) & Squashable$^{\S}$ & \(w_{\min}= \max\{2,d_x,d_y\}\), \(d_{\min}=1\) \\
    \midrule
    \textbf{Ours} (Theorem \ref{Theorem-Main1}) & \(L^p(\mathcal{K},\mathbb{R}^{d_y})\) & G-LReLU$^{\bigstar}$ & \(w_{\min}= \max\{d_x,d_y\}\), \(d_{\min}=1\) \\
    (Theorem \ref{Theorem-Main1}) & \(L^p(\mathcal{K},\mathbb{R}^{d_y})\) & LReLU & \(w_{\min}= \max\{2,d_x,d_y\}\), \(d_{\min}=1\) \\
    \bottomrule
  \end{tabular}
  \begin{tablenotes}
      \footnotesize
      \item[$^{\diamond}$] ReLU-Like activations as defined by \citet{kim2024minimumwidthuniversalapproximation}; this class is a subset of squashable functions (see Remark~\ref{Remark-Squashable_functions}).
      \item[$^{\S}$] Squashable activation functions as defined by \citet{shin2025minimumwidthuniversalapproximation} (see Definition~\ref{Definition-Squashable_function}), a large class that includes all non-affine analytic functions as well as certain piecewise functions such as LReLU and Hardswish.
      \item[$^{\bigstar}$] LReLU variants as in Definition~\ref{Definition-LReLU}. Here these refer to single-parameter sets (fixed \(\alpha \in (0,1) \cup (1,\infty)\)): LReLU = \(\{\sigma_\alpha\}\), G-LReLU = \(\{\sigma_{-\alpha}, \sigma_\alpha\}\).
  \end{tablenotes}
  \end{minipage}
  }
  \caption{State-of-the-art results on the minimal width and interior dimension of FNNs for $L^p$ universal approximation, \(p\in[1,\infty)\). The optimal minimal width is \(w_{\min}^*:=\max\{d_x,d_y\}\), where \(d_x\) and \(d_y\) denote the input and output dimensions. Here \(\mathcal{K}\subset\mathbb{R}^{d_x}\) is a non-empty compact set, ABS denotes the absolute value, and ``cont.\ non-poly.'' refers to continuous non-polynomial activations. 
  }
  \label{Table-UA_integrable_d_min}
  \end{threeparttable}
\end{table}

    \medskip

    In recent years, several important UA results have been established for narrow FNNs of arbitrary depth. We discuss the development of results for universal approximation of \(L^p(\mathbb{R}^{d_x},\mathbb{R}^{d_y})\) with respect to the \(L^p\) norm in what follows. Table~\ref{Table-UA_integrable_d_min} summarizes the current state-of-the-art. Early work by \citet{Lu} established upper and lower bounds for the minimum width of ReLU FNNs for global universal approximation of functions in \(L^1(\mathbb{R}^{d_x},\mathbb{R}^{d_y})\), while \citet{lyons_kidger} proved an upper bound for \(L^p(\mathbb{R}^{d_x},\mathbb{R}^{d_y})\). \citet{park2020minimum} fully characterized the minimum width for ReLU FNNs, showing \(w_{\min}=\max\{d_x+1,d_y\}\) for global universal approximation of \(L^p(\mathbb{R}^{d_x},\mathbb{R}^{d_y})\), achieving tighter bounds than previous work. For continuous non-polynomial activation functions, they also derived an upper bound \(w_{\min}\leq\max\{d_x+2,d_y+1\}\) for universal approximation of \(L^p(\mathcal{K},\mathbb{R}^{d_y})\) on compact sets \(\mathcal{K}\subset \mathbb{R}^{d_x}\). \citet{cai2023achieve} established a lower bound of \(\max\{d_x,d_y\}\) for the minimal width on compact domains and showed that LReLU activations achieve \(w_{\min}=\max\{2,d_x,d_y\}\), attaining optimal width when \(\max\{d_x,d_y\}\geq 2\). By augmenting LReLU with the absolute value function, they further achieved \(w_{\min}=\max\{d_x,d_y\}\), optimal even in the one-dimensional case. \citet{kim2024minimumwidthuniversalapproximation} showed that ReLU FNNs of width $\max\{2,d_x,d_y\}$ are universal approximators of $L^p(\mathbb{R}^{d_x},\mathbb{R}^{d_y})$ on compact sets, and then generalized this result to show that FNNs with any ReLU-like activation also achieve $w_{\min}=\max\{2,d_x,d_y\}$ in their Theorem 2. Most recently, \citet{shin2025minimumwidthuniversalapproximation} proved that FNNs with squashable activations achieve $w_{\min}=\max\{2,d_x,d_y\}$ for $L^p$ universal approximation on compact sets. Since the class of squashable functions includes all non-affine analytic functions and piecewise functions such as LReLU, this result generalizes Theorem 2 of \citet{kim2024minimumwidthuniversalapproximation} for ReLU-Like activations to a much larger class of activations, though it does not cover ReLU itself.

\begin{table}[ht]
    \centering
    \begin{threeparttable}
    \scalebox{.83}{
    \begin{minipage}{\textwidth}
  \begin{tabular}{llll}
    \toprule
    Reference     & Function class     & Activation   & \(w_{\min}\) \& \(d_{\min}\) \\
    \midrule
    \cite{hanin2018approximating} & \(C(\mathcal{K},\mathbb{R}^{d_y})\) & ReLU & \(d_x+1\leq w_{\min}\leq d_x+d_y\) \\
    \midrule 
    \cite{lyons_kidger} & \(C(\mathcal{K},\mathbb{R}^{d_y})\) & cont.\,non-poly.$^{1}$ & \(w_{\min}\leq d_x+d_y+1 \)\\ 
    & \(C(\mathcal{K},\mathbb{R}^{d_y})\) & non-aff.\,poly. & \(w_{\min}\leq d_x+d_y+2 \)\\
    \midrule
    \cite{park2020minimum} & \(C(\mathcal{K},\mathbb{R}^{d_y})\) & ReLU+STEP & \(w_{\min}=\max\{d_x+1,d_y\}\), \(d_{\min}=1\) \\
    & \(C([0,1],\mathbb{R}^2)\) & ReLU & \(\max\{d_x+1,d_y\}< w_{\min}=3\)\\
    \midrule 
    \cite{cai2023achieve} & \(C(\mathcal{K},\mathbb{R}^{d_y})\) &  Arbitrary & \(w_{\min}^{*}:=\max\{d_x,d_y\}\leq w_{\min}\)\\ 
     & \(C(\mathcal{K},\mathbb{R}^{d_y})\) & ReLU+\(\FLOOR\) & \(w_{\min}=\max\{2,d_x,d_y\}\)\\
    \midrule
    \cite{li2024minimumwidthleakyreluneural} & \(C(\mathcal{K},\mathbb{R})\) & LReLU & \(w_{\min} =d_x+1\)\\ 
    \midrule
    \textbf{Ours} (Theorem~\ref{Theorem-Main_sup}) & \(C(\mathcal{K},\mathbb{R}^{d_y})\) &  SG-LReLU$^{\bigstar}$ & \(w_{\min}=\max\{d_x,d_y\}\), \(d_{\min}=1\)\\
     (Theorem~\ref{Theorem-Main_sup}) & \(C(\mathcal{K},\mathbb{R}^{d_y})\) &  G-LReLU$^{\bigstar}$+FLOOR & \(w_{\min}=\max\{d_x,d_y\}\), \(d_{\min}=1\)\\
     (Theorem~\ref{Theorem-Main_sup})& \(C(\mathcal{K},\mathbb{R}^{d_y})\) &  S-LReLU$^{\bigstar}$ & \( w_{\min}=\max\{2,d_x,d_y\}\), \(d_{\min}=1\)\\
     (Theorem~\ref{Theorem-Main_sup})& \(C(\mathcal{K},\mathbb{R}^{d_y})\) &  LReLU+FLOOR & \( w_{\min}=\max\{2,d_x,d_y\}\), \(d_{\min}=1\)\\
     (Theorem~\ref{Theorem-Squashable})& \(C(\mathcal{K},\mathbb{R}^{d_y})\) &  Squashable$^{\S}$+FLOOR+Id & \( w_{\min}\leq\max\{3,d_x,d_y\}\), \(d_{\min}=1\)\\
    (Theorem~\ref{Theorem-Squashable})& \(C(\mathcal{K},\mathbb{R}^{d_y})\) &  STEP+FLOOR+Id & \(w_{\min}\leq \max\{3,d_x,d_y\}\), \(d_{\min}=1\)\\
     (Theorem~\ref{Theorem-Main5}) & $^{\ddag}$\(C(\mathcal{K},\mathbb{R}^{d_y})\) &  $C_{\mon}^{0}(\mathbb{R},\mathbb{R})^{\dag}$ & \(\max\{d_x,d_y\} +1\leq w_{\min}\)\\
    \bottomrule
  \end{tabular}
    \begin{tablenotes}
        \footnotesize
        \item[$^{1}$] Continuous non-polynomial functions \(f:\mathbb{R}\rightarrow \mathbb{R}\) that are differentiable at least at one point \(z\in \mathbb{R}\) with \(f'(z)\neq 0\).
        \item[$^{\bigstar}$] LReLU variants as in Definition~\ref{Definition-LReLU}. Here these refer to single-parameter sets (fixed \(\alpha \in (0,1) \cup (1,\infty)\)): LReLU = \(\{\sigma_\alpha\}\), G-LReLU = \(\{\sigma_{-\alpha}, \sigma_\alpha\}\), S-LReLU = \(\mathcal{F}_{\alpha,\mathfrak{s}}\), SG-LReLU = \(\mathcal{F}_{\pm\alpha,\mathfrak{s}}\).
        \item[$^{\S}$] Squashable activation functions as defined by \citet{shin2025minimumwidthuniversalapproximation} (see Definition~\ref{Definition-Squashable_function}), a large class that includes all non-affine analytic functions as well as certain piecewise functions such as LReLU and Hardswish.
        \item[$^{\ddag}$] The lower bound holds only for \(d_y\leq 2d_x\) and on compact sets \(\mathcal{K}\subset \mathbb{R}^{d_x}\) with non-empty interior.
        \item[$^{\dag}$] The set \(C^0_{\mon}(\mathbb{R},\mathbb{R})\) consists of all continuous monotone functions. By Remark~\ref{Lemma-Relation_one_to_one_mon}, these are exactly the continuous functions for which a uniformly converging sequence of one-to-one functions exists.
    \end{tablenotes}
    \end{minipage}
    }
    \caption{Summary of known results for the minimal width of FNNs with different activation functions for uniform universal approximation of continuous functions. The minimal interior dimension is included only when non-trivial, i.e., when \(1 \leq d_{\min} < w_{\min}\). Here \(\mathcal{K}\subset \mathbb{R}^{d_x}\) denotes a non-empty compact set.}
    \label{Table-UA_continuous_d_min} 
    \end{threeparttable}
\end{table}

    \medskip
    In this paragraph, we discuss relevant results for uniform universal approximation of continuous functions \(C^0(\mathcal{K},\mathbb{R}^{d_y})\) on compact sets \(\mathcal{K}\subset \mathbb{R}^{d_x}\). Table~\ref{Table-UA_continuous_d_min} summarizes the state-of-the-art results. \citet{hanin2018approximating} derived that \(d_x+1\leq w_{\min}\leq d_x+d_y\) for uniform universal approximation with ReLU FNNs. \citet{lyons_kidger} showed that FNNs with continuous non-polynomial activations achieve \(w_{\min}\leq d_x+d_y+1\), while non-affine polynomial activations achieve \(w_{\min} \leq d_x+d_y+2\). \citet{park2020minimum} showed that FNNs with ReLU+STEP activations achieve \(w_{\min}=\max\{d_x+1,d_y\}\). \citet{cai2023achieve} established that FNNs with ReLU+FLOOR activations achieve $w_{\min}=\max\{2,d_x,d_y\}$, which is optimal when $\max\{d_x,d_y\}\geq 2$. Moreover, \citet{li2024minimumwidthleakyreluneural} showed that LReLU FNNs of width $d_x+1$ can universally approximate $C^0(\mathbb{R}^{d_x},\mathbb{R})$ on compact sets. 
    
    \citet{johnson2018deep} derived a lower bound of \(d_x+1\) on the minimal width for uniformly continuous activations that can be uniformly approximated by one-to-one (injective and continuous) functions when \(d_y=1\), which generalizes directly to \(d_y \leq d_x\). \citet{kim2024minimumwidthuniversalapproximation} extended this to the case \(d_x < d_y \leq 2d_x\), showing the minimal width is lower bounded by \(d_y+1\) for continuous activations that can be uniformly approximated by one-to-one functions, relaxing Johnson's uniformly continuous requirement to merely continuous.

    \paragraph{Our main contributions} are divided into four parts. In the first part, we prove Theorem~\ref{Theorem-Main1}, which establishes upper bounds of \(\max\{2,d_x,d_y\}\) and \(\max\{d_x,d_y\}\) for universal approximation of \(L^p(\mathcal{K},\mathbb{R}^{d_y})\) on arbitrary compact sets \(\mathcal{K} \subset \mathbb{R}^{d_x}\) with FNNs equipped with LReLUs or generalized LReLUs, combining LReLUs with positive and negative parameters, respectively. This theorem can be seen as an alternative proof of results of \cite{cai2023achieve}, however, our proof relies only on elementary analysis and explicitly constructs the approximation sequences using the coding scheme of \citet{park2020minimum}.
    This explicit construction also establishes that the approximators can be realized with minimal interior dimension \(d_{\min}=1\). Building on this foundation, in Theorem \ref{Theorem-Main_sup} we show that by adapting LReLU activations with suitable discontinuities, either by including the FLOOR activation or by adding a step at zero, we can achieve minimal widths of \(\max\{d_x,d_y\}\) or \(\max\{2,d_x,d_y\}\) for uniform universal approximation of \(C^0(\mathbb{R}^{d_x},\mathbb{R}^{d_y})\) on compact sets, depending on whether generalized or standard LReLUs are used. Furthermore, in Theorem \ref{Theorem-Squashable} we prove that FNNs with squashable+FLOOR+\(\Id\) activations achieve minimal width \(\max\{3,d_x,d_y\}\) for uniform universal approximation of continuous functions by exploiting the defining properties of squashable functions to construct STEP functions precisely, then summing suitable STEP functions. The same bound holds for STEP+FLOOR+\(\Id\) FNNs.

\medskip
    In the second part, we generalize the strong universal approximation results of Theorem~\ref{Theorem-Main1} and Theorem 1 of \citet{shin2025minimumwidthuniversalapproximation} to show that LU-decomposable neural networks with invertible leaky ReLU and invertible squashable activations can universally approximate \(L^p(\mathcal{K},\mathbb{R}^{d})\) on arbitrary compact sets \(\mathcal{K} \subset \mathbb{R}^d\). 
    Building on these results, we prove that the normalizing flow LU-Net \citep{chan2023lunet} can transform any absolutely continuous source distribution into a sequence of distributions that converges in law to any predefined target distribution when using LReLU or invertible squashable activations. These results are stated in Theorems~\ref{Theorem-Main3} and~\ref{Theorem-DUAP_squashable}. We complement these theoretical results with numerical experiments demonstrating that the theoretically optimal LReLU activation with a single parameter does not necessarily yield the best numerical performance, and that sets of LReLUs with continuous parameters can improve performance.

\medskip
    In the third part, we prove Theorem~\ref{Theorem-Main4}, which establishes that the set of smooth diffeomorphisms in the form of smoothed LU-decomposable LReLU neural networks, where the LReLU is smoothed with suitable mollifiers, are universal approximators of \(L^p(\mathcal{K},\mathbb{R}^d)\) on any compact set \(\mathcal{K}\subset \mathbb{R}^d\).
    Consequently, Theorem~\ref{Theorem-Main4} immediately implies the well-known Theorem 2.5 (i) of \citet{Brenier_Gangbo}, which establishes that diffeomorphisms are universal approximators of \(L^p(\mathcal{K},\mathbb{R}^d)\) on compact sets. This demonstrates that neural networks are not merely objects of practical interest, but can also serve as theoretical tools to prove complex mathematical results.

\medskip
    In the fourth part, in Theorem~\ref{Theorem-Main5} we establish that \(\max\{d_x,d_y\}+1\) is a lower bound for uniform universal approximation of continuous functions with FNNs that employ monotone continuous activation functions when \(d_y\leq 2d_x\).
    This result strengthens and unifies Theorem 1 of \citet{johnson2018deep} and Theorem 3 of \citet{kim2024minimumwidthuniversalapproximation}. For a nuanced discussion of how our proof relates to these prior works, see Remark~\ref{Remark-Johnson}. An important consequence is that FNNs with commonly used monotone activation functions \citep{dubey2022activationfunctionsdeeplearning,Activation_functions}, such as ReLU, LReLU, sigmoid, hyperbolic tangent, and logistic sigmoid, cannot achieve uniform universal approximation of continuous functions if their width is restricted to \(\max\{d_x,d_y\}\), when \(d_y\leq 2d_x\).

\medskip
        To illustrate the core of ideas and make them more attainable to the reader, we have implemented our LReLU coding scheme constructions in 1D, i.e. target functions \(f:[0,1] \rightarrow [0,1]\). It includes Jupyter notebooks that guide interested readers through various approximation concepts and can be readily adapted 
        to visualize or evaluate these ideas for different functions, or to explore the effect of varying accuracy parameters. Moreover, we conduct numerical 2D experiments with LU-Net, a generative model based on an invertible architecture, to demonstrate the practical impact of the improved capacity presented in this paper.
        The implementations are publicly available at \url{https://github.com/DennisRTUB/universal-approximation-minimal-width.git}.

\paragraph{Structure of the paper}
    Section \ref{Section-Notation_and_main_results} is split into four subchapters corresponding to the contributions listed above. Each subchapter introduces the relevant notation for one of the associated theorems, explains its significance, and presents key implications. Section \ref{Section-coding_scheme} defines and motivates the coding scheme of \citet{park2020minimum}, which is essential for the constructions in our Theorems \ref{Theorem-Main1}, \ref{Theorem-Main_sup}, and \ref{Theorem-Squashable}, and shows that it can uniformly approximate continuous functions on the unit cube. Section \ref{Section-Piecewise_linear} analyzes the capabilities of LReLUs for approximating piecewise linear and zig-zag functions and describes how width-one FNNs using LReLUs with one or two parameters can uniformly approximate LReLUs with arbitrary parameters.
    
    Section \ref{Section-Approximating_coding_scheme} provides several preliminary results showing how leaky ReLU FNNs can approximate the components of the coding scheme. The section is organized according to these components: we first construct the encoder, then the memorizer, and finally the decoder. This is our longest and most detailed section, as our theorems rely on several distinct coding-scheme constructions that must be established rigorously.
    
    In Section \ref{Section-Proof_Main1}, we combine the results of previous sections \ref{Section-coding_scheme}--\ref{Section-Approximating_coding_scheme} 
    to give detailed proofs of Theorems \ref{Theorem-Main1} and \ref{Theorem-Main_sup}. Section \ref{Section-Squash_floor_coding_scheme} then shows that width-one squashable+FLOOR FNNs can be used to exactly construct STEP functions on compact sets and uses this to construct the memorizer exactly, while the encoder and decoder are implemented directly with FLOOR+$\Id$ networks. In Section \ref{Section-Proof_Theorem_Squashable}, we use these constructions to establish that squashable+FLOOR+$\Id$ and STEP+FLOOR+$\Id$ FNNs can uniformly approximate the continuous functions $C^0(\mathbb{R}^{d_x},\mathbb{R}^{d_y})$ on compact sets with width $\max\{3,d_x,d_y\}$.
    
    Section \ref{Section-Proof_main2} proves Theorems \ref{Theorem-Main2} and \ref{Theorem-LU_Squashable}, which extend universal approximation results for $L^p$-integrable functions with LReLU and squashable activations---given by our Theorem \ref{Theorem-Main1} and Theorem~1 of \cite{shin2025minimumwidthuniversalapproximation}---from FNNs to LU-decomposable networks. We complement this theoretical analysis with numerical experiments in Section~\ref{Section-LU_Net_numerical_results}, comparing LU-Net with a single fixed LReLU parameter against LU-Net with learnable LReLU parameters across neurons. The experiments demonstrate that allowing continuous LReLU parameters can significantly improve performance and convergence, despite the single-parameter case being sufficient for universal approximation.
    
    We then apply these results in Section \ref{Section-Proof_main3} to show that the normalizing flow LU-Net has the distributional universal approximation property with suitable LReLU or invertible squashable activations. Section \ref{Section-Proof_main4} further refines these results by smoothing the LReLU activations with mollifiers and generalizing Theorem \ref{Theorem-Main2} to obtain a class of smooth diffeomorphisms realized by smoothed LU-decomposable LReLU networks that universally approximate $L^p(\mathbb{R}^{d},\mathbb{R}^d)$. 
    
    Finally, Section \ref{Section-Proof_main5} provides a formal proof of Theorem \ref{Theorem-Main5}, showing that the minimal width for uniform universal approximation of the continuous functions $C^0(\mathbb{R}^{d_x},\mathbb{R}^{d_y})$ using continuous and monotone activations must be at least $\max\{d_x,d_y\}+1$ whenever $d_y \leq 2d_x$. A comprehensive conclusion and outlook on universal approximation for narrow neural networks and distributional universal approximation with LU-Net is given in Section \ref{Section-Conclusion_and_outlook}.

\section{Notation and main results}\label{Section-Notation_and_main_results}

    \subsection{Leaky ReLU neural networks as universal approximators}

    \begin{definition}
        We use the notation \(\mathbb{N}:=\{1,2,\hdots\}\) and \(\mathbb{N}_0:=\{0\}\cup \mathbb{N}\). Moreover, let \(k,n\in \mathbb{N}_{0}\) with \(k\leq n\), then \([k,n]:=\{k,...,n\}\) and \([n]:=[1,n]:=\{1,..,n\}\). Note that we use this exclusively for denoting ranges of indices and we still use the notation \([a,b]\) for intervals from \(a\in \mathbb{R}\) to \(b\in \mathbb{R}\). Thus it should be clear by the context, which of both is meant.
    \end{definition}

    \begin{definition}\label{Definition-Affine_transformations}
        We define the set of affine transformation from \(\mathbb{R}^{d_x}\) to \(\mathbb{R}^{d_y}\) as \(\Aff(d_x,d_y)=\{W:\mathbb{R}^{d_x}\rightarrow \mathbb{R}^{d_y}, W(x)=Ax+b \,|A\in \mathbb{R}^{d_y\times d_x},b\in \mathbb{R}^{d_y} \}\) and for \(d_x=d_y=d\) we define the shorthand \(\Aff(d):=\Aff(d,d)\). Moreover, we define the set of \(d\)-dimensional invertible affine transformations as \(\Aff_{\GL}(d)=\{W:\mathbb{R}^{d}\rightarrow \mathbb{R}^{d}, W(x)=Ax+b \,|A\in \GL(d),b\in \mathbb{R}^{d} \}\), where \(\GL(d)\) denotes the set of invertible $d\times d$ matrices.
    \end{definition}

    Now, we give a detailed definition of FNNs of different widths and introduce our neural network notation for this paper. 

    \begin{definition}\label{Definition-FNN}
        A feedforward neural network (FNN) with activations in a set \(\mathcal{A}\) is a function of the form
        \begin{align*}
            \phi = W_n \circ \sigma_{n-1} \circ W_{n-1} \circ \cdots \circ \sigma_1 \circ W_1\;,
        \end{align*}
        where \(n \in \mathbb{N}\) is the \textit{depth}, \(W_i \in \Aff(d_{i-1}, d_i)\) for \(i \in [n]\) with \(d_j \in \mathbb{N}\) for \(j \in [0,n]\), and the \(i\)-th activation layer is \(\sigma_i := (\sigma_1^{(i)}, \ldots, \sigma_{d_i}^{(i)})^T\) for \(i \in [n-1]\) with \(\sigma_k^{(i)} \in \mathcal{A}\) for \(k \in [d_i]\). Here \(d_x := d_0\) is the \textit{input dimension} and \(d_y := d_n\) the \textit{output dimension}, so \(\phi: \mathbb{R}^{d_x} \to \mathbb{R}^{d_y}\). When \(\sigma_i: \mathbb{R} \to \mathbb{R}\), we interpret it as acting component-wise. For small activation sets, e.g.\ \(\{\sigma_1, \sigma_2\}\), we refer to such networks as \(\sigma_1 + \sigma_2\) FNNs.
        
        \medskip
        The \textit{width} of \(\phi\) is \(w_{\max}^\phi := \max\{d_i \mid i \in [0,n]\} \geq \max\{d_x, d_y\}\), where the inequality follows from Lemma~1 of \cite{cai2023achieve}. The \textit{minimal interior dimension} is \(d_{\min}^\phi := \min\{d_i \mid i \in [n-1]\}\) if \(n > 1\), and \(d_{\min}^\phi := \min\{d_x, d_y\}\) if \(n = 1\). For a set \(\mathcal{M}\) of FNNs, we define \(w_{\max}^{\mathcal{M}} := \sup_{\phi \in \mathcal{M}} w_{\max}^\phi\) and \(d_{\min}^{\mathcal{M}} := \sup_{\phi \in \mathcal{M}} d_{\min}^\phi\).
        
        \medskip
        We denote by \(\FNN_{\mathcal{A}}^{(w)}(d_x, d_y; n)\) the set of FNNs with input dimension \(d_x\), output dimension \(d_y\), activations in \(\mathcal{A}\), width at most \(w\), and depth \(n\). The set of all such FNNs of arbitrary depth is \(\FNN_{\mathcal{A}}^{(w)}(d_x, d_y) := \bigcup_{n \in \mathbb{N}} \FNN_{\mathcal{A}}^{(w)}(d_x, d_y; n)\). We write \(\FNN_{\mathcal{A}}^{(w,k)}(d_x, d_y)\) for the subset with \(d_{\min}^\phi = k\).
    \end{definition}

        Autoencoders are a special type of FNN that use an encoder to transform a high-dimensional input into a lower-dimensional feature representation. A decoder then maps this feature back to a high-dimensional output that should approximate a target function applied to the original input. Since encoders and decoders are typically FNNs, we formally define autoencoders as FNNs with \(d = d_x = d_y > 1\) and minimal interior dimension less than \(d\).

    \begin{definition}\label{Definition-Autoencoders}
        Let the input and output dimension be given by \(d = d_x = d_y\) with \(d \geq 2\). We define the set of autoencoders of dimension \(d\) with feature dimension \(d_{\min} < d\) and activations in \(\mathcal{A} \subset \{f \mid f:\mathbb{R} \to \mathbb{R}\}\) as 
        \begin{align*}
            \AEN_{\mathcal{A}}^{(d_{\min})}(d) := \FNN_{\mathcal{A}}^{(d, d_{\min})}(d, d)\;.
        \end{align*}
        Thus any autoencoder \(\phi \in \AEN_{\mathcal{A}}^{(d_{\min})}(d)\) is a narrow neural network from \(\mathbb{R}^d\) to \(\mathbb{R}^d\) with width bounded by \(d\) and minimal interior dimension \(d_{\min}^\phi = d_{\min}\).
    \end{definition}
    
    The next definition gives the formal definition of \(L^p\) and uniform universal approximators.
    
    \begin{definition}\label{Definition-Universality_Lp_sup}
        For \(p \in [1,\infty)\) and \(d_x, d_y \in \mathbb{N}\), let \(\mathcal{M}\) and \(\mathcal{G}\) be sets of measurable functions from \(\mathbb{R}^{d_x}\) to \(\mathbb{R}^{d_y}\), with \(\mathcal{G} \subset L^p(\mathbb{R}^{d_x}, \mathbb{R}^{d_y})\). For \(\mathcal{X} \subset \mathbb{R}^{d_x}\) and a bounded function \(f: \mathcal{X} \to \mathbb{R}^{d_y}\), we define the supremum norm by
        \[
            \lVert f \rVert_{\mathcal{X},\sup} := \sup_{x \in \mathcal{X}} \lVert f(x) \rVert_{\max} := \sup_{x \in \mathcal{X}} \max_{i \in [d_y]} |f_i(x)|\;,
        \]
        and write \(\lVert f \rVert_{\sup} := \lVert f \rVert_{\mathbb{R}^{d_x},\sup}\). For \(f \in L^p(\mathcal{X}, \mathbb{R}^{d_y})\), we define the \(L^p\) norm by
        \[
            \lVert f \rVert_{\mathcal{X},p} := \left( \int_{\mathcal{X}} \lVert f(x) \rVert_p^p \, dx \right)^{1/p}, \quad \text{where} \quad \lVert x \rVert_p := \left( \sum_{i=1}^{d_y} |x_i|^p \right)^{1/p} \text{ for } x \in \mathbb{R}^{d_y}\;,
        \]
        and write \(\lVert f \rVert_p := \lVert f \rVert_{\mathbb{R}^{d_x},p}\).
        
        \medskip
        We call \(\mathcal{M}\) an \(L^p\) \textit{universal approximator} of \(\mathcal{G}\) on compact sets if for every \(g \in \mathcal{G}\), every compact \(\mathcal{K} \subset \mathbb{R}^{d_x}\), and every \(\epsilon > 0\), there exists \(f \in \mathcal{M}\) such that
        \[
            \lVert f - g \rVert_{\mathcal{K},p} = \left( \int_{\mathcal{K}} \lVert f(x) - g(x) \rVert_p^p \, dx \right)^{1/p} < \epsilon\;.
        \]
        If the same holds for the supremum norm, we call \(\mathcal{M}\) a \textit{uniform universal approximator} of \(\mathcal{G}\) on compact sets.
        
        \medskip
        When we say \(\mathcal{M}\) is an \(L^p\) universal approximator without specifying \(\mathcal{G}\), we mean universal approximation of \(\mathcal{G} = L^p(\mathbb{R}^{d_x}, \mathbb{R}^{d_y})\) on compact sets. Similarly, if we say \(\mathcal{M}\) is a uniform universal approximator without specifying \(\mathcal{G}\), we mean universal approximation of \(\mathcal{G} = C^0(\mathbb{R}^{d_x}, \mathbb{R}^{d_y})\) on compact sets. Throughout this paper, universal approximation always refers to approximation on compact sets unless explicitly stated otherwise.
    \end{definition}

        Having introduced the terms minimal width and minimal interior dimension for universal approximation in the introduction, we now give their formal mathematical definition.

    \begin{definition}\label{Definition-w_d_min}
        Let \(\mathcal{M}\) be a set of FNNs that universally approximates \(\mathcal{G}\) with respect to a given norm \(\lVert \cdot \rVert\), which is the \(L^p\) norm in the case of \(\mathcal{G} = L^p(\mathbb{R}^{d_x}, \mathbb{R}^{d_y})\) and the supremum norm in the case of \(\mathcal{G} = C^0(\mathbb{R}^{d_x}, \mathbb{R}^{d_y})\). We define the \textit{minimal width} of \(\mathcal{M}\) for universal approximation of \(\mathcal{G}\) as
        \begin{align*}
            w_{\min}(\mathcal{M}, \mathcal{G}, \lVert \cdot \rVert) := \inf\{w_{\max}^{\mathcal{U}} \mid \mathcal{U} \subset \mathcal{M} \text{ is a universal approximator of } \mathcal{G} \text{ on compact sets}\}\;,
        \end{align*}
        and the \textit{minimal interior dimension} of \(\mathcal{M}\) for universal approximation of \(\mathcal{G}\) as
        \begin{align*}
            d_{\min}(\mathcal{M}, \mathcal{G}, \lVert \cdot \rVert) := \inf\{d_{\min}^{\mathcal{U}} \mid \mathcal{U} \subset \mathcal{M} \text{ is a universal approximator of } \mathcal{G} \text{ on compact sets}\}\;.
        \end{align*}
        Since it is usually clear from context what \(\mathcal{M}\), \(\mathcal{G}\), and \(\lVert \cdot \rVert\) are, we often write simply \(w_{\min}\) and \(d_{\min}\).
    \end{definition}
    
    We now introduce the notation for classes of leaky ReLUs used throughout the paper.

    \begin{definition}\label{Definition-LReLU}
        For \(\alpha \in \mathbb{R} \setminus \{-1, 0, 1\}\), the corresponding \textit{leaky rectified linear unit} (LReLU) is defined by
        \begin{align*}
            \sigma_\alpha: \mathbb{R} \to \mathbb{R}, \quad \sigma_\alpha(x) = \begin{cases}
                \alpha x & x < 0\\
                x & x \geq 0
            \end{cases}\;.
        \end{align*}
        For most of our results, we consider LReLU parameters in \((-1,1) \setminus \{0\}\). We define the following sets:
        \begin{itemize}
            \item \(\mathcal{F}_+ := \{\sigma_\alpha \mid \alpha \in (0,1)\}\), the set of \textit{invertible LReLUs},
            \item \(\mathcal{F}_- := \{\sigma_{-\beta} \mid \beta \in (0,1)\}\), and
            \item \(\mathcal{F}_\pm := \mathcal{F}_- \cup \mathcal{F}_+\), the set of \textit{generalized LReLUs} (G-LReLUs).
        \end{itemize}
        
        For \(\alpha \in [0,\infty)\) and \(s \in \mathbb{R}\), we define the \textit{stepped leaky ReLU} (S-LReLU) by
        \begin{align*}
            \rho_{\alpha,s}: \mathbb{R} \to \mathbb{R}, \quad \rho_{\alpha,s}(x) = \begin{cases}
                \alpha x & x < 0 \\
                x + s & x \geq 0
            \end{cases}\;.
        \end{align*}
        We define the following sets of S-LReLUs:
        \begin{itemize}
            \item \(\mathcal{F}_{+,\mathfrak{s}} := \{\rho_{\alpha,s} \mid \alpha \in (0,1),\, s \in [0,1]\}\), the set of S-LReLUs,
            \item \(\mathcal{F}_{\pm,\mathfrak{s}} := \mathcal{F}_{+,\mathfrak{s}} \cup \mathcal{F}_-\), the set of \textit{stepped generalized LReLUs} (SG-LReLUs).
        \end{itemize}
        For a fixed \(\alpha \in (0,1) \cup (1,\infty)\), we define the single-parameter variants:
        \begin{itemize}
            \item \(\mathcal{F}_{\alpha,\mathfrak{s}} := \{\rho_{\alpha,s} \mid s \in [0,1]\}\),
            \item \(\mathcal{F}_{\pm\alpha,\mathfrak{s}} := \{\sigma_{-\alpha}\} \cup \mathcal{F}_{\alpha,\mathfrak{s}}\).
        \end{itemize}
        Finally, we define sets of S-LReLUs with step size restricted to \(\{0,1\}\). These sets require an expanded parameter range \(\alpha \in (0,\infty)\) to enable the construction of arbitrary step sizes from fixed ones (see Corollary~\ref{Corollary-UAP_sup_fixed_stepsize}):
        \begin{itemize}
            \item \(\mathcal{F}_{+,1}^{*} := \{\rho_{\alpha,s} \mid \alpha \in (0,\infty),\, s \in \{0,1\}\}\),
            \item \(\mathcal{F}_{\pm,1}^{*} := \mathcal{F}_{+,1}^{*} \cup \mathcal{F}_-\).
        \end{itemize}
        Note that \(\mathcal{F}_+ \subset \mathcal{F}_{+,1}^{*}\) and \(\mathcal{F}_\pm \subset \mathcal{F}_{\pm,1}^{*}\).
    \end{definition}
    
        The LReLU generalizes ReLU by introducing a non-zero slope for negative inputs. Since LReLUs with \(\alpha \neq 0\) are bijective, they can be used to construct invertible neural networks (INNs), which we later use to extend universal approximation results to diffeomorphisms and LU-Net \citep{chan2023lunet}.
    \medskip
    \begin{definition}\label{Definition-C_0_mon}
        Let \(C_{\mon}^{0}(\mathbb{R},\mathbb{R})\) be the set of continuous monotone functions from \(\mathbb{R}\) to \(\mathbb{R}\), where monotone means either monotonically increasing or monotonically decreasing. A function is \textit{strictly monotone} if it is either strictly monotonically increasing or strictly monotonically decreasing.
    \end{definition}
    
    The following lemma shows that width-one FNNs with LReLU activations of fixed parameters can uniformly approximate LReLUs with arbitrary parameters, and that the approximating sequences can be chosen with even depth for all positive LReLU parameters. This allows us to first establish our main results for variable LReLU parameters and then extend them to FNNs that use only one or two LReLUs with fixed parameters. The proof of Lemma~\ref{Lemma-Single_Parameter_LReLU_approx_mon} is given at the end of Section~\ref{Section-Piecewise_linear}.
    
    \begin{lemma}\label{Lemma-Single_Parameter_LReLU_approx_mon}
        Let \(\alpha \in (0,1) \cup (1,\infty)\), \(\sigma_{\alpha}\) the corresponding LReLU, \(f \in C_{\mon}^{0}(\mathbb{R},\mathbb{R})\), and \(\mathcal{K} \subset \mathbb{R}\) compact. Then there exists \((k_n)_{n \in \mathbb{N}} \subset \mathbb{N}\) with \(\lim_{n \to \infty} k_n = \infty\) and \((\phi_n)_{n \in \mathbb{N}} \subset \FNN_{\{\sigma_{\alpha}\}}^{(1)}(1,1;2k_n)\) such that \,\(\lim_{n \to \infty} \lVert \phi_n - f \rVert_{\mathcal{K},\sup} = 0\). Moreover, for any \(\alpha, \beta \in (0,1) \cup (1,\infty)\) and \(\mathcal{K} \subset \mathbb{R}\) compact, there exists \((\phi_n)_{n \in \mathbb{N}} \subset \FNN_{\{\sigma_{\shortminus\alpha},\sigma_{\alpha}\}}^{(1)}(1,1)\) such that \,\(\lim_{n \to \infty} \lVert \phi_n - \sigma_{-\beta} \rVert_{\mathcal{K},\sup} = 0\).
    \end{lemma}
    
    Our main proof strategy is to first carry out the constructions for sets containing infinitely many LReLU activation functions and then generalize the results to sets consisting of only one or two LReLUs with a fixed parameter by applying Lemma~\ref{Lemma-Single_Parameter_LReLU_approx_mon}.
    
    In our first main result, we show that FNNs with invertible LReLU activations universally approximate \(L^p(\mathcal{K}, \mathbb{R}^{d_y})\) on compact sets \(\mathcal{K} \subset \mathbb{R}^{d_x}\).

     \begin{figure}[t]
        \includegraphics[width=0.5\textwidth]{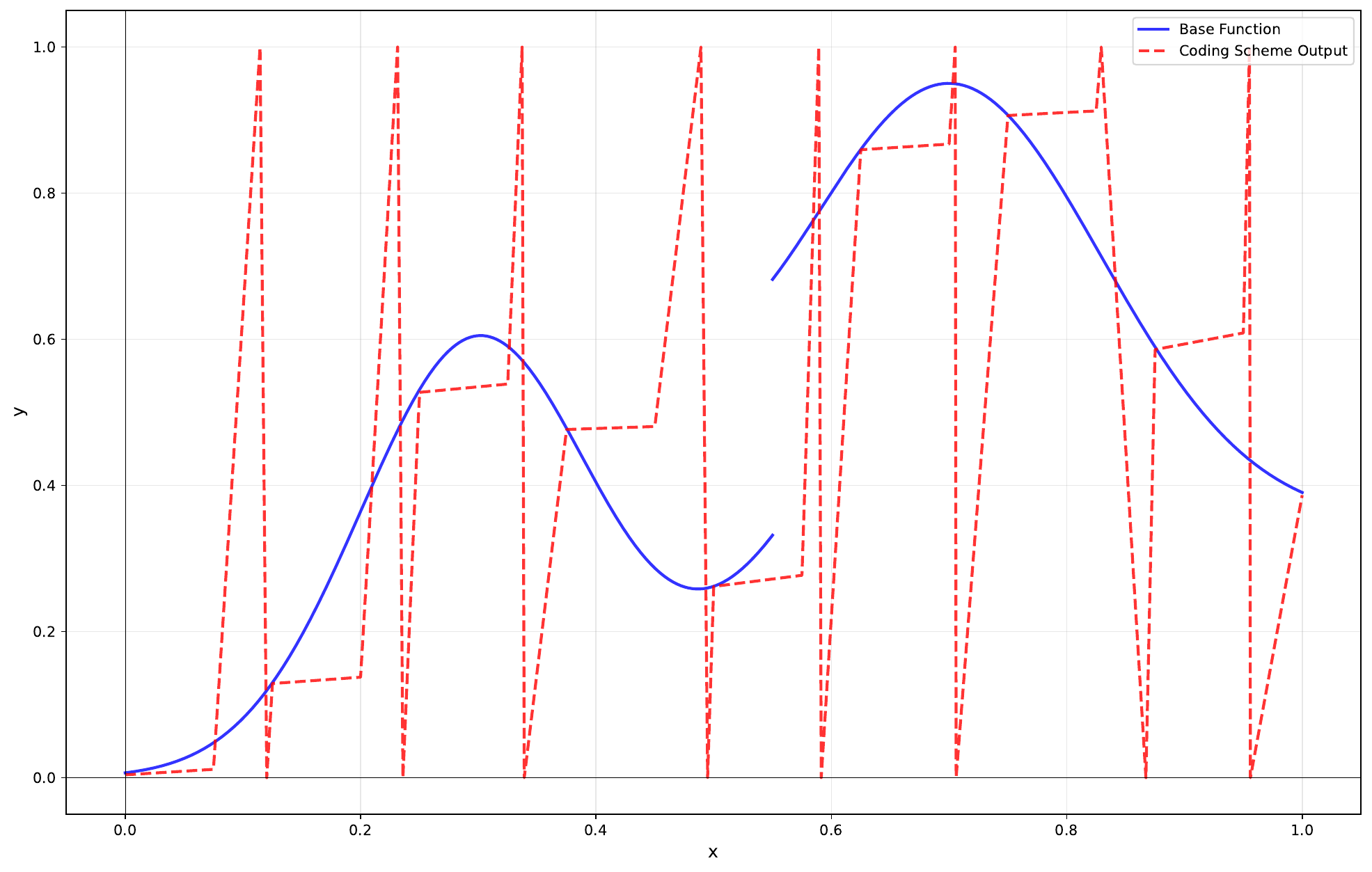}
        \hfill
        \includegraphics[width=0.5\textwidth]{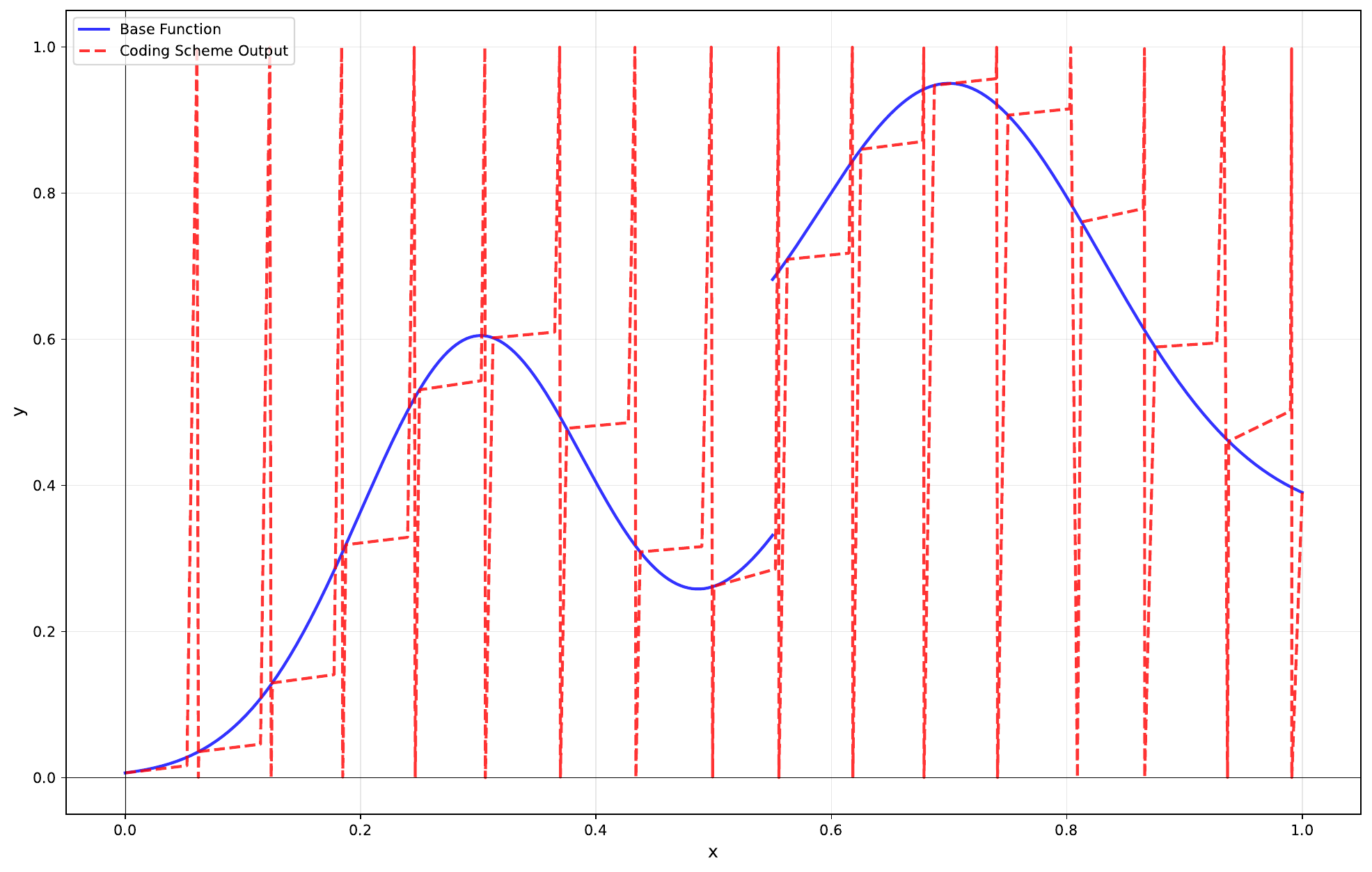}
        \caption{Coding scheme approximations from Theorem~\ref{Theorem-Main1} for \(f(x) = 0.6e^{-50(x-0.3)^2} + 0.6e^{-30(x-0.7)^2} + 0.35\cdot\mathbbm{1}_{[0.55,\infty)}(x)\) using G-LReLU activations from \(\mathcal{F}_{\pm}\) with a FNN of width \(1\). Left: \(K=3\), \(M=8\), \(\gamma=0.05\). Right: \(K=4\), \(M=14\), \(\gamma=0.01\). Here \(K, M\) are the coding scheme parameters (Definition~\ref{Definition-Coding_scheme}) and \(\gamma\) is the width of the exceptional intervals of the quantizer approximation with large error magnitudes (Lemma~\ref{Lemma-Approximate_q_k}).}
        \label{Figure-Main_theorem_lp_1}
    \end{figure}

     \begin{theorem}\label{Theorem-Main1}
        Let \(d_x, d_y \in \mathbb{N}\), \(\alpha \in (0,1) \cup (1,\infty)\), and let \(\sigma_{\alpha}\), \(\sigma_{\shortminus\alpha}\) denote the corresponding (generalized) LReLUs. Then \(\FNN_{\{\sigma_{\shortminus\alpha},\sigma_{\alpha}\}}^{(\max\{d_x,d_y\},1)}(d_x,d_y)\) and \(\FNN_{\{\sigma_{\alpha}\}}^{(\max\{2,d_x,d_y\},1)}(d_x,d_y)\) are \(L^p\)-universal approximators of \(L^p(\mathbb{R}^{d_x},\mathbb{R}^{d_y})\) on compact sets.
    \end{theorem}
    
    Theorem~\ref{Theorem-Main1} provides an alternative proof of Theorems~2 and~3 of \citet{cai2023achieve}. The equivalence follows from the fact that G-LReLUs can express LReLU+ABS activations and vice versa (clear by Lemma \ref{Lemma-LReLU_properties} (2.)), so both activation sets yield the same class of representable functions. Their approach is indirect: it leverages the universal approximation properties of Neural ODEs \citep{ruiz2023neural, li2020deeplearningdynamicalsystems} combined with the result that LReLU FNNs can approximate flow maps of ordinary differential equations \citep{duan2024vanillafeedforwardneuralnetworks}. Since these flow maps are themselves \(L^p\) universal approximators (analogous to the classical diffeomorphism result of \citet{Brenier_Gangbo}), this establishes the desired universal approximation property for LReLU FNNs. In contrast, our proof directly constructs approximating networks via a coding scheme and additionally guarantees minimal interior dimension \(d_{\min} = 1\).

    \begin{figure}[t]
        \includegraphics[width=0.5\textwidth]{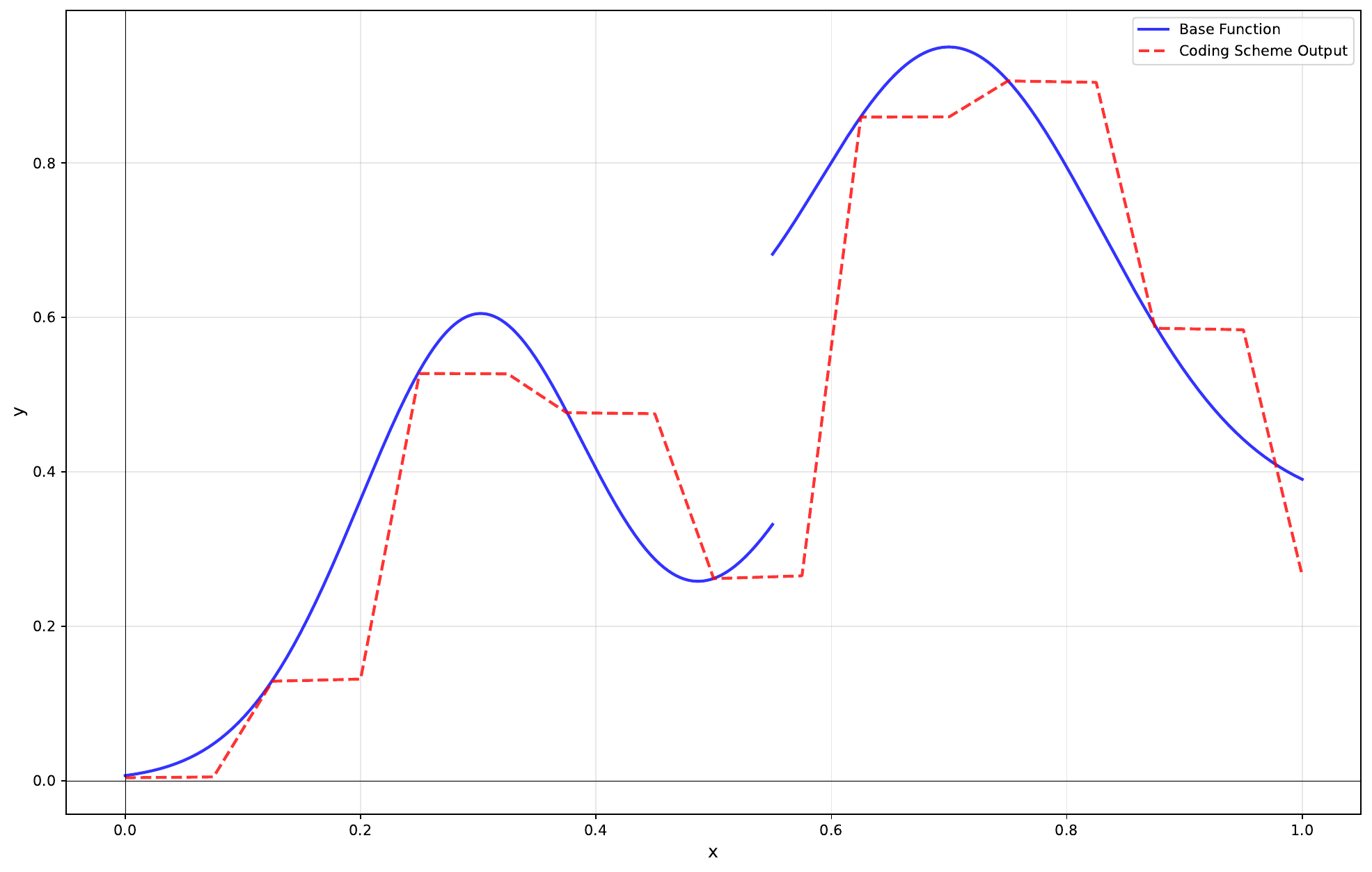}
        \hfill
        \includegraphics[width=0.5\textwidth]{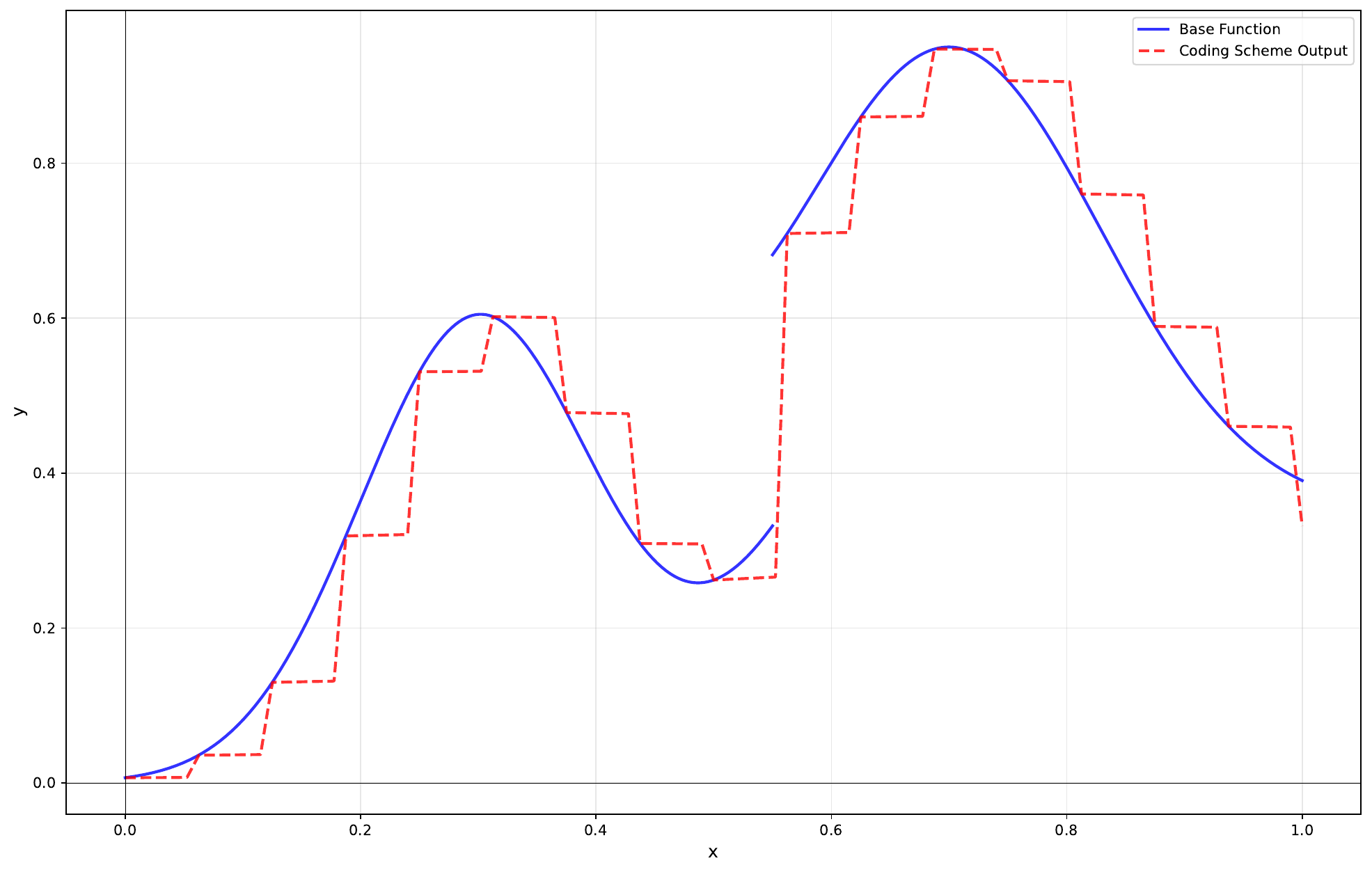}
        \caption{Coding scheme approximations from Theorem~\ref{Theorem-Main1} for \(f(x) = 0.6e^{-50(x-0.3)^2} + 0.6e^{-30(x-0.7)^2} + 0.35\cdot\mathbbm{1}_{[0.55,\infty)}(x)\) using LReLU activations from \(\mathcal{F}_{+}\) and a FNN of width 2. Left: \(K=3\), \(M=8\), \(\gamma=0.05\). Right: \(K=4\), \(M=14\), \(\gamma=0.01\). Here \(K, M\) are the coding scheme parameters (Definition~\ref{Definition-Coding_scheme}) and \(\gamma\) is the width of the exceptional intervals of the quantizer approximation with large error magnitudes (Lemma~\ref{Lemma-Approximate_q_k}).}
        \label{Figure-Main_theorem_lp_2}
    \end{figure}

    \begin{theorem}\label{Theorem-Main_sup}
        Let \(d_x, d_y \in \mathbb{N}\), \(\alpha \in (0,1) \cup (1,\infty)\), and let \(\sigma_{\alpha}\), \(\sigma_{\shortminus\alpha}\) denote the corresponding LReLUs. Let \(\mathcal{A}_1 \in \{\mathcal{F}_{\pm\alpha,\mathfrak{s}}, \{\sigma_{\shortminus\alpha}, \sigma_\alpha, \FLOOR\}\}\) and \(\mathcal{A}_2 \in \{\mathcal{F}_{\alpha,\mathfrak{s}}, \{\sigma_\alpha, \FLOOR\}\}\). Then \(\FNN_{\mathcal{A}_1}^{(\max\{d_x,d_y\},1)}(d_x,d_y)\) and \(\FNN_{\mathcal{A}_2}^{(\max\{2,d_x,d_y\},1)}(d_x,d_y)\) are uniform universal approximators of \(C^0(\mathbb{R}^{d_x},\mathbb{R}^{d_y})\) on compact sets.
    \end{theorem}

    \begin{figure}[t]
        \includegraphics[width=0.5\textwidth]{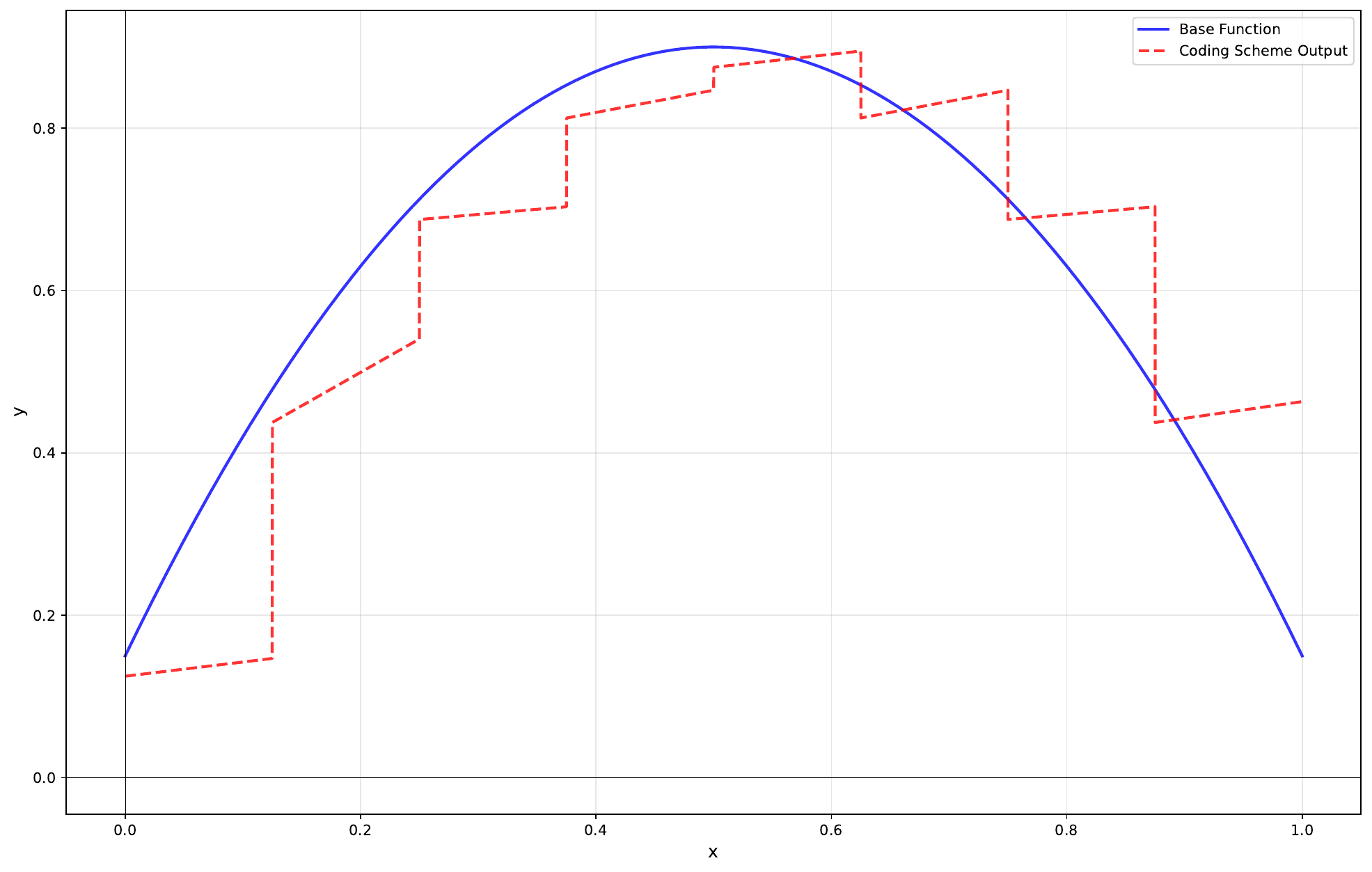}%
        \hfill
        \includegraphics[width=0.5\textwidth]{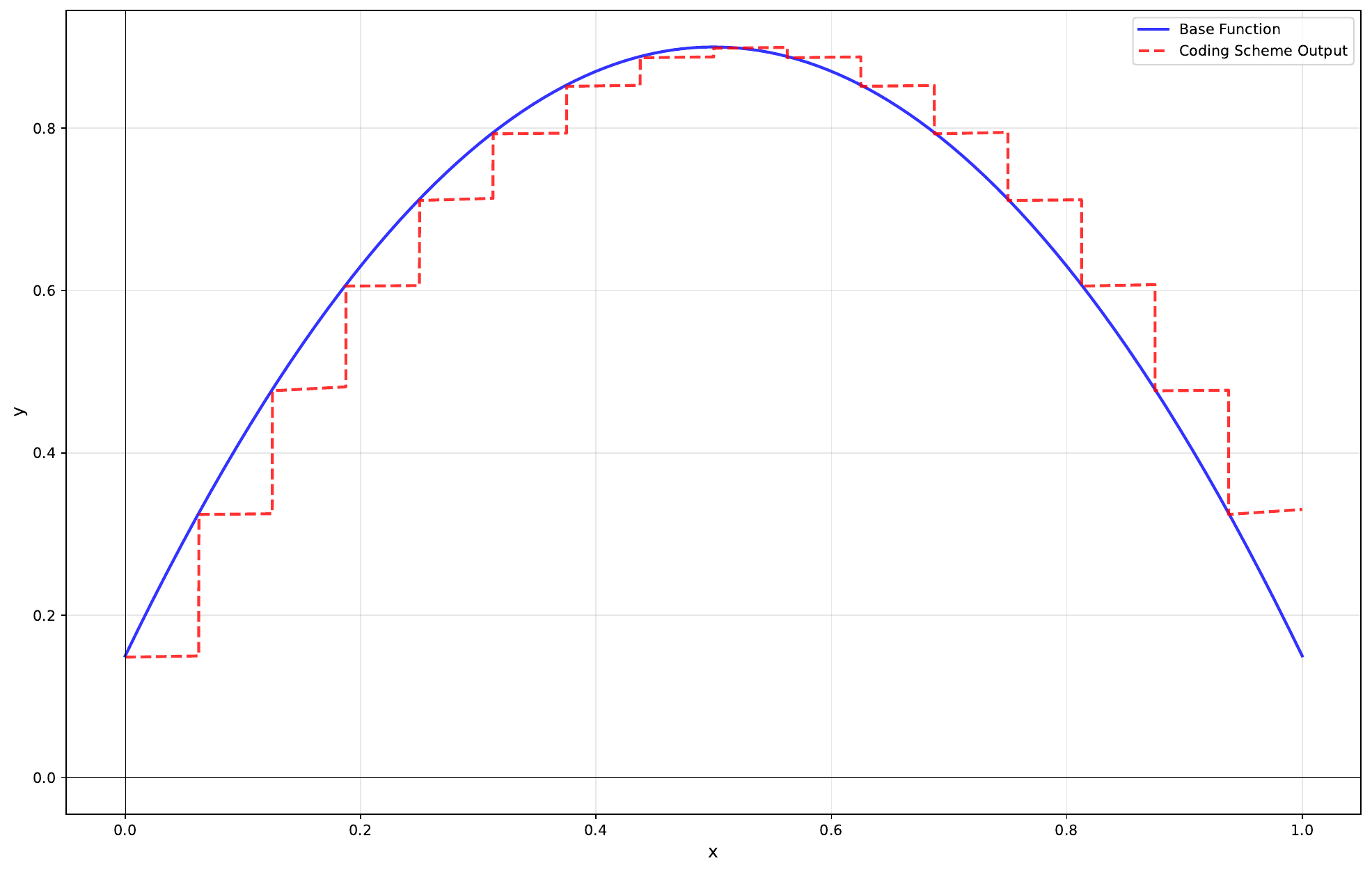}
        \caption{Coding scheme approximations from Theorem~\ref{Theorem-Main_sup} for \(f(x) = -3(x - 0.5)^2 + 0.9\) using SG-LReLU activations from \(\mathcal{F}_{\pm,\mathfrak{s}}\) with a width-1 FNN. Left: \(K=3\), \(M=4\), \(\alpha=0.02\). Right: \(K=4\), \(M=8\), \(\alpha=0.001\). Here \(K, M\) are the coding scheme parameters (Definition~\ref{Definition-Coding_scheme}) and \(\alpha\) is the slope of the quantizer approximation (Lemma~\ref{Lemma-Approximate_q_k_stepped}).}
        \label{Figure-Main_theorem_sup_1}
    \end{figure}

    Note that for discontinuous activations, uniform convergence of a composition does not automatically follow from uniform convergence of its components, as the standard continuity-based arguments break down. Consequently, even when two activation classes appear similar or can mutually approximate each other via width-one FNNs, they may not share the same universal approximation properties when discontinuities are present—the convergence behavior under composition must be established independently for each class. For this reason, Theorem~\ref{Theorem-Main_sup} is proven separately for two activation classes at each minimal width: one combining LReLUs with a step function, and another mixing LReLUs directly with FLOOR. While both achieve the same width bounds, we provide independent proofs to ensure rigor, as the discontinuity structures differ. Additionally, one formulation may prove more suitable for future theoretical extensions or practical architectures.

    The proofs of Theorems~\ref{Theorem-Main_sup} and~\ref{Theorem-Main1} rely heavily on the coding scheme constructions developed in Sections~\ref{Section-Piecewise_linear} and~\ref{Section-Approximating_coding_scheme}, with detailed arguments presented in Section~\ref{Section-Proof_Main1}.
    
    \medskip
    While theoretically appealing, the FLOOR activation suffers from severe information loss: it maps entire intervals to the same output, rendering gradients zero almost everywhere and making gradient-based optimization impractical. S-LReLUs avoid this pathology but introduce a different challenge: each unit requires a distinct step-size parameter that does not participate in gradient flow. Since step sizes do not affect local derivatives, gradient descent cannot learn them.
    
    The following corollary shows that our uniform universal approximation results extend to a more practical setting: function classes using only a single fixed step size combined with infinitely many LReLU slope parameters. Crucially, unlike step sizes, LReLU slopes directly affect the activation's derivative and thus receive meaningful gradient signals. This enables smooth parameterization of the activation class during training, optimization via standard gradient descent, and practical implementation in neural network architectures. In other words, by fixing the step size and allowing gradient descent to explore the space of LReLU slopes, we obtain a trainable activation class that retains the theoretical guarantees of Theorem~\ref{Theorem-Main_sup}. The proof of Corollary~\ref{Corollary-UAP_sup_fixed_stepsize} is presented in Section~\ref{Section-Proof_Main1}.
    
    \begin{corollary}\label{Corollary-UAP_sup_fixed_stepsize}
        Let \(d_x, d_y \in \mathbb{N}\). Then \(\FNN_{\mathcal{F}_{\pm,1}^{*}}^{(\max\{d_x,d_y\},1)}(d_x,d_y)\) and \(\FNN_{\mathcal{F}_{+,1}^{*}}^{(\max\{2,d_x,d_y\},1)}(d_x,d_y)\) are uniform universal approximators of \(C^0(\mathbb{R}^{d_x},\mathbb{R}^{d_y})\) on compact sets. These activation classes use a variable LReLU parameter but a fixed step size in \(\{0, 1\}\), making them amenable to training via gradient descent.
    \end{corollary}
        \medskip
        While the coding scheme approach is highly versatile and applicable to various constructive proofs in theory, its practical utility is limited, as the number of layers increases exponentially with the desired approximation accuracy.
        
        \begin{remark}\label{Remark-First_layer_depth_coding_scheme}
            The number of layers required for the universal approximators in Theorems~\ref{Theorem-Main_sup} and~\ref{Theorem-Main1} grows exponentially as the target accuracy increases when the activation sets are large families with infinitely many parameters. When one considers only the specific activation sets consisting of one or two fixed LReLU parameters, as in the statements of these theorems, the number of layers generally grows even faster than exponentially with increasing accuracy. For a more detailed analysis, see Remark~\ref{Remark-Layer_depth_coding_scheme}.
        \end{remark}
        
        As a direct consequence of Theorems~\ref{Theorem-Main_sup} and~\ref{Theorem-Main1}, the following corollary shows that autoencoders with LReLU and S-LReLU activations and feature dimension \(1\) are universal approximators of \(L^p\)-integrable and continuous functions on compact sets, respectively. Note that these universal approximating autoencoders arise from our constructions, and consequently their depth also grows exponentially with increasing accuracy, as stated in Remark~\ref{Remark-First_layer_depth_coding_scheme}.
        
        \begin{corollary}\label{Corollary-AE}
            Let \(d \in \mathbb{N}\) and \(\alpha \in (0,1) \cup (1,\infty)\). Then \(\AEN_{\{\sigma_\alpha\}}^{(1)}(d)\) is an \(L^p\) universal approximator of \(L^p(\mathbb{R}^{d},\mathbb{R}^{d})\) on compact sets. Moreover, \(\AEN_{\mathcal{F}_{\alpha,\mathfrak{s}}}^{(1)}(d)\) and \(\AEN_{\mathcal{F}_{+,1}^{*}}^{(1)}(d)\) are uniform universal approximators of \(C^0(\mathbb{R}^{d},\mathbb{R}^{d})\) on compact sets. These results generalize directly to any feature dimension \(d_{\min} < d\) by padding the affine transformations with zeros.
        \end{corollary}
        
        \begin{proof}
            This follows directly from Theorems~\ref{Theorem-Main_sup} and~\ref{Theorem-Main1} and Corollary~\ref{Corollary-UAP_sup_fixed_stepsize}, as their corresponding constructions have minimal interior dimension \(1\).
        \end{proof}    
    
        \medskip
        Note that we define autoencoders only when input and output dimensions are equal (\(d = d_x = d_y\)), which is the most common case in practice. However, since Theorems~\ref{Theorem-Main_sup} and~\ref{Theorem-Main1} and Corollary~\ref{Corollary-UAP_sup_fixed_stepsize} hold for arbitrary dimensions, Corollary~\ref{Corollary-AE} extends naturally to the general case.
        
        \begin{remark}
            Corollary~\ref{Corollary-AE} implies the distributional universal approximation property for variational autoencoders (VAEs) with LReLU activation of arbitrary internal dimension \(d_{\min} \geq 1\). As long as layers become narrower with invertible activations, continuous distributions are mapped to continuous distributions layer by layer. Starting with a continuous data distribution, we obtain a continuous distribution internally at the prescribed minimal layer size \(d_{\min}\). We can then continue mapping this internal continuous distribution to a Gaussian distribution, e.g.\ by approximating a continuous transport map, which is known to exist \citep{santambrogio2015optimal, villani2009optimal}. By Corollary~\ref{Corollary-AE}, one approximately reaches an internal normal distribution in dimension \(d_{\min}\) with layers of width \(d_{\min}\) in the sense of weak convergence by approximating the transport map with LReLU networks. This process can then be reversed by approximating the inverse transport map. If this module is inserted in the narrowest layer of an autoencoder that approximates the identity function with internal size \(d_{\min}\), we achieve distributional universal approximation with the VAE. We leave the details to the reader.
        \end{remark}
        
        \begin{definition}\label{Definition-Squashable_function}
            A function \(\sigma: \mathbb{R} \to \mathbb{R}\) is called \textit{squashable} if conditions (C1) and (C2) below hold. Here \(\STEP: \mathbb{R} \to \mathbb{R}\) denotes the step function with \(\STEP(x) = \mathbbm{1}_{[0,\infty)}(x)\). We define \(\mathcal{S}\) as the set of all squashable functions.
            \begin{enumerate}
                \item[\textbf{(C1)}] \textbf{Identity Approximation:} There exists \(z \in \mathbb{R}\) such that \(\sigma\) is continuously differentiable at \(z\) and \(\sigma'(z) \neq 0\).
                
                \item[\textbf{(C2)}] \textbf{Step Approximation:} \(\sigma\) is continuous and for any compact set \(\mathcal{K} \subset \mathbb{R}\) and \(\varepsilon, \zeta > 0\), there exists a width-\(1\) \(\sigma\)-network \(\rho_{\varepsilon,\zeta}: \mathbb{R} \to \mathbb{R}\) satisfying:
                \begin{itemize}
                    \item \(\max_{x \in \mathcal{K} \setminus (-\zeta, \zeta)} |\rho_{\varepsilon,\zeta}(x) - \STEP(x)| \leq \varepsilon\),
                    \item \(\rho_{\varepsilon,\zeta}\) is strictly increasing on \(\mathcal{K}\) and satisfies \(\rho_{\varepsilon,\zeta}(\mathcal{K}) \subseteq [0,1]\).
                \end{itemize}
            \end{enumerate}
        \end{definition}
        
        Note that \citet{shin2025minimumwidthuniversalapproximation} showed that FNNs with squashable activations are \(L^p\) universal approximators with width \(\max\{2, d_x, d_y\}\). We formulate a similar result for uniform universal approximation. The set of squashable functions contains many important examples, which we summarize in the following remark.
        
        \begin{remark}\label{Remark-Squashable_functions}
            By definition, squashable functions can locally approximate the identity (C1) and enable the construction of STEP function approximations when used as activations in width-one FNNs (C2). A broad class of functions is squashable. For instance, all non-affine analytic functions are squashable (see Lemma~4 in \citep{shin2025minimumwidthuniversalapproximation}). Moreover, all strictly monotonic functions that are non-differentiable at some point \(c \in \mathbb{R}\), but differentiable on an open interval \((a,b) \setminus \{c\}\) surrounding that point, and for which the left and right limits of the derivative exist, differ in magnitude, yet share the same sign, are also squashable (see Lemma~5 in \citep{shin2025minimumwidthuniversalapproximation}). This class includes strictly monotone piecewise linear functions with at least one non-differentiable breakpoint, such as LReLU, as well as piecewise-defined functions with a breakpoint at which the derivatives on both sides have the same sign, such as Hardswish. Note that all ReLU-like functions, as defined by \citet{kim2024minimumwidthuniversalapproximation}, are squashable.
        \end{remark}
        
        \begin{theorem}\label{Theorem-Squashable}
            Let \(d_x, d_y \in \mathbb{N}\) and let \(\sigma \in \mathcal{S}\) be squashable. Let \(\STEP(x) := \mathbbm{1}_{[0,\infty)}(x)\) be the step function and \(\Id\) denote the identity on \(\mathbb{R}\). Then \(\FNN_{\{\sigma,\,\FLOOR,\,\Id\}}^{(\max\{3,d_x,d_y\},1)}(d_x,d_y)\) and \(\FNN_{\{\STEP,\,\FLOOR,\,\Id\}}^{(\max\{3,d_x,d_y\},1)}(d_x,d_y)\) are uniform universal approximators of \(C^0(\mathbb{R}^{d_x},\mathbb{R}^{d_y})\) on compact sets.
        \end{theorem}
        
    \subsection{LU-decomposable invertible neural networks as \(L^p\) universal approximators and LU-Net as distributional universal approximator}
    
    \begin{definition}\label{Definition-LU_decomposition}\label{Definition-LU(d)}
        We say that an invertible matrix \(A\in \GL(d)\) is lower-upper- (\(LU\)-) decomposable if there exist a lower triangular matrix \(L\) with ones on the diagonal and an upper triangular matrix \(U\) such that \(A=LU\) holds. We define \(\LU(d)\) as the set of all invertible LU-decomposable matrices \(A\in \GL(d)\) and the set of their inverses by \(\LU^{-1}(d):=\{A^{-1}|\,A\in \LU(d)\}\). Note that one can show that all elements in \(\LU^{-1}(d)\) have an upper-lower (UL-) decomposition, thus can also be considered as the set of UL-decomposable matrices.
    \end{definition}

        Note that \(LU\)-decompositions can also be obtained for singular matrices, but for our purposes in this paper it is sufficient to only consider \(LU\)-decompositions of invertible matrices. 

    \begin{remark}\label{Remark-LU_characterization}
        A matrix $A\in \mathbb{R}^{d\times d}$ is said to have non-zero leading principal minors if all leading submatrices $A_k=(a_{i,j})_{i,j\in [k]}$, with $k\in [d]$, have non-zero determinants, and non-zero trailing principal minors if all trailing submatrices $A_k=(a_{i,j})_{i,j\in [k,d]}$, with $k\in [d]$, have non-zero determinants. It can be shown that the invertible $LU$-decomposable matrices are exactly the matrices that have non-zero leading principal minors (Corollary 1 in \cite{okunev2005necessarysufficientconditionsexistence}), and similarly, $\LU^{-1}(d)$ consists of exactly the matrices that have non-zero trailing principal minors.
    \end{remark}
    
    \begin{lemma}\label{Lemma-LU_dense}
    Both $\LU(d)$ and $\LU^{-1}(d)$ are dense subsets of $\mathbb{R}^{d\times d}$ with respect to any operator norm $\|\cdot\|_{\op}$.
    \end{lemma}
    \begin{proof}
    We prove the statement for $\LU(d)$; the argument for $\LU^{-1}(d)$ is analogous using the characterization via trailing principal minors from Remark~\ref{Remark-LU_characterization}, instead.
    
    Let $A\in \mathbb{R}^{d\times d}$ and $B\in\LU(d)$, and let $A_k, B_k$ denote their leading $k \times k$ submatrices for $k\in [d]$. Define $p_k:\mathbb{R}\rightarrow \mathbb{R}$ by $p_k(t):=\det((1-t)A_k+tB_k)$ for all $k\in [d]$. Each $p_k$ is a non-zero polynomial since $p_k(1)=\det(B_k)\neq 0$ (as $B$ is LU-decomposable), implying that each $p_k$ has only finitely many real zeros. Let $Z := \bigcup_{k=1}^{d} \{t > 0 : p_k(t) = 0\}$ denote the set of all positive zeros of the polynomials $p_1,\ldots,p_d$. Since each $p_k$ has finitely many real zeros, $Z$ is a finite set. Define $t_* := \inf Z$ if $Z \neq \emptyset$, and $t_* := 1$ otherwise. Since $Z$ is finite, we have $t_* > 0$. It follows that for any sequence $(t_n)_{n\in \mathbb{N}}\subset (0,t_*)$ with $\lim_{n\to \infty}t_n=0$, the matrices $A^{(n)}:=(1-t_n)A+t_n B$ satisfy $A^{(n)}\in \LU(d)$, since all leading principal minors are non-zero: for each $k\in[d]$, we have $\det(A_k^{(n)}) = p_k(t_n) \neq 0$ because $t_n \in (0,t_*)$ and $t_*$ was defined to exclude all positive zeros of the $p_k$. Moreover, these matrices satisfy
    \begin{align*}
        \lVert A-A^{(n)}\rVert_{\op}=t_n\lVert A-B\rVert_{\op} \xrightarrow{n\to\infty} 0\;.
    \end{align*}
    Thus, for every matrix $A\in \mathbb{R}^{d\times d}$, there exists a sequence of LU-decomposable matrices converging to it in operator norm, establishing density.
    \end{proof}    
        
    \begin{definition}\label{Definition-Affine_LU_transformations}
        We define the set of $d$-dimensional $LU$-decomposable affine transformations as $\Aff_{\LU}(d)=\{W:\mathbb{R}^{d}\rightarrow \mathbb{R}^{d}, W(x)=Ax+b \,|A\in \LU(d), b\in \mathbb{R}^{d} \}$ and $\Aff_{\LU^{-1}}(d)=\{W:\mathbb{R}^{d}\rightarrow \mathbb{R}^{d}, W(x)=Ax+b \,|A\in \LU^{-1}(d), b\in \mathbb{R}^{d} \}$.
    \end{definition}
    
    With the set of affine LU-decomposable transformations and their inverses, we now generalize Definition~\ref{Definition-FNN} of general FNNs to LU-decomposable neural networks and their inverses.
    
    \begin{definition}\label{Definition-LU_neural_networks}
        Let $\mathcal{A}$ be a set of activation functions and let $\phi\in \FNN_{\mathcal{M}}^{(w)}(d,d)$ be of the form $\phi=W_n\circ \sigma_{n-1}\circ W_{n-1}\circ \ldots \circ \sigma_1\circ W_1$ for some $n\in \mathbb{N}$. If $W_i\in \Aff_{\LU}(d)$ for all $i\in [n]$, then we call $\phi$ a $d$-dimensional LU-decomposable neural network, or if instead $W_i\in \Aff_{\LU^{-1}}(d)$ for all $i\in [n]$, then we call $\phi$ a $d$-dimensional UL-decomposable neural network, as the inverse of a LU-decomposable matrix is a UL-decomposable matrix. We denote the set of all $d$-dimensional $LU$- and $UL$-decomposable neural networks with activations in $\mathcal{A}$ by $\LU_{\mathcal{A}}(d)$ and $\LU^{-1}_{\mathcal{A}}(d)$, respectively. Note that if $\mathcal{A}$ consists of invertible functions, then these sets consist of invertible functions as compositions of invertible functions.
    \end{definition}
    
    We generalize the universal approximation for $L^p$ integrable functions on compact sets from Theorem~\ref{Theorem-Main1} for leaky ReLU activations and Theorem 1 of \cite{shin2025minimumwidthuniversalapproximation} for squashable activations to LU-decomposable neural networks with these activations in the following theorems.
    
    \begin{theorem}\label{Theorem-Main2}
         Let $d\in \mathbb{N}$, $\alpha\in (0,1)\cup(1,\infty)$, and let $\sigma_\alpha$ be the corresponding invertible leaky ReLU. Then, $\LU_{\{\sigma_\alpha\}}(d)$ and $\LU^{-1}_{\{\sigma_\alpha\}}(d)$ are $L^p$ universal approximators of $L^p(\mathbb{R}^{d},\mathbb{R}^{d})$ on compact sets.
     \end{theorem}
     
     \begin{theorem}\label{Theorem-LU_Squashable}
         Let $d\in \mathbb{N}\setminus \{1\}$ and let $\sigma\in \mathcal{S}$ be a squashable activation. Then, $\LU_{\{\sigma\}}(d)$ and $\LU^{-1}_{\{\sigma\}}(d)$ are $L^p$ universal approximators of $L^p(\mathbb{R}^{d},\mathbb{R}^{d})$ on compact sets. If, additionally, $\sigma$ is bijective, then $\LU_{\{\sigma\}}(d)$ and $\LU^{-1}_{\{\sigma\}}(d)$ are also sets of bijective functions.
     \end{theorem}
     
     We give the proofs of Theorem~\ref{Theorem-Main2} and Theorem~\ref{Theorem-LU_Squashable} in Section~\ref{Section-Proof_main2}. They give us helpful tools to conclude strong approximation results for the normalizing flow LU-Net (see Definition~\ref{Definition-LU-Net}), whose elements take the form of LU-decomposable invertible neural networks. Furthermore, we use Theorem~\ref{Theorem-Main2} as an important building block to generalize the $L^p$ universal approximation result of Theorem~\ref{Theorem-Main1} to suitable classes of neural networks that are smooth diffeomorphisms, which we do later in the proof of Theorem~\ref{Theorem-Main4}.

   \begin{definition}\label{Definition-Diffeomorphisms}
   For $k\in \mathbb{N}_{0}\cup \{\infty\}$, a $C^k$-diffeomorphism from $\mathbb{R}^d$ to $\mathbb{R}^d$ is an invertible and $k$-times continuously differentiable function $f:\mathbb{R}^d\rightarrow \mathbb{R}^d$ with a $k$-times continuously differentiable inverse $f^{-1}$, and we call a $C^\infty$-diffeomorphism a smooth diffeomorphism. For $k\in \mathbb{N}_{0}\cup \{\infty\}$, we define $\mathcal{D}^k(\mathbb{R}^d,\mathbb{R}^d)$ as the set of all $C^k$-diffeomorphisms from $\mathbb{R}^d$ to $\mathbb{R}^d$. As an important special case, 
    $\mathcal{D}^0(\mathbb{R}^d, \mathbb{R}^d)$ consists of all $C^0$-diffeomorphisms, also known as homeomorphisms, i.e. it contains exactly those maps $\mathbb{R}^d \to \mathbb{R}^d$ that are continuous, bijective, and have a continuous inverse. Equivalently, this is simply the set of continuous bijections $\mathbb{R}^d \to \mathbb{R}^d$, since the inverse of any such map is automatically continuous (see Remark~\ref{Remark-D0_Inverse_continuous}).
\end{definition}

\begin{definition}\label{Definition-Normalizing_flow}
    Let $d\in \mathbb{N}$ and $\mathcal{M}$ be a set of invertible functions from $\mathbb{R}^d$ to $\mathbb{R}^d$. The corresponding normalizing flow (NF) of dimension $d$ is then defined by $\NF({\mathcal{M}}):=\{(f,f^{-1})|\,f\in \mathcal{M}\}$ and we call $\mathcal{M}$ the set of normalizing directions and $\mathcal{M}^{-1}:=\{f^{-1}|f\in \mathcal{M}\}$ the set of generative directions of $\NF(\mathcal{M})$. Additionally, the normalizing flow consists of $C^k$-diffeomorphisms if $\mathcal{M}$ consists of $C^k$-diffeomorphisms. 
\end{definition}

Note that although NFs can be defined for invertible activations, it is usually required that the activations are also differentiable almost everywhere. This is necessary to be able to obtain the densities of the push-forward measures corresponding to the normalizing flow almost everywhere with the change of variables formula (see e.g. Theorem 3.7.1 in \cite{Bogachev2007}). For these and other details to normalizing flows and their application we refer to a survey by \citet{Kobyzev_2021}.

In the following we define the NF LU-Net, introduced in 2023 by \cite{chan2023lunet}, for which we show the distributional universal approximation property (see Definition~\ref{Definition-DUAP}).

\begin{definition}\label{Definition-LU-Net}
	Let $d\in \mathbb{N}$ denote the dimension, $\mathcal{A}$ a set of invertible real-valued activations and $\LU_{\mathcal{A}}(d)$ the corresponding set of LU-decomposable neural networks of dimension $d$ with activations in $\mathcal{A}$ as introduced in Definition~\ref{Definition-LU_neural_networks}. Then we define the $d$-dimensional LU-Net with activations in $\mathcal{A}$ to be the normalizing flow given by $\NF(\LU_{\mathcal{A}}(d))$. Thus, the normalizing directions are elements of $\LU_{\mathcal{A}}(d)$ and take the form of $LU$-decomposable neural networks with activations in $\mathcal{A}$.
\end{definition}

\begin{remark}\label{Remark-LU-Net}
    LU-Net offers several practical advantages: First, it can be initialized with LU-decomposed matrices, enabling efficient determinant computation of the normalizing directions (see equation (8) in \citet{chan2023lunet}). This allows efficient maximum likelihood training of the corresponding normalizing flow (see equation (9) in \citet{chan2023lunet}). Second, its FNN structure makes practical stabilization techniques readily applicable and enables implementation using highly optimized matrix multiplication kernels.
\end{remark}

Since the normalizing and generative directions of LU-Net have the structure of LU-de\-composable networks, they inherit the strong approximation capabilities of neural networks with expressive activation functions and embed them into a normalizing flow framework. This structural property further allows us to establish strong approximation results for the LU-Net by invoking universal approximation theorems for neural networks.

The following definition introduces distributional universal approximators, which essentially follows the notion of Definition 3 by \citet{teshima2020couplingbased}.

\begin{definition}\label{Definition-Distributional_universality}
Let $\mathcal{P}(\mathbb{R}^d)$ and $\mathcal{P}_{\text{abs}}(\mathbb{R}^d)$ denote the set of all probability and absolutely continuous probability measures, respectively, and denote convergence in distribution by $\stackrel{\mathcal{D}}{\longrightarrow}$. For $d\in \mathbb{N}$, let $\mathcal{M}$ be a function class consisting of measurable mappings $g:\mathbb{R}^d\rightarrow \mathbb{R}^d$. We say that $\mathcal{M}$ is a distributional universal approximator, or that it achieves/attains distributional universal approximation (DUA) if for every $\mu\in \mathcal{P}_{\text{abs}}(\mathbb{R}^d)$ and $\nu \in\mathcal{P}(\mathbb{R}^d)$, there exists a sequence $(g_n)_{n\in\mathbb{N}}\subset\mathcal{M}$ such that $(g_n)_{*}\mu:=\mu \circ g_n^{-1} \stackrel{\mathcal{D}}{\rightarrow}\nu$. 
\end{definition}

\begin{definition}\label{Definition-DUAP}
We say that a normalizing flow $\NF(\mathcal{M})$ of dimension $d\in \mathbb{N}$ has the distributional universal approximation property (DUAP) if the corresponding set of generative directions is a distributional universal approximator according to Definition~\ref{Definition-Distributional_universality}.
\end{definition}

In other words, a function class achieves DUA if, for any absolutely continuous probability measure and any target probability measure, there exists a sequence of transformations in the class such that the corresponding sequence of push-forward measures converges to the target measure in distribution. Therefore, NFs with the DUAP can transform any absolutely continuous probability measure into any target measure with respect to the topology induced by convergence in distribution. This guarantees that, at least in theory, the architecture is capable of learning mappings to arbitrary probability distributions in the sense of weak convergence.

To provide intuition for what function classes need to universally approximate in order to attain DUA, we introduce the following function classes.

\begin{definition}\label{Definition-T_infty}
A mapping $\tau=(\tau_1,\ldots,\tau_d)^T: \mathbb{R}^d\rightarrow \mathbb{R}^d$ is called increasing triangular if each component $\tau_i$ depends only on the first $i$ coordinates $x_{1:i}$ and is strictly increasing with respect to $x_i$ for all $i\in[d]$. We define $\mathcal{T}^\infty(d)$ as the function class of all $C^\infty$-increasing triangular mappings from $\mathbb{R}^d$ to $\mathbb{R}^d$.
\end{definition}

The following lemma, whose first part corresponds to Lemma 1 in \cite{teshima2020couplingbased}, is our essential tool for establishing the DUAP. It shows that any function class that universally approximates $\mathcal{T}^\infty(d)$ on compact sets with respect to the $L^p$ norm (for any $p\in [1,\infty)$) is automatically a distributional universal approximator. In other words, the universal approximation of a sufficiently rich function class like $\mathcal{T}^\infty(d)$ directly implies that the function class achieves DUA, which in turn implies the DUAP for any normalizing flow architecture whose generative directions are given by this function class.

\begin{lemma}\label{Lemma-T_infinity_approximators_universal}
Let $\mathcal{M}$ be a function class from $\mathbb{R}^d$ to $\mathbb{R}^d$ that is an $L^p$ universal approximator of $\mathcal{T}^\infty(d)$ on compact sets for any $p\in [1,\infty)$, then $\mathcal{M}$ achieves DUA. Consequently, any function class $\mathcal{M}$ from $\mathbb{R}^d$ to $\mathbb{R}^d$ that is a $L^p$ universal approximator of $L^p(\mathbb{R}^d,\mathbb{R}^d)$ for any $p\in [1,\infty)$ on compact sets achieves DUA.
\end{lemma}

\begin{proof}
For the first part we refer to the proof of Lemma 1 in \cite{teshima2020couplingbased}. For the second part, let $\mathcal{K}\subset \mathbb{R}^d$ be compact and note that $\mathcal{T}^{\infty}(d)|_\mathcal{K}:=\{f|_\mathcal{K} \mid f\in \mathcal{T}^\infty(d)\}\subset L^p(\mathcal{K},\mathbb{R}^d)$ since triangular transforms are smooth. Combining this with the assumption, we obtain that for $f\in \mathcal{T}^{\infty}(d)|_\mathcal{K}$ there exists $(\phi_n)_{n\in \mathbb{N}}\subset \mathcal{M}$ such that $\lim_{n\to\infty}\lVert f-\phi_n\rVert_{\mathcal{K},p}=0$. Since the compact set was chosen arbitrarily, it follows that $\mathcal{M}$ is an $L^p$ universal approximator of $\mathcal{T}^{\infty}(d)$ on compact sets and by the first part, it achieves DUA.
\end{proof}

\begin{remark}\label{Remark-T_infty_strengthening}
By Theorem 1(A) in \cite{teshima2020couplingbased}, Lemma~\ref{Lemma-T_infinity_approximators_universal} can be further strengthened: for invertible neural networks whose activations are all piecewise $C^1$ diffeomorphisms, it suffices to universally approximate a smaller class of $C^\infty$-diffeomorphisms called single-coordinate transforms. These are diffeomorphisms $\tau: \mathbb{R}^d\rightarrow \mathbb{R}^d$ of the form $\tau(x)=(x_1,\ldots,x_{d-1},\tau_d(x_d))^T$ that are equal to the identity outside of a compact set. Denoting this class by $\mathcal{S}_c^\infty(d)$, we have $\mathcal{S}_c^\infty(d)\subset \mathcal{T}^\infty(d)$, showing that approximating only transformations that alter the last coordinate suffices for DUA under these conditions. This generalization can be useful when analyzing normalizing flows with weaker approximation properties. However, for our architectures such as LU-Net, we establish universal approximation of all $L^p$ integrable functions on compact sets, which directly implies universal approximation of $\mathcal{T}^\infty(d)$ and thus the DUAP.
\end{remark}

    The following two theorems establish the DUAP for LU-Nets with either invertible leaky ReLU or more general squashable activations. Their proofs essentially rely on the underlying universal approximation results for $L^p$-integrable functions, namely Theorem~\ref{Theorem-Main1} and Theorem~1 of \cite{shin2025minimumwidthuniversalapproximation}, together with their generalizations to LU- or UL-decomposable networks as provided by Theorem~\ref{Theorem-Main2} and Theorem~\ref{Theorem-LU_Squashable}. We present their proofs in Section~\ref{Section-Proof_main3}.

\begin{theorem}\label{Theorem-Main3}
    Let $d\in \mathbb{N}$, $\alpha\in (0,1)\cup (1,\infty)$, and let $\sigma_\alpha$ be the corresponding leaky ReLU. Then $\NF(\LU_{\{\sigma_{\alpha}\}}(d))$ has the DUAP, i.e., LU-Net with any non-identical leaky ReLU activation has the DUAP.
\end{theorem}

\begin{theorem}\label{Theorem-DUAP_squashable}
    Let $d\in \mathbb{N}\setminus\{1\}$ and let $\sigma \in \mathcal{S}\cap \mathcal{D}^0(\mathbb{R},\mathbb{R})$, i.e., a squashable and bijective function from $\mathbb{R}$ to $\mathbb{R}$. Then $\NF(\LU_{\{\sigma^{-1}\}}(d))$ has the DUAP.
\end{theorem}

    \subsection{Smooth almost leaky ReLU neural networks are universal approximators}

    The following definition of standard mollifiers is equivalent to Definition 4.1.27 in \cite{PapageorgiouWinkert+2018}. As usual, we use mollifiers in our proofs to smooth out functions when necessary.

    \begin{definition}\label{Definition-Mollifier}
        We define \( p:\mathbb{R}\rightarrow \mathbb{R},\;p(x):=\begin{cases}
               c\cdot\exp\left(\frac{-1}{1-|x|^2}\right) & \text{if }|x|<1\\
               0 & \text{if } |x|\geq 1
           \end{cases}\), where the constant \(c>0\) is chosen such that \(\int_{-\infty}^{\infty}p(x)\,dx=1\). For \(n\in \mathbb{N}\) we define the standard mollifiers corresponding to \(p\) by \(p_n(x):=n\cdot p(nx)\), which then satisfy \(\int_{-\infty}^{\infty}p_n(x)\,dx=1\) and \(\operatorname{supp}(p_n) = \left[-\frac{1}{n},\frac{1}{n}\right]\).
    \end{definition}
    
    \begin{lemma}\label{Lemma-Convolution}
        We call \(f:\mathbb{R}^d\rightarrow \mathbb{R}^k\) locally \(L^p\), or equivalently write \(f\in L_{\loc}^p(\mathbb{R}^d,\mathbb{R}^k)\), if \(\lVert f \rVert_{\mathcal{K},p}<\infty\) for all compact \(\mathcal{K}\subset \mathbb{R}^d\). For \(n\in \mathbb{N}\), the convolution operator \((\cdot *p_n):L^p_{\loc}(\mathbb{R},\mathbb{R})\rightarrow C^\infty(\mathbb{R},\mathbb{R})\) defined by \((f *p_n)(x):=\int_{-\infty}^{\infty} f(y) p_n(y-x)\, dy\) is well-defined, where \(C^\infty(\mathbb{R},\mathbb{R})\) denotes the set of smooth real-valued functions.
    \end{lemma}
    \begin{proof}
        See Proposition 4.1.28 in \cite{PapageorgiouWinkert+2018}.
    \end{proof}
    
    \begin{definition}[Regularizing functions]\label{Definition-Mollifier_2}
        For any function \(f\in L^1_{\loc}(\mathbb{R},\mathbb{R})\), the regularizing functions of \(f\) are defined as \(f_n:=f * p_n\) for \(n\in \mathbb{N}\).
    \end{definition}
    
    We use their properties shown in \cite{PapageorgiouWinkert+2018} to regularize leaky ReLUs in the proof of Theorem~\ref{Theorem-Main4}.
    
    \begin{remark}\label{Remark-Mollifier}
        One can show that \(p\) is smooth, which also implies that \(p_n\) is smooth for all \(n\in \mathbb{N}\) as a composition of smooth functions. Moreover, for \(f\in L^1_{\loc}(\mathbb{R},\mathbb{R})\), the sequence \((f_n)_{n\in \mathbb{N}}\) defined by \(f_n:=f*p_n\) consists of smooth functions, i.e., \((f_n)_{n\in \mathbb{N}}\subset C^\infty(\mathbb{R},\mathbb{R})\), by Lemma~\ref{Lemma-Convolution}. Additionally, by definition, \(\operatorname{supp}(p)=[-1,1]\) and therefore \(\operatorname{supp}(p_n)=[-\frac{1}{n}, \frac{1}{n}]\) for all \(n\in \mathbb{N}\).    
    \end{remark}
    
    \begin{lemma}[\citet{PapageorgiouWinkert+2018}]\label{Lemma-Mollifier}
        For any compact \(\mathcal{K}\subset \mathbb{R}\) and \(f\in C^0(\mathcal{K},\mathbb{R})\), it holds that \(\lim_{n\to \infty}\lVert f - (f*p_n) \rVert_{\mathcal{K},\sup}=0\).
    \end{lemma}
    \begin{proof}
        See Proposition 4.1.29 in \cite{PapageorgiouWinkert+2018}.
    \end{proof}

    \begin{definition}\label{Definition-Smoothed_LReLUs}
        Let \(p_n\) be the mollifiers from Definition~\ref{Definition-Mollifier} and for \(\alpha\in (0,1)\cup (1,\infty)\) let \(\sigma_\alpha\) be the corresponding invertible LReLU. We define \(\mathcal{F}_{\alpha}^{\infty}:=\{\sigma_{\alpha}*p_n\mid n\in \mathbb{N}\}\) as the set of smoothed versions of this LReLU and note that by Lemma~\ref{Lemma-Convolution} it holds that \(\mathcal{F}_{\alpha}^\infty\subset C^\infty(\mathbb{R},\mathbb{R})\). Moreover, for \(d\in \mathbb{N}\) we call \(\LU_{\mathcal{F}_{\alpha}^{\infty}}(d)\) the set of LU-decomposable, smoothed LReLU networks with parameter \(\alpha\).
    \end{definition}
    
    Now, we show that smooth diffeomorphisms in the form of LU-decomposable, smoothed LReLU networks with parameter \(\alpha\), for any \(\alpha\in (0,1)\cup (1,\infty)\), can universally approximate \(L^p\)-integrable functions on compact sets. We present the proof of Theorem~\ref{Theorem-Main4} later in Section~\ref{Section-Proof_main4}.
    
    \begin{theorem}\label{Theorem-Main4}
        Let \(\alpha\in (0,1)\cup (1,\infty)\) and let \(\LU_{\mathcal{F}_{\alpha}^\infty}(d)\) be the set of LU-decomposable, smoothed LReLU networks from Definition~\ref{Definition-Smoothed_LReLUs} with parameter \(\alpha\). Then \(\LU_{\mathcal{F}_{\alpha}^\infty}(d)\subset\)\\ 
        \(\mathcal{D}^{\infty}(\mathbb{R}^d,\mathbb{R}^d)\) and it is an \(L^p\)-universal approximator of \(L^p(\mathbb{R}^{d},\mathbb{R}^{d})\) on compact sets. More precisely, for any \(f\in L^p(\mathbb{R}^{d},\mathbb{R}^{d})\) and any compact set \(\mathcal{K}\subset \mathbb{R}^d\), there exists a sequence of LU-decomposable, smoothed LReLU networks with parameter \(\alpha\), denoted by \((\phi_n)_{n\in\mathbb{N}}\subset \LU_{\mathcal{F}_{\alpha}^\infty}(d)\), such that \(\lVert f-\phi_n \rVert_{\mathcal{K},p}\overset{n \to \infty}{\longrightarrow}0\).
    \end{theorem}
    
    Note that Theorem~\ref{Theorem-Main4} shows that smooth diffeomorphisms achieve UA of \(L^p(\mathbb{R}^d,\mathbb{R}^d)\) with respect to the \(L^p\) norm on compact sets, which also implies the famous Theorem 2.5(i) of \citet{Brenier_Gangbo}. Moreover, unlike \citet{Brenier_Gangbo}, we actually obtain an explicit form for the approximating sequences of diffeomorphisms, namely that they can be constructed as LU-decomposable, smoothed LReLU networks with any parameter \(\alpha\in (0,1)\cup (1,\infty)\). This shows that neural networks are not only interesting for various applications due to their flexibility and strong performance, but they are also a versatile theoretical tool for proving rigorous results in the field of mathematical approximation theory.

\subsection{Minimal width for universal approximation of continuous functions with monotone continuous activations}

         Theorem \ref{Theorem-Main4} demonstrates that diffeomorphisms serve as universal approximators of \\\(L^p(\mathbb{R}^d,\mathbb{R}^d)\) with respect to the induced \(L^p\) norm. However, the following lemma shows that the set of continuous bijections, which is a superset of the set of diffeomorphisms, is not a universal approximator of \(C^0(\mathbb{R}^d,\mathbb{R}^d)\) with respect to the supremum norm. More importantly, continuous functions can even be found such that the supremum norm distance to any continuous bijection is arbitrarily large. It follows that a result similar to that of Theorem \ref{Theorem-Main4}, which constructs a dense set of diffeomorphisms given by neural networks, cannot hold for uniform universal approximation. 

    \begin{remark}\label{Remark-D0_Inverse_continuous}
        Note that by the invariance of domain theorem (see e.g. Corollary 19.8 in \cite{Bredon1993}), every continuous bijection \(\phi:\mathbb{R}^d\to\mathbb{R}^d\) is an open mapping. Therefore, for any open set \(U\subset \mathbb{R}^d\), the preimage \((\phi^{-1})^{-1}(U)=\phi(U)\) is open, showing that the inverse \(\phi^{-1}\) is continuous. Hence \(\phi\) is a homeomorphism, i.e., \(\mathcal{D}^0(\mathbb{R}^d,\mathbb{R}^d)\) is exactly the set of all continuous bijections from \(\mathbb{R}^d\) to \(\mathbb{R}^d\).
    \end{remark}

    \begin{lemma}\label{Lemma-Diffeos_not_dense_in_C0}
        Let \(\mathcal{K}\subset \mathbb{R}^{d}\) be a compact set with non-empty interior, then for any \(c>0\) there exists \(f\in C^0(\mathbb{R}^{d},\mathbb{R}^{d})\) such that for all \(\phi\in \mathcal{D}^0(\mathbb{R}^{d},\mathbb{R}^d)\) it holds that
        \begin{align*}
            \lVert f-\phi \rVert_{\mathcal{K},\sup} \geq \lVert f_1-\phi_1 \rVert_{\mathcal{K},\sup} \geq c\;,
        \end{align*}where \(\phi_1,f_1\) denote the first components of \(\phi\) and \(f\). Consequently, the set \(\mathcal{D}^0(\mathbb{R}^{d},\mathbb{R}^d)\), which consists of all continuous bijections from \(\mathbb{R}^d\) to \(\mathbb{R}^d\), is not a universal approximator of \(C^0(\mathbb{R}^{d},\mathbb{R}^{d})\) on compact sets w.r.t. the supremum norm.
    \end{lemma} 

    \begin{lemma}\label{Lemma-Relation_one_to_one_mon}
        It holds that \(\sigma\in C_{\mon}^{0}(\mathbb{R},\mathbb{R})\) if and only if \(\sigma\) is continuous and there exists a sequence of one-to-one functions, i.e. they are continuous and injective functions from \(\mathbb{R}\) to \(\mathbb{R}\), converging uniformly to \(\sigma\) on \(\mathbb{R}\).
    \end{lemma}

    \begin{proof}
        Let \(\sigma\in C_{\mon}^{0}(\mathbb{R},\mathbb{R})\) and w.l.o.g. assume that it is monotonically increasing, then \(\sigma_n(x):=\sigma(x)+\frac{1}{n}\tanh(x)\), \(n\in \mathbb{N}\), is monotonically increasing, hence one-to-one, and it holds that \(\lVert\sigma_n-\sigma\rVert_{\sup}=\frac{1}{n}\lVert\tanh\rVert_{\sup}=\frac{1}{n}\xrightarrow{n\to \infty}0\). For the other direction, let \((\sigma_n)_{n\in \mathbb{N}}\) be a sequence of one-to-one functions converging uniformly to \(\sigma\). Then for \(x\leq y\) we have \(\sigma(x)=\lim\limits_{n\to\infty}\sigma_n(x)\leq \lim\limits_{n\to\infty}\sigma_n(y)=\sigma(y)\), thus \(\sigma\) is monotone. The decreasing case follows analogously using \(\sigma_n(x):=\sigma(x)-\frac{1}{n}\tanh(x)\).
    \end{proof}
    \bigskip
    The following Theorem \ref{Theorem-Main5} combines 
        Lemma \ref{Lemma-Diffeos_not_dense_in_C0} with auxiliary results concerning activations in \(C_{\mon}^{0}(\mathbb{R},\mathbb{R})\) to conclude that neural networks with continuous and monotone activations cannot be uniform universal approximators of continuous functions with width less than or equal to \(\max\{d_x,d_y\}\) if \(d_y\leq 2d_x\). 

    \begin{theorem}\label{Theorem-Main5}
         Let \(\mathcal{A}\subset C_{\mon}^{0}(\mathbb{R},\mathbb{R})\), i.e. the activations are  monotone and continuous. Let \(d_x,d_y\in \mathbb{N}\), \(d_y \leq  2d_x\), then for any compact \(\mathcal{K}\subset \mathbb{R}^{d_x}\) with non-empty interior, there exists \(c>0\), \(f\in C^0(\mathbb{R}^{d_x},\mathbb{R}^{d_y})\) s.t. \(\lVert f- \phi \rVert_{\mathcal{K},\sup}\geq c\) for all \(\phi \in \FNN_{\mathcal{A}}^{(\max\{d_x,d_y\})}(d_x,d_y)\). Therefore, \(\FNN_{\mathcal{A}}^{(\max\{d_x,d_y\})}(d_x,d_y)\) does not universally approximate \(C^0(\mathbb{R}^{d_x},\mathbb{R}^{d_y})\) on compact sets with respect to the supremum norm, and consequently it is not a global universal approximator either.
    \end{theorem}

    We present the proof of Lemma \ref{Lemma-Diffeos_not_dense_in_C0} and Theorem \ref{Theorem-Main5} in Section \ref{Section-Proof_main5}.

    \begin{remark}\label{Remark-Johnson}
        Theorem~1 of \cite{johnson2018deep} establishes the same lower bound as Theorem~\ref{Theorem-Main5} for the case \(d_x\leq d_y\), but requires the activation to be uniformly continuous and uniformly approximable by one-to-one functions. We strengthen this in two ways: relaxing uniform continuity to mere continuity --- e.g., allowing activations with unbounded derivative such as \(\sigma(x)=\exp(x)\) --- and extending to networks using arbitrary unions of such activations. Both of these combined, by Lemma \ref{Lemma-Relation_one_to_one_mon}, are equivalent to the set of activations being a subset of \(C_{\mon}^{0}(\mathbb{R},\mathbb{R})\). Upon reviewing Johnson's proof, we believe it may already work for continuous activations without relying on uniform continuity, suggesting the original theorem was formulated more restrictively than necessary. Our proof adapts Johnson's approach via unbounded level sets, providing a detailed and rigorous treatment through properties of homeomorphisms that definitively establishes the result for all continuous monotone activations.

        Theorem~3 of \cite{kim2024minimumwidthuniversalapproximation} already handles continuous activations for the case \(d_x < d_y \leq 2d_x\), but considers only single-activation networks. We extend their result to arbitrary unions of activations in \(C_{\mon}^{0}(\mathbb{R},\mathbb{R})\), adapting key elements of their proof while relying on their main argument that injections are not dense in \(C^0\) for this regime.
        
        Thus, Theorem~\ref{Theorem-Main5} unifies and strengthens both prior results under a common framework for multiple-activation networks.
    \end{remark}

    Also note that the Whitney embedding theorem (Proposition 1.0 in Chapter 2 of \cite{hirsch1976differential}) shows that embeddings are dense in $C^0(\mathcal{K},\mathbb{R}^{d_y})$ for any compact $\mathcal{K}\subset \mathbb{R}^{d_x}$ when $d_y > 2d_x$. This is exactly where the proof structure of \cite{kim2024minimumwidthuniversalapproximation} breaks, and therefore also where our proof breaks, as we follow their idea. Their argument constructs a continuous function that cannot be uniformly approximated by injective functions, which exists precisely because embeddings are not dense for $d_x < d_y \leq 2d_x$. For $d_y > 2d_x$, no such counterexample can exist, as every continuous function can be approximated by embeddings. Establishing lower bounds on the minimal width in this regime would therefore require finding structural obstructions specific to neural network function classes, rather than relying on the topological argument that embeddings are not dense.

    Theorem~\ref{Theorem-Main5} applies to most common activation functions \citep{dubey2022activationfunctionsdeeplearning,Activation_functions}, as these are typically continuous and monotone (e.g., sigmoids, ReLU variants). Importantly, it holds even when combining multiple such activations, showing that uniform universal approximation of continuous functions requires networks of width at least \(\max\{d_x,d_y\}+1\) for \(d_y\leq 2d_x\). Since leaky ReLUs with parameters in \((0,1) \cup (1,\infty)\) belong to \(C_{\mon}^{0}(\mathbb{R},\mathbb{R})\), this demonstrates that the minimal width from Theorem~\ref{Theorem-Main1} cannot be achieved for uniform approximation without introducing discontinuous or non-monotone activations in this case.

    Therefore, Theorem~\ref{Theorem-Main5} tells us that achieving optimal minimal width for uniform approximation requires discontinuous activations (e.g., FLOOR or STEP) or non-monotone activations (e.g., \(\swish(x):=\frac{x}{1+\exp(-\beta x)}\)), at least when \(d_y\leq 2d_x\). For the discontinuous case, \citet{park2020minimum,cai2023achieve} showed that FNNs with ReLU+STEP or ReLU+FLOOR can uniformly approximate \(C^0(\mathbb{R}^{d_x},\mathbb{R}^{d_y})\) on compact sets with width \(\max\{d_x+1,d_y\}\) or \(\max\{d_x,d_y\}\) respectively. Moreover, Theorem~\ref{Theorem-Main_sup} shows that FNNs using FLOOR+G-LReLU or SG-LReLU achieve uniform universal approximation on compact sets with width \(\max\{d_x,d_y\}\), while Theorem~\ref{Theorem-Squashable} extends this to FNNs combining FLOOR with squashable or STEP activations, achieving width \(\max\{3,d_x,d_y\}\).

\section{Coding scheme}
\label{Section-coding_scheme}

    The idea of our proof of Theorem \ref{Theorem-Main_sup}, Theorem \ref{Theorem-Main1} and Theorem \ref{Theorem-Squashable} are based on a coding scheme, which was originally presented by \citet{park2020minimum} to show that FNNs with ReLU activations are global approximators of \(L^p\) functions (for details see Theorem 1 in \cite{park2020minimum}). In the following sections we construct the coding scheme to obtain the result that FNNs with leaky ReLU activations universally approximate \(L^p\) functions on compact sets.

\begin{definition}[Coding scheme, section 3 in \citet{park2020minimum}]\label{Definition-Coding_scheme}
		Let \(f^{*}:[0,1]^{d_x}\rightarrow [0,1]^{d_y}\) be a continuous function, which we want to approximate. In what follows, we define the coding scheme for such an arbitrary function with input dimension \(d_x\in \mathbb{N}\) and output dimension \(d_y\in \mathbb{N}\).\\

    \begin{enumerate}
        \item \textbf{Encoder}: For \(n\in \mathbb{N}\) we set \(\mathcal{C}_n:=\{0\cdot 2^{-n}, 2^{-n}, 2\cdot 2^{-n},...,1-2^{-n}\}\) and define quantization functions by \(q_n:\mathbb{R}\rightarrow \mathbb{R},\) \(q_n(x):=\max\{z\in \mathcal{C}_n| z\leq x\}\). For \(K\in \mathbb{N}\) the encoder \(\mathfrak{e}_K:\mathbb{R}^{d_x}\rightarrow \mathcal{C}_{Kd_x}\) is a function that, beginning at the \((K+1)\)-th bit, drops all bits of the binary representation of each component of a vector and then pushes all components into one-dimensional binary representation with a large amount of bits. It is defined by
		\begin{align}\label{formula-encode}
			\mathfrak{e}_K(x):= \sum_{i=1}^{d_x}2^{-(i-1)K} \cdot q_K(x_i)\,.
		\end{align}
        \item \textbf{Memorizer:} The memorizer applies the continuous function \(f^{*}\) to the places of the encoded value corresponding to each dimension, quantizes the results and encodes it again. Thus for \(M\in \mathbb{N}\) representing the places of quantization
		\begin{align}\label{formula-Memorizer}
			\mathfrak{m}_{K,M,f^{*}}:\mathcal{C}_{Kd_x}\longrightarrow \mathcal{C}_{Md_y},\;\mathfrak{m}_{K,M,f^{*}}(\mathfrak{e}_K(x)):=\mathfrak{e}_M(f^{*}(q_K(x))),
		\end{align}where \(q_K(x):=(q_K(x_1),...,q_K(x_{d_x}))\) is the component-wise application of the quantization. As any encoded value corresponds to an unique quantization \(q_K(x)\), the memorizer is well-defined for each \(\mathfrak{e}(x)\in \mathcal{C}_{Kd_x}\).
        \item \textbf{Decoder:} The decoder is used to transform the memorized message back to a quantized vector in \(\mathbb{R}^{d_y}\). For \(M\in \mathbb{N}\) it is given by 
		\begin{align}
			\mathfrak{d}_M:\mathcal{C}_{Md_y}\rightarrow \mathcal{C}_M^{d_y}, \mathfrak{d}_M(c):=\mathfrak{e}_M^{-1}(c)\cap \mathcal{C}_M^{d_y},
		\end{align}where \(\mathcal{C}_M^{d_y}:=\underbrace{\mathcal{C}_M\times ... \times \mathcal{C}_M}_{d_y\text{-times}}\) and \(\mathfrak{e}_M^{-1}\) is the pre-image of the encoder.\\~\\Now we define the coding scheme by composing the three functions from above for a fixed continuous \(f^{*}\), hence for \(K,M\in \mathbb{N}\) we set 
		\begin{align*}
			\mathfrak{c}_{K,M,f^{*}}:\mathbb{R}^{d_x}\rightarrow \mathbb{R}^{d_y}, \mathfrak{c}_{K,M,f^{*}}(x):=\mathfrak{d}_M\circ \mathfrak{m}_{K,M,f^{*}}\circ \mathfrak{e}_K(x)\,.
		\end{align*}
    \end{enumerate}
    
	\end{definition}

\tikzset{
    startstop/.style =
        {rectangle, rounded corners, minimum width=5cm, minimum height=1cm, text centered, draw=black, fill=red!30},
    io/.style = 
        {trapezium, trapezium left angle=70, trapezium right angle=110, minimum width=5cm, minimum height=1cm, text centered, text width=3cm,draw=black, fill=blue!30},
    process/.style = {rectangle, rounded corners, minimum width=1cm, minimum height=1cm, text centered, text width=2.9cm, draw=black, fill=violet!10},
    process_equation/.style = {rectangle, rounded corners, minimum width=3cm, minimum height=1cm, text centered, text width=4.9cm, draw=black, fill=violet!10},
    process_equation_1/.style = {rectangle, rounded corners, minimum width=3cm, minimum height=1cm, text centered, text width=3.4cm, draw=black, fill=violet!10},
    decision/.style = {diamond, minimum width=3cm, minimum height=1cm, text centered, draw=black, fill=green!30},
    arrow/.style = {thick,->,>=stealth},
    myfit/.style={draw,dashed,blue, inner xsep=7pt, inner ysep=12pt, rounded corners=5pt},
    mytitle/.style={draw,densely dashed,blue, fill=violet!50, inner sep=5pt, right, xshift=50pt}
    }

     \begin{figure}[htbp]
\centering
\begin{tikzpicture}[
    node distance=3.4cm
  ]
\node (input) [process, xshift=0cm, yshift=0cm] {x=\(\begin{pmatrix}
    \textcolor{red}{0.1101111}\\
    \textcolor{green}{0.0100010}\\
    \textcolor{blue}{0.1010101}
\end{pmatrix}\)};
\node (quant_input) [process, right of=input, xshift=1.8cm] {
\(\begin{pmatrix}
    \textcolor{red}{0.1101}\\
    \textcolor{green}{0.0100}\\
    \textcolor{blue}{0.1010}
\end{pmatrix}\)};
\node (combine) [process, right of=quant_input, xshift=2.5cm] {0.\textcolor{red}{1101}\textcolor{green}{0100}\textcolor{blue}{1010}};

\node (memorize_decode) [process_equation, below of=combine, xshift=-0.95cm, yshift=0cm] {\(f^{*}\bigg(\begin{pmatrix}
    \textcolor{red}{0.1101}\\
    \textcolor{green}{0.0100}\\
    \textcolor{blue}{0.1010}
\end{pmatrix}\bigg)=\begin{pmatrix}
     \textcolor{red}{0.011}\\
    \textcolor{green}{0.001}\\
    \textcolor{blue}{0.111}\\
    \textcolor{orange}{0.110}
\end{pmatrix}\)};

\node (encode_memorize) [process, left of=memorize_decode, xshift=-1.7cm] {0.\textcolor{red}{01}\textcolor{green}{00}\textcolor{blue}{11}\textcolor{orange}{11}};
\node (decode) [process_equation_1, left of=encode_memorize, xshift= -1.5cm] {\(\mathfrak{c}_{4,2,f^{*}}(x)=\begin{pmatrix}
    \textcolor{red}{0.01}\\
    \textcolor{green}{0.00}\\
    \textcolor{blue}{0.11}\\
    \textcolor{orange}{0.11}
\end{pmatrix}\)};

\draw [arrow] (input) -- node[anchor=south] {\(\mathfrak{q}_{4}(\cdot)\)} (quant_input);
\draw [arrow] (quant_input) -- node[anchor=south] {\textbf{combine}} (combine);
\draw [arrow] (combine) -- node[anchor=south east] {\(f^{*}(\cdot)\)} (memorize_decode);
\draw [arrow] (memorize_decode) -- node[anchor=south] {\(\mathfrak{e}_2(\cdot)\)} (encode_memorize);
\draw [arrow] (encode_memorize) -- node[anchor=south] {\(\mathfrak{d}_2(\cdot)\)} (decode);

\node[fit=(quant_input)(combine),myfit] (encoder) {};
\node[mytitle] at (encoder.north west) {\(\mathfrak{e}_{4}(\cdot)\)};

\node[fit=(memorize_decode)(encode_memorize),myfit] (memorizer) {};
\node[mytitle] at (memorizer.north west) {\(\mathfrak{m}_{4,2,f^{*}}(\cdot)\)};

\node[fit=(decode),myfit] (decoder) {};
\node[mytitle] at (decoder.north west) {\(\mathfrak{d}_{2}(\cdot)\)};

\end{tikzpicture}

\caption{A visualization of the coding scheme on an explicit example: Note that the considered \(f^{*}\) fulfills \(f^{*}(\mathfrak{q}_4(x))=(0.011,\hdots,0.110)^T\), which leads to the shown output of the memorizer. Moreover the parameter of the encoder needs to coincide with the first parameter of the memorizer and the second parameter of the memorizer needs to coincide with the parameter of the decoder for the coding scheme to be well-defined. For details about the different parts of the coding scheme we refer the reader to Definition \ref{Definition-Coding_scheme}.\label{fig:coding_scheme}}
\end{figure}

        The general idea of the coding scheme can be broken down in several key parts, which are visualized in figure \ref{fig:coding_scheme}. First we note that we can bijectively transform points on a multi-dimensional grid into points on a one-dimensional grid that has the same amount of grid points as the multi-dimensional grid, which is done by the \textit{encoder}. Points that are not on the grid are usually transformed to suitable approximations of some grid point, which depends on the explicit construction of the encoder. By defining a function between the encoded points on the multi-dimensional grid and their corresponding encoded values under a target function \(f^*:[0,1]^{d_x}\rightarrow [0,1]^{d_y}\) one obtains the memorizer corresponding to \(f^*\) and the grids. Therefore the memorizer of \(f^*\) is a one-dimensional real-valued function that holds the information of all values of \(f^*\) on a multi-dimensional grid. The last part of the coding scheme is given by the decoder, which reconstructs the corresponding multi-dimensional function values of \(f^*\) given their corresponding encoded values as obtained by the memorizer. Choosing the multi-dimensional grids (thus also the one-dimensional) sufficiently fine, any such function \(f^*:[0,1]^{d_x}\rightarrow [0,1]^{d_y}\) can be approximated with arbitrary precision in terms of the \(L^p\) norm for any \(p\in [1,\infty)\) as we show shortly in Lemma \ref{Lemma-Accuracy_coding_scheme}. This is possible as the function \(f^*\) is uniformly continuous and therefore the perfect reconstruction of \(f^*\) on the grids, achieved by the coding scheme is sufficient to approximate it well on the majority of the volume. Altogether this makes the coding scheme an enormously helpful tool for the approximation of multi-dimensional continuous functions w.r.t. the \(L^p\) norms for \(p\in[1,\infty)\).

    \begin{definition}[Modulus of continuity]\label{Definition-Modulus_of_cont}
        Let \(\lVert \cdot \rVert_{(1)}\) be a norm on \(\mathbb{R}^{d_x}\), \(\lVert \cdot \rVert_{(2)}\) a norm on \(\mathbb{R}^{d_y}\), \(\mathcal{K} \subseteq \mathbb{R}^{d_x}\), and \(f:\mathcal{K}\rightarrow \mathbb{R}^{d_y}\) a continuous function. For \(r\geq 0\) we define the modulus of continuity of \(f\) on \(\mathcal{K}\) by
        \[\omega_{f,\mathcal{K}}(r) = \sup_{\substack{x,y \in \mathcal{K} \\ \|x-y\|_{(1)} \leq r}} \|f(x) - f(y)\|_{(2)}.\]
    \end{definition}
    
    \begin{remark}
        The modulus of continuity is monotone increasing and satisfies \(\omega_{f,\mathcal{K}}(0) = 0\). Moreover, \(f\) is uniformly continuous on \(\mathcal{K}\) if and only if \(\omega_{f,\mathcal{K}}(r) \to 0\) as \(r \to 0\). In particular, if \(\mathcal{K}\) is compact, then every continuous function \(f:\mathcal{K} \to \mathbb{R}^{d_y}\) is uniformly continuous, hence \(\omega_{f,\mathcal{K}}\) is continuous at \(0\).
    \end{remark}
    
        The following lemma demonstrates that the coding scheme is capable of uniformly approximating any continuous function on a multidimensional unit cube with arbitrary precision and determines its accuracy. 
        By applying suitable linear transformations, this result can usually be extended to functions defined on and taking values in arbitrary compact sets, 
        a generalization that we establish later in Lemma \ref{Lemma-Reduce_universal_approximation_to_unit_cube}.

        \begin{lemma}\label{Lemma-Accuracy_coding_scheme}
        For compactly-supported, continuous \(f^{*}:\mathbb{R}^{d_x}\rightarrow \mathbb{R}^{d_y}\), with support in \([0,1]^{d_x}\) it holds that 
        \begin{align*}
            \lVert f^{*}-\mathfrak{c}_{K,M,f^{*}}\rVert_{[0,1]^{d_x},\sup} \leq  \omega_{f^{*},[0,1]^{d_x}}(2^{-K})+2^{-M}\;,
        \end{align*}
        where \(\omega_{f^{*},[0,1]^{d_x}}\) is the modulus of continuity of \(f^{*}\) on \([0,1]^{d_x}\) w.r.t. the \(\ell^\infty\)-norm (maximum norm) on both \(\mathbb{R}^{d_x}\) and \(\mathbb{R}^{d_y}\). Thus for any \(\epsilon>0\), by choosing \(K,M\) sufficiently large, we find 
        \begin{align*}
            \lVert f^{*}-\mathfrak{c}_{K,M,f^{*}}\rVert_{[0,1]^{d_x},\sup}<\epsilon\;\,. 
        \end{align*}
        This additionally implies that for any \(\epsilon>0\), by choosing \(K,M\) sufficiently large, we find 
        \begin{align*}
            \lVert f^{*}-\mathfrak{c}_{K,M,f^{*}}\rVert_{[0,1]^{d_x},p}<\epsilon\;\,. 
        \end{align*}
        \end{lemma}
        
        \begin{proof}
        Let \(\epsilon>0\), first we note that for \(x_i\in[0,1]\) the quantizer \(q_K(x_i)\) discards all but the first \(K\) places of the binary representation, thus for \(K\in \mathbb{N}\)
        \begin{align}\label{formula-Quantization1}
            \sup_{x_i\in[0,1]}|x_i-q_K(x_i)|\leq \sum_{i=K+1}^{\infty}2^{-i}=\frac{2^{-K-1}}{2^{-1}}=2^{-K}\,.
        \end{align}
        Hence, by choosing \(K,M\) large enough s.t. \(\omega_{f^{*},[0,1]^{d_x}}(2^{-K})<\frac{\epsilon}{2}\), \(2^{-M}<\frac{\epsilon}{2}\), we obtain
        
        \begin{align*}
            &\lVert f^{*}-\mathfrak{c}_{K,M,f^{*}}\rVert_{[0,1]^{d_x},\sup}\leq \lVert f^{*}-f^{*}\circ q_K  \rVert_{[0,1]^{d_x},\sup} +\lVert f^{*}\circ q_K -q_M\circ f^{*}\circ q_K \rVert_{[0,1]^{d_x},\sup}\\
            &=\sup_{x\in [0,1]^{d_x}}\lVert f^{*}(x)-f^{*}\circ q_K(x) \rVert_{\max} +\sup_{x\in [0,1]^{d_x}}\lVert f^{*}\circ q_K(x) -q_M\circ f^{*}\circ q_K(x) \rVert_{\max}\\
            &\leq \sup_{x\in [0,1]^{d_x}}\omega_{f^{*},[0,1]^{d_x}}(\underbrace{\lVert x  - q_K(x)\rVert_{\max}}_{\leq  2^{-K}  \text{ by (\ref{formula-Quantization1})}}) +\sup_{x\in [0,1]^{d_x}}\underbrace{\lVert f^{*}\circ q_K(x) -q_M\circ f^{*}\circ q_K(x)\rVert_{\max}}_{\leq 2^{-M} \text{ by (\ref{formula-Quantization1})}}\\
            &\leq \omega_{f^{*},[0,1]^{d_x}}(2^{-K})+ 2^{-M}<\frac{\epsilon}{2}+\frac{\epsilon}{2}=\epsilon\;.
        \end{align*}
        \end{proof}

    Lemma \ref{Lemma-Accuracy_coding_scheme} serves as a central tool for constructing universal approximators. 
    In our approaches, we first construct feedforward neural networks that accurately approximate the coding scheme itself on arbitrary continuous functions. 
    Subsequently, we use these coding scheme approximators and their structural properties in conjunction with the lemma to obtain universal approximation results for larger classes of functions.

    \begin{remark}\label{Remark-Practice_coding_scheme}
        As already discussed in Remark~\ref{Remark-First_layer_depth_coding_scheme} (see also Remark~\ref{Remark-Layer_depth_coding_scheme} for a detailed explanation), 
        the depth of the approximations constructed via the coding scheme increases exponentially with the desired approximation accuracy. 
        Consequently, this behavior may cause rapid growth in runtime in practical implementations. 
        The present remark highlights another significant limitation of the coding scheme in practice, namely the issue of numerical representation: 
        a datatype would require an impractically high level of precision to accurately encode the scheme. 
        The intention here is to emphasize that, although the coding scheme serves as a powerful theoretical tool for establishing various intricate universal approximation results, 
        its explicit constructions are not intended for direct practical use and will typically fail unless adapted suitably. 
    \end{remark}

\section{Expressing piecewise linear functions with leaky ReLUs}
\label{sec:piecewise-linear}\label{Section-Piecewise_linear}
	
		In this section, we thoroughly investigate the expressive power of width-one FNNs with LReLU activations, where we also look at G-LReLUs, allowing negative slopes, and at S-LReLUs with steps. Moreover, we show that one-dimensional FNNs with LReLU activations can be used to exactly construct monotone piecewise linear continuous functions. Furthermore, any piecewise linear function can be decomposed into a sum of two strictly monotone, i.e. either strictly increasing or decreasing, piecewise linear functions. Therefore, we show that one can express arbitrary one-dimensional piecewise linear functions with two streams of one-dimensional leaky ReLU FNNs. These results are fundamental for obtaining suitable approximations of the coding schemes in the following sections.

      \begin{definition}
        For \(a,b,c,d\in \mathbb{R}\) we define 
        \begin{align*}
            \sigma_{a,b}^{c,d}(x):=\begin{cases}
                a(x-c)+d & x<c\\
                b(x-c)+d & x\geq c
            \end{cases},
        \end{align*}
        which are all continuous piecewise linear functions with two pieces. We set \(\sigma_{a,b}:=\sigma_{a,b}^{0,0}\) and additionally define 
        \begin{align*}
            \rho_{a, s}^{c,d}(x):=\begin{cases}
                a x + d & x < c\\
                x + d + s & x\geq c
            \end{cases}.
        \end{align*}
    \end{definition}
    
    In the first result of this section, we prove some of the most fundamental properties that leaky ReLUs satisfy. Conceptually, part (2.) of the next lemma shows that leaky ReLUs can be rescaled with affine functions such that the slopes for positive and negative values can be arbitrary positive numbers. The reader should also note that leaky ReLUs with non-zero slope parameters are invertible functions, which is obvious from their structure, but conceptually important for the generalization of our results to diffeomorphisms in the later sections of this paper.

        \begin{lemma}\label{Lemma-LReLU_properties}
        ~\\(1.) LReLUs are positively homogenous, thus for \(a>0\), \(\alpha\in (0,\infty)\) it holds that \(\sigma_{\alpha}(ax)=a\sigma_{\alpha}(x)\).
        \\~\\(2.) For any fixed \(\alpha\in (0,1)\cup (1,\infty)\) it holds that \(\Id=\sigma_1\in \FNN_{\{\sigma_{\alpha}\}}^{(1)}(1,1;3)\) and \(\ABS(\cdot)=\sigma_{\shortminus 1}\in \FNN_{\{\sigma_{\shortminus\alpha},\sigma_\alpha\}}^{(1)}(1,1;5)\).
        \\~\\(3.)  For \(a, b\in \mathbb{R}\setminus \{0\}\), \(c,d\in \mathbb{R}\) it holds that \(\sigma_{a,b}^{c,d}\in \FNN^{(1)}_{\mathcal{F}_{\pm}}(1,1)\) and if \(\sgn(a)=\sgn(b)\), then we also have \(\sigma_{a,b}^{c,d}\in \FNN^{(1)}_{\mathcal{F}_{+}}(1,1)\).
            ~\\(4.) Let \(\alpha\in (0,\infty)\), \(s\in [0,1]\), then we have \(\rho_{1,s}\in \FNN^{(1)}_{\mathcal{F}_{\alpha,\mathfrak{s}}}(1,1)\) and \(\rho_{1,s}\in \FNN^{(1)}_{\mathcal{F}_{+,1}^{*}}(1,1)\).
        ~\\(5.) Let \(\alpha\in (0,1)\cup (1,\infty)\), \(n\in \mathbb{N}\) and denote with \(\sigma_\alpha^n:=\underbrace{\sigma_\alpha\circ \hdots \sigma_\alpha}_{n\text{ times}}\in \FNN_{\{\sigma_{\alpha}\}}^{(1)}(1,1;n+1)\), then, we have \(\sigma_\alpha^n=\sigma_{\alpha^n}\) and there exists \((\sigma_n)_{n\in \mathbb{N}}\subset \FNN_{\{\sigma_{\alpha}\}}^{(1)}(1,1;n+1)\), such that \(\lim\limits_{n\to\infty}\lVert \sigma_n -\ReLU \rVert_{\mathcal{K},\sup}=0 \) for any compact \(\mathcal{K}\subset \mathbb{R}\).
    \end{lemma}
        \begin{proof}
            \textbf{(Proof of (1.)):} This is immediately clear as each piecewise linear function is positively homogenous and as \(\sgn(ax)=\sgn(x)\) for all \(x\in \mathbb{R}\), \(a>0\).
            \\\textbf{(Proof of (2.)):} This follows from
            \begin{align*}
                &\Id=\frac{1}{\alpha}\sigma_{\alpha}(-\sigma_{\alpha}(-x))\overset{1.}{=}\sigma_{\alpha}(-\alpha^{-1}\sigma_{\alpha}(-x))\in \FNN_{\{\sigma_{\alpha}\}}^{(1)}(1,1;3)\;,\\
                &\ABS(x)\overset{(1.)}{=} \sigma_{-\alpha}\circ \sigma_{\alpha}\circ (- \alpha^{-2}\sigma_{\alpha})\circ\sigma_{\alpha}(-x))\in \FNN_{\{\sigma_{\shortminus\alpha},\sigma_\alpha\}}^{(1)}(1,1;5)\;.
            \end{align*}
            \\\textbf{(Proof of (3.)):} Let \(a,b\in \mathbb{R}\setminus\{0\}\), set \(\alpha_1:=\frac{\min\{|a|,|b|\}}{\max\{|a|,|b|\}}\), \(\alpha_2:=\sgn(a)\sgn(b)\), \(\beta_1:=\max\{|a|,|b|\}\) and \(\beta_2:=\begin{cases}
                \sgn(b) & |b| \geq |a| \\
                -\sgn(a) & |b| < |a|
            \end{cases}\). Then it holds that
            \begin{align*}
                \sigma_{a,b}^{c,d} = (\beta_2\Id+ d)\circ \sigma_{\alpha_2} \circ (\beta_1\Id) \circ \sigma_{\alpha_1}\circ (\Id - c) \in \FNN^{(1)}_{\mathcal{F}_{\pm}}(1,1)\;,
            \end{align*}where we used that \(\sigma_1,\sigma_{-1}\in \FNN^{(1)}_{\mathcal{F}_{\pm}}(1,1)\) by 2., which is necessary to construct \(\sigma_{\alpha_i},i\in [2]\), as \(\alpha_1\in (0,1]\), \(\alpha_2\in \{-1,1\}\). Additionally we see that \(\sigma_{a,b}^{c,d}\in \FNN^{(1)}_{\mathcal{F}_{+}}(1,1)\) if \(\sgn(a)=\sgn(b)\), because then \(\alpha_2=1\), hence we don't need negative LReLU parameters for the construction anymore.
            \\\textbf{(Proof of (4.)):} For \(\alpha\in (0,\infty)\), \(s\in [0,1]\) we have 
            \begin{align*}
                \rho_{1,s}=(-\Id)\circ\sigma_{\alpha}\circ (-\alpha^{-1}\Id)\circ \rho_{\alpha,s}\in \FNN_{\mathcal{F}_{\alpha,\mathfrak{s}}}^{(1)}(1,1)\;.
            \end{align*}Moreover, for \(a\in\mathbb{R}\), \(c,d\in \mathbb{R}\) we have that 
            \begin{align*}
                \rho_{1,s} = (-s\Id+s)\circ \sigma_{s^{-1}}\circ (-\Id+1)\circ \rho_{s^{-1},1} \in \FNN^{(1)}_{\mathcal{F}^{*}_{+,1}}(1,1) \;.
            \end{align*}
            \\\textbf{(Proof of (5.)):} One can directly see that \(\sigma_\alpha^n = \sigma_{\alpha^n}\), now, if \(\alpha \in (0,1)\) we set \(\sigma_n:=\sigma_{\alpha^n}\in \FNN_{\{\sigma_{\alpha}\}}^{(1)}(1,1;n+1)\) and if \(\alpha \in (1,\infty)\) we set \(\sigma_n:=\sigma_{\alpha^{-n}}=(-\alpha^{-n}\Id)\circ \sigma_{\alpha^n}\circ (-\Id)\) with \(\sigma_n\in \FNN_{\{\sigma_{\alpha}\}}^{(1)}(1,1;n+1)\) for all \(n\in \mathbb{N}\). Hence, if \(\alpha \in (0,1)\) we have \(\lim \limits_{n\to\infty}\alpha^n=0\) and if \(\alpha \in (1,\infty)\) we have \(\lim \limits_{n\to\infty}\alpha^{-n}=0\), therefore, for any compact \(\mathcal{K}\subset \mathbb{R}\) it holds that
            \begin{align*}
                \lim\limits_{n\to\infty}\lVert \sigma_n -\ReLU \rVert_{\mathcal{K},\sup}=|\inf(\mathcal{K})|\lim\limits_{n\to\infty} | \beta_n|=0\;.
            \end{align*}
        \end{proof}      

       \begin{definition}[Zig-zag functions]\label{Definition-Zig_zag}
            Let \(a, b \in \mathbb{R}\), \(a<b\), for some \(n\in \mathbb{N}\) let \(a_0<b_0<a_1<\hdots <b_{n-1}<a_n \) and define \(I_0:=(-\infty,a_0]\), \(I_i:=[a_i,b_i]\), \(i\in [n]\), \(J_i:=[b_i,a_{i+1}]\), \(i\in [0,n-1]\), \(J_n:=[b_n,\infty)\). Then we call the function 
            \begin{align*}
                h(x):=\begin{cases}
                    \frac{b - a}{b_i - a_i} (x - a_i) + a\;, & \forall i\in [0,n],\;\forall x\in I_i,\\
                    \frac{a-b}{a_{i+1} - b_i} (x - b_i) + b\;, & \forall i\in [0,n],\;\forall x\in J_i,
                \end{cases}
            \end{align*}
            a zig-zag function between \(a\) and \(b\) with \(n\) bumps (corresponding to the points \(a_i,b_i\)). Additionally we define \(\mathcal{Z}^{(n)}_{a,b}:=\{h:\mathbb{R}\rightarrow \mathbb{R}\,|\,h \text{ is a zig-zag function between \(a\) and \(b\) with \(n\) bumps}\}\) and \(\mathcal{Z}_{a,b}:=\bigcup_{n=1}^{\infty}\mathcal{Z}^{(n)}_{a,b}\).
        \end{definition}

        \begin{lemma}\label{Lemma-Zig_zag_LReLU}
            For \(n\in \mathbb{N}\), \(a,b\in \mathbb{R}\), \(a<b\) and any \(h\in \mathcal{Z}_{a,b}\) corresponding to \(a_0<b_0<\hdots <b_{n-1}<a_{n}\), there exists \(f\in\FNN_{\mathcal{F}_{\pm}}^{(1)}(1,1)\) such that for \(I:=[a_0,\infty)\) it holds that \(f\vert_{I}=h\vert_{I}\).
        \end{lemma}
        
        \begin{proof}
            We show by induction over \(n\in \mathbb{N}\), that for all \(a,b\in \mathbb{R}\), \(a<b\) and any \(h\in \mathcal{Z}^{(n)}_{a,b}\), w.r.t. \(a_0,b_0,\hdots, b_{n-1},a_{n}\), there exists \(f\in\FNN_{\mathcal{F}_{\pm}}^{(1)}(1,1)\) such that \(f\vert_{[a_0,\infty)}=h\vert_{[a_0,\infty)}\), which directly implies the claim by definition of \(\mathcal{Z}_{a,b}\).
            
            \textbf{(n=1):} For \(a_0,b_0,a_1\), any zig-zag function between \(a\) and \(b\) with one bump is given by
            \begin{align*}
                h(x)=\begin{cases}
                    \frac{b-a}{b_0-a_0}(x-a_0) + a & x\in [a_0, b_0] \\
                    \frac{a-b}{a_1-b_0}(x-b_0) + b & x\in [b_0, \infty),
                \end{cases}
            \end{align*}
            which, by setting \(\alpha:=\frac{b-a}{b_0-a_0}\), \(\beta:=\frac{a-b}{a_1-b_0}\), is equal to \(\sigma_{\alpha, \beta}^{b_0,b}\), which is an element of \(\FNN_{\mathcal{F}_{\pm}}^{(1)}(1,1)\) by Lemma \ref{Lemma-LReLU_properties} (3), since \(\alpha > 0\) and \(\beta < 0\) (as \(a < b\)).
            
            \textbf{(n → n+1):} Consider a zig-zag function \(h\) with \(n+1\) bumps between \(a\) and \(b\) corresponding to \(a_0,b_0,\hdots, b_{n},a_{n+1}\), and let \(g\) be the zig-zag function with only the first \(n\) bumps of \(h\) fulfilling \(h|_{[a_0, a_n]} = g|_{[a_0, a_n]}\). Then by induction hypothesis there is \(\psi_1\in \FNN_{\mathcal{F}_{\pm}}^{(1)}(1,1)\) such that \(\psi_1\vert_{[a_0,\infty)} = g\vert_{[a_0,\infty)}\). 
            
            We adapt \(\psi_1\) by applying generalized leaky ReLUs such that it fits \(h\) on \([a_0,\infty)\). First, we fold \(\psi_1\) up at \(a_{n}\) such that it runs through \(b_{n}\), which can be done by applying 
            \begin{align*}
                \sigma_{\alpha,1}^{a,a} = \begin{cases}
                    \alpha (y-a) + a & y< a\\
                    (y-a) + a = y & y\geq a
                \end{cases}
            \end{align*}
            with \(\alpha = \frac{-(a_n-b_{n-1})}{b_n-a_n}<0\). 
            
            Set \(\psi_2:=\sigma_{\alpha,1}^{a,a} \circ \psi_1\), which is in \(\FNN_{\mathcal{F}_{\pm}}^{(1)}(1,1)\) by Lemma \ref{Lemma-LReLU_properties} (3) and fulfills
            \begin{align*}
                \psi_2(x)=\sigma_{\alpha,1}^{a,a}(\psi_1(x))=\psi_1(x)=g(x)=h(x) \;\forall x\in [a_0,a_n]
            \end{align*}
            as \(\psi_1(x)\geq a\) on \([a_0,a_n]\). For \(x\geq a_n\), we have \(g(x) = \frac{a-b}{a_n - b_{n-1}}(x-a_n)+a < a\), thus
            \begin{align*}
                \psi_2(x)&=\sigma_{\alpha,1}^{a,a}(g(x))=\alpha(g(x)-a) + a\\
                &=\alpha\left(\frac{a-b}{a_n - b_{n-1}}(x-a_n)\right) + a\\
                &=\frac{b-a}{b_n-a_n}(x-a_n) + a\;.
            \end{align*}
            
            For constructing the \((n+1)\)-th bump, we fold down by applying 
            \begin{align*}
                \sigma_{1,\beta}^{b,b}=\begin{cases}
                    (y-b) + b = y & y\leq b\\
                    \beta (y-b) + b & y> b
                \end{cases}
            \end{align*}
            with \(\beta:=\frac{-(b_n-a_n)}{a_{n+1}-b_n}<0\). Set \(\psi:=\sigma_{1,\beta}^{b,b}\circ \psi_2\), which is in \(\FNN_{\mathcal{F}_{\pm}}^{(1)}(1,1)\) by Lemma \ref{Lemma-LReLU_properties} (3). It fulfills
            \begin{align*}
                \psi(x)=\sigma_{1,\beta}^{b,b}(\psi_2(x))=\psi_2(x)=h(x) \;\forall x\in [a_0,b_n]
            \end{align*}
            as \(\psi_2(x)\leq b\) on \([a_0,b_n]\). For \(x\geq b_n\), we have \(\psi_2(x) = \frac{b-a}{b_n-a_n}(x-a_n)+a\), and at \(x=b_n\): \(\psi_2(b_n) = b\). For \(x > b_n\), \(\psi_2(x) > b\), thus
            \begin{align*}
                \psi(x)&=\beta(\psi_2(x)-b) + b=\beta\left(\frac{b-a}{b_n-a_n}(x-a_n)+a-b\right) + b\\
                &=\beta\left(\frac{b-a}{b_n-a_n}(x-b_n)\right) + b=\frac{a-b}{a_{n+1}-b_n}(x-b_n) + b\;,
            \end{align*}
            therefore \(\psi \in \FNN_{\mathcal{F}_{\pm}}^{(1)}(1,1)\) fulfills \(\psi\vert_{[a_0,\infty)}=h\vert_{[a_0,\infty)}\), which concludes the proof.
        \end{proof}

    \bigskip
    Now, we introduce our main classes of piecewise linear functions used in the upcoming proofs.
    
    \begin{definition}\label{Definition-PLC}
        We define \(\PLC\) as the set of all piecewise linear and continuous functions \(f:\mathbb{R}\rightarrow \mathbb{R}\) with finitely many slopes. For \(n\in \mathbb{N}\), let \(\PLC_n\subset \PLC\) be the set of piecewise linear and continuous functions with \(n\) slopes. Thus for \(f\in \PLC_n\) there exist slopes \(a_1,...,a_n\in \mathbb{R}\) with \(a_{i}\neq a_{i+1}\) for \(i\in [n-1]\), and partition points \(-\infty=:x_0<x_1<...<x_{n-1}<x_n:=\infty\), where \(x_1,...,x_{n-1}\in\mathbb{R}\) are called the kinks of \(f\), such that \(f\vert_{(x_{i-1},x_{i})}(x)=a_ix+b_i\) for \(i\in [n]\) with suitable \(b_i\in \mathbb{R}\) ensuring \(f\) is continuous. By definition it holds that \(\PLC=\bigcup_{n=1}^{\infty}\PLC_n\).
    \end{definition}
    
    \begin{definition}\label{Definition-PLCSM}
        We define \(\PLCSM\subset \PLC\) as the set of all piecewise linear, continuous and strictly monotone functions \(f:\mathbb{R}\rightarrow \mathbb{R}\) with finitely many slopes. For \(n\in \mathbb{N}\), let \(\PLCSM_n\subset \PLCSM\) be the set of piecewise linear, continuous and strictly monotone functions with \(n\) slopes. Thus for \(f\in \PLCSM_n\) there exist slopes \(a_1,...,a_n\in \mathbb{R}\setminus \{0\}\) with \(a_{i}\neq a_{i+1}\) for \(i\in[n-1]\), which are either all positive or all negative, and partition points \(-\infty=:x_0<x_1<...<x_{n-1}<x_n:=\infty\), where \(x_1,...,x_{n-1}\in\mathbb{R}\) are called the kinks of \(f\), such that \(f\vert_{(x_{i-1},x_{i})}(x)=a_ix+b_i\) for \(i\in [n]\) with suitable \(b_i\in \mathbb{R}\) ensuring \(f\) is continuous. By definition it holds that \(\PLCSM=\bigcup_{n=1}^{\infty}\PLCSM_n\).
    \end{definition}
    
    \begin{remark}\label{Remark-Piecewise_linear_LU_FNN}
        For one-dimensional functions, LU-decomposable weight matrices correspond to invertible scalars in \(\mathbb{R}\setminus\{0\}\). Thus \(\LU_{\mathcal{M}}(1)\) denotes the set of width-1 FNNs with activations in \(\mathcal{M}\) and invertible (non-zero) weight scalars. Note that for each set of activation functions \(\mathcal{M}\) we have \(\LU_{\mathcal{M}}(1)\subset \FNN_\mathcal{M}(1,1)\) by definition and \(\LU_{\mathcal{M}}(1)\subset \FNN_\mathcal{M}^{(1)}(1,1)\), as each \(f\in \LU_{\mathcal{M}}(1)\) corresponds to a width-1 FNN with activations in \(\mathcal{M}\).
    \end{remark}
    
    In the following lemma, we obtain strong approximation results for affine concatenations of invertible leaky ReLUs for piecewise linear, continuous, strictly monotone functions by using the properties from Lemma \ref{Lemma-LReLU_properties}. This result helps us significantly in further proofs, since it allows us to perfectly fit any function in \(\PLCSM\) with a width-1 FNN using invertible leaky ReLU activations. Note that by comparison, ReLUs can approximate any piecewise linear continuous function, but only with a FNN of width \(2\). This difference in the necessary width for fitting piecewise linear functions is the crucial reason that allows us to reduce the width of the considered FNNs with leaky ReLU activations by \(1\) in many cases compared to the similar Theorem 1 in \cite{park2020minimum}, which considers only FNNs with ReLU activations.

	\begin{lemma}\label{Lemma-PLCSM_in_LU}
        For the set \(\mathcal{F}_{+}=\{\sigma_\alpha\,|\alpha\in (0,1)\}\) of invertible leaky ReLUs it holds that \(\PLCSM\subset \LU_{\mathcal{F}_{+}}(1)\subset \FNN_{\mathcal{F}_{+}}^{(1)}(1,1)\).
    \end{lemma}
    
    \begin{proof}
        We show by induction over \(n\in \mathbb{N}\) that for every strictly increasing \(f_n\in \PLCSM_n\) it holds that \(f_n\in \LU_{\mathcal{F}_{+}}(1)\). As for every strictly decreasing function \(f_n\in \PLCSM_n\), \(-f_n\) is strictly increasing and as \(\LU_{\mathcal{F}_{+}}(1)\) is closed under scalar multiplication by non-zero scalars, this implies that \(\PLCSM_n\subset \LU_{\mathcal{F}_{+}}(1)\). Because \(\PLCSM=\bigcup_{n=1}^{\infty}\PLCSM_n\) and by definition \(\LU_{\mathcal{F}_{+}}(1)\subset \FNN_{\mathcal{F}_{+}}^{(1)}(1,1)\), this finishes the proof.
        
        \textbf{(n=1):} \(f\in \PLCSM_1\) is given by \(f(x)=ax+b\) for some \(a\in (0,\infty)\), \(b\in\mathbb{R}\), hence choosing \(g(x)=ax+b\in \Aff(1)\) we find that \(f(x)=g(x)\circ \sigma_1 (x)\in \LU_{\mathcal{F}_{+}}(1)\), where \(\sigma_1\in\FNN_{\mathcal{F}_+}^{(1)}(1,1)\) by Lemma \ref{Lemma-LReLU_properties} (2).
        
        \textbf{(n → n+1):} Let \(f_{n+1}\in \PLCSM_{n+1}\) be strictly increasing, hence as written in Definition \ref{Definition-PLCSM}, it has positive slopes \(a_1,...,a_{n+1}\in (0,\infty)\) with \(a_{i}\neq a_{i+1}\) for \(i\in [n]\), and kinks \(x_1<...<x_n\) such that \(f_{n+1}\vert_{(x_{i-1},x_{i})}(x)=a_ix+b_i\) for \(i\in [n+1]\) with suitable \(b_i\in \mathbb{R}\) ensuring \(f_{n+1}\) is continuous. 
        
        We define 
        \begin{align*}
            f_n(x):=\begin{cases}
                a_2(x-x_2)+f_{n+1}(x_2)=a_2x+b_2 & x\in (-\infty,x_2)\\
                f_{n+1}(x) & x\in [x_2,\infty)
            \end{cases},
        \end{align*}
        which has \(n\) slopes \(a_2,...,a_{n+1}\) with corresponding kinks \(x_2<...<x_n\), thus \(f_n\in \PLCSM_n\subset \LU_{\mathcal{F}_{+}}(1)\), where the last inclusion follows from the induction hypothesis. 
        
        As \(a_1,a_2>0\) by assumption, it follows from Lemma \ref{Lemma-LReLU_properties} (3) that \(\sigma_{a_1,a_2}^{0,0}\in\FNN_{\mathcal{F}_{+}}^{(1)}(1,1)\). Note that because \(f_n\) is strictly increasing, we know that \(f_n(x)-f_n(x_1)\geq 0\) if and only if \(x\geq x_1\). Hence we obtain 
        \begin{align*}
            g(x):=\left(\frac{1}{a_2}\cdot \Id\right)\circ \sigma_{a_1,a_2}^{0,0} \circ (f_n(x)-f_n(x_1))=\begin{cases}
                \frac{a_1}{a_2}(f_n(x)-f_n(x_1)) & x\in (-\infty,x_1)\\
                f_n(x)-f_n(x_1) & x\in [x_1,\infty)
            \end{cases}.
        \end{align*}
        
        Since \(f_n\in \LU_{\mathcal{F}_+}(1)\) by the induction hypothesis, and \(g\) is obtained by composing \(f_n\) with affine transformations (subtracting the constant \(f_n(x_1)\)), then applying \(\sigma_{a_1,a_2}^{0,0}\in\FNN_{\mathcal{F}_+}^{(1)}(1,1)\), and finally scaling by the non-zero scalar \(1/a_2\), we have \(g\in \FNN_{\mathcal{F}_{+}}^{(1)}(1,1)\). Moreover, all weight scalars in this composition are non-zero, preserving the LU property, so \(g\in\LU_{\mathcal{F}_+}(1)\).
        
        We compute the slopes of \(g\): 
        \(g'\vert_{(-\infty,x_1)}(x)=\frac{a_1}{a_2} \cdot a_2 = a_1\), and 
        \(g'\vert_{(x_i,x_{i+1})}(x)=f_n'\vert_{(x_i,x_{i+1})}(x)=a_{i+1}\) for \(i\in [n]\), 
        meaning that \(g\in \PLCSM_{n+1}\) has the same slopes as \(f_{n+1}\). 
        
        Combining this with the continuity of \(f_{n+1}\), we verify that \(g(x)+f_{n}(x_1)=f_{n+1}(x)\) for all \(x\in\mathbb{R}\):
        
        For \(x\in (-\infty,x_1)\):
        \begin{align*}
            g(x)+f_{n}(x_1)&=\frac{a_1}{a_2}(f_n(x)-f_n(x_1))+f_n(x_1)\\
            &=\frac{a_1}{a_2}(a_2x+b_2-a_2x_1-b_2)+a_2x_1+b_2\\
            &=a_1(x-x_1)+a_2x_1+b_2\\
            &\stackrel{f_{n+1}\text{ cont.}}{=}a_1(x-x_1)+a_1x_1+b_1=a_1x+b_1=f_{n+1}(x).
        \end{align*}
        
        For \(x\in [x_1,\infty)\):
        \begin{align*}
            g(x)+f_n(x_1)=f_n(x)-f_n(x_1)+f_n(x_1)=f_n(x)=f_{n+1}(x),
        \end{align*}
        where the last equality follows from the definition of \(f_n\) for \(x\in [x_2,\infty)\), and from continuity at \(x_2\) which ensures \(f_n(x)=f_{n+1}(x)\) also for \(x\in[x_1,x_2)\).
        
        Therefore \(f_{n+1}(x)=g(x)+f_n(x_1)\in \LU_{\mathcal{F}_{+}}(1)\), which completes the induction.
    \end{proof}

        We now use the previous lemma to show that any real-valued function defined on finitely many points can be expressed by a one-dimensional FNN with leaky ReLU activations. 

        \begin{lemma}\label{Lemma-PLCSM_exactly_fit_finite_points}
            For \(n\in \mathbb{N}\) let \(\mathcal{X}=\{x_1,...,x_n\}\subset \mathbb{R}\) with \(x_1<...<x_n\) be a set of finitely many points and \(g:\mathcal{X}\rightarrow \mathbb{R}\). Then there exists \(f\in \PLC_{n-1}\) such that \(f(x)=g(x)\) for all \(x\in \mathcal{X}\). Additionally if \(g\) is either strictly decreasing or strictly increasing, then \(f\) can be chosen such that \(f\in \PLCSM_{n-1}\subset \LU_{\mathcal{F}_{+}}(1) \subset \FNN_{\mathcal{F}_{+}}^{(1)}(1,1)\).
        \end{lemma}
        
        \begin{proof}
            W.l.o.g. we can assume that \(|\mathcal{X}|\geq 3\), as otherwise we can add two arbitrary elements \(x_1<x_2<x_3\) and define \(g(x_2), g(x_3)\) arbitrarily. For \(n\geq 3\) define
            \begin{align*}
                f(x)=
                \begin{cases}
                    \frac{g(x_{2})-g(x_1)}{x_{2}-x_1}(x-x_1) + g(x_1) & x\in (-\infty, x_2)\\
                    \frac{g(x_{i+1})-g(x_i)}{x_{i+1}-x_i}(x-x_i) + g(x_i)  & x\in [x_i,x_{i+1}),\;i\in[2,n-1]\\
                    \frac{g(x_{n})-g(x_{n-1})}{x_{n}-x_{n-1}}(x-x_{n-1}) + g(x_{n-1}) & x\in [x_n,\infty)
                \end{cases},
            \end{align*}
            which is piecewise-linear and continuous by construction, hence \(f\in \PLC_{n-1}\), and \(f(x_i)=g(x_i)\) holds for all \(i\in [n]\) by direct verification. If \(g\) is strictly monotone, i.e. strictly decreasing or strictly increasing, then either \(\frac{g(x_{i+1})-g(x_i)}{x_{i+1}-x_i}<0\) or \(\frac{g(x_{i+1})-g(x_i)}{x_{i+1}-x_i}>0\) for all \(i\in [n-1]\), thus \(f\) is also strictly decreasing or strictly increasing, respectively. Therefore, we have \(f\in \PLCSM_{n-1}\) if \(g\) is strictly monotone and by Lemma \ref{Lemma-PLCSM_in_LU} it follows that \(f\in \LU_{\mathcal{F}_{+}}(1)\subset \FNN_{\mathcal{F}_{+}}^{(1)}(1,1)\), which concludes the proof.
        \end{proof}

       Note that the memorizer \(\mathfrak{m}_{K,M,f^{*}}(\cdot)\) for fixed grid sizes \(K,M\in \mathbb{N}\) and a continuous \(f^{*}: \mathbb{R}^{d_x}\rightarrow \mathbb{R}^{d_y}\) is itself a one-dimensional function that is defined on a finite set and therefore could be expressed exactly by an element in \(\FNN_{\mathcal{F}_{+}}^{(1)}(1,1)\) by Lemmas \ref{Lemma-PLCSM_exactly_fit_finite_points} and \ref{Lemma-PLCSM_in_LU}, if the memorizer were a monotone function. However, since this is not the case for most non-monotone functions \(f^{*}\), we need a way to express the memorizer in terms of strictly monotone functions, even if it is not monotone itself. For this, we introduce the monotone function decomposition for functions in \(\PLC\), which later allows us to express any piecewise linear continuous function with a FNN with leaky ReLU activations of width \(2\) by decomposing it into two monotone functions.

    \begin{definition}
        We say that a function \(f:\mathbb{R}\rightarrow \mathbb{R}\) has a monotone function decomposition if there exist \(f_1:\mathbb{R}\rightarrow \mathbb{R}\) strictly decreasing and \(f_2:\mathbb{R}\rightarrow \mathbb{R}\) strictly increasing, such that \(f(x)=f_1(x)+f_2(x)\) for all \(x\in \mathbb{R}\).
    \end{definition}
    
    \begin{lemma}\label{Lemma-function_decomposition_PLC}
        Let \(f:\mathbb{R}\rightarrow \mathbb{R}\in\PLC\) be a piecewise-linear and continuous function. Then it has a monotone function decomposition with piecewise-linear, continuous, strictly decreasing \(f_1\in \PLCSM\) and piecewise-linear, continuous and strictly increasing \(f_2\in \PLCSM\).
    \end{lemma}
    
    \begin{proof}
        Let \(f:\mathbb{R}\rightarrow \mathbb{R}\) be a piecewise-linear and continuous function with \(n\in \mathbb{N}_0\) kinks \(x_1,...,x_{n}\in \mathbb{R}\), where \(x_0:=-\infty\) and \(x_{n+1}:=\infty\), and slopes \(a_i\in \mathbb{R}\) for \(i\in [0,n]\), such that \(f'(x)|_{(x_i,x_{i+1})}=a_i\) for \(i\in [0,n]\). 
        
        For \(n=0\), the function \(f\) is linear, i.e., \(f(x) = a_0 x + b\) for some \(b \in \mathbb{R}\). Let \(\alpha > 0\) and define \(f_1(x) := -\alpha x + b/2\) and \(f_2(x) := (a_0+\alpha)x + b/2\). Then \(f_1\) is strictly decreasing, \(f_2\) is strictly increasing, and \(f_1(x) + f_2(x) = a_0 x + b = f(x)\).
        
        For \(n \geq 1\), let \(J_1, J_2\) be two disjoint index sets with \(J_1\cup J_2=[0,n]\) such that \(i\in J_1\) if \(a_i<0\) and \(i\in J_2\) if \(a_i\geq 0\). Let \(\alpha>0\), and define the slopes \(a_{j,i}:= a_i+(-1)^{j} \alpha\) if \(i\in J_j\) and \(a_{j,i}:=(-1)^j \alpha\) if \(i\in [0,n]\setminus J_j\) for \(j\in [2]\). Then we inductively define 
        \begin{align*}
            f_j(x)=\begin{cases}
                a_{j,0}(x-x_1) +\frac{f(x_1)}{2} & x\in (-\infty,x_1]\\
                a_{j,i}(x-x_{i}) + f_j(x_i) & x\in [x_i,x_{i+1}],\;i\in [n-1]\\
                a_{j,n}(x-x_{n}) + f_j(x_n) & x\in [x_n,\infty)
            \end{cases}
        \end{align*}
        for \(j\in [2]\). Moreover, by definition of \(J_1,J_2\), the slopes satisfy \(a_{1,i}<0\) and \(a_{2,i}>0\) for all \(i\in [0,n]\), thus \(f_1\) is strictly decreasing and \(f_2\) is strictly increasing. 
        
        We show by induction over the intervals \(I_0:=(-\infty,x_1]\), \(I_i:=[x_i,x_{i+1}]\) for \(i\in [n-1]\), and \(I_n:=[x_{n},\infty)\) that \(f(x)=f_1(x)+f_2(x)\) for all \(x\in \mathbb{R}\). For \(i=0\) we directly compute that
        \begin{align*}
            f_1(x)+f_2(x)=(a_{1,0}+a_{2,0})(x-x_1) +\frac{2f(x_1)}{2}=a_0 (x-x_1) + f(x_1)=f(x)
        \end{align*}
        for all \(x\in I_0\). For the induction step, we assume that \(f(x)=f_1(x)+f_2(x)\) for all \(x\in \bigcup_{i=0}^{m}I_i\), where \(m\in [0,n-1]\), and show that \(f(x)=f_1(x)+f_2(x)\) on \(I_{m+1}\). This holds since \(x_{m+1}\in I_m\) and therefore by the induction hypothesis:
        \begin{align*}
            &f_1(x)+f_2(x)=(a_{1,m+1}+a_{2,m+1})(x-x_{m+1}) + f_1(x_{m+1}) + f_2(x_{m+1})\\
            &\stackrel{\text{Ind. hyp.}}{=} a_{m+1}(x-x_{m+1}) + f(x_{m+1})=f(x)
        \end{align*}
        for all \(x\in I_{m+1}\).
    \end{proof}

        Since neural networks are given by compositions of functions, the question naturally arises under which conditions compositions of converging sequences converge to the composition of limit functions. The following Lemma \ref{Lemma-Approximate_function_compositions_sup_unif_cont} classifies this and gives us a helpful tool for many of the upcoming proofs.

    \begin{lemma}\label{Lemma-Approximate_function_compositions_sup_unif_cont}
        For \(k\in \mathbb{N}\), \(d_i\in \mathbb{N}\), \(i\in [0,k+1]\) and a compact set \(\mathcal{K}\subset \mathbb{R}^{d_{0}}\), let \(g:\mathcal{K}\rightarrow \mathbb{R}^{d_1}\) be bounded and \((g^{(n)})_{n\in \mathbb{N}}\) a sequence of bounded functions converging uniformly to \(g\) on \(\mathcal{K}\). By the boundedness and uniform convergence of the functions, the set \(\mathcal{K}_1:=\overline{g(\mathcal{K})\cup\bigcup_{n\in \mathbb{N}}g^{(n)}(\mathcal{K})}\subset \mathbb{R}^{d_1}\) is compact. Moreover, for \(i\in [k]\), let \(f_i:\mathcal{K}_i\rightarrow \mathbb{R}^{d_{i+1}}\) and \(f_i^{(n)}:\mathcal{K}_i\rightarrow \mathbb{R}^{d_{i+1}}\) be continuous functions, where the compact sets \(\mathcal{K}_{i+1}:=\overline{f_i(\mathcal{K}_i)\cup \bigcup_{n\in \mathbb{N}}f_i^{(n)}(\mathcal{K}_i)}\subset \mathbb{R}^{d_{i+1}}\) are defined recursively for \(i\in [k-1]\), which are indeed compact as the functions are continuous and the uniform convergence ensures boundedness of the union. Moreover, assume that \(\lim\limits_{n\to\infty}\lVert f_i - f_i^{(n)}\rVert_{\mathcal{K}_i,\sup} = 0\) for all \(i\in [k]\). Then, it holds that \(f_k^{(n)}\circ \hdots \circ f_1^{(n)} \circ g^{(n)}\) converges uniformly to \(f_k\circ \hdots \circ f_1\circ g\) on \(\mathcal{K}\) for \(n\to\infty\).
    \end{lemma}
    
    \begin{proof}
        We show the proof iteratively and begin with \(k=1\). Choose an arbitrary but fixed \(\epsilon>0\). By the uniform convergence of \(f_1^{(n)}\) there exists \(n_1\in \mathbb{N}\) with \(\lVert f_1^{(m)}-f_1\rVert_{\mathcal{K}_1,\sup}<\frac{\epsilon}{2}\) for all \(m\geq n_1\), and as \(f_1\) is uniformly continuous on the compact set \(\mathcal{K}_1\) there exists \(\delta>0\) such that for \(x,y\in \mathcal{K}_1\) with \(y\in B_{\delta}(x)\) we have \(\lVert f_1(x) -f_1(y)\rVert<\frac{\epsilon}{2}\). By the uniform convergence of \(g^{(n)}\) we can choose \(n_2\geq n_1\) large enough such that \(\lVert g^{(n)}-g \rVert_{\mathcal{K},\sup}<\delta\) for \(n\geq n_2\). Since \(g(\mathcal{K}), g^{(n)}(\mathcal{K})\subset \mathcal{K}_1\) by construction, combining all of these we obtain
        \begin{align*}
            \lVert f_1^{(n)}\circ g^{(n)}-f_1\circ g \rVert_{\mathcal{K},\sup}\leq \underbrace{\lVert f_1^{(n)}\circ g^{(n)}-f_1\circ g^{(n)} \rVert_{\mathcal{K},\sup}}_{<\frac{\epsilon}{2} \text{ (unif. conv.)}} + \underbrace{\lVert f_1\circ g^{(n)}-f_1\circ g \rVert_{\mathcal{K},\sup}}_{<\frac{\epsilon}{2} \text{ (unif. cont.)}}<\epsilon\;.
        \end{align*}
        For \(k\geq 2\) we iterate the argument by applying the first part to \(f_k^{(n)}\) and \(\tilde{g}_k^{(n)}:=f_{k-1}^{(n)}\circ \hdots \circ f_1^{(n)}\circ g^{(n)}\), \(\tilde{g}_k:=f_{k-1}\circ \hdots \circ f_1\circ g\), beginning at \(k=2\). The boundedness of \((\tilde{g}_k^{(n)})_{n\in \mathbb{N}}\) and \(\tilde{g}_k\) follows from the continuity of \(f_i^{(n)}\) and \(f_i\) for \(i\in [k-1]\), and by the recursive construction we have \(\tilde{g}_k(\mathcal{K}), \tilde{g}_k^{(n)}(\mathcal{K})\subset \mathcal{K}_k\). The uniform convergence of \((\tilde{g}_k^{(n)})_{n\in \mathbb{N}}\) to \(\tilde{g}_k\) follows from the result for \(k-1\).
    \end{proof}

        \medskip 
        Lemma~\ref{Lemma-Approximate_function_compositions_sup_unif_cont} implies that substituting uniformly convergent FNN sequences into a fixed FNN yields uniform convergence of the composition. This allows us to replace activations or subnetworks by approximating FNNs while preserving convergence—a key step that we use in our single-parameter reduction arguments.

    Now, we prove Lemma \ref{Lemma-Single_Parameter_LReLU_approx_mon}, which is essential to generalize the universal approximation results for FNNs with many LReLU activations to ones that just use one or two LReLUs with fixed parameters as activations.
    \bigskip
    \begin{proof}[Proof of Lemma \ref{Lemma-Single_Parameter_LReLU_approx_mon}]
        Let \(f\in C^0_{\mon}(\mathbb{R},\mathbb{R})\) and let \(\mathcal{K}\subset\mathbb{R}\) be compact. There exists an interval \(I\subset \mathbb{R}\) such that \(\mathcal{K}\subset I\), and by Lemma 1 in \cite{duan2024vanillafeedforwardneuralnetworks} there exists \((\phi_n)_{n\in \mathbb{N}} \subset \FNN_{\{\sigma_{\alpha}\}}^{(1)}(1,1)\) such that \(\lim\limits_{n\to\infty}\lVert \phi_n - f\rVert_{I,\sup}=0\). 
        
        By following the proofs of Lemmas 2 and 3 in \cite{duan2024vanillafeedforwardneuralnetworks}, we recognize that the construction is done via \(\alpha\)-power piecewise linear functions with slopes in \(\{c\alpha^k\,|\,c\in\mathbb{R},\,k\in \mathbb{Z}\}\), which can be constructed by width-1 LReLU FNNs. By the construction of Lemma 2 in \cite{duan2024vanillafeedforwardneuralnetworks}, there exists a sequence \((\tilde{\phi}_n)_{n\in \mathbb{N}}\subset \FNN_{\{\sigma_\alpha\}}^{(1)}(1,1;m_n)\) of \(\alpha\)-power functions converging to \(f\) uniformly on \(I\), with \((m_n)_{n\in \mathbb{N}}\subset \mathbb{N}\) and \(\lim_{n\to \infty}m_n=\infty\). If \(m_n\) is odd, by adapting the power of a single piece of \(\tilde{\phi}_n\) by one and adjusting the double slice in their construction appropriately, we obtain \((\phi_n)_{n\in\mathbb{N}}\subset\FNN_{\{\sigma_\alpha\}}^{(1)}(1,1;2k_n)\) with \((k_n)_{n\in \mathbb{N}}\subset \mathbb{N}\), \(\lim_{n\to\infty}k_n=\infty\), such that \(\lim_{n\to\infty}\lVert f-\phi_n\rVert_{I,\sup}=0\). This proves the first part. 
        
        For the second part, let \(\beta \in (0,1)\cup(1,\infty)\). Since \(\sigma_{\beta}\) is monotone increasing, applying the first part yields a sequence \((\phi_n)_{n\in \mathbb{N}} \subset \FNN_{\{\sigma_{\alpha}\}}^{(1)}(1,1)\) such that \(\lim_{n\to\infty}\lVert \phi_n - \sigma_{\beta}\rVert_{I,\sup}=0\). By Lemma \ref{Lemma-LReLU_properties} (2.) we have \(\sigma_{-1}\in \FNN_{\{\sigma_{\shortminus\alpha},\sigma_{\alpha}\}}^{(1)}(1,1)\). Since \(\sigma_{-1}\circ \sigma_{\beta}=\sigma_{-\beta}\), we define \(\psi_n:=\sigma_{-1}\circ \phi_n\in \FNN_{\{\sigma_{\shortminus\alpha},\sigma_{\alpha}\}}^{(1)}(1,1)\) for all \(n\in \mathbb{N}\). As \(\sigma_{-1}\) is continuous, Lemma \ref{Lemma-Approximate_function_compositions_sup_unif_cont} implies that \(\lim_{n\to\infty}\lVert \psi_n - \sigma_{-\beta}\rVert_{\mathcal{K},\sup}=0\).
    \end{proof}

\section{Approximating the coding scheme with leaky ReLU FNNs}
        \label{sec:coding-scheme-lReLU}\label{Section-Approximating_coding_scheme}
            In this section, we construct an approximation of the coding scheme for FNNs with leaky ReLU activations. We introduce lemmas with appropriate implementations of the three parts of the coding scheme, namely the encoder, the decoder, and the memorizer. Later, these can be combined with Lemma \ref{Lemma-Accuracy_coding_scheme}, by which the coding scheme can approximate any continuous function arbitrarily well in the \(L^p\) norm, to show the universal approximation of \(L^p(\mathbb{R}^{d_x},\mathbb{R}^{d_y})\).
        
            The following Lemma \ref{Lemma-Approximate_q_k} shows that leaky ReLU FNNs of width \(1\) can approximate the quantizations while being exact on the set \(\mathcal{C}_{K}\). This is possible because FNNs with leaky ReLU activations of width \(1\) can exactly express functions in \(\PLCSM\), as proved in Lemma \ref{Lemma-PLCSM_in_LU}. However, it should be noted that when using FNNs with ReLU activations, as demonstrated in Lemma 8 in \cite{park2020minimum}, a width of \(2\) is necessary.
        
        \begin{lemma}\label{Lemma-Approximate_q_k}
            Let \(K\in \mathbb{N}\), \(\epsilon>0\) and for \(\gamma\in (0,1)\) consider the Borel-measurable set \(D_\gamma:= \bigcup_{i=1}^{2^{K}}[i\cdot 2^{-K}-\beta, i\cdot 2^{-K})\) with \(\beta\in (0,\frac{\gamma}{2^K})\), which fulfills \(\lambda(D_\gamma)<\gamma\). Then there exists \(\phi\in \FNN_{\mathcal{F}_{+}}^{(1)}(1,1)\) such that
            \begin{align*}
                q_K(c)=\phi(c)\;\,\forall  c\in \mathcal{C}_{K},\;\lVert q_K-\phi\rVert_{[0,1]\setminus D_\gamma,\sup}<\epsilon \;,
            \end{align*}
            \(\phi([0,1])\subset [0,1]\) and \(\phi([0,1]\setminus D_\gamma)\subset [0,1]\setminus D_\gamma\). 
        \end{lemma}
        
        \begin{proof}
            We define piecewise linear, continuous functions approximating the quantizer \(q_K\). Let \(\alpha,\beta\in (0,2^{-K})\) and for all \(i\in [2^K]\) define the piecewise function
            \begin{align*}
                \phi_{\alpha,\beta}(x):=\begin{cases}
                    \frac{\alpha}{2^{-K}-\beta} (x-(i-1)2^{-K}) + (i-1)2^{-K} & x\in [(i-1)2^{-K}, i\cdot 2^{-K}-\beta)\\
                    \frac{2^{-K}-\alpha}{\beta}(x- (i\cdot 2^{-K}-\beta)) + (i-1)2^{-K} + \alpha & x\in [i\cdot 2^{-K}-\beta, i\cdot 2^{-K}]
                \end{cases},
            \end{align*}
            where the last interval for \(i=2^K\) is \([1-\beta, 1]\). By construction, \(\phi_{\alpha,\beta}\) is continuous on \([0,1]\), satisfies \(\phi_{\alpha,\beta}([0,1])\subset [0,1]\), and \(\phi_{\alpha,\beta}(c)=q_K(c)\) for all \(c\in \mathcal{C}_{K}\). Choose \(\beta<\frac{\gamma}{2^{K}}\) and \(\alpha < \min\{\epsilon, 2^{-K}-\beta\}\), and set \(\phi:=\phi_{\alpha,\beta}\). Defining \(D_\gamma:= \bigcup_{i=1}^{2^{K}}[i\cdot 2^{-K}-\beta, i\cdot 2^{-K})\) yields \(\lambda(D_\gamma)=2^{K}\beta<\gamma\). Moreover, we obtain \(\lVert q_K-\phi\rVert_{[0,1] \setminus D_\gamma,\sup} \leq \alpha <\epsilon\) and note that 
            \begin{align*}
                \phi([0,1]\setminus D_\gamma)=\bigcup_{i=1}^{2^K}[(i-1)2^{-K}, (i-1)2^{-K}+\alpha)\subset \bigcup_{i=1}^{2^K}[(i-1)2^{-K}, i\cdot 2^{-K}-\beta)=[0,1]\setminus D_\gamma\;.
            \end{align*}
            By construction, \(\phi\) is piecewise linear, continuous, and strictly increasing with \(2^{K+1}\) pieces and alternating positive slopes, hence \(\phi\in \PLCSM_{2^{K+1}}\). By Lemma \ref{Lemma-PLCSM_in_LU} it follows that \(\phi\in \LU_{\mathcal{F}_{+}}(1)\subset\FNN^{(1)}_{\mathcal{F}_{+}}(1,1)\), which concludes the proof.
        \end{proof}

        \begin{figure}[h]
            \includegraphics[width=0.5\textwidth]{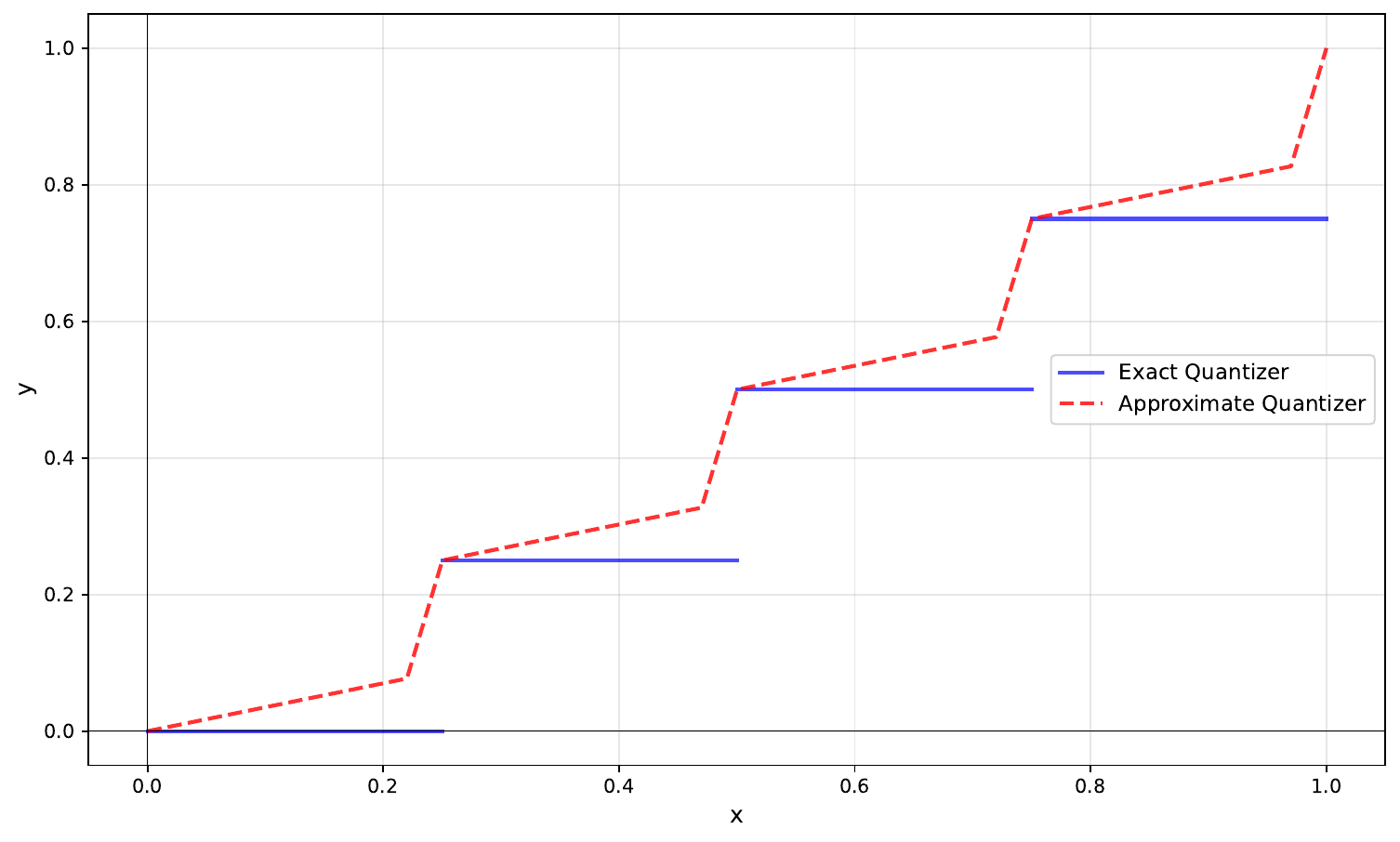}%
            \hfill
            \includegraphics[width=0.5\textwidth]{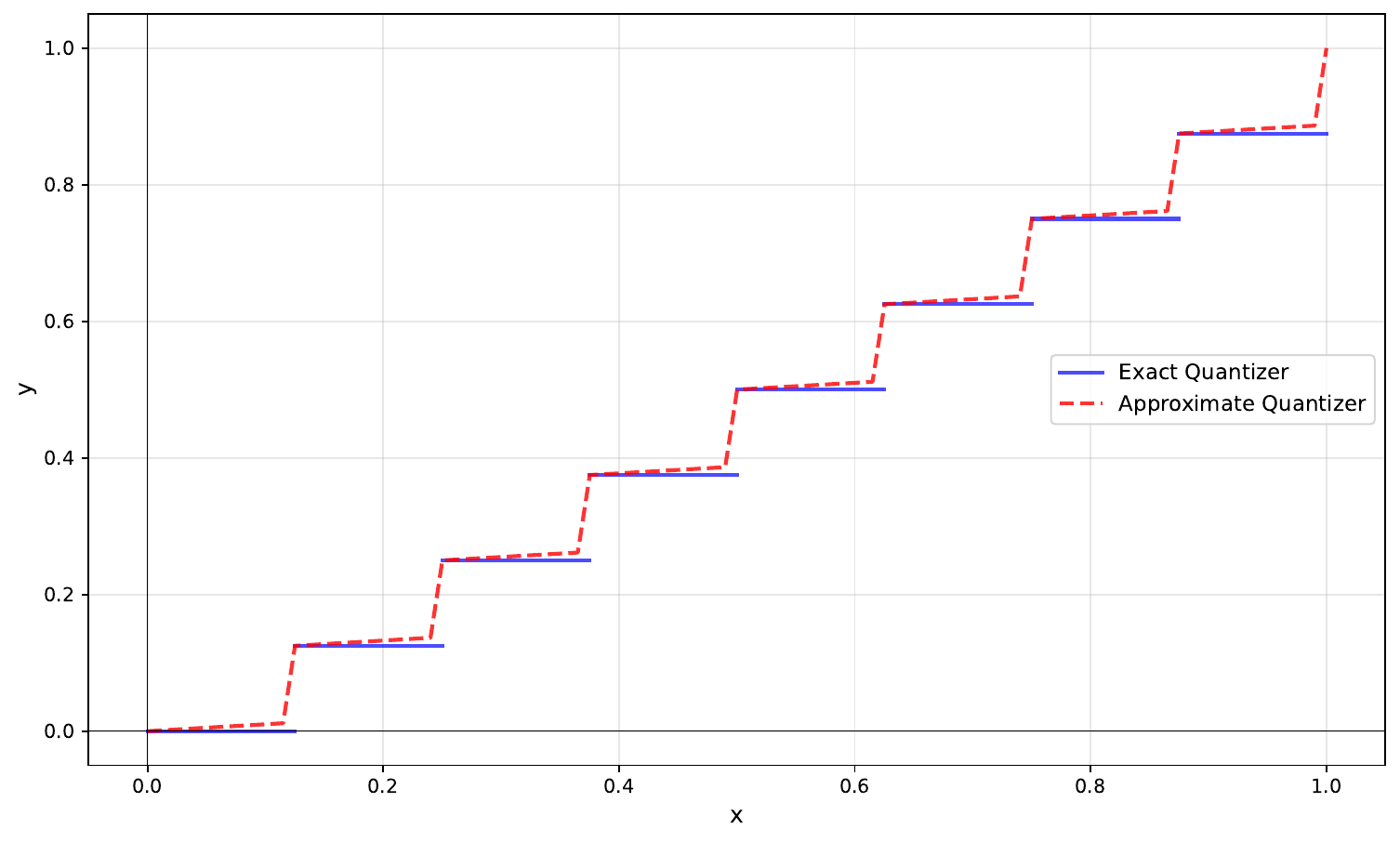}
            \caption{Quantizer approximations from Lemma~\ref{Lemma-Approximate_q_k}. Left: \(K=2\), \(\alpha=0.35\), \(\gamma=0.03\). Right: \(K=3\), \(\alpha=0.1\), \(\gamma=0.01\). Here \(\alpha\) is the flat slope where the approximation closely matches the true quantizer, and \(\gamma\) is the width of the exceptional intervals with larger error. To obtain our approximating quantizers, we decrease \(\alpha\) and \(\gamma\) to \(0\) while increasing \(K\) to infinity.}
            \label{Figure-q_K}
        \end{figure}

    With Lemma \ref{Lemma-Approximate_q_k}, we can construct the approximation of the encoder, which is the first component of the coding scheme, in the following lemma. It is important to note that the width of the encoder can be reduced by one by using leaky ReLU activations, in comparison to the results of Lemma 11 in \cite{park2020minimum} that considers FNNs with the ReLU activation.

    \begin{lemma}\label{Lemma-Approximate_Encoder}
        For any \(d,K\in \mathbb{N}\), \(\gamma>0\), \(\epsilon>0\) there exists \(\phi\in \FNN_{\mathcal{F}_{+}}^{(d)}(d,1)\) and \(S_\gamma\subset [0,1]^{d}\) with \(\lambda(S_\gamma)<\gamma\) such that
        \begin{align*}
            \mathfrak{e}_K(c)=\phi(c)\,\;\forall c\in \mathcal{C}_{K}^{d},\;\lVert \mathfrak{e}_K - \phi \rVert_{[0,1]^{d}\setminus S_\gamma,\sup}<\epsilon,
        \end{align*}
        and \(\phi([0,1]^d)\subset [0,2)\). Moreover, for any \(\alpha\in (0,1)\cup (1,\infty)\) there exists \((\phi_n)_{n\in \mathbb{N}}\subset \FNN_{\{\sigma_\alpha\}}^{(d)}(d,1)\) such that \(\lim\limits_{n\to\infty}\lVert \phi_n -\phi \rVert_{[0,1]^d,\sup} =0\).
    \end{lemma}
    
    \begin{proof}
        Let \(d,K\in \mathbb{N}\), \(\epsilon>0\), \(\gamma>0\) be arbitrary and for \(\delta, \bar{\gamma}>0\) choose \(\psi\in \FNN_{\mathcal{F}_{+}}^{(1)}(1,1)\) to be the approximation of \(q_K\) from Lemma \ref{Lemma-Approximate_q_k} such that
        \begin{align}\label{formula-Approximate_encoder_1}
            q_K(c)=\psi(c)\,\;\forall c\in \mathcal{C}_{K},\;\lVert q_K-\psi\rVert_{[0,1]\setminus D_{\bar{\gamma}},\sup}<\delta
        \end{align}
        and \(\psi([0,1])\subset [0,1]\). Define \(\zeta:\mathbb{R}^{d}\rightarrow \mathbb{R}^{d}\) by \(\zeta(x):=(\psi(x_1),...,\psi(x_d))^{T}\), which applies \(\psi\) component-wise. By stacking \(d\) copies of the width-1 network \(\psi\) in parallel, we obtain \(\zeta\in \FNN_{\mathcal{F}_{+}}^{(d)}(d,d)\). 
        
        Set \(S_\gamma:=[0,1]^{d}\setminus ([0,1]\setminus D_{\bar{\gamma}})^d\in \mathcal{B}(\mathbb{R}^{d})\) and let \(\bar{\gamma} < \gamma/d\). Since \(\lambda(D_{\bar{\gamma}})<\bar{\gamma}\) by Lemma \ref{Lemma-Approximate_q_k}, we have \(\lambda([0,1]\setminus D_{\bar{\gamma}})> 1-\bar{\gamma}\), and by Bernoulli's inequality \((1-\bar{\gamma})^{d}\geq 1-d\bar{\gamma}\), thus
        \begin{align*}
            \lambda(S_\gamma)=1-\lambda(([0,1]\setminus D_{\bar{\gamma}})^d)< 1-(1-\bar{\gamma})^{d}\leq d\bar{\gamma}<\gamma \;.
        \end{align*}
        By choice of \(S_\gamma\), for \(x\in [0,1]^{d}\setminus S_\gamma\) we have \(x_i\in [0,1]\setminus D_{\bar{\gamma}}\) for all \(i\in [d]\). Denoting by \(q_K^{\otimes d}(x):=(q_K(x_1),...,q_K(x_d))^T\) the component-wise application of \(q_K\), we obtain
        \begin{align}\label{formula-Approximate_encoder_2}
            \lVert q_K^{\otimes d}-\zeta\rVert_{[0,1]^{d}\setminus S_\gamma,\sup} = \sup_{x\in [0,1]^d\setminus S_\gamma} \max_{i\in[d]}|q_K(x_i)-\psi(x_i)|\stackrel{(\ref{formula-Approximate_encoder_1})}{<}\delta.
        \end{align}
        Consider \(A:\mathbb{R}^{d}\rightarrow \mathbb{R}\), \(A(x)=\sum_{i=1}^{d}2^{-K(i-1)}\cdot x_i\), which is linear and therefore Lipschitz-continuous with constant \(\mathcal{L}=\lVert A \rVert_{\op}\), and let \(\delta<\frac{\epsilon}{\mathcal{L}}\). Define \(\phi:=A\circ \zeta\in\FNN_{\mathcal{F}_{+}}^{(d)}(d,1)\). Since \(\mathfrak{e}_K = A \circ q_K^{\otimes d}\), we have
        \begin{align*}
            \lVert \mathfrak{e}_K - \phi \rVert_{[0,1]^{d}\setminus S_\gamma,\sup}&=\lVert A\circ q_K^{\otimes d} - A\circ \zeta \rVert_{[0,1]^{d}\setminus S_\gamma,\sup}\\
            &\leq \mathcal{L}\lVert q_K^{\otimes d} - \zeta \rVert_{[0,1]^{d}\setminus S_\gamma,\sup}\stackrel{(\ref{formula-Approximate_encoder_2})}{<}\mathcal{L}\delta<\epsilon.
        \end{align*}
        Furthermore, \(\mathfrak{e}_K(c)=A\circ q_K^{\otimes d}(c)\stackrel{(\ref{formula-Approximate_encoder_1})}{=}A\circ \zeta(c)=\phi(c)\) for all \(c\in \mathcal{C}_K^{d}\), and since \(\psi([0,1])\subset [0,1]\) we obtain
        \begin{align*}
            0\leq \phi(x)=\sum_{i=1}^{d}2^{-K(i-1)}\psi(x_i)\leq \sum_{i=1}^{d}2^{-K(i-1)}< \sum_{i=0}^{\infty}2^{-iK}=\frac{1}{1-2^{-K}}\leq 2\;,
        \end{align*}
        where the last inequality is strict for \(x\in [0,1]^d\), thus \(\phi([0,1]^d)\subset [0,2)\).
        
        For the second part, denote \(\psi=W_L\circ \sigma_{\alpha_{L-1}}\circ W_{L-1}\circ \hdots \circ \sigma_{\alpha_1}\circ W_1\) for some \(L\in \mathbb{N}\) with \(\alpha_j\in (0,1)\). By Lemma \ref{Lemma-Single_Parameter_LReLU_approx_mon}, for each \(j\in [L-1]\) there exist sequences \((\sigma_j^{(n)})_{n\in \mathbb{N}}\subset \FNN_{\{\sigma_\alpha\}}^{(1)}(1,1)\) converging uniformly to \(\sigma_{\alpha_j}\) on any compact set. Define \(\psi_n:=W_L\circ \sigma_{L-1}^{(n)}\circ W_{L-1}\circ \hdots \circ \sigma_1^{(n)}\circ W_1\) and \(\zeta_n(x):=(\psi_n(x_1),...,\psi_n(x_d))^T\). Set \(\phi_n:=A\circ \zeta_n\in \FNN_{\{\sigma_\alpha\}}^{(d)}(d,1)\) for all \(n\in \mathbb{N}\). By Lemma \ref{Lemma-Approximate_function_compositions_sup_unif_cont}, \(\psi_n\to\psi\) uniformly on \([0,1]\), which implies \(\zeta_n\to\zeta\) uniformly on \([0,1]^d\). Since \(A\) is Lipschitz continuous, we obtain \(\lim\limits_{n\to\infty}\lVert \phi_n -\phi \rVert_{[0,1]^d,\sup} =0\).
    \end{proof}
    
    \begin{lemma}\label{Lemma-Approximate_q_k_stepped}
        Let \(\alpha\in (0,1)\cup(1,\infty)\), \(K\in \mathbb{N}\), \(\epsilon>0\), then there exists \(\phi \in \FNN_{\mathcal{F}_{\alpha,\mathfrak{s}}}^{(1)}(1,1)\cap \FNN^{(1)}_{\mathcal{F}^{*}_{+,1}}(1,1)\) such that 
        \begin{align*}
            q_K(c)=\phi(c) \;\forall c\in \mathcal{C}_{K},\; \lVert q_K-\phi \rVert_{[0,1],\sup} < \epsilon\;.
        \end{align*}
    \end{lemma}
    
    \begin{proof}
        For \(a \in (0,1)\) we define the piecewise linear function \(\phi_{a}\) by
        \begin{align*}
            \phi_{a}(x):= a (x-(i-1)2^{-K})+(i-1)2^{-K} \quad \text{for } x\in [(i-1)2^{-K}, i\cdot 2^{-K}),\; i\in [2^{K}]\;,
        \end{align*}
        and \(\phi_a(1):= 1 - (1-a)2^{-K}\). By setting \(a < 2^K \epsilon\), we have for \(x \in [(i-1)2^{-K}, i\cdot 2^{-K})\):
        \begin{align*}
            |\phi_{a}(x) - q_K(x)| = |a(x-(i-1)2^{-K})| < a \cdot 2^{-K} < \epsilon\;,
        \end{align*}
        and similarly \(|\phi_a(1) - q_K(1)| \leq a \cdot 2^{-K} < \epsilon\), thus \(\lVert \phi_{a} - q_K \rVert_{[0,1],\sup} < \epsilon\).
        
        We now construct \(\phi_{a}\) as an FNN with stepped LReLU activations. Let \(b:=(1-a)2^{-K} \in (0,2^{-K}) \subset [0,1]\) and \(\gamma_j:=(a+j-1)\cdot 2^{-K}\) for \(j\in [2^{K}-1]\). By Lemma \ref{Lemma-LReLU_properties} (2.) and (4.), we have \(\Id\in \FNN_{\mathcal{F}_{\alpha,\mathfrak{s}}}^{(1)}(1,1)\cap \FNN^{(1)}_{\mathcal{F}^{*}_{+,1}}(1,1)\) and \(\rho_{1,b} \in \FNN_{\mathcal{F}_{\alpha,\mathfrak{s}}}^{(1)}(1,1) \cap \FNN^{(1)}_{\mathcal{F}^{*}_{+,1}}(1,1)\). Since
        \[
            \rho_{1,b}^{\gamma_j,0} = (\Id + \gamma_j) \circ \rho_{1,b} \circ (\Id - \gamma_j)\;,
        \]
        it follows that \(\rho_{1,b}^{\gamma_j,0} \in \FNN_{\mathcal{F}_{\alpha,\mathfrak{s}}}^{(1)}(1,1) \cap \FNN^{(1)}_{\mathcal{F}^{*}_{+,1}}(1,1)\). Therefore
        \begin{align*}
            \phi_{a} = \rho_{1, b}^{\gamma_{{2^K}-1},0}\circ \hdots \circ \rho_{1,b}^{\gamma_1,0}\circ (a \cdot\Id) \in \FNN_{\mathcal{F}_{\alpha,\mathfrak{s}}}^{(1)}(1,1)\cap \FNN^{(1)}_{\mathcal{F}^{*}_{+,1}}(1,1)\;.
        \end{align*}
        The initial scaling by \(a\) sets the slope, and each \(\rho_{1,b}^{\gamma_j,0}\) shifts the function up by \(b\) when the accumulated value exceeds \(\gamma_j\), which occurs precisely at \(x = j \cdot 2^{-K}\). By direct computation, for \(c = j \cdot 2^{-K} \in \mathcal{C}_K\) where \(j \in [0, 2^K-1]\):
        \begin{align*}
            \phi_{a}(c) = a \cdot j \cdot 2^{-K} + j(1-a)2^{-K} = j \cdot 2^{-K} = q_K(c)\;.
        \end{align*}
        Setting \(\phi := \phi_a\) with \(a < 2^K \epsilon\) completes the proof.
    \end{proof}

    \begin{figure}[h]
        \includegraphics[width=0.5\textwidth]{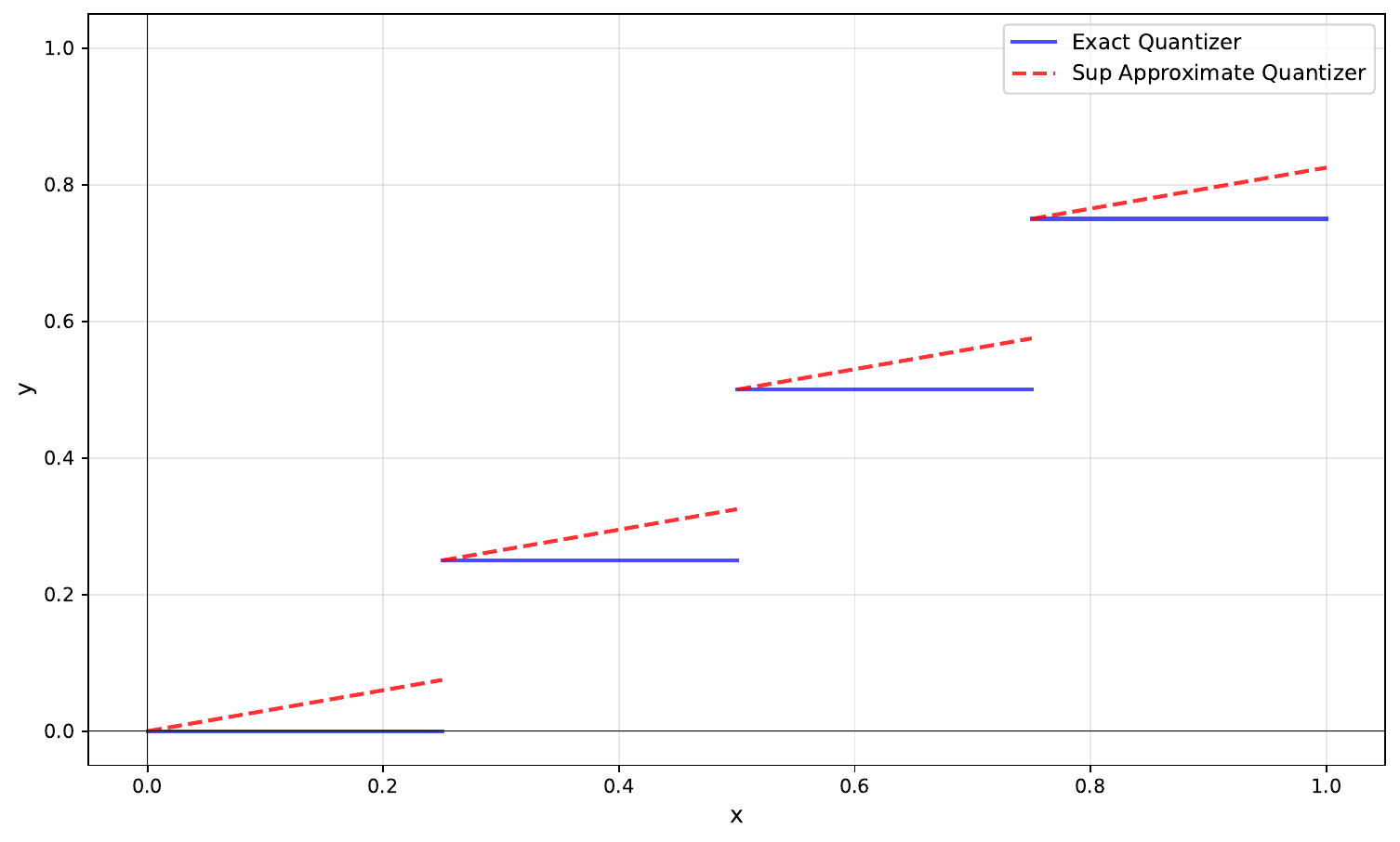}%
        \hfill
        \includegraphics[width=0.5\textwidth]{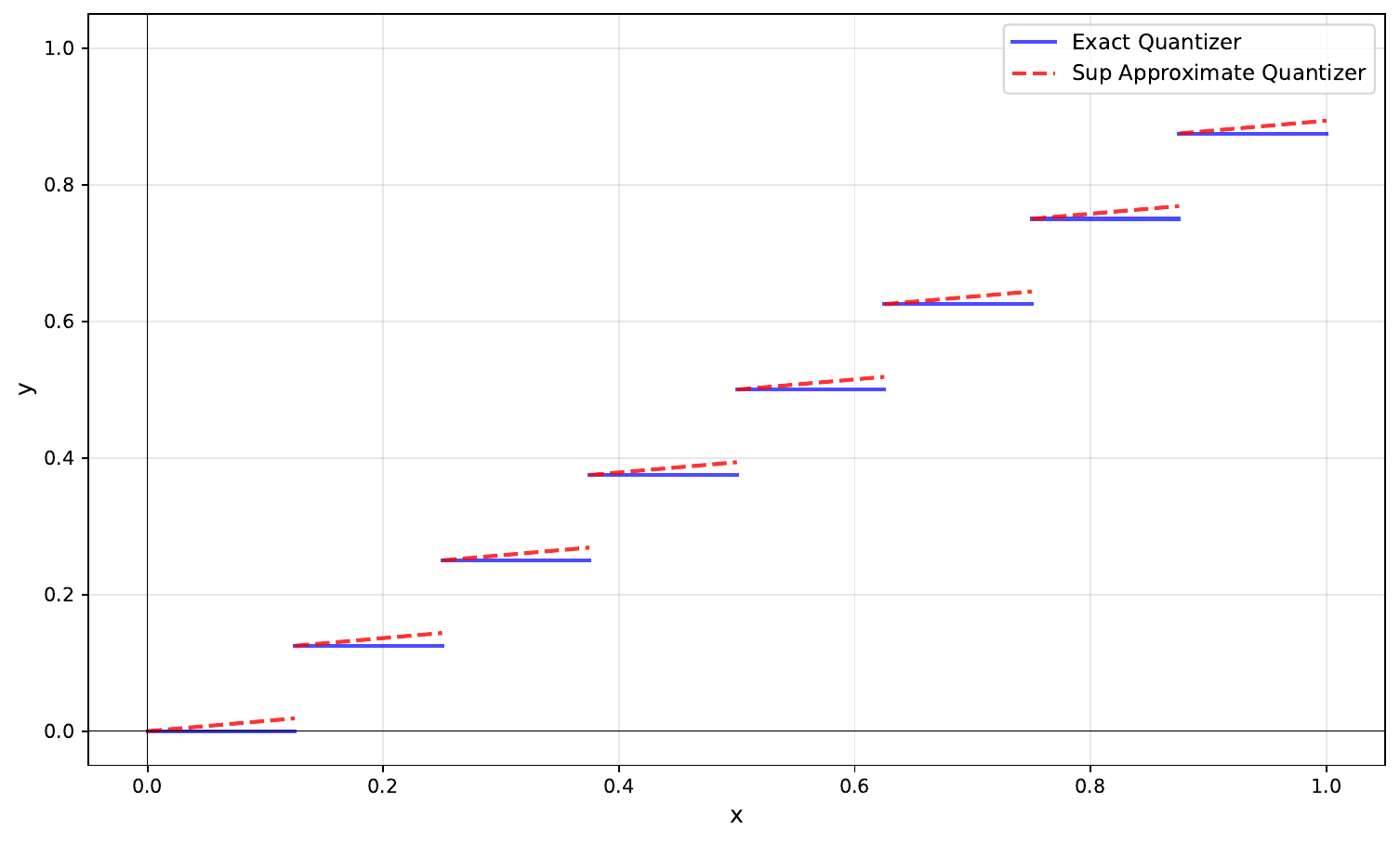}
        \caption{Discontinuous quantizer approximations from Lemma~\ref{Lemma-Approximate_q_k_stepped}. Left: \(K=2\), \(\alpha=0.3\). Right: \(K=3\), \(\alpha=0.15\). Here \(\alpha\) is the flat slope of the approximation, ensuring uniform closeness to the true encoder on the entire domain. To obtain our approximating quantizers, we decrease \(\alpha\) to \(0\) while increasing \(K\) to infinity.}
        \label{Figure-q_K_stepped}
    \end{figure}

    \begin{lemma}\label{Lemma-Approximate_encoder_stepped}
        Let \(\alpha\in (0,1)\cup(1,\infty)\). For any \(d,K\in \mathbb{N}\), \(\epsilon>0\) there exists \(\phi\in \FNN_{\mathcal{F}_{\alpha,\mathfrak{s}}}^{(d)}(d,1)\cap \FNN^{(d)}_{\mathcal{F}^{*}_{+,1}}(d,1)\) of width \(d\) such that
        \begin{align*}
            \mathfrak{e}_K(c)=\phi(c)\,\;\forall c\in \mathcal{C}_{K}^{d},\;\lVert \mathfrak{e}_K - \phi \rVert_{[0,1]^{d},\sup}<\epsilon\;,
        \end{align*}
        with \(\phi([0,1]^d)\subset [0,2)\). Moreover, if \(\alpha\in (0,1)\), then \(\phi\in \FNN_{\mathcal{F}_{\alpha,\mathfrak{s}}}^{(d)}(d,1)\subset \FNN_{\mathcal{F}_{+,\mathfrak{s}}}^{(d)}(d,1)\).
    \end{lemma}
    
    \begin{proof}
        Let \(d,K\in \mathbb{N}\), \(\epsilon>0\) be arbitrary. For \(\delta > 0\) choose \(\psi\in \FNN_{\mathcal{F}_{\alpha,\mathfrak{s}}}^{(1)}(1,1) \cap \FNN^{(1)}_{\mathcal{F}^{*}_{+,1}}(1,1)\) to be the approximation of \(q_K\) from Lemma \ref{Lemma-Approximate_q_k_stepped} such that
        \begin{align}\label{Formula-Approximate_encoder_1_stepped}
            q_K(c)=\psi(c)\,\;\forall c\in \mathcal{C}_{K},\;\lVert q_K-\psi\rVert_{[0,1],\sup}<\delta\;,
        \end{align}
        and note that by the construction in Lemma \ref{Lemma-Approximate_q_k_stepped} we have \(\psi([0,1])\subset [0,1]\). Define \(\zeta:\mathbb{R}^{d}\rightarrow \mathbb{R}^{d}\) by \(\zeta(x):=(\psi(x_1),...,\psi(x_d))^{T}\), which applies \(\psi\) component-wise. By stacking \(d\) copies of the width-1 network \(\psi\) in parallel, we obtain \(\zeta\in \FNN_{\mathcal{F}_{\alpha,\mathfrak{s}}}^{(d)}(d,d)\cap \FNN^{(d)}_{\mathcal{F}^{*}_{+,1}}(d,d)\). Denoting by \(q_K^{\otimes d}(x):=(q_K(x_1),...,q_K(x_d))^T\) the component-wise application of \(q_K\), we have
        \begin{align}\label{Formula-Approximate_encoder_2_stepped}
            \lVert q_K^{\otimes d}-\zeta\rVert_{[0,1]^{d},\sup} = \sup_{x\in [0,1]^d} \max_{i\in[d]}|q_K(x_i)-\psi(x_i)|\stackrel{(\ref{Formula-Approximate_encoder_1_stepped})}{<}\delta\;.
        \end{align}
        Consider \(A:\mathbb{R}^{d}\rightarrow \mathbb{R}\), \(A(x)=\sum_{i=1}^{d}2^{-K(i-1)}\cdot x_i\), which is linear and therefore Lipschitz-continuous with constant \(\mathcal{L}=\lVert A \rVert_{\op}\), and let \(\delta<\frac{\epsilon}{\mathcal{L}}\). Define \(\phi:=A\circ \zeta\in\FNN_{\mathcal{F}_{\alpha,\mathfrak{s}}}^{(d)}(d,1)\cap \FNN^{(d)}_{\mathcal{F}^{*}_{+,1}}(d,1)\). Since \(\mathfrak{e}_K = A \circ q_K^{\otimes d}\), we have
        \begin{align*}
            \lVert \mathfrak{e}_K - \phi \rVert_{[0,1]^{d},\sup}&=\lVert A\circ q_K^{\otimes d} - A\circ \zeta \rVert_{[0,1]^{d},\sup}\\
            &\leq \mathcal{L}\lVert q_K^{\otimes d} - \zeta \rVert_{[0,1]^{d},\sup}\stackrel{(\ref{Formula-Approximate_encoder_2_stepped})}{<}\mathcal{L}\delta<\epsilon\;.
        \end{align*}
        Furthermore, \(\mathfrak{e}_K(c)=A\circ q_K^{\otimes d}(c)\stackrel{(\ref{Formula-Approximate_encoder_1_stepped})}{=}A\circ \zeta(c)=\phi(c)\) for all \(c\in \mathcal{C}_K^{d}\). Since \(\psi([0,1])\subset [0,1]\), we obtain
        \begin{align*}
            0\leq \phi(x)=\sum_{i=1}^{d}2^{-K(i-1)}\psi(x_i)\leq \sum_{i=1}^{d}2^{-K(i-1)}= \frac{1-2^{-Kd}}{1-2^{-K}}<\frac{1}{1-2^{-K}}\leq 2\;,
        \end{align*}
        with strict inequality for all \(x \in [0,1]^d\), thus \(\phi([0,1]^d)\subset [0,2)\). If \(\alpha \in (0,1)\), then \(\mathcal{F}_{\alpha,\mathfrak{s}} \subset \mathcal{F}_{+,\mathfrak{s}}\), hence \(\phi\in \FNN_{\mathcal{F}_{\alpha,\mathfrak{s}}}^{(d)}(d,1)\subset\FNN_{\mathcal{F}_{+,\mathfrak{s}}}^{(d)}(d,1)\).
    \end{proof}
    
    Additionally, for our theorems that use FLOOR activations, we use the fact that the encoder can be implemented exactly with FLOOR FNNs, which the following lemma establishes.

    \begin{lemma}\label{Lemma-Encoder_FLOOR}
        For any \(K\in \mathbb{N}\), \(d\in \mathbb{N}\) it holds that \(q_K\in \FNN_{\{\FLOOR\}}^{(1)}(1,1;2)\). Moreover, the encoder \(\mathfrak{e}_K\) on \(\mathbb{R}^d\) satisfies \(\mathfrak{e}_K\in \FNN_{\{\FLOOR\}}^{(d)}(d,1;2)\).
    \end{lemma}
    
    \begin{proof}
        It holds that \(q_K = (2^{-K}\cdot\Id) \circ \FLOOR \circ (2^{K}\cdot\Id) \in \FNN_{\{\FLOOR\}}^{(1)}(1,1;2)\). By defining \(A:\mathbb{R}^{d}\rightarrow \mathbb{R}\), \(A(x)=\sum_{i=1}^{d}2^{-K(i-1)}\cdot x_i\), we obtain \(\mathfrak{e}_K = A\circ (q_K,\hdots,q_K)^T\in \FNN_{\{\FLOOR\}}^{(d)}(d,1;2)\).
    \end{proof}
    
    \bigskip
    Now, we provide the different constructions of the memorizers that we use in the proof of Theorem \ref{Theorem-Main_sup} and Theorem \ref{Theorem-Main1}. The following lemma shows that we can construct it with a width-1 FNN using generalized LReLUs with negative parameters.

    \begin{lemma}\label{Lemma-Memorizer_zig_zag}
        Let \(d_x,d_y\in \mathbb{N}\), \(f^{*}\in C^0([0,1]^{d_x},[0,1]^{d_y})\) and let \(K,M\in \mathbb{N}\). Then there exists \(\phi\in \FNN_{\mathcal{F}_{\pm}}^{(1)}(1,1)\) such that \(\phi(c)=\mathfrak{m}_{K,M,f^{*}}(c)\) for all \(c\in \mathcal{C}_{d_xK}\). Additionally, for any \(\alpha\in(0,1)\cup (1,\infty)\) there exists \((\phi_m)_{m\in \mathbb{N}}\subset \FNN_{\{\sigma_{-\alpha},\sigma_{\alpha}\}}^{(1)}(1,1)\) such that \(\lim \limits_{m\to\infty}\lVert \phi_m-\phi \rVert_{\mathcal{K},\sup}=0\) for any compact \(\mathcal{K}\subset\mathbb{R}\).
    \end{lemma}  
    
    \begin{proof}
        For \(f^{*}\in C^0([0,1]^{d_x},[0,1]^{d_y})\) we construct a zig-zag function \(h\) such that \(h(c)=\mathfrak{m}_{K,M,f^{*}}(c)\) for all \(c\in \mathcal{C}_{d_xK}\). Denote the increasingly sorted elements of the finite set \(\mathcal{C}_{d_xK}\) by \(c_1<c_2<\hdots<c_n\) for some \(n\in \mathbb{N}\) and set \(y_i:=\mathfrak{m}_{K,M,f^{*}}(c_i)\in [0,1]\) for \(i\in [n]\). 
        
        We construct \(h\) iteratively, adding one bump per point, proceeding from right to left. Each bump rises from anchor \(a_i\) to peak \(b_i\) (where \(h(b_i) = 1\)), then falls to the next anchor \(a_{i+1}\) (where \(h(a_{i+1}) = 0\)). By choosing \(a_i\) appropriately, we ensure that \(c_i\) lies on the rising part of bump \(i\) at height \(y_i\).
        
        Fix an arbitrary \(a_{n+1} > c_n\) and set \(c_0 := -\infty\). For \(i = n, n-1, \ldots, 1\), define \(a_i\) and \(b_i\) as follows:
        \begin{itemize}
            \item If \(y_i = 0\): Set \(a_i := c_i\) and choose any \(b_i \in (c_i, a_{i+1})\).
            \item If \(y_i = 1\): Choose any \(a_i \in (c_{i-1}, c_i)\) and set \(b_i := c_i\).
            \item If \(y_i \in (0,1)\): Choose \(a_i\) in the interval
            \[
                a_i \in \left( \max\left\{c_{i-1}, \frac{c_i - y_i a_{i+1}}{1-y_i}\right\}, c_i \right)
            \]
            and set \(b_i := a_i + \frac{c_i - a_i}{y_i}\).
        \end{itemize}
        
        We verify that for \(y_i \in (0,1)\), this interval is non-empty. Since \(a_{i+1} > c_i\) (which holds inductively, starting from \(a_{n+1} > c_n\)), we have
        \[
            \frac{c_i - y_i a_{i+1}}{1-y_i} < \frac{c_i - y_i c_i}{1-y_i} = c_i\;,
        \]
        so the interval is non-empty. Note that for small \(y_i > 0\), the lower bound approaches \(c_i\), requiring \(a_i\) to be chosen close to \(c_i\). This produces a steep slope, which is necessary to reach a small target value \(y_i\) before the peak.
        
        On each interval \([a_i, a_{i+1}]\), define
        \begin{align*}
            h\vert_{[a_i, a_{i+1}]}(x) := \begin{cases}
                \frac{x - a_i}{b_i - a_i} & x \in [a_i, b_i] \\[4pt]
                \frac{a_{i+1} - x}{a_{i+1} - b_i} & x \in [b_i, a_{i+1}]
            \end{cases}\;.
        \end{align*}
        
        We verify \(h(c_i) = y_i\):
        \begin{itemize}
            \item If \(y_i = 0\): Then \(a_i = c_i\), so \(h(c_i) = (c_i - a_i)/(b_i - a_i) = 0\).
            \item If \(y_i = 1\): Then \(b_i = c_i\), so \(h(c_i) = (c_i - a_i)/(c_i - a_i) = 1\).
            \item If \(y_i \in (0,1)\): Since \(a_i < c_i \leq b_i\), we have
            \[
                h(c_i) = \frac{c_i - a_i}{b_i - a_i} = \frac{c_i - a_i}{(c_i - a_i)/y_i} = y_i\;.
            \]
        \end{itemize}
        
        We verify \(b_i < a_{i+1}\):
        \begin{itemize}
            \item If \(y_i = 0\): By choice, \(b_i \in (c_i, a_{i+1})\).
            \item If \(y_i = 1\): Then \(b_i = c_i < a_{i+1}\).
            \item If \(y_i \in (0,1)\): By the constraint \(a_i > \frac{c_i - y_i a_{i+1}}{1-y_i}\), we obtain
            \[
                b_i = \frac{c_i - (1-y_i)a_i}{y_i} < \frac{c_i - (c_i - y_i a_{i+1})}{y_i} = a_{i+1}\;.
            \]
        \end{itemize}
        
        Finally, \(a_i > c_{i-1}\) by construction, ensuring that \(c_{i-1}\) lies in the interval \([a_{i-1}, a_i]\) where \(h(c_{i-1}) = y_{i-1}\) was set in the previous iteration.
        
        By Lemma \ref{Lemma-Zig_zag_LReLU}, there exists \(\phi \in \FNN_{\mathcal{F}_{\pm}}^{(1)}(1,1)\) such that \(\phi\vert_{[a_1,\infty)} = h\vert_{[a_1,\infty)}\). Since \(\mathcal{C}_{d_xK} \subset [a_1, a_{n+1}]\), we have \(\phi(c) = \mathfrak{m}_{K,M,f^{*}}(c)\) for all \(c \in \mathcal{C}_{d_xK}\).
        
        For the second part, note that \(\phi = W_L \circ \sigma_{-\beta_{L-1}} \circ W_{L-1} \circ \hdots \circ \sigma_{-\beta_1} \circ W_1\) with \(\beta_j \in (0,1) \cup (1,\infty)\) for all \(j \in [L-1]\). By Lemma \ref{Lemma-Single_Parameter_LReLU_approx_mon}, for each \(j \in [L-1]\) there exist sequences \((\varphi_j^{(m)})_{m \in \mathbb{N}} \subset \FNN_{\{\sigma_{-\alpha}, \sigma_\alpha\}}^{(1)}(1,1)\) such that \(\lim_{m \to \infty} \lVert \varphi_j^{(m)} - \sigma_{-\beta_j} \rVert_{\mathcal{K}_j, \sup} = 0\), where the compact sets \(\mathcal{K}_j \subset \mathbb{R}\) are chosen large enough to contain the images of all intermediate compositions. Define \(\phi_m := W_L \circ \varphi_{L-1}^{(m)} \circ W_{L-1} \circ \hdots \circ \varphi_1^{(m)} \circ W_1 \in \FNN_{\{\sigma_{-\alpha}, \sigma_\alpha\}}^{(1)}(1,1)\). By Lemma \ref{Lemma-Approximate_function_compositions_sup_unif_cont}, \(\lim_{m \to \infty} \lVert \phi_m - \phi \rVert_{\mathcal{K}, \sup} = 0\) for any compact \(\mathcal{K} \subset \mathbb{R}\).
    \end{proof}
    
        \begin{figure}[h]
            \includegraphics[width=0.5\textwidth]{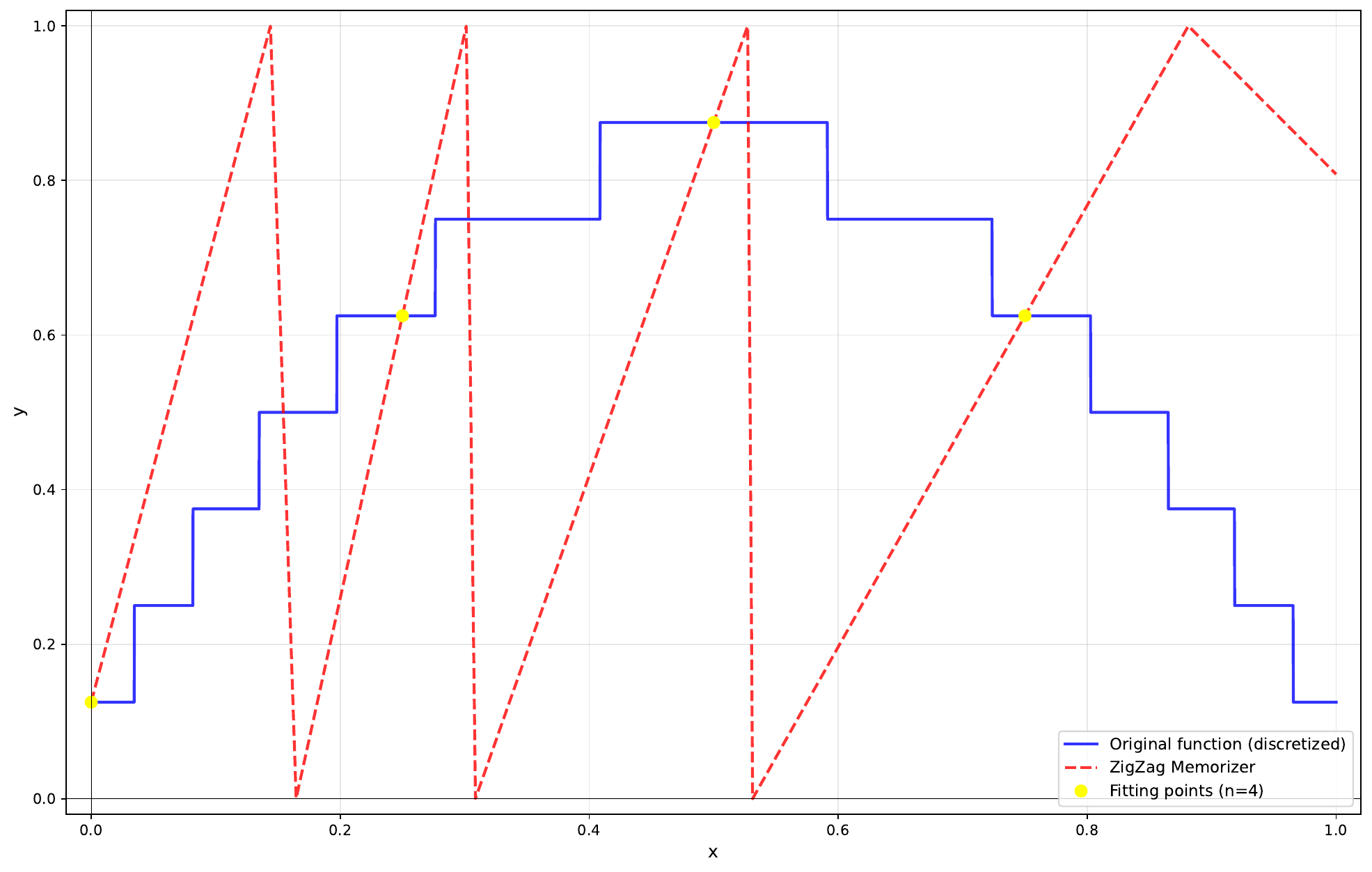}%
            \hfill
            \includegraphics[width=0.5\textwidth]{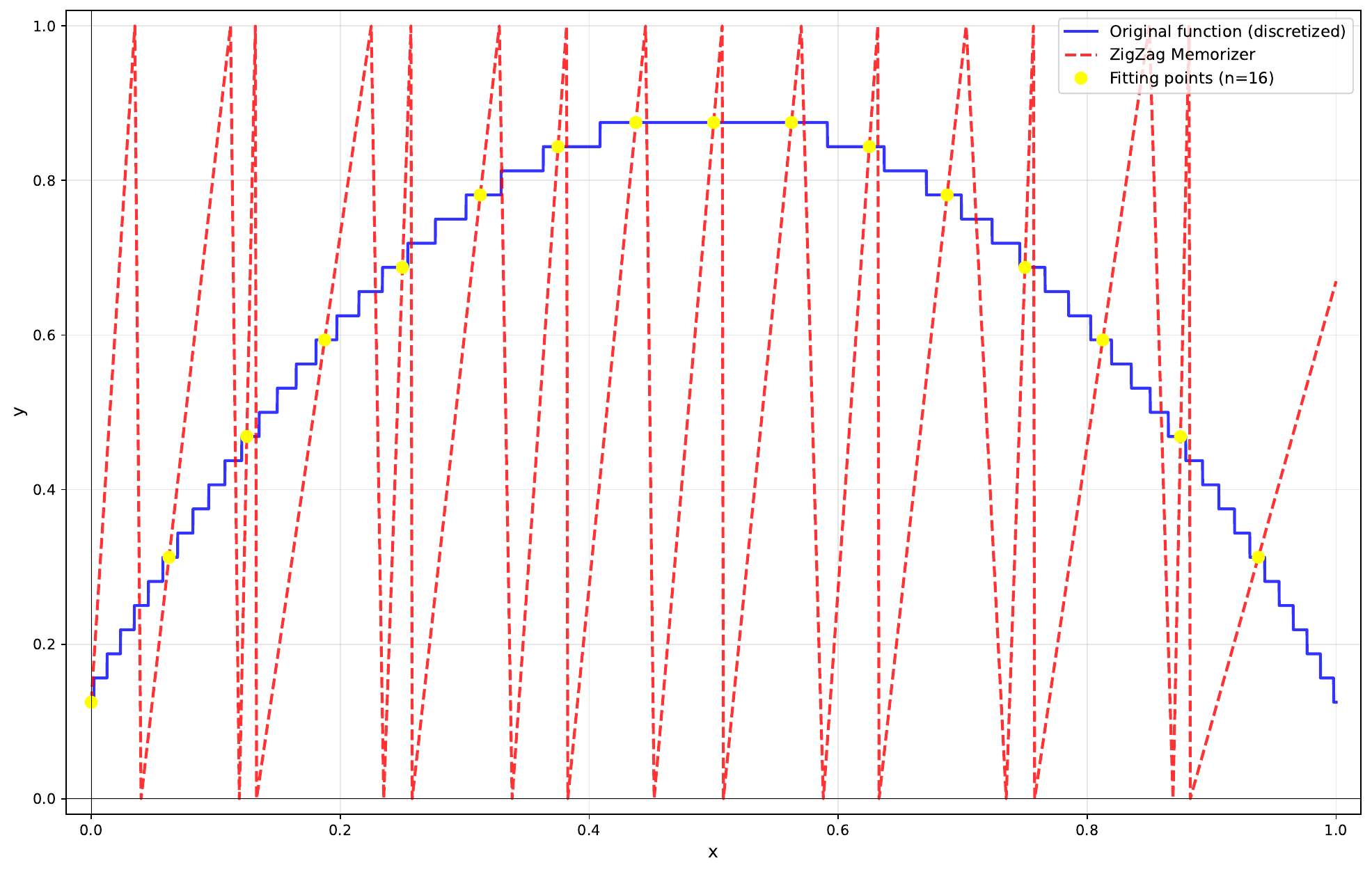}
            \caption{Exact memorizer constructions from Lemma~\ref{Lemma-Memorizer_zig_zag}. Left: \(K=2\), \(M=3\). Right: \(K=4\), \(M=5\). The underlying function is \(f(x)=-3(x - 0.5)^2 + 0.9\), but the memorizer fits \(q_M \circ f\), which produces the stepped structure.}
            \label{Figure-Memorizers}
        \end{figure}

        \bigskip
		 The following lemma shows that we can implement the memorizer with FNNs of width \(2\) with leaky ReLU activations. The proof is based on the idea that the memorizer, as a one-dimensional function defined on a finite set, can be extended to a continuous piecewise linear function on the real numbers. This function can then be written as a sum of continuous monotone piecewise linear functions by the monotone function decomposition of Lemma \ref{Lemma-function_decomposition_PLC}, each of which can be expressed by a FNN of width \(1\) with leaky ReLU activations.

        \begin{lemma}\label{Lemma-Memorizer-Lp}
            Let \(d_x,d_y\in \mathbb{N}\), \(f^{*}\in C^0([0,1]^{d_x},[0,1]^{d_y})\) and let \(K,M\in \mathbb{N}\). Then there exists \(\phi\in \FNN_{\mathcal{F}_{+}}^{(2)}(1,1)\) such that \(\phi(c)=\mathfrak{m}_{K,M,f^{*}}(c)\) for all \(c\in \mathcal{C}_{d_xK}\). Moreover, for any \(\alpha\in (0,1)\cup (1,\infty)\) there exists \((\phi_n)_{n\in \mathbb{N}}\subset \FNN_{\{\sigma_\alpha\}}^{(2)}(1,1)\) such that \(\lim\limits_{n\to\infty}\lVert \phi_n-\phi \rVert_{[0,1],\sup}=0\).
        \end{lemma}
        
        \begin{proof}
            For \(f^{*}\in C^0([0,1]^{d_x},[0,1]^{d_y})\) we apply Lemma \ref{Lemma-PLCSM_exactly_fit_finite_points} to \(\mathfrak{m}_{K,M,f^{*}}\) with \(\mathcal{X}=\mathcal{C}_{d_xK}\) and obtain \(\psi\in \PLC\) such that \(\mathfrak{m}_{K,M,f^{*}}(c)=\psi(c)\) for all \(c\in \mathcal{C}_{d_xK}\). By Lemma \ref{Lemma-function_decomposition_PLC}, there exist \(\psi_1,\psi_2\in \PLCSM\) satisfying \(\psi(x)=\psi_1(x)+\psi_2(x)\), and by Lemma \ref{Lemma-PLCSM_in_LU} we have \(\psi_1,\psi_2\in \FNN_{\mathcal{F}_{+}}^{(1)}(1,1;L)\) for some depth \(L\in \mathbb{N}\). Now, we can define \(\phi\in \FNN_{\mathcal{F}_{+}}^{(2)}(1,1)\) by 
            \[
                \phi(x):=(1,1) \begin{pmatrix} \psi_1(x)\\ \psi_2(x) \end{pmatrix}=\psi_1(x)+\psi_2(x)=\psi(x)\;,
            \]
            which shows the first claim.
            
            For the second part, let \(\alpha\in (0,1)\cup (1,\infty)\). By Lemma \ref{Lemma-Single_Parameter_LReLU_approx_mon}, for \(i\in [2]\) there exist sequences \((\psi_i^{(n)})_{n\in \mathbb{N}}\subset \FNN_{\{\sigma_\alpha\}}^{(1)}(1,1;2k_i^{(n)})\) with \(\lim_{n\to\infty}k_i^{(n)}=\infty\) such that \(\lim_{n\to\infty}\lVert \psi_i^{(n)} - \psi_i \rVert_{[0,1],\sup}=0\). 
            
            To stack these networks, we need them to have the same depth. By Lemma \ref{Lemma-LReLU_properties} (2.), \(\Id \in \FNN_{\{\sigma_\alpha\}}^{(1)}(1,1;3)\), and composing a depth-\(L\) FNN with \(\Id\) yields a depth-\((L+2)\) FNN that computes the same function, as we add two activations, increasing depth by two. 
            
            For each \(n\), let \(L_n := \max\{2k_1^{(n)}, 2k_2^{(n)}\}\). For \(i \in [2]\), if the depth \(2k_i^{(n)} < L_n\), we compose \(\psi_i^{(n)}\) with \(\Id\) repeatedly until the depth equals \(L_n\). Since the depth difference \(L_n - 2k_i^{(n)}\) is even, this is always possible. Denote the resulting networks by \(\tilde{\psi}_i^{(n)} \in \FNN_{\{\sigma_\alpha\}}^{(1)}(1,1;L_n)\), which satisfy \(\tilde{\psi}_i^{(n)} = \psi_i^{(n)}\) as functions.
            
            Since both \(\tilde{\psi}_1^{(n)}\) and \(\tilde{\psi}_2^{(n)}\) have the same depth \(L_n\) we can stack them to define
            \[
                \phi_n := (1,1) \circ \begin{pmatrix} \tilde{\psi}_1^{(n)}\\ \tilde{\psi}_2^{(n)} \end{pmatrix} \circ \begin{pmatrix} 1 \\ 1 \end{pmatrix} \in \FNN_{\{\sigma_\alpha\}}^{(2)}(1,1)\;.
            \]
            Then
            \begin{align*}
                &\lVert \phi_n - \phi \rVert_{[0,1],\sup} \leq \lVert \tilde{\psi}_1^{(n)} - \psi_1 \rVert_{[0,1],\sup} + \lVert \tilde{\psi}_2^{(n)} - \psi_2 \rVert_{[0,1],\sup} 
                \\&= \lVert \psi_1^{(n)} - \psi_1 \rVert_{[0,1],\sup} + \lVert \psi_2^{(n)} - \psi_2 \rVert_{[0,1],\sup} \xrightarrow[]{n\to\infty} 0\;.
            \end{align*}
        \end{proof}

    \begin{figure}[h]
        \includegraphics[width=0.5\textwidth]{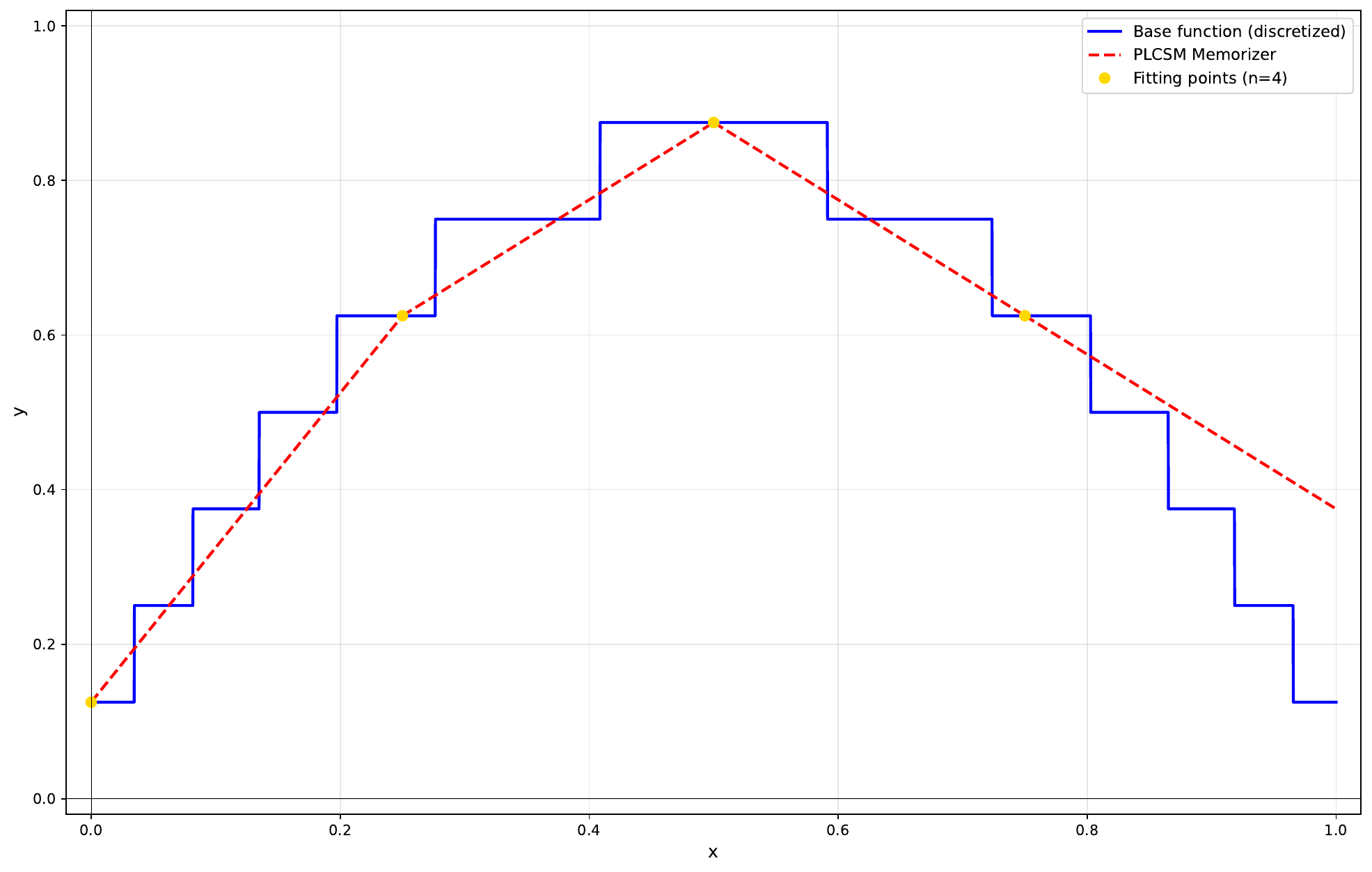}%
        \hfill
        \includegraphics[width=0.5\textwidth]{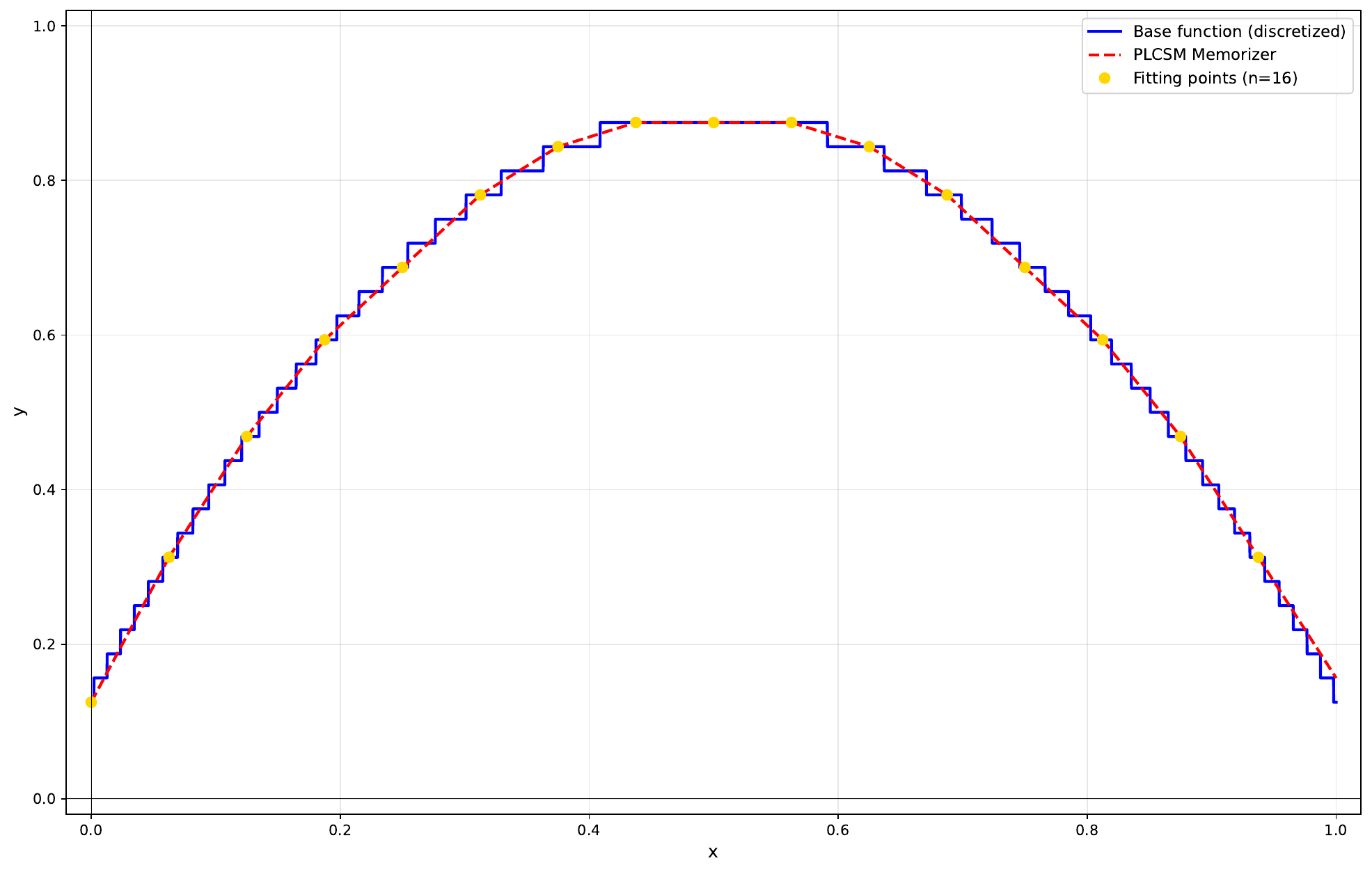}
        \caption{Exact memorizer constructions from Lemma~\ref{Lemma-Memorizer-Lp}. Left: \(K=2\), \(M=3\). Right: \(K=4\), \(M=5\). The underlying function is \(f(x)=-3(x - 0.5)^2 + 0.9\), but the memorizer fits \(q_M \circ f\), which produces the stepped structure.}
        \label{Figure-Memorizers-Lp}
    \end{figure}

       In their Lemma 10, \citet{park2020minimum} obtained an exact implementation of the decoder by FNNs of width \(d_y\) with ReLU activations. Since invertible leaky ReLUs can uniformly approximate the ReLU function on compact sets, we can use their result to obtain an FNN of width \(d_y\) that approximates the decoder arbitrarily well. Moreover, Lemma 1 of \citet{cai2023achieve} shows that \(\max\{d_x,d_y\}\) is a lower bound on the minimum width for the universal approximation of \(L^p(\mathcal{K},\mathbb{R}^{d_y})\) for compact \(\mathcal{K}\subset \mathbb{R}^{d_x}\). Since it is not possible to achieve universal approximation of \(L^p\) integrable functions on compact sets using FNNs with invertible leaky ReLU activations of width less than \(d_y\), approximating the decoder (rather than implementing it exactly) is sufficient for our purposes.

        \medskip
        Thus, for completeness, we present below the proof of Lemma 10 and the therein used Lemma 15 and Lemma 17 (Proposition 2 in \citet{hanin2018approximating}) as introduced in \cite{park2020minimum}, and use them together with our improved results for encoder and memorizer in the proofs of some of our main theorems.
        
        \medskip
        The following definition corresponds to Definition 1 in \citet{hanin2018approximating} and is necessary to formulate Lemma \ref{Lemma-Min_max_strings}, which corresponds to Proposition 2 in \citet{hanin2018approximating}.

	\begin{definition}\label{Definition-Max_min_string}
        A function \(h:\mathbb{R}^{d_x}\rightarrow \mathbb{R}^{d_y}\) is called a max-min string of length \(L\), if there exist affine transformations \(h_i\in \Aff(d_x,d_y)\) for \(i\in [L]\) such that
        \begin{align*}
            h(x)=\tau_{L-1}(h_L(x),\tau_{L-2}(h_{L-1}(x),\tau_{L-3}(\ldots,\tau_2(h_3(x),\tau_1(h_2(x),h_1(x)))\ldots))), 
        \end{align*}
        where each \(\tau_{i}\), \(i\in [L-1]\), is either a coordinate-wise \(\max\{\cdot,\cdot\}\) or \(\min\{\cdot,\cdot\}\).
    \end{definition}
    
    \begin{lemma}\label{Lemma-Min_max_strings}
        For any max-min string \(f^{*}:\mathbb{R}^{d_x}\rightarrow \mathbb{R}^{d_y}\) of length \(L\) and any compact \(\mathcal{K}\subset \mathbb{R}^{d_x}\), there exists \(\phi\in \FNN_{\{\ReLU\}}^{(d_x+d_y)}(d_x,d_x+d_y)\) of depth \(L\) such that for all \(x\in \mathcal{K}\) it holds that
        \begin{align*}
            \phi(x)=(\phi_1(x),\phi_2(x))^T, \;\text{ where } \;\phi_1(x)=x\;\text{ and }\; \phi_2(x)=f^{*}(x)\;.
        \end{align*}
    \end{lemma}
    
    \begin{proof}
        We refer to the proof of Proposition 2 in \cite{hanin2018approximating}.
    \end{proof}
    
    \medskip
    The following lemma was originally formulated as Lemma 15 in \citet{park2020minimum}.
    
    \begin{lemma}\label{Lemma-Min_max_string_by_FNN}
        For any \(M\in \mathbb{N}\), \(\delta>0\), there exists \(\phi\in \FNN_{\{\ReLU\}}^{(2)}(1,2)\) such that for all \(x\in [0,1]\setminus P_{M,\delta}\), with \(P_{M,\delta}:=\bigcup_{i=1}^{2^M-1}(i2^{-M}-\delta, i2^{-M})\), it holds that
        \begin{align}\label{formula-Min_max_string_by_FNN_1}
            \phi(x)=(q_M(x), 2^M(x-q_M(x)))^T,
        \end{align}
        and furthermore \(\phi(\mathbb{R})\subset [0,1-2^{-M}]\times [0,1]\).
    \end{lemma}
    
    \begin{proof}
        The proof structure and idea is taken from the proof of Lemma 15 in \cite{park2020minimum}. We define 
        \begin{align*}
            \zeta(x)=\min\{\max\{x,0\},1\}=1-\sigma(1-\sigma(x))\in \FNN_{\{\ReLU\}}^{(1)}(1,1)
        \end{align*}
        to obtain only values in \([0,1]\) after applying \(\zeta\). Moreover, we define \(\psi_l:[0,1]\rightarrow [0,1]^2\) by \((\psi_l(x))_1:=x\) and
        \begin{align*}
            (\psi_l(x))_2:=\begin{cases}
                0 & x\in [0,2^{-M}-\delta]\\
                \delta^{-1}2^{-M}(x-2^{-M}+\delta) & x\in (2^{-M}-\delta, 2^{-M})\\
                2^{-M} & x\in [2^{-M},2\cdot 2^{-M}-\delta]\\
                \delta^{-1}2^{-M}(x-2\cdot 2^{-M}+\delta) + 2^{-M} & x\in (2\cdot 2^{-M}-\delta, 2\cdot 2^{-M})\\
                \vdots & \vdots \\
                (l-1)2^{-M} & x\in [(l-1)2^{-M},1]
            \end{cases},
        \end{align*}
        and note that \((\psi_{2^M}(x))_2 = q_M(x)\) for \(x\in [0,1]\setminus P_{M,\delta}\). Moreover, we observe that \((\psi_1(x))_2=0\) and that we can rewrite 
        \begin{align*}
            (\psi_{l+1}(x))_2=\min\{l\cdot 2^{-M},\max\{\delta^{-1}2^{-M}(x-l\cdot 2^{-M}+\delta)+(l-1)2^{-M}, (\psi_l(x))_2\}\}
        \end{align*}
        for all \(l\in \mathbb{N}\), hence \(\psi_{2^M}\) is a max-min string as introduced in Definition \ref{Definition-Max_min_string}. Therefore, by Lemma \ref{Lemma-Min_max_strings}, there exists \(\tilde{\phi}\in \FNN_{\{\ReLU\}}^{(2)}(1,2)\) satisfying 
        \begin{align*}
            \tilde{\phi}(x)=(x, q_M(x))^T \quad \forall x\in [0,1]\setminus P_{M,\delta},
        \end{align*}
        and we obtain
        \begin{align*}
            \phi(x):=\begin{pmatrix}
                0 & 1\\
                2^{M} & -2^{M}
            \end{pmatrix}
            \circ \tilde{\phi} \circ \zeta(x)=
            \begin{pmatrix}
                (\tilde{\phi}(\zeta(x)))_2\\
                2^{M}((\tilde{\phi}(\zeta(x)))_1-(\tilde{\phi}(\zeta(x)))_2)
            \end{pmatrix},
        \end{align*}
        thus \(\phi(x)=(q_M(x), 2^{M}(x-q_M(x)))^T\) for all \(x\in [0,1]\setminus P_{M,\delta}\) and \(\phi \in \FNN_{\{\ReLU\}}^{(2)}(1,2)\). Moreover, \(\phi(\mathbb{R})\subset [0,1-2^{-M}]\times [0,1]\) since
        \begin{align*}
            \max\{0,\zeta(x)-2^{-M}\}\leq (\psi_{2^M}(\zeta(x)))_2\leq \zeta(x) = (\psi_{2^M}(\zeta(x)))_1
        \end{align*}
        and \(\zeta(\mathbb{R})\subset [0,1]\).
    \end{proof}

    The following lemma corresponds to Lemma 10 in \cite{park2020minimum}.

    \begin{lemma}\label{Lemma-Decoder}
        For any \(d,M\in \mathbb{N}\), there exists \(\phi\in \FNN_{\{\ReLU\}}^{(d)}(1,d)\) such that for all \(c\in \mathcal{C}_{dM}\) it holds that \(\phi(c)=\mathfrak{d}_M(c)\) and \(\phi(\mathbb{R})\subset [0,1]^{d}\).
    \end{lemma}
    
    \begin{proof}
        The proof structure and idea is taken from the proof of Lemma 10 in \cite{park2020minimum}. We show the claim by iteratively applying Lemma \ref{Lemma-Min_max_string_by_FNN}. For \(M\in \mathbb{N}\) and \(\delta<2^{-dM}\), let \(\psi\in \FNN_{\{\ReLU\}}^{(2)}(1,2)\) be the FNN from Lemma \ref{Lemma-Min_max_string_by_FNN} and note that \(\mathcal{C}_{dM}\subset [0,1] \setminus P_{M,\delta}\). Hence formula (\ref{formula-Min_max_string_by_FNN_1}) holds for \(c\in \mathcal{C}_{dM}\) and additionally \(\psi(\mathbb{R})\subset [0,1-2^{-M}]\times [0,1]\). 
        
        We define \(\phi^{(0)}:=\Id\), \(\phi^{(1)}:=\psi\), and for \(i\in [2,d-1]\) we define \(\phi^{(i)}:\mathbb{R}\rightarrow [0,1-2^{-M}]^{i}\times [0,1]\) by 
        \begin{align}
            \phi^{(i)}(x):=\begin{pmatrix}
                \sigma_0((\phi^{(i-1)}(x))_1)\\
                \vdots \\
                \sigma_0((\phi^{(i-1)}(x))_{i-1})\\
                \psi((\phi^{(i-1)}(x))_i)
            \end{pmatrix}
            \stackrel{\phi^{(i-1)}(x)\geq 0}{=}\begin{pmatrix}
                (\phi^{(i-1)}(x))_1\\
                \vdots \\
                (\phi^{(i-1)}(x))_{i-1}\\
                \psi((\phi^{(i-1)}(x))_i)
            \end{pmatrix},\label{formula-Decoder_1}
        \end{align}
        where we used that \(\phi^{(i-1)}(x)\geq 0\) for \(i\in [2,d-1]\) and that \(\sigma_0\) equals the identity on \([0,\infty)\). We set \(\phi:=\phi^{(d-1)}\) and note that by definition and choice of \(\psi\) we have \(\phi^{(i)}\in \FNN_{\{\ReLU\}}^{(i+1)}(1,i+1)\). 
        
        For \(c\in \mathcal{C}_{dM}\), the component \((\psi(c))_2=2^{M}(c-q_M(c))\) is the number where the first \(M\) bits of \(c\) were extracted and the remaining bits were shifted \(M\) places to the left. By construction, \(\phi^{(i)}\) applies this procedure \(i\) times to an input \(c\in \mathcal{C}_{dM}\), hence \((\phi^{(i)}(c))_{i+1}\in \mathcal{C}_{(d-i)M}\) is the binary number obtained by extracting the first \(i\cdot M\) bits from \(c\) and shifting the remaining \((d-i)M\) bits by \(i \cdot M\) binary places to the left. Since \(\mathcal{C}_{(d-i)M}\subset [0,1]\setminus P_{M,\delta}\), formula (\ref{formula-Min_max_string_by_FNN_1}) applies at each step.
        
        By definition of the decoder, for \(c\in \mathcal{C}_{dM}\) we have
        \begin{align}\label{formula-Decoder_2}
            (\phi^{(i)}(c))_{i}=(\psi((\phi^{(i-1)}(c))_i))_1=q_M((\phi^{(i-1)}(c))_i)=(\mathfrak{d}_M(c))_i\quad \forall i\in [d-1]\;,
        \end{align}
        where for \(i=1\) we use \((\phi^{(0)}(c))_1 = c\). Since \((\phi^{(d-2)}(c))_{d-1}\in \mathcal{C}_{2M}\), it follows that 
        \begin{align}\label{formula-Decoder_3}
            (\phi^{(d-1)}(c))_{d}=(\psi((\phi^{(d-2)}(c))_{d-1}))_2=2^{M}((\phi^{(d-2)}(c))_{d-1}-q_M((\phi^{(d-2)}(c))_{d-1}))=(\mathfrak{d}_M(c))_{d}\;.
        \end{align}
        Combining (\ref{formula-Decoder_2}) and (\ref{formula-Decoder_3}), and noting that by (\ref{formula-Decoder_1}) the function \(\phi^{(i)}\) keeps the values of \(\phi^{(i-1)}\) in the first \(i-1\) coordinates unchanged, we obtain for all \(c\in \mathcal{C}_{dM}\):
        \begin{align*}
            \phi(c)
            \stackrel{(\ref{formula-Decoder_1}),\, (\ref{formula-Decoder_2})}{=}
            \begin{pmatrix}
                (\mathfrak{d}_M(c))_1\\
                \vdots \\
                (\mathfrak{d}_M(c))_{d-1}\\
                (\psi((\phi^{(d-2)}(c))_{d-1}))_2
            \end{pmatrix}
            \stackrel{(\ref{formula-Decoder_3})}{=}
            \begin{pmatrix}
                (\mathfrak{d}_M(c))_1\\
                \vdots \\
                (\mathfrak{d}_M(c))_{d-1}\\
                (\mathfrak{d}_M(c))_{d}
            \end{pmatrix}
            =\mathfrak{d}_M(c)\;.
        \end{align*}
    \end{proof}

		The following lemma ensures that we can uniformly approximate the decoder with FNNs that use a single non-identical invertible LReLU activation.

        \begin{lemma}\label{Lemma-Decoder_leaky}
            For any \(\alpha\in (0,1)\cup (1,\infty)\) and \(d,M\in \mathbb{N}\), there exists \((\phi_n)_{n\in \mathbb{N}}\subset \FNN^{(d)}_{\{\sigma_{\alpha}\}}(1,d)\) such that \(\lim_{n\to \infty}\lVert \phi_n - \mathfrak{d}_M \rVert_{\mathcal{C}_{dM},\sup}=0\).
        \end{lemma}
        
        \begin{proof}
            Let \(\phi\in \FNN^{(d)}_{\{\ReLU\}}(1,d)\) be the ReLU network of width \(d\) from Lemma \ref{Lemma-Decoder}, satisfying \(\phi(c) = \mathfrak{d}_M(c)\) for all \(c\in \mathcal{C}_{dM}\), and let its structure be given by 
            \begin{align*}
                \phi=W_L\circ \ReLU \circ W_{L-1}\circ \hdots \circ W_2 \circ \ReLU \circ W_1\;,
            \end{align*}
            where \(L\in \mathbb{N}\), \(W_i\in \Aff(d_{i-1}, d_i)\) with \(d_0:=1\), \(d_L:=d\), and \(d_i\leq d\) for all \(i\in [L]\). By Lemma \ref{Lemma-LReLU_properties} (5.), there exists \((\sigma^{(n)})_{n\in \mathbb{N}}\subset \FNN^{(1)}_{\{\sigma_{\alpha}\}}(1,1)\) such that \(\lim_{n\to \infty}\lVert \sigma^{(n)} - \ReLU \rVert_{\mathcal{K},\sup}=0\) for all compact \(\mathcal{K}\subset \mathbb{R}\). Define \(\phi_n:=W_L\circ \sigma^{(n)} \circ W_{L-1} \circ \hdots \circ W_2 \circ \sigma^{(n)}\circ W_1\in \FNN^{(d)}_{\{\sigma_{\alpha}\}}(1,d)\), where \(\sigma^{(n)}\) is applied component-wise. By Lemma \ref{Lemma-Approximate_function_compositions_sup_unif_cont}, \(\lim_{n\to \infty}\lVert \phi_n - \phi \rVert_{[0,1],\sup}=0\), and since \(\phi(c) = \mathfrak{d}_M(c)\) for all \(c\in \mathcal{C}_{dM}\), we conclude \(\lim_{n\to \infty}\lVert \phi_n - \mathfrak{d}_M \rVert_{\mathcal{C}_{dM},\sup}=0\).
        \end{proof}

\section{Universal approximation of FNNs with LReLU activations}
\label{Section-Proof_Main1}

    The following lemma is used to generalize our universal approximation results from the unit cube to compact sets with non-empty interior.

    \begin{lemma}\label{Lemma-Reduce_universal_approximation_to_unit_cube}
        Let \(k\in \mathbb{N}\), \(n\in [\max\{d_x,d_y\}]\), then: 
        \begin{enumerate}
            \item[(1.)] Let \(\mathcal{A}\) be a set of Lebesgue-measurable activations from \(\mathbb{R}\) to \(\mathbb{R}\) and assume that \(\FNN_{\mathcal{A}}^{(k,n)}(d_x,d_y)\) is a universal approximator of \(C^0([0,1]^{d_x},[0,1]^{d_y})\) on compact sets with respect to the \(L^p\) norm. Then \(\FNN_{\mathcal{A}}^{(k,n)}(d_x,d_y)\) is a universal approximator of \(C^0(\mathcal{K},\mathbb{R}^{d_y})\) on compact sets with respect to the \(L^p\) norm. 
            \item[(2.)] Let \(\mathcal{A}\) be a set of activations from \(\mathbb{R}\) to \(\mathbb{R}\), such that each element of \(\mathcal{A}\) is bounded on any compact set. Additionally, assume that \(\FNN_{\mathcal{A}}^{(k,n)}(d_x,d_y)\) is a universal approximator of \(C^0([0,1]^{d_x},[0,1]^{d_y})\) on compact sets with respect to the supremum norm. Then \(\FNN_{\mathcal{A}}^{(k,n)}(d_x,d_y)\) is a universal approximator of \(C^0(\mathcal{K},\mathbb{R}^{d_y})\) on compact sets with respect to the supremum norm. 
        \end{enumerate}
    \end{lemma}
    
    \begin{proof}
        Let \(\mathcal{K}\subset \mathbb{R}^{d_x}\) be a compact set, which without loss of generality is assumed to have non-empty interior (otherwise we consider a compact superset with non-empty interior), and let \(\epsilon>0\), \(f\in C^0(\mathcal{K},\mathbb{R}^{d_y})\). There exists an invertible affine mapping \(V\in \Aff_{\GL}(d_x)\) such that \(V(\mathcal{K})\subset [0,1]^{d_x}\), and \(W(x)=Ax+b\), \(W\in \Aff_{\GL}(d_y)\), such that \(f(\mathcal{K})\subset W([0,1]^{d_y})\), which is possible since \(f(\mathcal{K})\) is compact as the image of a compact set under a continuous function. We define \(g:=W^{-1}\circ f \circ V^{-1}:V(\mathcal{K})\rightarrow [0,1]^{d_y}\), which is well-defined since \(g(V(\mathcal{K}))=W^{-1}(f(\mathcal{K}))\subset [0,1]^{d_y}\) by the choice of \(W\). By the Tietze extension theorem, there exists a continuous extension \(g^*:[0,1]^{d_x}\rightarrow [0,1]^{d_y}\) of \(g\).
        
        \medskip
        \textbf{(Proof of (1.))} We assume that our class of FNNs is an \(L^p\) universal approximator for functions in \(C^0([0,1]^{d_x},[0,1]^{d_y})\), which \(g^{*}\) is an element of. Thus, for all \(\epsilon>0\) we find \(\psi_\epsilon\in \FNN_{\mathcal{A}}^{(k,n)}(d_x,d_y)\) such that 
        \begin{align}\label{Formula-Inequality_matrix_norms}
            \lVert g^{*}-\psi_\epsilon \rVert_{V(\mathcal{K}),p}\leq \lVert g^{*}-\psi_\epsilon \rVert_{[0,1]^{d_x},p}<\frac{|\det(V)|\cdot\epsilon}{\lVert A \rVert_{\op}}\;.
        \end{align}
        Note that \(W(x)=Ax+b\) with \(A\in \GL(d_y)\), thus 
        \begin{align}\label{formula-Affine-operatornorm}
            \lVert W(x)-W(y) \rVert_p = \lVert A(x-y) \rVert_p\leq \lVert A \rVert_{\op} \lVert x-y \rVert_p\;,
        \end{align}
        where the operator norm of \(A\) is chosen with respect to the \(p\)-norms. Moreover, \(V\) is a \(C^{\infty}\) diffeomorphism, which makes the change of variables formula (see Theorem 3.7.1 in \cite{Bogachev2007}) applicable. For \(\epsilon>0\) define \(\phi_\epsilon:=W\circ \psi_\epsilon \circ V\in \FNN_{\mathcal{A}}^{(k,n)}(d_x,d_y)\) and we obtain 
        \begin{align*}
            &\lVert f-\phi_\epsilon \rVert_{\mathcal{K},p}=\bigg(\int_{\mathcal{K}}\lVert W(g(V(x)))-W(\psi_\epsilon(V(x))) \rVert_{p}^p\,dx\bigg)^{1/p}\\
            &\stackrel{(\ref{formula-Affine-operatornorm})}{\leq} \lVert A \rVert_{\op} \bigg(\int_{\mathcal{K}}\lVert g(V(x))-\psi_\epsilon(V(x)) \rVert_{p}^p\,dx\bigg)^{1/p}\\
            &=\frac{\lVert A \rVert_{\op}}{|\det(V)|}\bigg(\int_{V(\mathcal{K})}\lVert g^{*}(y)-\psi_\epsilon(y) \rVert_{p}^p\,dy\bigg)^{1/p}\\
            &=\frac{\lVert A \rVert_{\op}}{|\det(V)|} \lVert g^{*}-\psi_\epsilon \rVert_{V(\mathcal{K}),p} \stackrel{(\ref{Formula-Inequality_matrix_norms})}{<}\frac{\lVert A \rVert_{\op}}{|\det(V)|}\cdot\frac{|\det(V)|\cdot\epsilon}{\lVert A \rVert_{\op}}=\epsilon\;.
        \end{align*}
        
        \medskip
        \textbf{(Proof of (2.))} We assume that our class of FNNs is a uniform universal approximator for functions in \(C^0([0,1]^{d_x},[0,1]^{d_y})\), which \(g^{*}\) is an element of. Since \(V(\mathcal{K})\subset [0,1]^{d_x}\), there exists \((\psi_n)_{n\in \mathbb{N}}\subset \FNN_{\mathcal{A}}^{(k,n)}(d_x,d_y)\), which is bounded on compact sets since the same holds for each activation, such that
        \begin{align*}
            \lim_{n\to\infty}\lVert g - \psi_n \rVert_{V(\mathcal{K}),\sup} \leq \lim_{n\to\infty}\lVert g^{*} - \psi_n \rVert_{[0,1]^{d_x},\sup} = 0\;.
        \end{align*}
        Let \((\phi_n)_{n\in\mathbb{N}}:=(W\circ \psi_n\circ V)_{n\in \mathbb{N}}\subset \FNN_{\mathcal{A}}^{(k,n)}(d_x,d_y)\). Then
        \begin{align*}
            \lVert f-\phi_n \rVert_{\mathcal{K},\sup}=\lVert W\circ g\circ V- W\circ \psi_n\circ V\rVert_{\mathcal{K},\sup}\leq \lVert A \rVert_{\op}\lVert g- \psi_n\rVert_{V(\mathcal{K}),\sup}\xrightarrow{n\to \infty} 0\;.
        \end{align*}
    \end{proof}

    \bigskip
    Having established the necessary constructions for encoder, memorizer, and decoder, along with their approximation properties, we are now ready to prove our main results.
    \bigskip

    \begin{proof}[Proof of Theorem \ref{Theorem-Main1}]
        We lead the proof for both function classes \(\FNN_{\mathcal{F}_{\pm}}^{(\max\{d_x,d_y\},1)}(d_x,d_y)\) and \(\FNN_{\mathcal{F}_{+}}^{(\max\{2,d_x,d_y\}, 1)}(d_x,d_y)\), as most steps work analogously. If a step in the proof needs different treatment, we note this explicitly. Let \(\mathcal{A}\in \{\mathcal{F}_{\pm}, \mathcal{F}_{+}\}\) be the set of activations used and the corresponding class of neural networks be \(\FNN_{\mathcal{A}}^{(k(\mathcal{A}), 1)}(d_x,d_y)\), where \(k(\mathcal{A})=\max\{d_x,d_y\}\) if \(\mathcal{A}=\mathcal{F}_{\pm}\) and \(k(\mathcal{A})=\max\{2,d_x,d_y\}\) otherwise. 
        
        \medskip
        \textbf{(1.)} For \(\delta>0\), \(\gamma>0\), \(K,M\in \mathbb{N}\), let \(\mathfrak{e}_K^\dag\in \FNN^{(d_x)}_{\mathcal{F}_{+}}(d_x,1)\) be our approximation of the encoder from Lemma \ref{Lemma-Approximate_Encoder} that fulfills 
        \begin{align}\label{formula-FNN_dense_in_M_1}
            \mathfrak{e}_K(c)=\mathfrak{e}_K^\dag(c)\;\forall c\in \mathcal{C}_{K}^{d_x},\quad\lVert \mathfrak{e}_K - \mathfrak{e}_K^\dag \rVert_{[0,1]^{d_x}\setminus S_\gamma,\sup}<\delta,
        \end{align}
        with \(\lambda(S_\gamma)<\gamma\). Furthermore, let \(\mathfrak{m}_{K,M,f^{*}}^\dag\) be the exact implementation of the memorizer, which is in \(\FNN_{\mathcal{F}_{\pm}}^{(1,1)}(1,1)\) by Lemma \ref{Lemma-Memorizer_zig_zag} if \(\mathcal{A}=\mathcal{F}_{\pm}\), or in \(\FNN_{\mathcal{F}_{+}}^{(2,1)}(1,1)\) by Lemma \ref{Lemma-Memorizer-Lp} if \(\mathcal{A}=\mathcal{F}_{+}\). Let \(\mathfrak{d}_M^\dag \in \FNN_{\{\ReLU\}}^{(d_y)}(1,d_y)\) be the exact decoder from Lemma \ref{Lemma-Decoder}, and let \((\mathfrak{d}_M^{(n)})_{n\in\mathbb{N}} \subset \FNN_{\{\sigma_\alpha\}}^{(d_y)}(1,d_y)\) be the approximating sequence from Lemma \ref{Lemma-Decoder_leaky}. For \(n\) large enough, we have
        \begin{align}\label{formula-FNN_dense_in_M_2}
            \lVert \mathfrak{d}_M^{(n)} - \mathfrak{d}_M^\dag \rVert_{\mathcal{C}_{d_yM},\sup}<\delta\;.
        \end{align}
        In steps (2.)–(4.) we show that 
        \begin{align}\label{formula-FNN_dense_in_M_3}
            \phi:=\mathfrak{d}_M^{(n)}\circ \mathfrak{m}_{K,M,f^{*}}^\dag\circ \mathfrak{e}_K^\dag\in \FNN_\mathcal{A}^{(k(\mathcal{A}), 1)}(d_x,d_y)
        \end{align}
        is a suitable approximation of \(f^{*}\). Note that the different widths of the memorizer for the two activation classes lead to the different widths \(k(\mathcal{A})\) of the entire network. The encoder and decoder are constructed equivalently for both classes.
        
        \medskip
        \textbf{(2.)} First, we note that the function \(\psi:=\mathfrak{d}_M^{(n)}\circ \mathfrak{m}^{\dag}_{K,M,f^{*}}\) from \(\mathbb{R}\) to \(\mathbb{R}^{d_y}\) is uniformly continuous on \([0,1]\) as a continuous function on a compact interval. Thus there exists \(\nu>0\) such that for \(|x-y|<\nu\) it holds
        \begin{align}\label{formula-FNN_dense_in_M_5}
            \lVert \psi(x) - \psi(y) \rVert_{p}<\frac{\epsilon}{8}\;.
        \end{align}
        Let \(\mathfrak{c}_{f^{*},K,M}\) be the exact coding scheme from Definition \ref{Definition-Coding_scheme}. We show that if \(\delta \leq \min\{\nu,\frac{\epsilon}{8d_y^{1/p}}\}\), then \(\lVert \phi-\mathfrak{c}_{f^{*},K,M}\rVert_{[0,1]^{d_x}\setminus S_\gamma,p} < \frac{\epsilon}{2}\).
        
        For \(i_j\in [0,2^{K}-1]\), \(j\in [d_x]\), and \(i:=(i_1,\ldots,i_{d_x})\), define
        \begin{align*}
            &S_{i,\gamma}:=[i_1\cdot 2^{-K}, (i_1+1)\cdot 2^{-K})\times\cdots\times [i_{d_x}\cdot 2^{-K}, (i_{d_x}+1)\cdot 2^{-K})\setminus S_\gamma\subset [0,1]^{d_x}\setminus S_\gamma\;,\\
            &c(i_1,\ldots,i_{d_x}):=(i_1\cdot 2^{-K}, i_2\cdot 2^{-K},\ldots,i_{d_x}\cdot 2^{-K})\in \mathcal{C}^{d_x}_{K}\;.
        \end{align*}
        Since \(\dot{\bigcup}_{i_1,\ldots,i_{d_x}=0}^{2^{K}-1}S_{i,\gamma}=[0,1]^{d_x}\setminus S_\gamma\), the sets \(S_{i,\gamma}\) form a disjoint partition of \([0,1]^{d_x}\setminus S_\gamma\). Each \(c\in \mathcal{C}^{d_x}_{K}\) can be written as \(c(i_1,\ldots,i_{d_x})\) for some \(i_j\in [0,2^{K}-1]\), \(j\in [d_x]\), and it holds that
        \begin{align}\label{formula-FNN_dense_in_M_6}
            \mathfrak{m}_{K,M,f^{*}}^\dag\circ \mathfrak{e}_K^\dag(c(i_1,\ldots,i_{d_x}))=\mathfrak{m}_{K,M,f^{*}}\circ \mathfrak{e}_K(c(i_1,\ldots,i_{d_x}))\;,
        \end{align}
        since our implementations of encoder and memorizer are exact on \(\mathcal{C}_{K}^{d_x}\) and \(\mathcal{C}_{d_xK}\), respectively. Furthermore, by definition of the coding scheme,
        \begin{align}\label{formula-FNN_dense_in_M_7}
            \mathfrak{c}_{K,M,f^{*}}(c(i_1,\ldots,i_{d_x}))=\mathfrak{c}_{K,M,f^{*}}(x)\quad\forall x\in S_{i,\gamma}
        \end{align}
        and \(i_j\in [0,2^K-1]\), \(j\in [d_x]\). For \(\delta \leq \min\{\nu,\frac{\epsilon}{8d_y^{1/p}}\}\) we obtain
        \begin{align*}
            &\lVert \phi - \mathfrak{c}_{K,M,f^{*}}\rVert_{[0,1]^{d_x}\setminus S_\gamma,p}=\bigg(\sum_{i_1,\ldots,i_{d_x}=0}^{2^{K}-1}\int_{S_{i,\gamma}}\lVert \phi(x)-\mathfrak{c}_{K,M,f^{*}}(x)\rVert_p^p\,dx \bigg)^{1/p}\\
            &\stackrel{(\ref{formula-FNN_dense_in_M_7})}{=}\bigg(\sum_{i_1,\ldots,i_{d_x}=0}^{2^{K}-1}\int_{S_{i,\gamma}}\lVert \phi(x)-\mathfrak{c}_{K,M,f^{*}}(c(i_1,\ldots,i_{d_x}))\rVert_p^p\,dx \bigg)^{1/p}  \\
            &=\bigg(\sum_{i_1,\ldots,i_{d_x}=0}^{2^{K}-1}\int_{S_{i,\gamma}}\lVert \mathfrak{d}_M^{(n)} \circ \mathfrak{m}_{K,M,f^{*}}^\dag\circ \mathfrak{e}_K^\dag(x)
            -\mathfrak{d}_M^\dag \circ \mathfrak{m}_{K,M,f^{*}}\circ \mathfrak{e}_K(c(i_1,\ldots,i_{d_x}))\rVert_p^p\, dx \bigg)^{1/p}\\
            &\stackrel{(\ref{formula-FNN_dense_in_M_6})}{=}\bigg(\sum_{i_1,\ldots,i_{d_x}=0}^{2^{K}-1}\int_{S_{i,\gamma}}\lVert \mathfrak{d}_M^{(n)} \circ \mathfrak{m}_{K,M,f^{*}}^\dag\circ \mathfrak{e}_K^\dag(x)-\mathfrak{d}_M^\dag \circ \mathfrak{m}_{K,M,f^{*}}^\dag\circ \mathfrak{e}_K^\dag(c(i_1,\ldots,i_{d_x}))\rVert_p^p\, dx \bigg)^{1/p}\\
            &\leq \bigg(2^p\sum_{i_1,\ldots,i_{d_x}=0}^{2^{K}-1}\int_{S_{i,\gamma}}\lVert \mathfrak{d}_M^{(n)} \circ \mathfrak{m}_{K,M,f^{*}}^\dag\circ \mathfrak{e}_K^\dag(x)- \mathfrak{d}_M^{(n)} \circ \mathfrak{m}_{K,M,f^{*}}^\dag\circ \mathfrak{e}_K^\dag(c(i_1,\ldots,i_{d_x}))\rVert_p^p \\
            &\quad+ \lVert  \mathfrak{d}_M^{(n)} \circ \mathfrak{m}_{K,M,f^{*}}^\dag\circ \mathfrak{e}_K^\dag(c(i_1,\ldots,i_{d_x}))-\mathfrak{d}_M^\dag \circ \mathfrak{m}_{K,M,f^{*}}^\dag\circ \mathfrak{e}_K^\dag(c(i_1,\ldots,i_{d_x}))\rVert_p^p\,dx \bigg)^{1/p}\\
            &\stackrel{(\ref{formula-FNN_dense_in_M_5}),\, (\ref{formula-FNN_dense_in_M_2})}{<} \bigg(\sum_{i_1,\ldots,i_{d_x}=0}^{2^{K}-1}\int_{S_{i,\gamma}}  \Big(\frac{2\epsilon}{8}\Big)^{p} + d_y \Big(\frac{2\epsilon}{8d_y^{1/p}}\Big)^{p}\, dx \bigg)^{1/p}\\
            &\leq \bigg(\sum_{i_1,\ldots,i_{d_x}=0}^{2^{K}-1}\int_{S_{i,\gamma}}  \Big(\frac{\epsilon}{2}\Big)^p\, dx \bigg)^{1/p}=\frac{\epsilon}{2} (1-\lambda(S_\gamma))^{1/p}\leq \frac{\epsilon}{2}\;,
        \end{align*}
        where we used that \((a+b)^p\leq (2\max\{a,b\})^p\leq 2^p(a^p+b^p)\) for \(a,b\geq 0\), \(p\geq 1\).
        
        \medskip
        \textbf{(3.)} Since \(f^{*}\in C^0([0,1]^{d_x},[0,1]^{d_y})\), by choosing \(M,K\) large enough, we can apply Lemma \ref{Lemma-Accuracy_coding_scheme} to obtain
        \begin{align}\label{formula-FNN_dense_in_M_8}
            \lVert f^{*}-\mathfrak{c}_{K,M,f^{*}}\rVert_{[0,1]^{d_x},p}< \frac{\epsilon}{4}\;.
        \end{align}
        Set \(r_1:=\max_{x\in [0,1]^{d_x}}\lVert \phi(x)\rVert_p<\infty\), which exists since \(\phi\) is continuous. Note that \(0\leq (\mathfrak{c}_{K,M,f^{*}})_j(x)\leq 1\) for \(x\in [0,1]^{d_x}\), \(j\in [d_y]\), hence 
        \begin{align*}
            r_2:=\max_{x\in [0,1]^{d_x}}\lVert \mathfrak{c}_{f^{*},K,M}(x)\rVert_p\leq  d_y^{1/p}<\infty\;.
        \end{align*}
        Let \(M,K\) be sufficiently large such that (\ref{formula-FNN_dense_in_M_5}) and (\ref{formula-FNN_dense_in_M_8}) hold, and let \(\gamma< \big(\frac{\epsilon}{4(r_1+r_2)}\big)^p\). Then
        \begin{align*}
            &\lVert f^{*}-\phi\rVert_{[0,1]^{d_x},p}\leq \lVert f^{*}-\mathfrak{c}_{K,M,f^{*}}\rVert_{[0,1]^{d_x},p}+\lVert \phi-\mathfrak{c}_{K,M,f^{*}}\rVert_{[0,1]^{d_x}\setminus S_\gamma,p}+\lVert \phi-\mathfrak{c}_{K,M,f^{*}}\rVert_{S_\gamma,p}\\
            &\stackrel{(\ref{formula-FNN_dense_in_M_8}), (2.)}{<} \frac{\epsilon}{4}+\frac{\epsilon}{2}+ \bigg(\int_{S_\gamma}(\lVert \phi(x)\rVert_p+\lVert \mathfrak{c}_{K,M,f^{*}}(x)\rVert_p)^p\,dx \bigg)^{1/p}\\
            &\leq  \frac{3\epsilon}{4} +\bigg(\int_{S_\gamma}(r_1+r_2)^p\,dx \bigg)^{1/p}=  \frac{3\epsilon}{4} +\lambda(S_\gamma)^{1/p}(r_1+r_2)< \epsilon\;.
        \end{align*}
        
        \medskip
        \textbf{(4.)} Let \(\alpha\in (0,1)\cup(1,\infty)\). Define \(\mathcal{A}_1:=\{\sigma_{-\alpha},\sigma_\alpha\}\) if \(\mathcal{A}=\mathcal{F}_{\pm}\), and \(\mathcal{A}_1:=\{\sigma_\alpha\}\) otherwise. By Lemma \ref{Lemma-Approximate_Encoder}, there exists \((\mathfrak{e}_K^{(n)})_{n\in \mathbb{N}}\subset \FNN_{\{\sigma_\alpha\}}^{(d_x)}(d_x,1)\) such that \(\lim_{n\to \infty}\lVert \mathfrak{e}^\dag_K - \mathfrak{e}_K^{(n)}\rVert_{[0,1]^{d_x},\sup}=0\). By Lemma \ref{Lemma-Memorizer_zig_zag} (if \(\mathcal{A}=\mathcal{F}_{\pm}\)) or Lemma \ref{Lemma-Memorizer-Lp} (if \(\mathcal{A}=\mathcal{F}_{+}\)), there exists \((\mathfrak{m}_{K,M,f^{*}}^{(n)})_{n\in \mathbb{N}}\subset \FNN_{\mathcal{A}_1}^{(k(\mathcal{A}),1)}(1,1)\) such that \(\lim_{n\to \infty}\lVert \mathfrak{m}_{K,M,f^{*}}^\dag - \mathfrak{m}_{K,M,f^{*}}^{(n)}\rVert_{\mathcal{K}_1,\sup}=0\) for an appropriately large compact set \(\mathcal{K}_1 \subset \mathbb{R}\). By Lemma \ref{Lemma-Decoder_leaky}, there exists \((\mathfrak{d}_M^{(n)})_{n\in \mathbb{N}}\subset \FNN_{\{\sigma_\alpha\}}^{(d_y)}(1,d_y)\) such that \(\lim_{n\to \infty}\lVert \mathfrak{d}^\dag_M - \mathfrak{d}_M^{(n)}\rVert_{\mathcal{K}_2,\sup}=0\) for an appropriately large compact set \(\mathcal{K}_2\subset \mathbb{R}\).
        
        Since all functions are continuous, we can apply Lemma \ref{Lemma-Approximate_function_compositions_sup_unif_cont} to obtain that 
        \[
            \phi_n:=\mathfrak{d}_M^{(n)}\circ \mathfrak{m}_{K,M,f^{*}}^{(n)}\circ \mathfrak{e}_K^{(n)}\in \FNN_{\mathcal{A}_1}^{(k(\mathcal{A}),1)}(d_x,d_y)
        \]
        satisfies \(\lim_{n\to \infty}\lVert \phi - \phi_n\rVert_{[0,1]^{d_x},\sup}=0\), where \(\phi\) is as in step (3.). Combining this with step (3.), we find that \(\FNN_{\mathcal{A}_1}^{(k(\mathcal{A}),1)}(d_x,d_y)\) is an \(L^p\) universal approximator of \(C^0([0,1]^{d_x},[0,1]^{d_y})\). By Lemma \ref{Lemma-Reduce_universal_approximation_to_unit_cube} (1.), \(\FNN_{\mathcal{A}_1}^{(k(\mathcal{A}),1)}(d_x,d_y)\) is an \(L^p\) universal approximator of \(C^0(\mathbb{R}^{d_x},\mathbb{R}^{d_y})\) on compact sets. 
        
        \medskip
        \textbf{(5.)} For given \(\epsilon>0\) and compact \(\mathcal{K}\subset\mathbb{R}^{d_x}\), let \(f\in L^p(\mathcal{K},\mathbb{R}^{d_y})\). There exists \(f^{*}\in C^0(\mathcal{K},\mathbb{R}^{d_y})\) such that \(\lVert f-f^{*} \rVert_{\mathcal{K},p}<\frac{\epsilon}{2}\), since continuous functions are dense in \(L^p\). By step (4.), there exists \(\phi \in \FNN_{\mathcal{A}_1}^{(k(\mathcal{A}), 1)}(d_x,d_y)\) satisfying \(\lVert f^{*}-\phi \rVert_{\mathcal{K},p}<\frac{\epsilon}{2}\). Thus
        \begin{align*}
            \lVert f-\phi \rVert_{\mathcal{K},p}\leq \lVert f-f^{*} \rVert_{\mathcal{K},p}+\lVert f^{*}-\phi \rVert_{\mathcal{K},p}<\frac{\epsilon}{2}+\frac{\epsilon}{2}=\epsilon\;,
        \end{align*}
        which concludes the proof by the definition of \(\mathcal{A}_1\).
    \end{proof}

    \bigskip
    \begin{proof}[Proof of Theorem \ref{Theorem-Main_sup}]
        We denote \(\mathcal{S}_1:=\{\mathcal{F}_{\pm,\mathfrak{s}}, \mathcal{F}_{+,\mathfrak{s}}\}\) and \(\mathcal{S}_2:=\{\mathcal{F}_+\cup \{\FLOOR\}, \mathcal{F}_{\pm}\cup \{\FLOOR\}\}\). Choose \(\mathcal{A}\in \mathcal{S}_1\cup\mathcal{S}_2\) and set \(k(\mathcal{A}):=\max\{d_x,d_y,2\}\), \(m(\mathcal{A}):=2\) if \(\mathcal{A}\in \{\mathcal{F}_{+,\mathfrak{s}}, \mathcal{F}_+\cup \{\FLOOR\}\}\), and \(k(\mathcal{A}):=\max\{d_x,d_y\}\), \(m(\mathcal{A}):=1\) otherwise. We show the result by distinguishing between the cases \(\mathcal{A}\in \mathcal{S}_1\) and \(\mathcal{A}\in \mathcal{S}_2\), proving the claim for two classes of activations simultaneously in each case.
        
        \medskip
        \textbf{(1.)} We show that \(\FNN_{\mathcal{A}}^{(k(\mathcal{A}),1)}(d_x,d_y)\) is a universal approximator of \(C^0([0,1]^{d_x},[0,1]^{d_y})\) on compact sets with respect to the supremum norm.
        
        \textbf{(Case: \(\mathcal{A}\in \mathcal{S}_1\)):} Let \(f\in C^0([0,1]^{d_x},[0,1]^{d_y})\), \(\epsilon>0\), and choose \(K,M\) large enough such that the coding scheme fulfills
        \begin{align}\label{Formula-main_sup_conv_1}
            \lVert \mathfrak{c}_{K,M,f} - f \rVert_{[0,1]^{d_x},\sup} < \frac{\epsilon}{3}\;,
        \end{align}
        which is possible by Lemma \ref{Lemma-Accuracy_coding_scheme}. By Lemma \ref{Lemma-Decoder_leaky}, choose \(\mathfrak{d}^{\dag}_{M}\in \FNN_{\mathcal{F}_{+}}^{(d_y)}(1,d_y)\) as the approximation of the decoder such that
        \begin{align}\label{Formula-main_sup_conv_2}
            \lVert \mathfrak{d}^{\dag}_{M} - \mathfrak{d}_{M}\rVert_{\mathcal{C}_{d_yM},\sup}<\frac{\epsilon}{6}\;,
        \end{align}
        and note that it is Lipschitz continuous with Lipschitz constant \(\mathcal{L}_M\) as a composition of Lipschitz continuous functions. Let \(\mathfrak{m}^{\dag}_{K,M,f}\in \FNN_{\mathcal{A}}^{(m(\mathcal{A}),1)}(1,1)\) be our exact implementation of the memorizer of width 1 from Lemma \ref{Lemma-Memorizer_zig_zag} if \(\mathcal{A}=\mathcal{F}_{\pm,\mathfrak{s}}\), or of width 2 from Lemma \ref{Lemma-Memorizer-Lp} otherwise. Note that \(\mathfrak{m}^{\dag}_{K,M,f}\) is Lipschitz continuous with Lipschitz constant \(\mathcal{L}_{K,M}\) and fulfills
        \begin{align}\label{Formula-main_sup_conv_3}
            \lVert \mathfrak{m}^{\dag}_{K,M,f} - \mathfrak{m}_{K,M,f} \rVert_{\mathcal{C}_{d_x K},\sup} = 0\;.
        \end{align}
        Let \(\mathfrak{e}^{\dag}_{K}\in \FNN_{\mathcal{F}_{+,\mathfrak{s}}}^{(d_x)}(d_x,1)\) be our approximation of the encoder from Lemma \ref{Lemma-Approximate_encoder_stepped}, chosen such that
        \begin{align}\label{Formula-main_sup_conv_4}
            \lVert \mathfrak{e}_{K} -\mathfrak{e}^{\dag}_{K} \rVert_{[0,1]^{d_x},\sup}<\frac{\epsilon}{6\mathcal{L}_M\mathcal{L}_{K,M}}\;.
        \end{align}
        We denote our coding scheme approximation by \(\phi:= \mathfrak{d}^{\dag}_{M} \circ \mathfrak{m}^{\dag}_{K,M,f} \circ \mathfrak{e}^{\dag}_{K}\in \FNN_{\mathcal{A}}^{(k(\mathcal{A}),1)}(d_x,d_y)\).
        
        For \(i_j\in [0,2^{K}-1]\), \(j\in [d_x]\), and \(i:=(i_1,\ldots,i_{d_x})\), define
        \begin{align*}
            &S_{i}:=[i_1\cdot 2^{-K}, (i_1+1)\cdot 2^{-K})\times\cdots\times [i_{d_x}\cdot 2^{-K}, (i_{d_x}+1)\cdot 2^{-K})\subset [0,1]^{d_x}\;,\\
            &c(i_1,\ldots,i_{d_x}):=(i_1\cdot 2^{-K}, i_2\cdot 2^{-K},\ldots,i_{d_x}\cdot 2^{-K})\in \mathcal{C}^{d_x}_{K}\;,\\
            &I:=\{i=(i_1,\ldots,i_{d_x})\mid i_j\in [0,2^{K}-1],\, j\in [d_x]\}\;.
        \end{align*}
        Then we obtain
        \begin{align*}
            &\lVert \phi - \mathfrak{d}^{\dag}_{M} \circ \mathfrak{m}_{K,M,f} \circ \mathfrak{e}_{K} \rVert_{[0,1]^{d_x},\sup}\\
            &\leq \mathcal{L}_{M}\max_{i\in I}\sup_{x\in S_i}\lVert \mathfrak{m}^{\dag}_{K,M,f} \circ \mathfrak{e}^{\dag}_{K}(x) - \mathfrak{m}_{K,M,f} \circ \mathfrak{e}_{K}(c(i_1,\ldots,i_{d_x})) \rVert \\
            &\leq \mathcal{L}_{M}\max_{i\in I}\sup_{x\in S_i}\bigg(\lVert \mathfrak{m}^{\dag}_{K,M,f} \circ\mathfrak{e}^{\dag}_{K}(x) - \mathfrak{m}^{\dag}_{K,M,f} \circ \mathfrak{e}_{K}(c(i_1,\ldots,i_{d_x})) \rVert \\
            &\quad+ \lVert \mathfrak{m}^{\dag}_{K,M,f} \circ \mathfrak{e}_{K}(c(i_1,\ldots,i_{d_x})) - \mathfrak{m}_{K,M,f} \circ \mathfrak{e}_{K}(c(i_1,\ldots,i_{d_x})) \rVert \bigg)\\
            &\leq \mathcal{L}_{M}\bigg(\mathcal{L}_{K,M}\underbrace{\lVert \mathfrak{e}^{\dag}_{K} - \mathfrak{e}_{K} \rVert_{[0,1]^{d_x},\sup}}_{\stackrel{(\ref{Formula-main_sup_conv_4})}{<}\frac{\epsilon}{6\mathcal{L}_M\mathcal{L}_{K,M}}} + \underbrace{\lVert \mathfrak{m}^{\dag}_{K,M,f} - \mathfrak{m}_{K,M,f} \rVert_{\mathcal{C}_{d_xK},\sup}}_{\stackrel{(\ref{Formula-main_sup_conv_3})}{=}0} \bigg) < \frac{\epsilon}{6}\;,
        \end{align*}
        and since the encoder \(\mathfrak{e}_K\) is constant on each \(S_i\),
        \begin{align*}
            &\lVert \mathfrak{d}^{\dag}_{M} \circ \mathfrak{m}_{K,M,f} \circ \mathfrak{e}_{K} - \mathfrak{c}_{K,M,f}\rVert_{[0,1]^{d_x},\sup}\\
            &= \max_{i\in I} \lVert \mathfrak{d}^{\dag}_{M} \circ \mathfrak{m}_{K,M,f} \circ \mathfrak{e}_{K}(c(i_1,\ldots,i_{d_x})) - \mathfrak{c}_{K,M,f}(c(i_1,\ldots,i_{d_x})) \rVert\\
            &= \lVert \mathfrak{d}^{\dag}_{M} - \mathfrak{d}_{M} \rVert_{\mathcal{C}_{d_yM},\sup} \stackrel{(\ref{Formula-main_sup_conv_2})}{<} \frac{\epsilon}{6}\;.
        \end{align*}
        Therefore,
        \begin{align*}
            \lVert \phi - \mathfrak{c}_{K,M,f}\rVert_{[0,1]^{d_x},\sup} < \frac{\epsilon}{6} + \frac{\epsilon}{6} = \frac{\epsilon}{3}\;.
        \end{align*}
        
        \textbf{(Case: \(\mathcal{A}\in \mathcal{S}_2\)):} We use exactly the same memorizer construction as in the first case: for \(\mathcal{A}=\mathcal{F}_+\cup\{\FLOOR\}\) we use the memorizer of width 2 from Lemma \ref{Lemma-Memorizer-Lp}, and for \(\mathcal{A}=\mathcal{F}_{\pm}\cup\{\FLOOR\}\) we use the memorizer of width 1 from Lemma \ref{Lemma-Memorizer_zig_zag}. For the decoder, we use the exact implementation \(\mathfrak{d}_M^\dag = \mathfrak{d}_M \in \FNN_{\{\ReLU\}}^{(d_y)}(1,d_y)\) from Lemma \ref{Lemma-Decoder}. Instead of the LReLU approximation of the encoder, we construct it exactly using Lemma \ref{Lemma-Encoder_FLOOR}, which shows that \(\mathfrak{e}_{K}\in \FNN_{\{\FLOOR\}}^{(d_x)}(d_x,1)\). 
        
        Define \(\phi:=\mathfrak{d}_M^\dag\circ \mathfrak{m}_{K,M,f}^\dag\circ \mathfrak{e}_K\). Since we use the exact encoder and a memorizer that coincides with the memorizer on its domain \(\mathcal{C}_{d_xK}\), we have
        \begin{align*}
            \mathfrak{m}_{K,M,f}^\dag\circ \mathfrak{e}_K(c)=\mathfrak{m}_{K,M,f}\circ \mathfrak{e}_K(c)\quad\forall c\in \mathcal{C}_{K}^{d_x}\;.
        \end{align*}
        Therefore,
        \begin{align*}
            \lVert \phi - \mathfrak{c}_{K,M,f} \rVert_{[0,1]^{d_x},\sup} = 0\;.
        \end{align*}
        
        \medskip
        \textbf{(2.)} We show that we can reduce the sets of activations to single-parameter families and still approximate the coding scheme uniformly. Let \(\alpha \in (0,1)\cup (1,\infty)\) be chosen arbitrarily.
        
        \textbf{(Case: \(\mathcal{A}\in \mathcal{S}_1\)):} Set \(\mathcal{A}_1(\mathcal{A}):=\mathcal{F}_{\alpha,\mathfrak{s}}\) if \(\mathcal{A}=\mathcal{F}_{+,\mathfrak{s}}\), and \(\mathcal{A}_1(\mathcal{A}):=\mathcal{F}_{\pm\alpha,\mathfrak{s}}\) if \(\mathcal{A}=\mathcal{F}_{\pm,\mathfrak{s}}\). By the second parts of Lemma \ref{Lemma-Approximate_encoder_stepped}, Lemma \ref{Lemma-Memorizer-Lp} or Lemma \ref{Lemma-Memorizer_zig_zag}, and Lemma \ref{Lemma-Decoder_leaky}, we can find sequences \((\mathfrak{e}_K^{(n)})_{n\in \mathbb{N}}\subset \FNN_{\{\sigma_\alpha\}}^{(d_x)}(d_x,1)\), \((\mathfrak{m}_{K,M,f}^{(n)})_{n\in \mathbb{N}}\subset \FNN_{\mathcal{A}_1(\mathcal{A})}^{(m(\mathcal{A}),1)}(1,1)\), and \((\mathfrak{d}_{M}^{(n)})_{n\in \mathbb{N}}\subset \FNN_{\{\sigma_\alpha\}}^{(d_y)}(1,d_y)\) converging uniformly to their respective counterparts on appropriately large compact sets. Then we set 
        \begin{align*}
            (\phi_n)_{n\in \mathbb{N}}:=(\mathfrak{d}_{M}^{(n)}\circ \mathfrak{m}_{K,M,f}^{(n)}\circ \mathfrak{e}_{K}^{(n)})_{n\in \mathbb{N}}\subset \FNN_{\mathcal{A}_1(\mathcal{A})}^{(k(\mathcal{A}),1)}(d_x,d_y)
        \end{align*}
        and apply Lemma \ref{Lemma-Approximate_function_compositions_sup_unif_cont}, which is applicable since only the encoder (as the first function of the composition) is bounded but not necessarily continuous, to obtain \(\lim_{n\to\infty}\lVert \phi_n -\phi \rVert_{[0,1]^{d_x},\sup}=0\).
        
        \textbf{(Case: \(\mathcal{A}\in \mathcal{S}_2\)):} Set \(\mathcal{A}_2(\mathcal{A}):=\{\sigma_\alpha,\FLOOR\}\) if \(\mathcal{A}=\mathcal{F}_{+}\cup \{\FLOOR\}\), and \(\mathcal{A}_2(\mathcal{A}):=\{\sigma_{-\alpha},\sigma_\alpha,\FLOOR\}\) if \(\mathcal{A}=\mathcal{F}_{\pm}\cup \{\FLOOR\}\). Note that we can still construct the encoder exactly since FLOOR is in the set of activations, i.e., \(\mathfrak{e}_K\in \FNN_{\mathcal{A}_2(\mathcal{A})}^{(d_x)}(d_x,1)\). By Lemma \ref{Lemma-Memorizer-Lp} or Lemma \ref{Lemma-Memorizer_zig_zag} and Lemma \ref{Lemma-Decoder_leaky}, we can find sequences \((\mathfrak{m}_{K,M,f}^{(n)})_{n\in \mathbb{N}}\subset \FNN_{\mathcal{A}_2(\mathcal{A})}^{(m(\mathcal{A}),1)}(1,1)\) and \((\mathfrak{d}_{M}^{(n)})_{n\in \mathbb{N}}\subset \FNN_{\{\sigma_\alpha\}}^{(d_y)}(1,d_y)\) converging uniformly to their respective counterparts on appropriately large compact sets. Then we set 
        \begin{align*}
            (\phi_n)_{n\in \mathbb{N}}:=(\mathfrak{d}_{M}^{(n)}\circ \mathfrak{m}_{K,M,f}^{(n)}\circ \mathfrak{e}_{K})_{n\in \mathbb{N}}\subset \FNN_{\mathcal{A}_2(\mathcal{A})}^{(k(\mathcal{A}),1)}(d_x,d_y)
        \end{align*}
        and apply Lemma \ref{Lemma-Approximate_function_compositions_sup_unif_cont} to obtain \(\lim_{n\to\infty}\lVert \phi_n -\phi \rVert_{[0,1]^{d_x},\sup}=0\).
        
        \medskip
        \textbf{(3.)} We choose \(K,M\) large enough such that in step (1.) it holds that \(\lVert \phi - \mathfrak{c}_{K,M,f}\rVert_{[0,1]^{d_x},\sup}<\frac{\epsilon}{3}\) and such that \(\lVert \mathfrak{c}_{K,M,f} - f\rVert_{[0,1]^{d_x},\sup}<\frac{\epsilon}{3}\) by Lemma \ref{Lemma-Accuracy_coding_scheme}. By step (2.), there exists \(\psi\in \FNN_{\mathcal{A}_i(\mathcal{A})}^{(k(\mathcal{A}),1)}(d_x,d_y)\) for \(i\in [2]\) such that \(\lVert \psi - \phi\rVert_{[0,1]^{d_x},\sup}<\frac{\epsilon}{3}\). Combining these three inequalities:
        \begin{align*}
            \lVert \psi - f\rVert_{[0,1]^{d_x},\sup}&\leq \lVert \psi - \phi\rVert_{[0,1]^{d_x},\sup}+ \lVert \phi - \mathfrak{c}_{K,M,f}\rVert_{[0,1]^{d_x},\sup}+\lVert \mathfrak{c}_{K,M,f} - f\rVert_{[0,1]^{d_x},\sup}\\
            &< \frac{\epsilon}{3} + \frac{\epsilon}{3}+ \frac{\epsilon}{3}=\epsilon\;,
        \end{align*}
        which shows that \(\FNN_{\mathcal{A}_i(\mathcal{A})}^{(k(\mathcal{A}),1)}(d_x,d_y)\) for \(i\in [2]\) are universal approximators of \\\(C^0([0,1]^{d_x},[0,1]^{d_y})\) with respect to the supremum norm.
        
        \medskip
        \textbf{(4.)} All sets of activations that we consider are bounded on any compact set, therefore, Lemma \ref{Lemma-Reduce_universal_approximation_to_unit_cube} (2.) is applicable. This implies that \(\FNN_{\mathcal{A}_i(\mathcal{A})}^{(k(\mathcal{A}),1)}(d_x,d_y)\) are universal approximators of \(C^0(\mathbb{R}^{d_x},\mathbb{R}^{d_y})\) on compact sets with respect to the supremum norm. Since \(\mathcal{A}\in \mathcal{S}_1\cup \mathcal{S}_2\) together with \(\mathcal{A}_i(\mathcal{A})\) cover exactly the four sets of activations for which Theorem \ref{Theorem-Main_sup} is formulated, this concludes the proof.
    \end{proof}

    \bigskip
    Note that the set of activations considered in the previous theorem includes a continuous range of step sizes at zero for the LReLUs. This feature renders standard gradient descent ineffective for training feedforward neural networks with these activations, since variations in the step size do not influence the gradients. The following corollary therefore refines the previous construction to yield a set of LReLUs with only two distinct step sizes, achieved by strategically employing a continuous range of slopes elsewhere in the construction.

    \bigskip
    \begin{proof}[Proof of Corollary\ref{Corollary-UAP_sup_fixed_stepsize}]
        By Lemma \ref{Lemma-Approximate_encoder_stepped}, the encoder approximation can be constructed via \(\mathfrak{e}_K^{\dag}\in \FNN_{\mathcal{F}_{+,1}^{*}}^{(d_x)}(d_x,1)\). By definition, \(\mathcal{F}_+\subset \mathcal{F}_{+,1}^{*}\) and \(\mathcal{F}_{\pm}\subset \mathcal{F}_{\pm,1}^{*}\). Since our memorizer and decoder constructions use only standard LReLUs (without steps), these inclusions simply relax the assumptions on the activation sets. Therefore, we can follow the proof of Theorem \ref{Theorem-Main_sup} to conclude that \(\FNN_{\mathcal{F}_{\pm,1}^{*}}^{(\max\{d_x,d_y\},1)}(d_x,d_y)\) and \(\FNN_{\mathcal{F}_{+,1}^{*}}^{(\max\{2,d_x,d_y\},1)}(d_x,d_y)\) are universal approximators of \(C^0(\mathbb{R}^{d_x},\mathbb{R}^{d_y})\) on compact sets with respect to the supremum norm.
    \end{proof}
    \bigskip

    \begin{remark}\label{Remark-Layer_depth_coding_scheme}
    Our coding scheme constructions in Theorems~\ref{Theorem-Main1}, \ref{Theorem-Main_sup}, and \ref{Theorem-Squashable} require a number of layers that grows exponentially with the approximation parameters \(K\) and \(M\). Table~\ref{Table-Coding_scheme_depth} summarizes the layer counts for each component across different activation classes. We emphasize that these are upper bounds achieved by our constructions; whether more depth-efficient implementations exist at minimal width remains an open question.
    
    \begin{table}[h]
    \centering
    \begin{tabular}{lccc}
    \toprule
    \textbf{Component} & \(\mathcal{F}_+, \mathcal{F}_{\pm}, \mathcal{F}_{+,\mathfrak{s}}, \mathcal{F}_{\pm,\mathfrak{s}}\) & \textbf{with FLOOR} & \(\{\sigma_\alpha\}, \{\sigma_{-\alpha}, \sigma_\alpha\}\) \\
    \midrule
    Encoder & \(\mathcal{O}(2^K)\) & \(2\) & \(\mathcal{O}(2^K \cdot g_1(\epsilon))\) \\
    Memorizer & \(\mathcal{O}(2^{d_xK})\) & \(\mathcal{O}(2^{d_xK})\) & \(\mathcal{O}(2^{d_xK} \cdot g_2(\epsilon))\) \\
    Decoder & \(\mathcal{O}(d_y 2^M)\) & \(\mathcal{O}(d_y 2^M)\) & \(\mathcal{O}(d_y 2^M \cdot g_3(\epsilon))\) \\
    \bottomrule
    \end{tabular}
    \caption{Layer counts for our coding scheme constructions. For single-parameter activations, \(g_i(\epsilon) \to \infty\) as \(\epsilon \to 0\) reflects additional depth from LReLU approximations.}
    \label{Table-Coding_scheme_depth}
    \end{table}
    
    For the activation classes \(\mathcal{F}_+\), \(\mathcal{F}_{\pm}\), \(\mathcal{F}_{+,\mathfrak{s}}\), and \(\mathcal{F}_{\pm,\mathfrak{s}}\), our constructions directly build the required piecewise linear functions:
    \begin{itemize}
        \item The encoder from Lemmas~\ref{Lemma-Approximate_Encoder} and \ref{Lemma-Approximate_encoder_stepped} approximates the quantizer \(q_K\), which has \(2^K\) linear pieces. By Lemma~\ref{Lemma-PLCSM_in_LU}, this requires \(\mathcal{O}(2^K)\) layers.
        \item The memorizer from Lemmas~\ref{Lemma-Memorizer-Lp} and \ref{Lemma-Memorizer_zig_zag} interpolates \(2^{d_xK}\) grid points via a piecewise linear function with a comparable number of pieces.
        \item The decoder from Lemma~\ref{Lemma-Decoder} applies min-max strings of length \(2^M\) across \(d_y\) output components.
    \end{itemize}
    
    When FLOOR is available, the encoder can be implemented exactly in depth \(2\) by Lemma~\ref{Lemma-Encoder_FLOOR}, while the memorizer and decoder constructions remain unchanged.
    
    For squashable+FLOOR+Id or STEP+FLOOR+Id activations (Theorem~\ref{Theorem-Squashable}), the encoder and decoder have improved depth: the encoder is exact in depth \(2\) using FLOOR, and the decoder requires only \(\mathcal{O}(d_y)\) layers as it is directly constructed from FLOOR and Id activations. However, the memorizer still requires \(\mathcal{O}(2^{d_xK})\) layers, which still dominates the overall depth and leads to exponential growth of its depth.
    
    For single-parameter activations \(\{\sigma_\alpha\}\) or \(\{\sigma_{-\alpha}, \sigma_\alpha\}\), additional layers are needed to approximate the various LReLUs used in our constructions. By Lemma~\ref{Lemma-Single_Parameter_LReLU_approx_mon}, these approximations introduce accuracy-dependent factors \(g_i(\epsilon)\) that grow as the target accuracy increases.
    
    In summary, the depth of our constructions scales exponentially (or worse) with the approximation parameters. Since achieving approximation accuracy \(\epsilon\) typically requires \(K, M = \mathcal{O}(\log(1/\epsilon))\), the depth grows polynomially in \(1/\epsilon\), which limits direct practical applicability. Our results are primarily of theoretical interest for establishing minimal width universal approximation.
\end{remark}

\section{Approximating the coding scheme with squashable and FLOOR activations}\label{Section-Squash_floor_coding_scheme}

    \begin{lemma}\label{Lemma-Squashable-Construct_STEP}
        Let \(\sigma \in \mathcal{S}\) be a squashable function, then, for all compact intervals \(I=[a,b]\subset \mathbb{R}\), \(c\in (a,b)\), \(\alpha\in \mathbb{R}\), there exists \(\phi \in \FNN_{\{\sigma, \FLOOR\}}^{(1)}(1,1)\) s.t. \(\phi\vert_{[a,b]}(x) = \alpha \cdot\STEP(x-c)\).
    \end{lemma}

    \begin{proof}
            For \(\chi,\epsilon>0\), by rescaling the width one approximation of step with \(1+2\epsilon\) from the definition of a squashable functions, there exists a strictly increasing \(\psi_{\epsilon, \chi} \in \FNN_{\{\sigma\}}^{(1)}(1,1)\) s.t. \(\lVert (1+2\epsilon)\STEP-\psi_{\epsilon, \chi} \rVert_{I\setminus (-\chi,\chi), \sup}<\epsilon\), \(\psi_{\epsilon, \chi}(I)\subset [0,1+2\epsilon]\). Because of the difference in supremum norm, \([\epsilon,1+\epsilon]\subset \psi_{\epsilon, \chi}(I)\) and as \(\psi_{\epsilon, \chi} \) is strictly increasing, by the mean value theorem there exists \(z\in I\) with \(\psi_{\epsilon, \chi}(z)=1\) s.t. for all \(x_1<z<x_2\) it holds that \(0\leq \psi_{\epsilon, \chi}(x_1)<1=\psi_{\epsilon, \chi}(z)\), \(\psi_{\epsilon, \chi}(z)=1<\psi_{\epsilon, \chi}(x_2)\). Therefore, we obtain that \(\phi(x):=(\alpha \Id)\circ\FLOOR \circ \psi_{\epsilon, \chi} \circ (\Id(x)+c-z)\) with \(\phi\in \FNN_{\{\sigma,\,\FLOOR\}}^{(1)}(1,1)\) fulfills \(\phi\vert_{I}(x)=\alpha\mathbbm{1}_{[z,b]}(x+z-c)=\alpha\cdot\STEP(x-c)\).
        \end{proof}

    \begin{lemma}\label{Lemma-Memorizer_STEP}
        Assume that \(\mathcal{A}\) is a set of activations with \(\Id\in \mathcal{A}\) such that for \(\alpha,c\in \mathbb{R}\) it holds that there exists \(\psi_{[0,1],\alpha,c}\in \FNN_{\mathcal{A}}^{(1)}(1,1)\) that fulfills \(\psi_{[0,1],\alpha,c}\vert_{[0,1]}(x)=\alpha\STEP(x-c)\). Then, for all \(K,M\in \mathbb{N}\), \(f\in C^0([0,1]^{d_x},[0,1]^{d_y})\) there exists \(\phi\in \FNN_{\mathcal{A}}^{(3)}(1,1)\) such that 
        \begin{align*}
            \phi(c)=\mathfrak{m}_{K,M,f}(c)\;\;\forall c\in \mathcal{C}_{d_xK}\;.
        \end{align*}
    \end{lemma}

    \begin{proof}
            We show that for \(0\leq x_1 < x_2<\hdots <x_n\), \(y_i\in [0,1], i\in [n]\) for \(n\in \mathbb{N}\) we can construct \(\phi\in \FNN_{\mathcal{A}}^{(3)}(1,1)\) such that \(\phi(x_i)=y_i\) for \(i\in [n]\), which implies the statement that we want to show. We lead the proof by induction: Set \(\phi_0(x):=(1,1,1)^T x\) with \(\phi_0\in \FNN_{\mathcal{A}}^{(3)}(1,3)\) as the identity is part of the set of activations. We show that for \(i\in [n]\) we can construct \(\phi_i\in \FNN_{\mathcal{A}}^{(3)}(3,3)\) and \(\Phi_i:=\phi_i\circ \hdots \circ \phi_0(x)\) with \(\Phi_i\in \FNN_{\mathcal{A}}^{(3)}(1,3)\) fulfills \(\Phi_i(x)_1=\sum_{k=1}^{i}\alpha_k\STEP(x-x_k)\) with some \(\alpha_k\in [-1,1]\) for all \(k\in [i]\) such that \(\Phi_i(x_i)_1=y_i\) for all \(i\in [n]\) and \(\Phi_i(x)_j=x\) for \(j\in \{2,3\}\). 
            
            \textbf{(\(i=1\)):} We choose \(\nu_1^{(1)}:=\psi_{[0,1],y_1,x_1}\), which is of some depth \(L_1\in \mathbb{N}\) and for \(\nu_2^{(1)},\nu_3^{(1)}\) we append affine identities and identity activations exactly as often such that \((\nu_i)_{i\in[3]}\subset \FNN_{\mathcal{A}}^{(3)}(1,1;L_1)\) and set \(\phi_1\left(\begin{pmatrix}
                v_1\\
                v_2\\
                v_3
            \end{pmatrix}\right):=\begin{pmatrix}
                    \nu_1(v_1) \\
                    \nu_2(v_2)\\
                    \nu_3(v_3)
                \end{pmatrix}\) with \(\phi_1\in\FNN_{\mathcal{A}}^{(3)}(3,3)\). Then we have \(\Phi_1(x)=\phi_1\circ \phi_0(x)=(\psi_{[0,1],y_1,x_1}(x),x,x)^T\), thus \(\Phi_1 (x_1)_1=\psi_{[0,1],y_1,x_1}(x_1)=y_1\cdot\STEP(0)=y_1\) and \(\Phi_1(x)_j=\nu_j(x)=x\) for \(j\in \{2,3\}\).
            
            \textbf{(\(i-1\rightarrow i\)):} Assume that the induction assumption (IA) is true for \(i-1\) with \(i\in [2,n]\), then we define \(z_i:=y_i-\sum_{j=1}^{i-1}y_j\) and set \(\nu_2^{(i)}:=\psi_{[0,1],z_i,x_i}\), which is of some depth \(L_i\in \mathbb{N}\) and for \(\nu_1^{(i)},\nu_3^{(i)}\) we append affine identities and identity activations exactly as often such that \((\nu_i)_{i\in[3]}\subset \FNN_{\mathcal{A}}^{(3)}(1,1;L_i)\) and define \(\phi_i\left(\begin{pmatrix}
                v_1\\
                v_2\\
                v_3
            \end{pmatrix}\right):= \begin{pmatrix}
                    1 & 1 & 0\\
                    0 & 0 & 1\\
                    0 & 0 & 1
                \end{pmatrix}\circ 
                \begin{pmatrix}
                    \nu_1(v_1) \\
                    \nu_2(v_2)\\
                    \nu_3(v_3)
                \end{pmatrix}\in \FNN_{\mathcal{A}}^{(3)}(3,3)\). 
            Then, by the IA we have \(\Phi_i(x)=(\Phi_{i-1}(x)+\psi_{[0,1],z_i,x_i}(x),x,x)^T\) and, moreover, for \(k\in [i]\) we have
            \begin{align*}
                \Phi_i(x_k)_1\overset{\text{IA}}{=}\begin{cases}
                    \Phi_{i-1}(x_k)=y_k &, k \in [i-1]\\
                    \Phi_{i-1}(x_i)+\psi_{[0,1],z_i,x_i}(x_i)=\bigg(\sum_{j=1}^{i-1}y_j\bigg) + z_i=y_i &, k=i
                \end{cases}\;.
            \end{align*}
            \\As the induction is concluded, we can define \(\bar{\Phi}_n:=(1,0,0)\circ\Phi_n\in \FNN_{\mathcal{A}}^{(3)}(1,1)\), which fulfills \(\bar{\Phi}_n(x_i)=\Phi_n(x_i)_1=y_i\) for all \(i\in [n]\) and shows the original claim.
        \end{proof}

    In the next lemma we construct the decoder using FNNs with \(\FLOOR\) and \(\Id\). Note that this follows a similar idea to that of \cite{park2020minimum}, extracting the quantized function values of the memorizer one dimension at a time; however, they constructed it using ReLU.

    \begin{lemma}\label{Lemma-Decoder_FLOOR}
        For \(M\in \mathbb{N}\), \(d\in \mathbb{N}\) let \(\mathfrak{d}_{M}\) be the decoder that decodes to \(\mathbb{R}^d\), then it holds that there exists \(\phi\in \FNN_{\{\FLOOR,\Id\}}^{(d)}(1,d)\) such that \(\phi(c)=\mathfrak{d}_M(c)\) for all \(c\in \mathcal{C}_{dM}\).
    \end{lemma}

    \begin{proof}
            First we note that for all \(n\) \(\mathfrak{d}_M(c)_n = q_M(2^{nM}c - \FLOOR(2^{nM}c))\).
            We construct a suitable \(\FNN\) of depth \(d\), i.e. \(\phi \in \FNN_{\{\FLOOR,\Id\}}^{(d)}(1,d;d)\). First, we define \(c_n:=2^{nM}c-\sum_{k=1}^{n-1}2^{(n+1-k)M}\mathfrak{d}_M(c)_k\) for all \(n\in [d]\) and observe that
            \begin{align}\label{Formular-Decoder_1}
            &\mathfrak{d}_M(c)_n 
                  = 2^{-M}\left(2^{nM}q_{nM}(c)-\sum_{k=1}^{n-1}2^{(n+1-k)M}\mathfrak{d}_{M}(c)_k\right) \nonumber \\
                & = 2^{-M}\left(\FLOOR(2^{nM}c)-\sum_{k=1}^{n-1}2^{(n+1-k)M}\mathfrak{d}_{M}(c)_k\right)=2^{-M}\FLOOR(c_n)\;,
            \end{align}where \(q_k\) denotes the quantizer for some \(k\in \mathbb{N}\). For all \(n\in [d-1]\) we have that \(c_{n+1}=2^{M}c_n-2^{2M}\mathfrak{d}_M(c)_n\). 
            
            For illustration of the idea of our construction, consider \(M=2\), \(d=2\), and \(c=0.1101\) (in binary). Then \(c_1 = 2^2c = 11.01\) (shifting left by \(2\) bits), thus \(2^{-2}\cdot\FLOOR(c_1) = 2^{-2}\cdot 11=0.11\) and \(\mathfrak{d}_M(c)_1 = 0.11\) (shifting right by \(2\) bits). Next, \(c_2 = 2^2\cdot 11.01 - 2^4\cdot0.11=1101.0-1100.0= 1.0\), thus \(2^{-2}\cdot\FLOOR(c_2) = 0.01\) and \(\mathfrak{d}_M(c)_2 = 0.01\). Hence \(\mathfrak{d}_M(0.1101) = (0.11, 0.01)^T=(2^{-2}\FLOOR(c_1),2^{-2}\FLOOR(c_2))^T\), where we extracted the \(2\) bits for each dimension with this procedure. 
            
            Now, we define
            \(W_n:=\begin{pmatrix}
                \Id_{1:n-1,1:n-1} & 0_{1:n-1,1}\\
                w_n^T\\
                w_n^T
            \end{pmatrix}\in \mathbb{R}^{n+1\times n}\) with \(w_1=2^{M}\in \mathbb{R}\) and \(w_n=(0,\hdots,0,-2^{2M},2^M)^T\in \mathbb{R}^n\) if \(n\in [2,d-1]\), and \(W_d=\begin{pmatrix}
                \Id_{1:d-1,1:d-1} & 0_{1:d-1,1}\\
                w_d^T
            \end{pmatrix}\in \mathbb{R}^{d\times d}\). Additionally, let \(V_n:=\diag(1,\hdots,1,2^{-M},1)\in \mathbb{R}^{n+1\times n+1}\) if \(n\in [d-1]\) and \(V_d=\diag(1,\hdots,1,2^{-M})\in \mathbb{R}^{d\times d}\). Moreover, set \(\sigma_n:=(\Id,\hdots,\Id,\FLOOR,\Id)^T\) defined on \(\mathbb{R}^{n+1}\) if \(n\in [d-1]\) and \(\sigma_d=(\Id, \hdots,\Id,\FLOOR)^T\) on \(\mathbb{R}^d\), otherwise. We set \(\phi_n=V_n\circ \sigma_n\circ W_n\circ \hdots \circ V_1\circ \sigma_1 \circ W_1\in \FNN_{\{\FLOOR,\Id\}}^{(n+1)}(1,n+1)\) for all \(n\in [d-1]\) and \(\phi_{d}:=V_d\circ \sigma_d\circ W_d\circ \phi_{d-1}\in \FNN_{\{\FLOOR,\Id\}}^{(d)}(1,d)\). Now, we show by induction that 
            \begin{align*}
                \phi_n(c)_{1:n}=\mathfrak{d}_M(c)_{1:n}\;\forall n\in [d]\;\;\text{ and }\;\; \phi_n(c)_{n+1}=c_n \;\forall n\in [d-1]\;,
            \end{align*}where the subscripts on vectors denote its corresponding component. Note that this then concludes the proof as \(\phi_d(c)=\mathfrak{d}_M(c)\). 
            
            \bigskip
            
            \textbf{(\(n=1\)):} This follows directly from \(\phi_1(c)_1=2^{-M}\FLOOR(2^Mc)=2^{-M}\FLOOR(c_1)=\mathfrak{d}_M(c)_1\) and \(\phi_1(c)_2=2^Mc=c_1\).

            \textbf{(\(n-1\rightarrow n):\)} Assume that the IA holds for \(n-1\in [d-1]\), then by construction \(\phi_n(c)_{1:n-1}=(V_n\circ \sigma_n\circ W_n\circ \phi_{n-1})(c)_{1:n-1}=\phi_{n-1}(c)_{1:n-1}\overset{\text{IA}}{=}\mathfrak{d}_M(c)_{1:n-1}\). Now, if \(n\in [d-1]\) we have that 
            \begin{align*}
                &W_n\circ \phi_{n-1}(c)_{n}=2^Mc_{n-1}-2^{2M}\phi_{n-1}(c)_{n-1}=2^Mc_{n-1}-2^{2M}\mathfrak{d}_M(c)_{n-1}=c_n
                \\&\overset{n<d}{=}W_n\circ \phi_{n-1}(c)_{n+1}\;,
            \end{align*}where the last equality only holds if \(n<d\). Therefore, it follows that 
            \begin{align*}
                \phi_n(c)_n = 2^{-M}\cdot\FLOOR(c_n)\overset{(\ref{Formular-Decoder_1})}{=}\mathfrak{d}_M(c)_n
            \end{align*}and if \(n<d\), additionally, \(\phi_n(c)_{n+1}=(V_n\circ\sigma_n\circ W_n\circ \phi_{n-1})(c)_{n+1}=c_n\). 
        \end{proof}

\section{Universal approximation with FLOOR and squashable functions and proof of Theorem \ref{Theorem-Squashable}}\label{Section-Proof_Theorem_Squashable}

    Having constructed the key components of our coding scheme using FLOOR+\(\Id\) combined with squashable or STEP activations—the encoder (Lemma \ref{Lemma-Encoder_FLOOR}), decoder (Lemma \ref{Lemma-Decoder_FLOOR}), and memorizer (Lemma \ref{Lemma-Memorizer_STEP})—we are now ready to prove our main universal approximation result. The following lemma provides a general framework that applies to any activation class capable of constructing scaled step functions on compact intervals.

    \begin{lemma}\label{Lemma-Unif_UAP_STEP}
        Let \(d_x,d_y\in \mathbb{N}\), \(\mathcal{A}\) be an activation class with \(\{\FLOOR,\Id\}\subset \mathcal{A}\) such that any activation function is bounded on compact sets and for all \(\alpha,c\in \mathbb{R}\) there exists \(\psi_{[0,1],\alpha,c}\in \FNN_{\mathcal{A}}^{(1)}(1,1)\) with \(\psi_{[0,1],\alpha,c}\vert_{[0,1]}(x)=\alpha\STEP(x-c)\). Then, \(\FNN_{\mathcal{A}}^{(\max\{3,d_x,d_y\},1)}(d_x,d_y)\) is a universal approximator of \(C^0(\mathbb{R}^{d_x},\mathbb{R}^{d_y})\) on compact sets w.r.t. the supremum norm.
    \end{lemma}

    \begin{proof}
            Let \(K,M\in \mathbb{N}\), \(f\in C^0([0,1]^{d_x},[0,1]^{d_y})\) and \(\epsilon>0\). By Lemma \ref{Lemma-Encoder_FLOOR} it holds that \(\mathfrak{e}_{K}\in \FNN_{\{\FLOOR\}}^{(d_x)}(d_x,1)\) and by Lemma \ref{Lemma-Decoder_FLOOR} we know that \(\mathfrak{d}_{M}^\dag\in \FNN_{\{\FLOOR,\,\Id\}}^{(d_y)}(1,d_y)\) with \(\mathfrak{d}_{M}^\dag(c)=\mathfrak{d}_{M}(c)\) for all \(c\in \mathcal{C}_{d_yM}\). Moreover, the assumptions on \(\mathcal{A}\) make Lemma \ref{Lemma-Memorizer_STEP} applicable, hence \(\mathfrak{m}_{K,M,f}\in \FNN_{\mathcal{A}}^{(3)}(1,1)\). This implies that \(\mathfrak{c}_{K,M,f}\in \FNN_{\mathcal{A}}^{(\max\{3,d_x,d_y\},1)}(d_x,d_y)\) and Lemma \ref{Lemma-Accuracy_coding_scheme} implies that there exist \(K,M\) large enough such that \(\lVert \mathfrak{c}_{K,M,f} - f\rVert_{[0,1]^{d_x},\sup}<\epsilon\). Therefore, \(\FNN_{\mathcal{A}}^{(\max\{3,d_x,d_y\},1)}(d_x,d_y)\) is a uniform universal approximator of \\\(C^0([0,1]^{d_x},[0,1]^{d_y})\). Moreover, as all activations are bounded on compact sets we can apply Lemma \ref{Lemma-Reduce_universal_approximation_to_unit_cube} (2.), which implies that \(\FNN_{\mathcal{A}}^{(\max\{3,d_x,d_y\},1)}(d_x,d_y)\) is a uniform universal approximator of \(C^0(\mathbb{R}^{d_x},\mathbb{R}^{d_y})\) on compact sets.
        \end{proof}
    \bigskip
    \begin{proof}[Proof of Theorem \ref{Theorem-Squashable}]
        Let \(\sigma\in \mathcal{S}\), then \(\{\sigma, \FLOOR, \Id \}\) fulfills the assumptions of Lemma \ref{Lemma-Unif_UAP_STEP} as \(\FLOOR\), \(\Id\) are included as elements and as Lemma \ref{Lemma-Squashable-Construct_STEP} showed that we can construct scaled \(\STEP\) functions on \([0,1]\) with width-one \(\sigma\) FNNs. The set of activations \(\{\STEP,\FLOOR,\Id\}\) also fulfills the assumptions of Lemma \ref{Lemma-Unif_UAP_STEP}, as for \(\alpha,c\in \mathbb{R}\) we have \(\alpha\, \STEP(\cdot - c)=(\alpha \,\Id) \circ \STEP \circ (\Id -c)\in \FNN_{\{\STEP\}}^{(1)}(1,1)\). Therefore, we can apply Lemma \ref{Lemma-Unif_UAP_STEP} to both sets of activations and obtain that \(\FNN_{\{\sigma,\, \FLOOR,\, \Id\}}^{(\max\{3,d_x,d_y\},1)}(d_x,d_y)\) and \(\FNN_{\{\STEP,\, \FLOOR,\, \Id\}}^{(\max\{3,d_x,d_y\},1)}(d_x,d_y)\) are universal approximators of \(C^0(\mathbb{R}^{d_x},\mathbb{R}^{d_y})\) on compact sets w.r.t. the supremum norm.
    \end{proof}

\section{\texorpdfstring{$LU$-decomposable}{LU-decomposable} neural networks with leaky ReLU activations are \(L^p\) universal approximators}\label{Section-Proof_main2}

        In this section we show how the universal approximation results for FNNs can be generalized to \(LU\)-decomposable neural networks. The idea is based on the fact that (invertible) \(LU\)-decomposable matrices are dense in the set of all matrices by Lemma \ref{Lemma-LU_dense}. Note that the following lemma can be applied to our UA results for continuous functions, however, we focus on applying it to LReLUs and squashable activations as these are most interesting for constructing invertible LU- or UL-decomposable networks.
    
    \begin{lemma}\label{Lemma-FNN_to_LU_compact}
        Let \(d\in \mathbb{N}\), \(\mathcal{K}\subset \mathbb{R}^d\) be a compact set, \(\mathcal{A}\) be a set of continuous functions from \(\mathbb{R}\) to \(\mathbb{R}\) and \(\psi\in\FNN_{\mathcal{A}}^{(s)}(d,d)\) with \(s\in [d]\). Then there exist \((\phi_n)_{n\in \mathbb{N}}\subset \LU_{\mathcal{A}}(d)\) with \(\lVert \psi-\phi_n \rVert_{\mathcal{K},\sup}\xlongrightarrow{n\to \infty}0\) and \((\nu_n)_{n\in \mathbb{N}}\subset \LU^{-1}_{\mathcal{A}}(d)\) with \(\lVert \psi-\nu_n \rVert_{\mathcal{K},\sup}\xlongrightarrow{n\to \infty}0\). Moreover, as \(\lVert \psi - \phi_n \rVert_{\mathcal{K},p} \leq C\lVert \psi - \phi_n \rVert_{\mathcal{K},\sup}\) on compact sets, it follows that \(\lVert \psi-\phi_n \rVert_{\mathcal{K},p}, \lVert \psi-\nu_n \rVert_{\mathcal{K},p}\xlongrightarrow{n\to \infty}0\) as well. Consequently, \(\LU_{\mathcal{A}}(d)\) and \(\LU^{-1}_{\mathcal{A}}(d)\) inherit all universal approximation properties of \(\FNN_{\mathcal{A}}^{(s)}(d,d)\) w.r.t. supremum and \(L^p\) norm.
    \end{lemma}	
    
    \begin{proof}
        The idea is to fill up the linear layers of the neural network with zeros to obtain matrices \(A_i\in \mathbb{R}^{d\times d}\), for which we can find arbitrarily good \(LU\)-decomposable approximations. By exchanging the filled up matrices with corresponding \(LU\)-decomposable ones, or their inverses, we obtain the desired approximation. Note that the \(\sigma\)'s in the proof correspond to arbitrary activations from \(\mathcal{A}\) and the subscripts do not refer to the slopes of the LReLUs, as they usually did in the proofs of the previous sections.
        
        \medskip
        \noindent(1) Let \(\psi\in \FNN_{\mathcal{A}}^{(s)}(d,d)\) have the form from Definition \ref{Definition-FNN}, thus there is \(L\in \mathbb{N}\) such that for all \(i\in [L]\) we have \(d_0=d=d_L\), \(d_i\in [s]\), \(W_i\in \Aff(d_{i-1},d_i)\), \(W_i(x)=A_ix+b_i\) with \(A_i\in \mathbb{R}^{d_i\times d_{i-1}}\) for \(i\in [L]\) and \(\varphi_i=(\sigma_1^{(i)},\ldots,\sigma_{d_i}^{(i)})\), \(\sigma_{j}^{(i)}\in \mathcal{A}\), \(j\in [d_i]\) for all \(i\in [L-1]\), such that \(\psi=W_L\circ \varphi_{L-1}\circ \cdots\circ W_2\circ \varphi_1\circ W_1\). Let \(0_{n,k}\in \mathbb{R}^{n\times k}\) be the zero matrix and define 
        \begin{align*}
            \bar{A}_i&:=\begin{pmatrix}
                A_i & 0_{d_i,d-d_{i-1}}\\
                0_{d-d_{i},d_{i-1}} & 0_{d-d_{i},d-d_{i-1}} 
            \end{pmatrix}\in \mathbb{R}^{d\times d},\quad \bar{W}_i(x):=\bar{A}_ix+\bar{b}_i,\quad\bar{b}_i:=\begin{pmatrix}
                b_i\\
                0_{d-d_i,1}
            \end{pmatrix},\quad i\in [L],
            \\
            \bar{\varphi}_i&:=(\sigma_1^{(i)},\ldots,\sigma_{d_i}^{(i)}, \bar{\sigma}^{(i)}_{d_i+1},\ldots,  \bar{\sigma}^{(i)}_{d}),\quad i\in [L-1],\quad\bar{\psi}:=\bar{W}_L\circ \bar{\varphi}_{L-1}\circ \cdots\circ \bar{W}_2\circ \bar{\varphi}_1\circ \bar{W}_1,
        \end{align*}
        where \(\bar{\sigma}^{(i)}_{d_i+k}\in \mathcal{A}\) can be chosen arbitrarily for \(k\in [d-d_i]\) and we used that \(d_i\leq d\) as \(\psi\) is of width \(s\). Note that by construction \(\psi=\bar{\psi}\) and that \(\bar{\psi}\in \FNN_{\mathcal{A}}^{(d)}(d,d)\) has layer dimensions \(\bar{d}_i=d\) for \(i\in [L]\).
        
        \medskip
        \noindent(2) Let \(\mathcal{K}\subset \mathbb{R}^d\) be compact. As \(LU\)-decomposable matrices are dense in \(\mathbb{R}^{d\times d}\) with respect to uniform convergence on compact sets by Lemma \ref{Lemma-LU_dense}, we can find sequences of matrices \((B_i^{(n)})_{n\in \mathbb{N}}\subset \LU(d)\) such that \(\lVert \bar{A}_i - B_i^{(n)} \rVert_{\op} \stackrel{n \to \infty}{\longrightarrow} 0\) for \(i\in [L]\), where the operator norm corresponds to the supremum norm for vectors on \(\mathbb{R}^d\). Then, it holds that
        \begin{align*}
            \lVert \bar{A}_i(x) - B_i^{(n)}(x) \rVert_{\max} = \lVert (\bar{A}_i - B_i^{(n)})(x) \rVert_{\max}\leq \lVert \bar{A}_i - B_i^{(n)} \rVert_{\op} \lVert x \rVert_{\max} \stackrel{n \to \infty}{\longrightarrow} 0
        \end{align*}
        for all \(x\in \mathbb{R}^d\), thus the matrices converge uniformly. We set \(V_i^{(n)}(x):=B_i^{(n)}x+\bar{b}_i\), which fulfill \((V_i^{(n)})_{n\in \mathbb{N}}\subset \Aff_{\LU}(d)\) for any \(i\in [L]\), and define
        \begin{align*}
            \phi_n:= V_L^{(n)}\circ \bar{\varphi}_{L-1}\circ\cdots\circ V_2^{(n)} \circ \bar{\varphi}_1\circ V_1^{(n)}\in \LU_{\mathcal{A}}(d)
        \end{align*}
        for all \(n\in \mathbb{N}\). As the sequences \((V_i^{(n)})_{n\in \mathbb{N}}\) converge uniformly to \(\bar{W}_i\) for all \(i\in [L]\) and as all composed functions of \(\phi_n\), \(\bar{\psi}\) are continuous, we can apply Lemma \ref{Lemma-Approximate_function_compositions_sup_unif_cont} to obtain that
        \begin{align*}
            \lVert \psi-\phi_n \rVert_{\mathcal{K},\sup}=\lVert \bar{\psi}-\phi_n \rVert_{\mathcal{K},\sup}\stackrel{n \to \infty}{\longrightarrow} 0\;.
        \end{align*}
        For \(\LU^{-1}_{\mathcal{A}}(d)\) we can use an analogous argument: The set of \(UL\)-decomposable matrices is dense in \(\mathbb{R}^{d\times d}\) by Lemma \ref{Lemma-LU_dense}, thus, similar to before, we can find \((U_i^{(n)})_{n\in \mathbb{N}}\subset \Aff_{\LU^{-1}}(d)\) that converge uniformly to \(\bar{W}_i\). Choosing 
        \begin{align*}
            \nu_n:= U_L^{(n)}\circ \bar{\varphi}_{L-1}\circ\cdots\circ U_2^{(n)} \circ \bar{\varphi}_1\circ U_1^{(n)}\in \LU^{-1}_{\mathcal{A}}(d)
        \end{align*}
        and applying Lemma \ref{Lemma-Approximate_function_compositions_sup_unif_cont} implies 
        \begin{align*}
            \lVert \psi-\nu_n \rVert_{\mathcal{K},\sup}=\lVert \bar{\psi}-\nu_n \rVert_{\mathcal{K},\sup}\stackrel{n \to \infty}{\longrightarrow} 0\;.
        \end{align*}
        \end{proof}
		
	   \bigskip
        With Lemma \ref{Lemma-FNN_to_LU_compact} we are now ready to prove Theorem \ref{Theorem-Main2} and Theorem \ref{Theorem-LU_Squashable}, which generalize the universal approximation results of FNNs with leaky ReLU or squashable activations to LU-decomposable neural networks with said activations.
    \bigskip
    \begin{proof}[Proof of Theorem \ref{Theorem-Main2}]
        Let \(\alpha\in(0,1)\cup(1,\infty)\), \(f\in L^p(\mathbb{R}^{d},\mathbb{R}^{d})\) and for \(\epsilon>0\), \(\mathcal{K}\subset \mathbb{R}^d\) compact choose \(\psi\in \FNN_{\{\sigma_\alpha\}}^{(d)}(d,d)\) as an approximation from Theorem \ref{Theorem-Main1} such that \(\lVert f  - \psi\rVert_{\mathcal{K},p}<\frac{\epsilon}{2}\). By applying Lemma \ref{Lemma-FNN_to_LU_compact} with \(\{\sigma_\alpha\}\), which is continuous, we can find \(\phi\in \LU_{\{\sigma_{\alpha}\}}(d)\) with \(\lVert \psi - \phi\rVert_{\mathcal{K},p}<\frac{\epsilon}{2}\) and obtain that \( \lVert f  - \phi\rVert_{\mathcal{K},p}\leq\lVert f  - \psi\rVert_{\mathcal{K},p}+ \lVert \psi - \phi\rVert_{\mathcal{K},p}<\epsilon\). With the same argument, again by Lemma \ref{Lemma-FNN_to_LU_compact}, we can find a UL-decomposable network \(\phi\in \LU^{-1}_{\{\sigma_{\alpha}\}}(d)\) with \(\lVert \psi - \phi\rVert_{\mathcal{K},p}<\frac{\epsilon}{2}\), thus \(\lVert f  - \phi\rVert_{\mathcal{K},p}<\epsilon\). Moreover, as all affine transformations in LU- or UL-decomposable networks are bijections and as \(\sigma_\alpha\) is bijective, it follows that these sets consist only of compositions of bijections, which are again bijective.
    \end{proof}
    
    \bigskip
    
    \begin{proof}[Proof of Theorem \ref{Theorem-LU_Squashable}]
        We follow the same structure as the previous proof: For \(\sigma\in\mathcal{S}\), i.e., \(\sigma\) is squashable, \(f\in L^p(\mathbb{R}^{d},\mathbb{R}^{d})\) and for \(\epsilon>0\), \(\mathcal{K}\subset \mathbb{R}^d\) compact, by Theorem 1 of \cite{shin2025minimumwidthuniversalapproximation} there exists \(\psi\in \FNN_{\{\sigma\}}^{(d)}(d,d)\) such that \(\lVert f  - \psi\rVert_{\mathcal{K},p}<\frac{\epsilon}{2}\). By applying Lemma \ref{Lemma-FNN_to_LU_compact} with \(\{\sigma\}\), which is continuous by definition, we can find \(\phi\in \LU_{\{\sigma\}}(d)\) with \(\lVert \psi - \phi\rVert_{\mathcal{K},p}<\frac{\epsilon}{2}\) and obtain that \( \lVert f  - \phi\rVert_{\mathcal{K},p}\leq\lVert f  - \psi\rVert_{\mathcal{K},p}+ \lVert \psi - \phi\rVert_{\mathcal{K},p}<\epsilon\). With the same argument, again by Lemma \ref{Lemma-FNN_to_LU_compact}, we can find a UL-decomposable network \(\phi\in \LU^{-1}_{\{\sigma\}}(d)\) with \(\lVert \psi - \phi\rVert_{\mathcal{K},p}<\frac{\epsilon}{2}\), thus \(\lVert f  - \phi\rVert_{\mathcal{K},p}<\epsilon\). Moreover, as all affine transformations in LU- or UL-decomposable networks are bijections, if the squashable function is additionally bijective, it follows that these sets consist only of compositions of bijections, which are again bijective.
    \end{proof}

    \section{LU-Net has the distributional universal approximation property}\label{Section-Proof_main3}

    In this section we show that the LU-Net has the DUAP, which is due to the strong approximation properties of the LU-decomposable neural networks with leaky ReLU activations.

    \begin{lemma}\label{Lemma-DUAP_inverse_set_classification}
        Let \(\sigma\) be an invertible function from \(\mathbb{R}\) to \(\mathbb{R}\), then we have \(\LU^{-1}_{\{\sigma^{-1}\}}(d)= (\LU_{\{\sigma\}}(d))^{-1}:=\{\phi^{-1}|\,\phi \in \LU_{\{\sigma\}}(d)\}\).
    \end{lemma}

    \begin{proof}
        \begin{align*}
            &(\LU_{\{\sigma\}}(d))^{-1}=\{\phi^{-1}|\,\phi=W_L\circ \sigma\circ W_{L-1}\circ \cdots \circ \sigma\circ W_1,\, W_i\in \LU(d), i\in [L], L\in \mathbb{N}\}\\
            &=\{W_1^{-1}\circ \sigma^{-1}\circ W_2^{-1}\circ \cdots \circ \sigma^{-1}\circ W_L^{-1}\,|\,W_i\in \LU(d), i\in [L], L\in \mathbb{N}\}\\
            &=\{V_L\circ \sigma^{-1}\circ V_{L-1}\circ \cdots \circ \sigma^{-1}\circ V_1\,|\,V_i\in \LU^{-1}(d), i\in [L], L\in \mathbb{N}\}=\LU^{-1}_{\{\sigma^{-1}\}}(d)\;,
        \end{align*}
        where the third equality follows because as the \(W_i\) range over all elements of \(\LU(d)\), their inverses \(W_i^{-1}\) range over all elements of \(\LU^{-1}(d)\), and the indexing is arbitrary.
    \end{proof}

    \bigskip

    Using the expressivity result from Theorem~\ref{Theorem-Main2} we have all the tools to prove Theorem~\ref{Theorem-Main3}, i.e., that LU-Net with invertible leaky ReLU activations has the DUAP. This shows that suitable generative directions of LU-Net can transform an arbitrary absolutely continuous probability measure into any probability measure, with respect to weak convergence of measures.

    \bigskip
    
    \begin{proof}[Proof of Theorem \ref{Theorem-Main3}]
    Let $\alpha\in (0,1)\cup (1,\infty)$ and let $\sigma_{\alpha}$ be the corresponding LReLU. We show that $\NF(\LU_{\{\sigma_\alpha\}}(d))$ has the DUAP. By Lemma~\ref{Lemma-DUAP_inverse_set_classification}, this is equivalent to showing that $(\LU_{\{\sigma_{\alpha}\}}(d))^{-1}=\LU^{-1}_{\{\sigma_{1/\alpha}\}}(d)$ is a distributional universal approximator.
    
    \medskip
    \noindent\textbf{Case $d=1$:} By Lemma~\ref{Lemma-Single_Parameter_LReLU_approx_mon}, $\FNN_{\{\sigma_{1/\alpha}\}}^{(1)}(1,1)$ is a uniform universal approximator of $C^{0}_{\mon}(\mathbb{R},\mathbb{R})$ on compact sets and as $\mathcal{T}^{\infty}(1)\subset C^{0}_{\mon}(\mathbb{R},\mathbb{R})$, $\FNN_{\{\sigma_{1/\alpha}\}}^{(1)}(1,1)$ is a uniform universal approximator of $\mathcal{T}^{\infty}(1)$. Now, Lemma~\ref{Lemma-FNN_to_LU_compact} implies that also $\LU^{-1}_{\{\sigma_{1/\alpha}\}}(1)$ is an $L^p$ universal approximator of $\mathcal{T}^{\infty}(1)$ and Lemma~\ref{Lemma-T_infinity_approximators_universal} implies that $\LU^{-1}_{\{\sigma_{1/\alpha}\}}(1)$ is a distributional universal approximator.
    
    \medskip
    \noindent\textbf{Case $d\geq 2$:} By Theorem~\ref{Theorem-Main2} it follows that $\LU^{-1}_{\{\sigma_{1/\alpha}\}}(d)$ is an $L^p$ universal approximator of $L^p(\mathbb{R}^d,\mathbb{R}^d)$ on compact sets as \(\frac{1}{\alpha} \in (0,1)\cup (1,\infty)\) by choice of alpha and the second part of Lemma~\ref{Lemma-T_infinity_approximators_universal} implies that $\LU^{-1}_{\{\sigma_{1/\alpha}\}}(d)$ is a distributional universal approximator.
    
    \medskip
    By combining both cases we find that the set of generative directions, which is given by $(\LU_{\{\sigma_{\alpha}\}}(d))^{-1}=\LU^{-1}_{\{\sigma_{1/\alpha}\}}(d)$, is a distributional universal approximator for all $d\in \mathbb{N}$, which shows that the normalizing flow $\NF(\LU_{\{\sigma_{\alpha}\}}(d))$ has the DUAP.
    \end{proof}
    
    \bigskip
    
    The following proof relies on our Theorem~\ref{Theorem-LU_Squashable}, which is fundamentally based on Theorem~2 of \cite{shin2025minimumwidthuniversalapproximation}, i.e., the fact that FNNs with squashable activations and a width of $d$ are universal approximators of $L^p(\mathbb{R}^d,\mathbb{R}^d)$ on compact sets if $d\geq 2$.
    
    \bigskip
    
    \begin{proof}[Proof of Theorem \ref{Theorem-DUAP_squashable}]
        By Lemma~\ref{Lemma-DUAP_inverse_set_classification} the generative directions are given by $(\LU_{\{\sigma^{-1}\}}(d))^{-1}$$\\$$=\LU^{-1}_{\{\sigma\}}(d)$ and as $d\geq 2$, this is an $L^p$ universal approximator of $L^p(\mathbb{R}^{d},\mathbb{R}^{d})$ on compact sets by Theorem~\ref{Theorem-LU_Squashable}. Now, we can conclude from the second part of Lemma~\ref{Lemma-T_infinity_approximators_universal} that $\LU^{-1}_{\{\sigma\}}(d)$ is a distributional universal approximator. Consequently, $\NF(\LU_{\{\sigma^{-1}\}}(d))$ has the DUAP.
    \end{proof}

\subsection{Numerical LU-Net Experiments}\label{Section-LU_Net_numerical_results}

    In this subsection, we conduct numerical experiments with LU-Net \citep{chan2023lunet} that complement our theoretical results. Theorem~\ref{Theorem-Main3} establishes that LU-Net with a single LReLU parameter suffices to achieve the DUAP. While this implies that larger activation sets are theoretically unnecessary, the following experiments demonstrate that learnable LReLU slopes can improve optimization in practice.

In \Cref{fig:2d-toy-gaussian} we provide results on learning a 2D distribution. 
In \Cref{tab:nll-mnist}, we report numerical results of various LU-Net configurations used as density estimator on the MNIST \citep{deng2012mnist} and FashionMNIST \citep{xiao2017fashionmnist} image datasets.

The vanilla LU-Net architecture from \citet{chan2023lunet} uses a single fixed slope throughout the network, which by Theorem~\ref{Theorem-Main3} already suffices for distributional universal approximation. We compare this with a refined version that uses learnable slopes per component, to examine whether additional flexibility—though theoretically unnecessary—improves practical performance.

\begin{figure}[h!]
    \includegraphics[width=.31\textwidth]{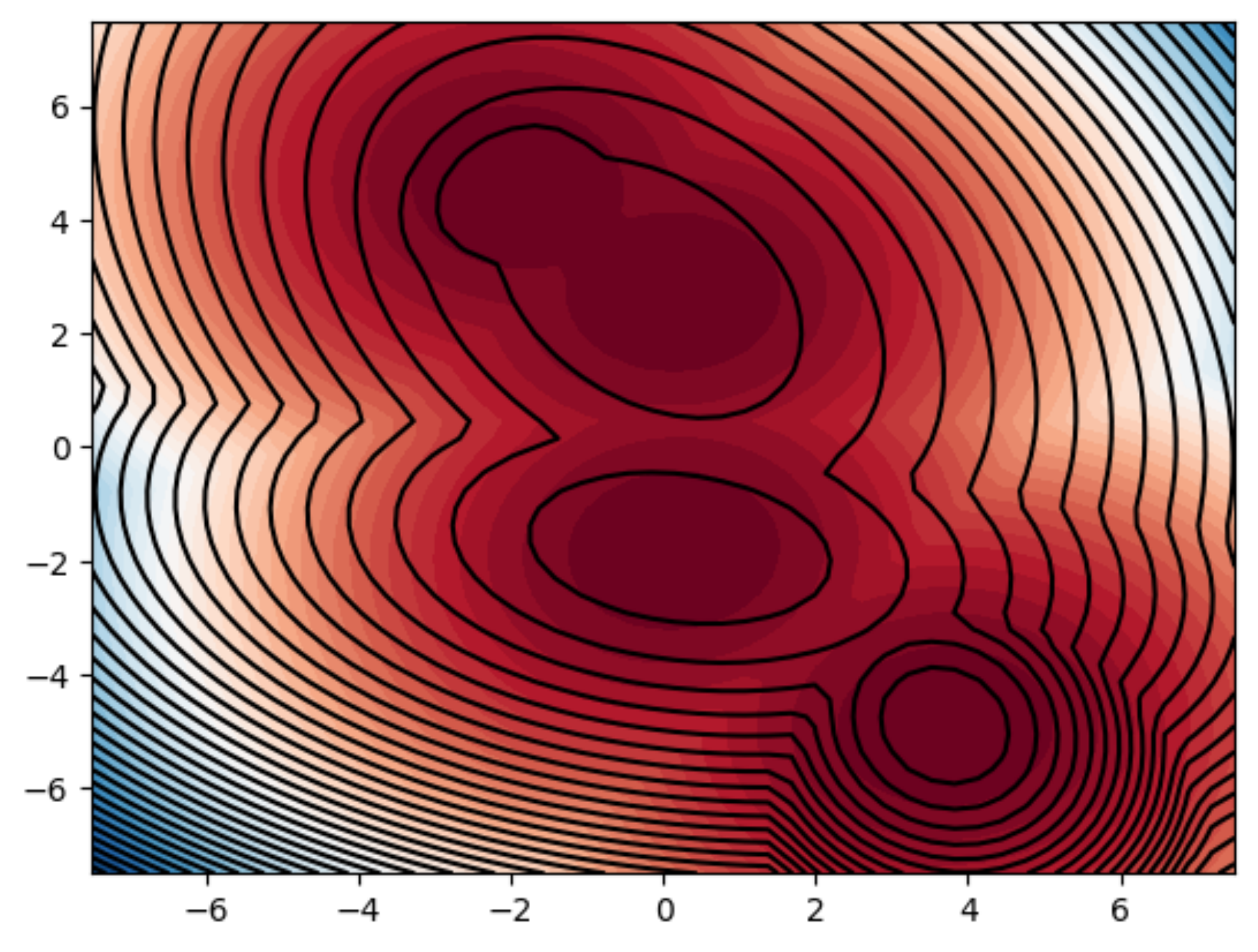}\hfill
    \includegraphics[width=.31\textwidth]{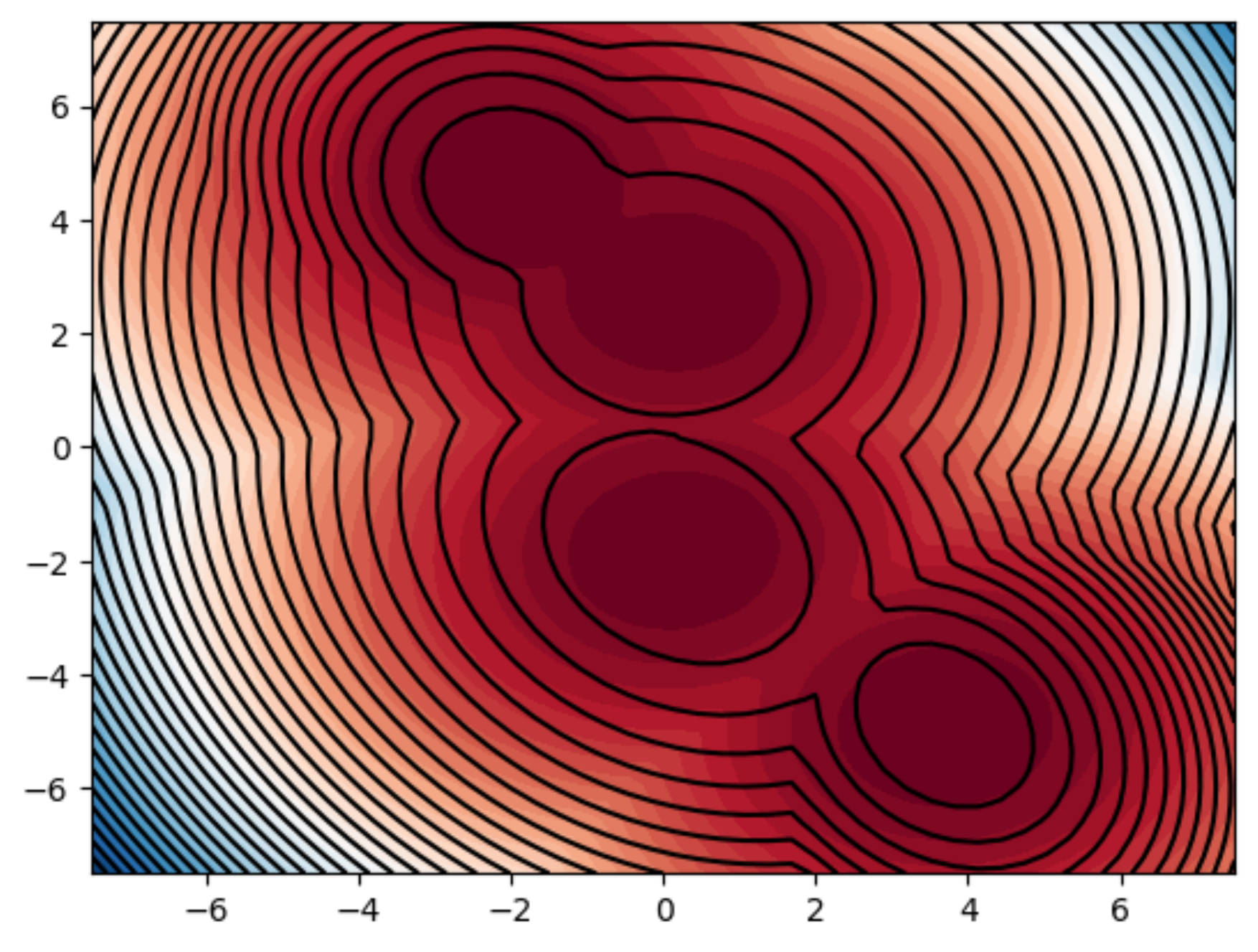}\hfill
    \includegraphics[width=.35\textwidth]{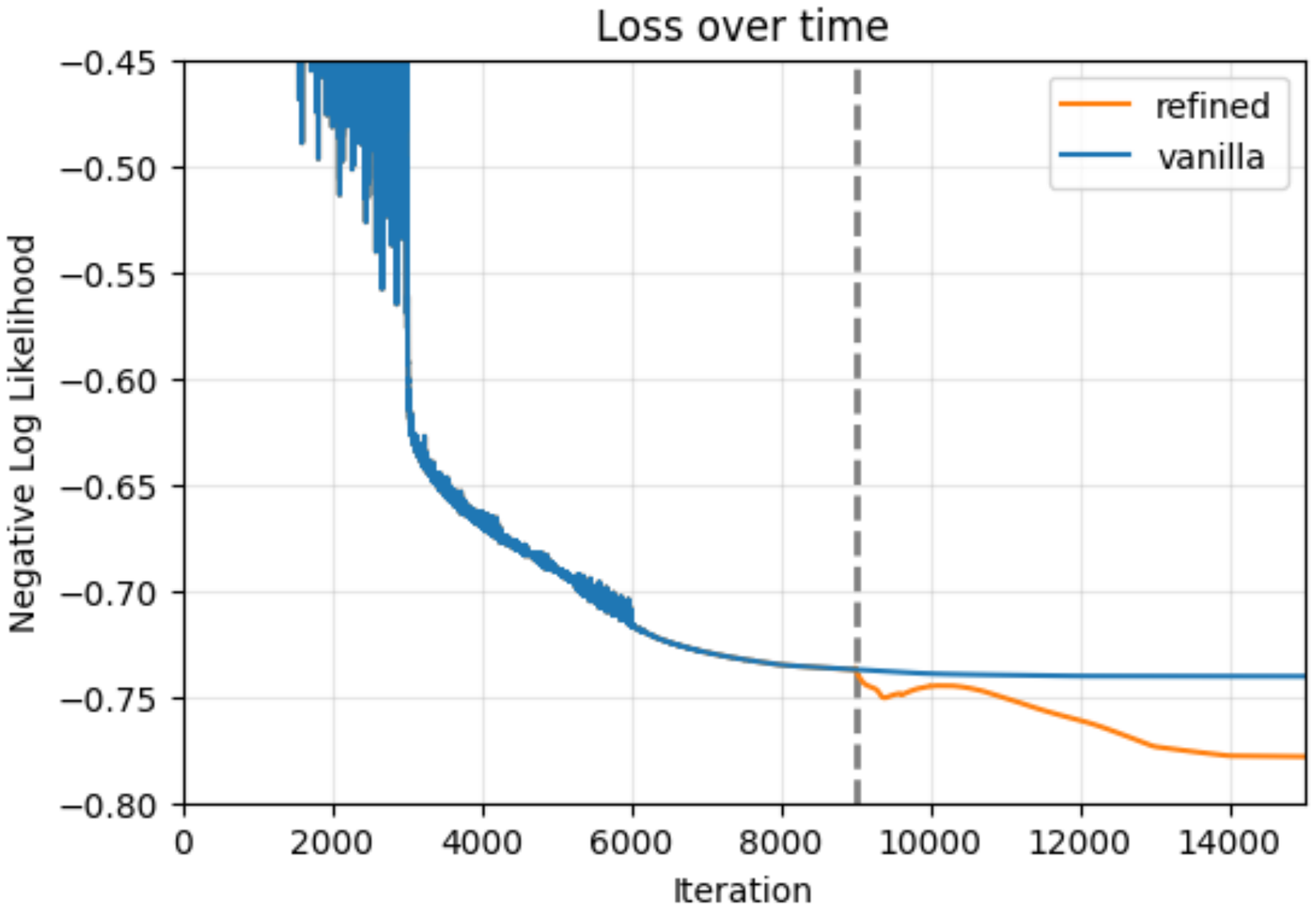}
    \caption{We train LU-Net on two-dimensional data. Starting with the original architecture as proposed in the original work \citep{chan2023lunet}, we then extend the model by making the slopes of the activation functions learnable parameters. Left: target distribution (heatmap) and learned distribution of the vanilla model (level curves). Middle: same target distribution (heatmap) and learned distribution of the refined model (level curves). Right: comparison of training loss over time.
    }\label{fig:2d-toy-gaussian}
\end{figure}

\begin{table}[h!]
\centering
\scalebox{0.8}{
\begin{tabular}{|c|c|c|c||c|c|c|c|}
    \hline
    \multicolumn{4}{|c||}{MNIST Test NLL in Bits / Pixel $\downarrow$} & \multicolumn{4}{|c|}{Fashion-MNIST Test NLL in Bits / Pixel $\downarrow$} \\
    \hline\hline
    Class & RealNVP & Vanilla  & Refined & Class & RealNVP & Vanilla  & Refined \\
    \hline
    Number 0 & 5.26 & 2.71 & \textbf{1.93} & T-Shirt    & 6.08 & 3.77 & \textbf{1.65} \\
    Number 1 & 5.08 & 2.47 & \textbf{1.75} & Trousers   & 5.51 & 7.25 & \textbf{1.62} \\
    Number 2 & 6.07 & 2.93 & \textbf{1.99} & Pullover   & 6.17 & 2.40 & \textbf{1.60} \\
    Number 3 & 5.01 & 2.74 & \textbf{1.95} & Dress      & 6.02 & 2.53 & \textbf{1.83} \\
    Number 4 & 5.09 & 2.77 & \textbf{1.93} & Coat       & 6.13 & 3.34 & \textbf{1.67} \\
    Number 5 & 5.42 & 2.74 & \textbf{2.00} & Sandal     & 5.93 & 2.88 & \textbf{2.09} \\
    Number 6 & 5.46 & 2.85 & \textbf{1.92} & Shirt      & 6.25 & 2.72 & \textbf{1.70} \\
    Number 7 & 5.77 & 2.79 & \textbf{1.90} & Sneaker    & 5.66 & 3.60 & \textbf{1.75} \\
    Number 8 & 5.40 & 2.79 & \textbf{1.96} & Bag        & 6.30 & 4.39 & \textbf{1.92} \\
    Number 9 & 5.36 & 2.65 & \textbf{1.89} & Ankle Boot & 5.87 & 4.44 & \textbf{1.80} \\
    \hline\hline
    Average & 5.37 & 2.75 & \textbf{1.93} & Average & 5.99 & 3.74 & \textbf{1.77} \\
    \hline 
\end{tabular}}
\caption{\small Class-wise negative log likelihood (NLL) when applying different LU-Net variants to MNIST and Fashion-MNIST test dataset. Moreover, the LU-Net variants are compared to \textbf{RealNVP} \cite{dinh2017density}. We report results of LU-Net in its original \textbf{vanilla} variant \citep{chan2023lunet} and \textbf{refined} variant with learnable slope of the leaky softplus activation per network component. Note that the NLL is reported in bits per pixel, which is the NLL with logarithm base 2 averaged over all pixels. Each LU-Net consists of three layers as also investigated in the original work from \cite{chan2023lunet}.}
\label{tab:nll-mnist}
\end{table}

Building on the vanilla model from \cite{chan2023lunet}, we consider the same LU-Net architecture but with additional learnable parameters that control the slopes of leaky softplus activations, which approximate leaky ReLU. In the refined model, each component within a layer has its own slope parameter.

Our experimental setup is as follows: we first train the vanilla model to near convergence. Then we (1) continue training under the vanilla configuration and (2) continue training under the refined configuration with flexible slopes. This allows us to evaluate the effect of introducing learnable slope parameters.

We find that introducing learnable slope parameters consistently enhances model performance. In the 2D experiments, the refined model achieves a lower negative log-likelihood than the vanilla version, yielding Kullback–Leibler divergences of $\mathrm{KL}(P|Q_v)=0.65$ where $P$ denotes the target distribution and $Q_v$ the distribution learned by the vanilla model, and $\mathrm{KL}(P|Q_r)=0.31$ where $Q_r$ is the distribution learned by the refined model. In addition, the LU-Net variant with learnable slope parameters outperforms the vanilla architecture on both MNIST and FashionMNIST datasets (see \Cref{tab:nll-mnist}), illustrating that the additional flexibility of learnable slopes provides optimization benefits beyond what the asymptotic expressivity theory captures.

 \section{Smooth diffemorphisms in form of \(LU\)-decomposable neural networks are universal approximators of \(L^p(\mathbb{R}^{d},\mathbb{R}^{d})\) }\label{Section-Proof_main4}

    In this section we prove Theorem \ref{Theorem-Main4}, i.e. that for any \(\alpha\in (0,1)\cup (1,\infty)\), smoothed LU-decomposable LReLU networks with parameter \(\alpha\) universally approximate \(L^p(\mathbb{R}^{d},\mathbb{R}^{d})\) on compact sets, by applying mollifiers to the approximating sequences obtained in Theorem \ref{Theorem-Main2} in a suitable manner.\\

        \begin{proof}[Proof of Theorem \ref{Theorem-Main4}]\label{Proof-Main4}
            \\\textbf{(1.)} First, we show that, for \(\alpha\in (0,1)\cup (1,\infty)\), \( \LU_{\mathcal{F}_\alpha^{\infty}}(d)\subset \mathcal{D}^\infty(\mathbb{R}^d,\mathbb{R}^d)\):  
            \\Let \((\sigma_\alpha^{(n)})_{n\in \mathbb{N}}:=(\sigma_\alpha * p_n)_{n\in \mathbb{N}}\subset \mathcal{F}_{\alpha}^{\infty} \) be the sequence of smoothed versions of \(\sigma_\alpha\), which are in \(C^{\infty}(\mathbb{R},\mathbb{R})\) by Lemma \ref{Lemma-Convolution}. Now, let \(\psi \in \LU_{\mathcal{F}_\alpha^{\infty}}(d)\), then it can be written as 
            \begin{align*}
                \psi:=W_{N}\circ\varphi_{\alpha,N-1} \circ W_{N-1}\circ ...\circ \varphi_{\alpha,1}\circ W_1\in \LU_{\mathcal{F}_\alpha^{\infty}}(d)
            \end{align*}with \(W_i\in \Aff_{\LU}(d), i\in [N], N\in \mathbb{N}\) and \(\varphi_{\alpha,k}=(\sigma_\alpha^{(m_{k,1})},\hdots,\sigma_{\alpha}^{(m_{k,d})})\) for all \(k\in [N-1]\) and \(m_{k,i}\in \mathbb{N}\) for all \(i\in [d]\). Then, we obtain for any \(\alpha\in (0,1)\cup (1,\infty)\), \(x\in \mathbb{R}\), that
            \begin{align*}
                &\frac{d}{dx}(\sigma_{\alpha} * p_n) (x)=\frac{d}{dx}\int_{B_{\frac{1}{n}}(0)}\sigma_{\alpha}(x-z)p_n(z)dz
                \\&=\lim\limits_{h\to 0}\frac{1}{h}\bigg(\int_{B_{\frac{1}{n}}(0)}\sigma_{\alpha}(x+h-z)p_n(z) dz - \int_{B_{\frac{1}{n}}(0)}\sigma_{\alpha}(x-z)p_n(z) dz \bigg)
                \\&=\lim\limits_{h\to 0}\frac{1}{h}\int_{B_{\frac{1}{n}}(0)\setminus\{x\}}(\sigma_{\alpha}(x+h-z)-\sigma_{\alpha}(x-z))p_n(z) dz
                \\&\stackrel{\text{dom. conv.}}{=}\int_{B_{\frac{1}{n}}(0)\setminus\{x\}}\frac{d}{dx}\sigma_{\alpha}(x-z)p_n(z)dz\geq \alpha \cdot \int_{B_{\frac{1}{n}}(0)\setminus\{x\}}p_n(z)dz
                = \alpha >0,
                \end{align*}where we used that \(\frac{d}{dx}\sigma_{\alpha}(x)\) is well-defined outside of zero with \(\frac{d}{dx}\sigma_{\alpha}(x)=\mathbbm{1}_{[0,\infty)}(x)+\alpha\mathbbm{1}_{(-\infty, 0)}(x)\geq \alpha \) and that the regularizing functions integrate to \(1\). Therefore, we know that functions of the form \(\sigma_\alpha^{(n)}=\sigma_{\alpha}*p_n\) are smooth and have a strictly positive derivative for all \(n\in \mathbb{N}\). Moreover, by the fundamental theorem of calculus, we obtain for \(a\leq x\) that
            \begin{align*}
                &\sigma_\alpha^{(n)}(x)=\sigma_\alpha^{(n)}(a)+\int_{a}^{x} (\sigma_\alpha^{(n)})'(z) dz 
                \geq \sigma_\alpha^{(n)}(a) + \alpha(x-a)\xlongrightarrow{x\to \infty}\infty\;,
            \end{align*}and an analogous inequality for \(a>x\) implies \(\lim\limits_{x\to -\infty} \sigma_\alpha^{(n)}(x)=-\infty\). By the intermediate value theorem, applicable because \(\sigma_\alpha^{(n)}, n\in \mathbb{N}\) is continuous, it follows that \(\sigma_\alpha^{(n)}(\mathbb{R})=\mathbb{R}\) and as it is monotonically increasing, it is a bijection from \(\mathbb{R}\) to \(\mathbb{R}\). Therefore, \(\varphi_{\alpha,k},k\in [N]\) are smooth bijections with globally invertible jacobians, as their jacobians are just diagonal matrices with non-zero elements (the derivatives of \(\sigma_\alpha^{(m_{k,i})}, i\in [d], k\in [N-1]\)) on the diagonal, hence, the Inverse Function Theorem 4.5 in \cite{lee2013smooth} implies that the inverse \((\varphi_{\alpha,k})^{-1}\) is also smooth, i.e. \(\varphi_{\alpha,k}\in \mathcal{D}^\infty(\mathbb{R}^d,\mathbb{R}^d), k\in [N-1]\). Additionally, for all \(i\in [N]\), \(W_i(x)=A_ix+b\), thus \((W_i)^{-1}(x)=A_i^{-1}x-A_i^{-1}b_i\) exists as \(A_i\in \LU(d)\subset \GL(d)\) and is smooth as it is again an affine transformation. Now, it follows that \(\psi:\mathbb{R}^d\rightarrow \mathbb{R}^d\) is a bijection, as a concatenation of bijections from \(\mathbb{R}^d\) to \(\mathbb{R}^d\) and that its inverse is given by 
            \begin{align*}
                \psi^{-1}=W_1^{-1}\circ \varphi_{\alpha,1}^{-1}\circ W_{2}\circ ...\circ \varphi_{\alpha,N-1}^{-1}\circ W_N^{-1}\;,
            \end{align*}which is smooth as a composition of smooth functions, i.e. \(\psi\in \mathcal{D}^\infty(\mathbb{R}^d,\mathbb{R}^d)\) and as it was chosen arbitrarily we have \(\LU_{\mathcal{F}_\alpha^{\infty}}(d)\subset \mathcal{D}^\infty(\mathbb{R}^d,\mathbb{R}^d)\).  
           \\~\\ \textbf{(2.)} Now, let \(\alpha \in (0,1)\cup (1,\infty)\), let \(\mathcal{K}\subset \mathbb{R}^d\) be a compact set and \(f\in L^p(\mathcal{K},\mathbb{R}^d)\). We show how a sequence of LU-decomposable \(\alpha\)-LReLU networks, converging uniformly to \(f\) on \(\mathcal{K}\), can be obtained: 
           \\By Theorem \ref{Theorem-Main2}  there exists a sequence \((\phi_n)_{n\in \mathbb{N}}\subset \LU_{\{\sigma_\alpha\}}(d)\) that satisfies 
            \begin{align}\label{Formular-Diffeos_dense_1}
                \lVert f-\phi_n \rVert_{\mathcal{K},p} \stackrel{n\to \infty}{\longrightarrow} 0 \;\;.
            \end{align}Thus for all \(n\in \mathbb{N}\) there exists \(N_n\in \mathbb{N}\) s.t.
            \begin{align*}
                \phi_n=W_{N_n}^{(n)}\circ \varphi_\alpha \circ W_{N_n-1}^{(n)}\circ ....\circ \varphi_\alpha\circ W_1^{(n)}\;,
            \end{align*}where \(W_i^{(n)}\in \Aff_{\LU}(d), i\in [N_n]\), \(\varphi_\alpha:\mathbb{R}^d\rightarrow \mathbb{R}^d\) and \(\varphi_\alpha(x)=(\sigma_\alpha(x_1),\hdots,\sigma_\alpha(x_d))^T\), i.e. \(\sigma_\alpha\) is applied component-wise. Define \(\varphi_\alpha^{(n)}:\mathbb{R}^d\rightarrow \mathbb{R}^d\) with \(\varphi_\alpha^{(n)}(x)=(\sigma_\alpha^{(n)}(x_1),\hdots,\sigma_\alpha^{(n)}(x_d))^T\), thus, \((\varphi_\alpha^{(n)})_{n\in \mathbb{N}} \subset  C^\infty(\mathbb{R}^d,\mathbb{R}^d)\) and we set \( \phi_{n,k}:=W_{N_n}^{(n)}\circ\varphi_\alpha ^{(k)}\circ W_{N_n-1}^{(n)}\circ ...\circ \varphi_\alpha^{(k)}\circ W_1^{(n)}\in \LU_{\mathcal{F}_\alpha^{\infty}}(d)\). By Lemma \ref{Lemma-Mollifier} we know that \(\sigma_\alpha^{(n)}\) converges uniformly to \(\sigma_\alpha\) on any compact set and consequently, \(\varphi_\alpha^{(n)}\) converges uniformly to \(\varphi_\alpha\). Therefore, as all composed functions are continuous, we can apply Lemma \ref{Lemma-Approximate_function_compositions_sup_unif_cont} for all \(n\in \mathbb{N}\) to obtain that 
            \begin{align}\label{Formular-Diffeos_dense_3}
                \lVert \phi_n -\phi_{n,k} \rVert_{\mathcal{K},\sup}\stackrel{k\to \infty}{\longrightarrow} 0 \;\; \forall n\in \mathbb{N}\;.
            \end{align}Now (\ref{Formular-Diffeos_dense_3}) enables us to obtain an increasing subsequence \((k_n)_{n\in \mathbb{N}}\subset \mathbb{N}\) with \(\lim \limits_{n\to \infty} k_n=\infty\) such that
            \begin{align}\label{Formular-Diffeos_dense_4}
                \lVert \phi_n - \phi_{n,k_n} \rVert_{\mathcal{K},p}\stackrel{n\to \infty} {\longrightarrow} 0 \;.
            \end{align}Therefore, we have found a sequence of LU-decomposable \(\alpha\)-LReLU networks \((\phi_{n,k_n})_{n\in \mathbb{N}}\subset \LU_{\mathcal{F}_\alpha^{\infty}}(d)\overset{(1.)}{\subset} \mathcal{D}^\infty(\mathbb{R}^d,\mathbb{R}^d)\), which consists of smooth diffeomorphisms and fulfills 
            \begin{align*}
                \lVert f-\phi_{n,k_n} \rVert_{\mathcal{K},p}\leq \lVert f-\phi_n \rVert_{\mathcal{K},p}+ \lVert \phi_{n}-\phi_{n,k_n} \rVert_{\mathcal{K},p}\stackrel{n\to \infty}{\longrightarrow} 0 
            \end{align*}because of (\ref{Formular-Diffeos_dense_1}) and (\ref{Formular-Diffeos_dense_4}). 
        \end{proof}

\section{\texorpdfstring{\(w_{\min}\geq \max\{d_x,d_y\}+1\)}{w\_min ≥ max(d\_x,d\_y)+1} for uniform universal approximation of continuous functions with FNNs using continuous monotone activations}
\label{Section-Proof_main5}

        The following two lemmas show some helpful properties with respect to the set \(\mathcal{D}^{0}(\mathbb{R}^d,\mathbb{R}^d)\) of homeomorphisms from \(\mathbb{R}^d\) to \(\mathbb{R}^d\), which are essential tools for the proof of Proposition \ref{Proposition-D0_not_supremum_dense_in_C0}. Recall from Remark \ref{Remark-D0_Inverse_continuous} that by the invariance of domain theorem, every continuous bijection \(\phi:\mathbb{R}^d\to\mathbb{R}^d\) is automatically a homeomorphism, i.e., \(\phi^{-1}\) is also continuous.

    \begin{lemma}\label{Lemma-Unbounded_values_image}
        Let \(\phi\in \mathcal{D}^0(\mathbb{R}^d,\mathbb{R}^d)\), \((x_n)_{n\in\mathbb{N}}\subset \mathbb{R}^d\) be a sequence, and define \((y_n)_{n\in \mathbb{N}}\subset \mathbb{R}^d\) by \(y_n:=\phi(x_n)\) for all \(n\in \mathbb{N}\). Then \((x_n)_{n\in\mathbb{N}}\) is unbounded if and only if \((y_n)_{n\in \mathbb{N}}\) is unbounded.
    \end{lemma}
    \begin{proof}
        By contraposition we can equivalently show that \((x_n)_{n\in\mathbb{N}}\) is bounded if and only if \((y_n)_{n\in\mathbb{N}}\) is bounded. Assume that \((y_n)_{n\in \mathbb{N}}\) is bounded, thus there exists \(r>0\) such that \((y_n)_{n\in \mathbb{N}}\subset \overline{B_r(0)}\). As \(\phi^{-1}\) is continuous and \(\overline{B_r(0)}\) is compact, the preimage \(\phi^{-1}(\overline{B_r(0)})\) is also compact, hence bounded. Since \(\phi^{-1}(\{y_n|\,n\in \mathbb{N}\})=\phi^{-1}(\{\phi(x_n)|\,n\in \mathbb{N}\})=\{x_n|\,n\in \mathbb{N}\} \subset \phi^{-1}(\overline{B_r(0)})\), it follows that \((x_n)_{n\in\mathbb{N}}\) is bounded. The other direction follows analogously by exchanging the roles of \(\phi, \phi^{-1}\) and \(x_n,y_n\), respectively.
    \end{proof}

    \begin{lemma}\label{Lemma-Diffeomorphism_level_set}
        Let \(\phi\in \mathcal{D}^0(\mathbb{R}^d,\mathbb{R}^d)\) and define the level set corresponding to \(y_1\in \mathbb{R}\) and the first component of \(\phi\) by \(\mathcal{L}_{y_1,\phi}:=\{x\in \mathbb{R}^d|\,\phi_1(x)=y_1\}\), where \(\phi_1\) denotes the first component of \(\phi\). Assume that \(\mathcal{L}_{y_1,\phi}\neq \emptyset\) for some \(y_1\in \mathbb{R}\). Then for any \(x_0\in \mathcal{L}_{y_1,\phi}\) there exists a continuous curve \(\gamma:[0,\infty)\rightarrow \mathcal{L}_{y_1,\phi}\) with \(\gamma(0)=x_0\) and \(\lim\limits_{t\to\infty}\lVert \gamma(t) \rVert =\infty\).
    \end{lemma}
    
    \begin{proof}
        By assumption it holds that \(\phi(x_0)=(y_1,\bar{y}_1)\) for some \(\bar{y}_1\in \mathbb{R}^{d-1}\). We define the curve \(\bar{\gamma}:[0,\infty)\rightarrow \mathbb{R}^{d}\) by \(\bar{\gamma}(t)=(y_1,\theta(t))\), where \(\theta:[0,\infty)\rightarrow \mathbb{R}^{d-1}\) is a continuous curve starting at \(\theta(0)=\bar{y}_1\) with \(\lim\limits_{t\to\infty}\lVert \theta(t) \rVert =\infty\), e.g. \(\theta(t):=\bar{y}_1+tw\) with \(w\in \mathbb{R}^{d-1}\setminus \{0\}\). As \(\phi^{-1}\) is continuous, \(\gamma:=\phi^{-1}\circ \bar{\gamma}\) is a continuous curve with \(\gamma(0)=\phi^{-1}(\bar{\gamma}(0))=\phi^{-1}(y_1,\bar{y}_1)=\phi^{-1}(\phi(x_0))=x_0\) and \(\phi(\gamma(t))=\bar{\gamma}(t)=(y_1,\theta(t))\), therefore \(\phi_1(\gamma(t))=y_1\) for all \(t\in[0,\infty)\), which shows \(\gamma:[0,\infty)\rightarrow\mathcal{L}_{y_1,\phi}\). By the choice of \(\theta\) it holds that \(\lim\limits_{t\to\infty}\lVert \bar{\gamma}(t) \rVert =\lim\limits_{t\to\infty}\lVert (y_1,\theta(t))\rVert=\infty\) and Lemma \ref{Lemma-Unbounded_values_image} implies that then also \(\gamma=\phi^{-1}\circ \bar{\gamma}\) must be unbounded, i.e. \(\lim\limits_{t\to\infty}\lVert \gamma(t) \rVert =\infty\).
    \end{proof}
    \bigskip
    The following proposition uses the fact that homeomorphisms do not alter the topology of sets and therefore any level set of any component of a homeomorphism must be unbounded. We use this to show that the set \(\mathcal{D}^0(\mathbb{R}^d,\mathbb{R}^d)\) does not universally approximate the continuous functions with respect to the supremum norm in dimension \(d\geq 2\). Therefore, it is essential to prove the corresponding Lemma \ref{Lemma-Diffeos_not_dense_in_C0}.

    \begin{proposition}\label{Proposition-D0_not_supremum_dense_in_C0}
       Let \(f\in C^0(\mathbb{R}^{d_x},\mathbb{R}^{d_y})\) with \(d_x\geq d_y\) such that there exists \(y_1\in \mathbb{R}\) with a compact level set \(\mathcal{L}_{y_1,f}\neq \emptyset\) of the first component, where \(\mathcal{L}_{y_1,f}:=\{x\in \mathbb{R}^{d_x}|\,f_1(x)=y_1\}\). (For any \(d_x \geq d_y \geq 1\), the proof of Lemma \ref{Lemma-Diffeos_not_dense_in_C0} below constructs such a function explicitly.) Let \(\mathcal{K}\subset \mathbb{R}^{d_x}\) be a compact set that fulfills \(\mathcal{L}_{y_1,f}\subset \interior(\mathcal{K})\) and \( \mathcal{L}_{y_1,f} \cap\partial \mathcal{K}  = \emptyset\). Then \(\epsilon:= \frac{1}{2}\inf_{x\in \partial \mathcal{K}}|f_1(x)-y_1|\) is positive and for all \(\phi \in \mathcal{D}^0(\mathbb{R}^{d_x},\mathbb{R}^{d_x})\) it holds that 
        \begin{align*}
            \lVert f_1 - \phi_1 \rVert_{\mathcal{K},\sup} \geq \epsilon\;.
        \end{align*}
    \end{proposition}
    \begin{proof}
         As \(\mathcal{K}\neq \emptyset\) is compact, \(\partial \mathcal{K}\) is also compact and non-empty. Therefore, the continuous function \(|f_1(x)-y_1|\) takes its minimum on \(\partial \mathcal{K}\) for some \(x^{*}\in \partial \mathcal{K}\). Because of \(\mathcal{L}_{y_1,f} \cap\partial \mathcal{K}  = \emptyset\) it follows that \(f_1(x)\neq y_1\) for all \(x\in\partial\mathcal{K}\), thus 
        \begin{align*}
            \epsilon=\frac{1}{2}\min_{x\in \partial \mathcal{K}}|f_1(x)-y_1|=\frac{1}{2}|f_1(x^{*})-y_1|>0\;.
        \end{align*}
        Now let \(x_0\in \mathcal{L}_{{y_1},f}\neq \emptyset\) and for an arbitrary fixed \(\phi \in \mathcal{D}^0(\mathbb{R}^{d_x},\mathbb{R}^{d_x})\) we set \(z:=\phi(x_0)\) and denote its first component by \(z_1:=\phi_1(x_0)\). Note that \(f_1(x_0) = y_1\) by definition of \(x_0 \in \mathcal{L}_{y_1,f}\). If \(|f_1(x_0) - z_1| = |y_1 - z_1| > \epsilon\) then the claim is true, otherwise \(|y_1 - z_1|\leq \epsilon\) holds and by Lemma \ref{Lemma-Diffeomorphism_level_set} there exists a continuous curve \(\gamma:[0,\infty)\rightarrow \mathcal{L}_{z_1,\phi}\) with \(\gamma(0)=x_0\) such that \(\lim\limits_{t\to \infty} \lVert \gamma(t) \rVert =\infty\), i.e. it holds that \(\phi_1(\gamma(t))=z_1\) for all \(t\in [0,\infty)\). We define \( t^{*}:=\inf\{t>0|\,\gamma(t) \not \in \interior(\mathcal{K})\}\) and proceed to show that \(t^{*}<\infty\) and \(\gamma(t^{*})\in \partial \mathcal{K}\), i.e. that the curve passes through \(\partial \mathcal{K}\). First, we note that \(\gamma(0)=x_0\in \mathcal{L}_{y_1,f}\subset \interior(\mathcal{K})\). Moreover, as \(\lim\limits_{t\to \infty} \lVert \gamma(t) \rVert =\infty\) and as \(\interior(\mathcal{K})\) is bounded as a subset of a compact set, there exists \(t_1>0\) such that \(\gamma(t_1)\not \in \interior(\mathcal{K})\), implying that \(t^{*}\leq t_1<\infty\). Thus there exists a sequence \((t_n)_{n\in\mathbb{N}}\) converging from below to \(t^{*}\) and it holds that \(\gamma(t_n)\in \interior(\mathcal{K})\) for all \(n\in \mathbb{N}\). By the continuity of \(\gamma\), the sequence \((\gamma(t_n))_{n\in \mathbb{N}}\subset \interior(\mathcal{K})\) converges to \(\gamma(t^{*})\not \in \interior(\mathcal{K})\), hence \(\gamma(t^{*})\in \partial \mathcal{K}\). Note that by definition of the curve, \(\phi_1(\gamma(t))=z_1\) for all \(t\in[0,\infty)\). Thus we obtain
        \begin{align*}
             &\lVert f_1 - \phi_1 \rVert_{\mathcal{K},\sup} \geq \sup_{x\in \partial\mathcal{K}}|f_1(x)-\phi_1(x)|
             \overset{\gamma(t^{*})\in \partial\mathcal{K}}{\geq} |f_1(\gamma(t^{*}))-\phi_1(\gamma(t^{*}))|
             \\&\geq |f_1(\gamma(t^{*}))-y_1|- |y_1-\phi_1(\gamma(t^{*}))|
             = |f_1(\gamma(t^{*}))-y_1|- |y_1-z_1| \geq 2\epsilon - \epsilon=\epsilon\;,
        \end{align*}
        which concludes the claim.
    \end{proof}

        \medskip
        Now, we are ready to prove Lemma \ref{Lemma-Diffeos_not_dense_in_C0}, which states that the set of continuous bijections with continuous inverse is not a uniform universal approximator of the continuous functions.

    \medskip
    \begin{proof}[Proof of Lemma \ref{Lemma-Diffeos_not_dense_in_C0}]
        We provide separate proofs for the cases \(d=1\) and \(d\geq 2\). While the proof for \(d\geq 2\) could be adapted to also cover the case \(d=1\), we give an elementary and more intuitive direct proof for the one-dimensional case.
        \\\textbf{(Case: \(d=1\)):} Let \(c>0\), \(\mathcal{K}\subset \mathbb{R}\) with non-empty interior, then there exists \([z-r,z+r]\subset \mathcal{K}\) with \(z\in \mathbb{R},r>0\). Choose \(f(x):=(\frac{\sqrt{2c}}{r}(x-z))^2\) which is in \(f\in C^0(\mathbb{R},\mathbb{R})\) and assume that there exists \(\phi\in \mathcal{D}^0(\mathbb{R},\mathbb{R})\) such that \(\lVert f-\phi \rVert_{\mathcal{K},\sup}<c\). Note that \(\phi\) must be strictly increasing or strictly decreasing as a continuous bijective function in dimension \(d=1\). Furthermore, it must hold that \(|\phi(z\pm r)-f(z\pm r)|=|\phi(z+\pm r)-2|<c\) and additionally \(|\phi(z)-f(z)|=|\phi(z)|<c\). If \(\phi\) is strictly increasing then \(\phi(z-r)<\phi(z)<c\) implies \(|\phi(z-r)-f(z-r)|=2c-\phi(z-r)> c\), which contradicts the assumption. If \(\phi\) is strictly decreasing then \(\phi(z+r)<\phi(z)=|\phi(z)-f(z)|<c\) implies \(|\phi(z+r)-f(z+r)|=2c-\phi(z+r)> c\), which also contradicts the assumption. Therefore it holds that \(\lVert f-\phi \rVert_{\mathcal{K},\sup} \geq \lVert f-\phi \rVert_{[z-r,z+r],\sup} > c\) for all \(\phi \in \mathcal{D}^0(\mathbb{R},\mathbb{R})\).
        
        ~\\~\textbf{(Case \(d\geq 2\)):} Let \(c>0\), \(\mathcal{K}\subset \mathbb{R}^d\) a compact set with non-empty interior, then there exist \(z\in \mathbb{R}^d\), \(r>0\) such that \(\overline{B_r(z)}\subset \mathcal{K}\). Let \(p:\mathbb{R}^d\rightarrow \mathbb{R}\), \(p(x):=(2\pi)^{-\frac{d}{2}} \cdot\exp(-\frac{1}{2}(x-z)^T(x-z))\) be the density of a multivariate standard normal distribution with identity covariance matrix \(I\) and mean \(z\). Then, for \(c_1:=(2\pi)^{-\frac{d}{2}}\) we have \(p(x)=c_1\) if and only if \(x=z\) and \(p(x)<c_1\) for all \(x\in \mathbb{R}^d\setminus\{z\}\) and \(p(z_r)=c_2<c_1\) for all \(z_r\in \partial\overline{B_r(z)}=\{x\in \mathbb{R}^d|\,\lVert z-x\rVert=r\}\) as the density is symmetric around \(z\). Furthermore, define \(f\in C^0(\mathbb{R}^{d},\mathbb{R}^{d})\) by \(f(x)=\left(\frac{2c}{c_1-c_2}p(x),x_2,...,x_d\right)\) and for \(c_3:=\frac{2cc_1}{c_1-c_2}\) it follows that \(\mathcal{L}_{c_3,f}=\{z\}\subset \interior(\overline{B_r(z)})\) and \(\mathcal{L}_{c_3,f}\cap \partial \overline{B_r(z)}=\emptyset\). Therefore, \(f\) fulfills the assumptions of Proposition \ref{Proposition-D0_not_supremum_dense_in_C0} with 
        \begin{align*}
        \epsilon = \frac{1}{2}\inf_{x\in \partial \overline{B_r(z)}}|f_1(x)-c_3| = \frac{1}{2}\left|\frac{2cc_2}{c_1-c_2}-\frac{2cc_1}{c_1-c_2}\right| = c\;,
        \end{align*}
        and thus, for all \(\phi \in \mathcal{D}^0(\mathbb{R}^d,\mathbb{R}^d)\) by Proposition \ref{Proposition-D0_not_supremum_dense_in_C0} it follows that 
        \begin{align*}
        \lVert f-\phi \rVert_{\mathcal{K},\sup} \geq \lVert f_1 - \phi_1 \rVert_{\mathcal{K},\sup} \geq \epsilon = c\;.
        \end{align*}
        Moreover, this shows that \(\mathcal{D}^0(\mathbb{R}^{d},\mathbb{R}^d)\) is not a universal approximator of \(C^0(\mathbb{R}^{d},\mathbb{R}^{d})\) w.r.t. the supremum norm.

    \end{proof}

        The following Lemma \ref{Lemma-Mon_cont_to_bijective} is essential for extending continuous monotone activations to bijections, which is necessary for proving Theorem \ref{Theorem-Main5}.

    \begin{lemma}\label{Lemma-Mon_cont_to_bijective}
        For any \(f\in C^{0}_{\mon}(\mathbb{R},\mathbb{R})\) there exists a sequence \((\varphi_n)_{n\in \mathbb{N}}\subset \mathcal{D}^0(\mathbb{R},\mathbb{R})\) such that for all \(\mathcal{K}\subset \mathbb{R}\) compact it holds that \(\lim \limits_{n\to \infty} \lVert f-\varphi_n \rVert_{\mathcal{K},\sup}=0\).
    \end{lemma}
    
    \begin{proof}
        For \(n\in \mathbb{N}\) we define \(\varphi_n:= f + \frac{x}{n}\) and assume w.l.o.g. that \(f\) is monotonically increasing (otherwise defining \(\varphi_n:= f - \frac{x}{n}\) works analogously), then \(\varphi_n\) is strictly monotonically increasing, continuous and bijective, thus \((\varphi_n)_{n\in \mathbb{N}}\subset \mathcal{D}^0(\mathbb{R},\mathbb{R})\). Moreover, we have
        \begin{align*}
            \lVert f-\varphi_n \rVert_{\mathcal{K},\sup} = \frac{1}{n}\lVert x \rVert_{\mathcal{K},\sup}= \frac{1}{n} (\max\{x|x\in \mathcal{K}\}-\min\{x|x\in \mathcal{K}\})\xlongrightarrow{n\to \infty} 0\;.
        \end{align*}
    \end{proof}
    
    With Proposition \ref{Proposition-D0_not_supremum_dense_in_C0} and Lemma \ref{Lemma-Mon_cont_to_bijective}, we can now prove Theorem \ref{Theorem-Main5}.

    \bigskip
    \begin{proof}[Proof of Theorem \ref{Theorem-Main5}]\label{Proof-Theorem-main5}
    \textbf{(Case \(d_x\geq d_y):\)} Let \(c>0\), \(\mathcal{K}\subset \mathbb{R}^{d_x}\) compact with non-empty interior, then by Lemma \ref{Lemma-Diffeos_not_dense_in_C0} there exists \(g\in C^0(\mathbb{R}^{d_x},\mathbb{R}^{d_x})\) such that for all \(\phi\in \mathcal{D}^0(\mathbb{R}^{d_x},\mathbb{R}^{d_x})\) it holds that \(\lVert g_1- \phi_1\rVert_{\mathcal{K},\sup} \geq c\). Define \(f:\mathbb{R}^{d_x}\rightarrow \mathbb{R}^{d_y}\), \(f(x):=(g_1(x),\hdots,g_{d_y}(x))^T\), where \(g_i\) denotes the \(i\)-th component of \(g\), then \(f\in C^0(\mathbb{R}^{d_x},\mathbb{R}^{d_y})\) and for all \(\phi\in \mathcal{D}^0(\mathbb{R}^{d_x},\mathbb{R}^{d_x})\) it also fulfills 
        \begin{align}\label{Formula-Proof_Theorem_5_1}
            \lVert f_1- \phi_1\rVert_{\mathcal{K},\sup} \geq c \;.
        \end{align}Now, let \(\psi\in \FNN_{\mathcal{A}}^{(d_x)}(d_x,d_y)\) of length \(L\in \mathbb{N}\), then, by filling up the last affine transformation of \(\psi\), which is in \(\Aff(d_{L-1},d_y)\), for some \(d_{L-1}\leq d_x\), with zeros to \(\Aff(d_{L-1},d_x)\) we can obtain \(\bar{\psi}\in \FNN_{\mathcal{A}}^{(d_x)}(d_x,d_x)\). Now, by Lemma \ref{Lemma-FNN_to_LU_compact}, as \(\mathcal{A}\) consists of continuous functions, there exists a sequence \((\Phi_{n})_{n\in \mathbb{N}}\subset \LU_\mathcal{A}(d_x)\) such that \(\lim \limits_{n\to\infty}\lVert \Phi_{n}-\bar{\psi} \rVert_{\mathcal{K},\sup}=0\). Moreover, by Lemma \ref{Lemma-Mon_cont_to_bijective}, for any activation \(\sigma\in \mathcal{A}\) we can find \((\sigma_k)_{k\in \mathbb{N}}\subset \mathcal{D}^0(\mathbb{R},\mathbb{R})\) that converge uniformly to \(\sigma\) on any compact set. Now, let \(\phi_n^{(k)} \in \LU_{\mathcal{D}^0(\mathbb{R},\mathbb{R})}(d_x)\) be the LU-decomposable network that we obtain when exchanging every activation \(\sigma\in \mathcal{A}\) in \(\Phi_n\) by the corresponding sequence element \(\sigma_k\in \mathcal{D}^0(\mathbb{R},\mathbb{R})\) for any \(k,n\in \mathbb{N}\). Then, by Lemma \ref{Lemma-Approximate_function_compositions_sup_unif_cont} it follows that \(\lim \limits_{k\to\infty}\lVert \phi_{n}^{(k)}-\Phi_n \rVert_{\mathcal{K},\sup}=0\). By choosing a suitable subsequence \((k_n)_{n\in \mathbb{N}}\) with \(\lim \limits_{n\to\infty}k_n=\infty\) and setting \(\phi_n:=\phi_n^{(k_n)}\) we can find \((\phi_n)_{n\in \mathbb{N}} \subset \LU_{\mathcal{D}^0(\mathbb{R},\mathbb{R})}(d_x)\subset \mathcal{D}^0(\mathbb{R}^{d_x},\mathbb{R}^{d_x})\) with \(\lim \limits_{n\to\infty}\lVert \phi_n- \Phi_n\rVert_{\mathcal{K},\sup} =0\), where the networks are bijections as they are compositions of bijective LU-decomposable affine transformations and of component-wise applications of bijective activations. We obtain 
        \begin{align*}
            &\lVert (\phi_n)_1 - \psi_1 \rVert_{\mathcal{K},\sup} \leq  \lVert (\phi_n)_1 - (\Phi_n)_1\rVert_{\mathcal{K},\sup} + \lVert (\Phi_n)_1- \psi_1\rVert_{\mathcal{K},\sup} 
            \\&\overset{\psi_1=\bar{\psi}_1}{=} \lVert (\phi_n)_1 - (\Phi_n)_1\rVert_{\mathcal{K},\sup} + \lVert (\Phi_n)_1- \bar{\psi}_1\rVert_{\mathcal{K},\sup} \xrightarrow{n\to\infty} 0\;.
        \end{align*}and by using this we can conclude 
        \begin{align*}
            &0<c \overset{(\ref{Formula-Proof_Theorem_5_1})}{\leq }\lVert (\phi_n)_1- f_1\rVert_{\mathcal{K},\sup} \leq \lVert (\phi_n)_1- \psi_1\rVert_{\mathcal{K},\sup} + \lVert \psi_1- f_1\rVert_{\mathcal{K},\sup}
            \\&\Longrightarrow c \leq \lim \limits_{n\to\infty} \left(\lVert (\phi_n)_1- \psi_1\rVert_{\mathcal{K},\sup}  + \lVert \psi_1- f_1\rVert_{\mathcal{K},\sup}\right)= \lVert \psi_1- f_1\rVert_{\mathcal{K},\sup}\leq \lVert \psi - f\rVert_{\mathcal{K},\sup}\;.
        \end{align*}This implies immediately that \(\FNN_{\mathcal{A}}^{(d_x)}(d_x,d_y)\) cannot be a universal approximator of the class \( C^0(\mathbb{R}^{d_x},\mathbb{R}^{d_y})\) w.r.t. the supremum norm on compact sets and therefore also not globally. Note that what we have shown in this case is even stronger: For any \(c>0\) and any \(\mathcal{K}\subset \mathbb{R}^{d_x}\) with non-empty interior there exists  \( f\in C^0(\mathbb{R}^{d_x},\mathbb{R}^{d_y})\) such that \(\lVert f -\phi \rVert_{\mathcal{K},\sup}\geq c\) for all \(\phi \in \FNN_{\mathcal{A}}^{(d_x)}(d_x,d_y)\).

        \bigskip

        \textbf{(Case \(d_x<d_y\)):} We lead the proof by contradiction. Suppose that \(\FNN_\mathcal{A}^{(d_y)}(d_x,d_y)\) is a universal approximator of \(C^0(\mathbb{R}^{d_x},\mathbb{R}^{d_y})\) on compact sets. Then for any \(f\in C^0(\mathbb{R}^{d_x},\mathbb{R}^{d_y})\) and any compact set \(\mathcal{K}\subset \mathbb{R}^{d_x}\) with non-empty interior, there exists a sequence \((\phi_n)_{n\in \mathbb{N}}\subset \FNN_\mathcal{A}^{(d_y)}(d_x,d_y)\) such that \(\lim\limits_{n\to\infty}\lVert f-\phi_n\rVert_{\mathcal{K},\sup}=0\). 
        
        As \(\mathcal{A}\subset C^0_{\mon}(\mathbb{R},\mathbb{R})\), we can apply Lemma \ref{Lemma-Relation_one_to_one_mon} to find, for each \(n\in\mathbb{N}\), a set \(\mathcal{I}_n\) of continuous injections such that for any \(\sigma\in \mathcal{A}\) there exists \(\sigma_n\in \mathcal{I}_n\) with \(\lim\limits_{n\to\infty}\lVert \sigma_n - \sigma \rVert_{\sup}=0\). Set \(\mathcal{I}:=\bigcup_{n\in\mathbb{N}}\mathcal{I}_n\). Then for any \(\phi_n \in \FNN_\mathcal{A}^{(d_y)}(d_x,d_y)\), by exchanging each of its activations \(\sigma\in\mathcal{A}\) with the corresponding convergent sequence \((\sigma_k)_{k\in \mathbb{N}}\) and approximating its affine transformations by sequences of invertible affine transformations converging uniformly on sufficiently large compact sets, Lemma \ref{Lemma-Approximate_function_compositions_sup_unif_cont} implies that there exists \((\phi_n^{(k)})_{k\in \mathbb{N}}\subset \FNN_\mathcal{I}^{(d_y)}(d_x,d_y)\) for all \(n\in \mathbb{N}\), which are injective and continuous as compositions of such functions, such that \(\lim\limits_{k\to\infty}\lVert \phi_n^{(k)}- \phi_n \rVert_{\mathcal{K},\sup}=0\). By extracting a suitable diagonal subsequence \((k_n)_{n\in \mathbb{N}}\) with \(\lim\limits_{n\to\infty} k_n=\infty\) such that \(\lim \limits_{n\to\infty}\lVert \phi_n -\phi_{n}^{(k_n)}\rVert_{\mathcal{K},\sup}=0\), our assumption implies that also
        \begin{align}\label{Formula-Injective_FNN_lower_bound_1}
            \lim \limits_{n\to\infty}\lVert f -\phi_{n}^{(k_n)}\rVert_{\mathcal{K},\sup} \leq \lim \limits_{n\to\infty}\left(\lVert f -\phi_n\rVert_{\mathcal{K},\sup} + \lVert \phi_n -\phi_{n}^{(k_n)}\rVert_{\mathcal{K},\sup}\right)=0\;.
        \end{align}
        
        Now, we note that the proof of Theorem 3 in \cite{kim2024minimumwidthuniversalapproximation} (in their Section 5) shows that one can construct \(f\in C^0([0,1]^{d_x},[0,1]^{d_y})\) such that 
        \begin{align}\label{Formula-Injective_FNN_lower_bound_2}
            \lVert f-\psi \rVert_{[0,1]^{d_x},\sup}\geq \frac{1}{3}\;,
        \end{align}
        for any continuous injection \(\psi:[0,1]^{d_x}\rightarrow [0,1]^{d_y}\). As \(\mathcal{K}\) has non-empty interior, there exists \(V\in \Aff_{\GL}(d_x)\) with \([0,1]^{d_x}\subset V(\mathcal{K})\). By the Tietze Extension Theorem there exists a continuous extension \(f^{*}:\mathbb{R}^{d_x}\rightarrow \mathbb{R}^{d_y}\) of \(f\). Set \(g^{*}:=f^{*}\circ V:\mathcal{K}\rightarrow \mathbb{R}^{d_y}\), then \(g^{*}\in C^0(\mathbb{R}^{d_x}, \mathbb{R}^{d_y})\). By \eqref{Formula-Injective_FNN_lower_bound_1}, there exists \((\psi_n)_{n\in \mathbb{N}}\subset \FNN_\mathcal{I}^{(d_y)}(d_x,d_y)\) such that \(\lim \limits_{n\to\infty}\lVert g^{*}-\psi_n \rVert_{\mathcal{K},\sup}=0\). Then \(\phi_n:=\psi_n\circ V^{-1}\) is again a continuous injection. For \(y \in [0,1]^{d_x}\), setting \(x = V^{-1}(y) \in \mathcal{K}\), we have
        \begin{align*}
            &\lim \limits_{n\to\infty}\lVert f -\phi_n\rVert_{[0,1]^{d_x},\sup} = \lim \limits_{n\to\infty}\sup_{y\in[0,1]^{d_x}}|f(y)-\phi_n(y)|
            = \lim \limits_{n\to\infty}\sup_{y\in[0,1]^{d_x}}|f(y)-\psi_n(V^{-1}(y))|
            \\&= \lim \limits_{n\to\infty}\sup_{x\in V^{-1}([0,1]^{d_x})}|f(V(x))-\psi_n(x)|
            = \lim \limits_{n\to\infty}\sup_{x\in V^{-1}([0,1]^{d_x})}|f^{*}(V(x))-\psi_n(x)|
            \\&= \lim \limits_{n\to\infty}\sup_{x\in V^{-1}([0,1]^{d_x})}|g^{*}(x)-\psi_n(x)|
            \leq \lim \limits_{n\to\infty}\lVert g^{*} -\psi_n\rVert_{\mathcal{K},\sup}=0\;.
        \end{align*}
        However, this contradicts \eqref{Formula-Injective_FNN_lower_bound_2}, and therefore \(\FNN_\mathcal{A}^{(d_y)}(d_x,d_y)\) cannot be a universal approximator of \(C^0(\mathbb{R}^{d_x},\mathbb{R}^{d_y})\) on compact sets.
    \end{proof}

\section{Conclusion and outlook}\label{Section-Conclusion_and_outlook}

         In this paper, we have established various universal approximation results for $L^p$ integrable functions and uniform approximation of continuous functions on compact sets. In Theorem \ref{Theorem-Main1}, we have derived several activation sets achieving the minimal widths $\max\{2,d_x,d_y\}$ with LReLUs and $\max\{d_x,d_y\}$ with G-LReLUs for $L^p$ approximation, providing alternative proofs to the results of \cite{cai2023achieve}. Theorem \ref{Theorem-Main_sup} has extended these results to S-LReLUs and FLOOR+LReLU networks at width $\max\{2,d_x,d_y\}$, and to SG-LReLUs and FLOOR+G-LReLUs at width $\max\{d_x,d_y\}$. We have further generalized these results to activations (including discontinuous and non-monotone functions) that can be trained in practice with standard gradient descent. For uniform approximation of continuous functions, we have shown that FNNs combining FLOOR, $\Id$, and a squashable activation achieve the minimal width $\max\{3,d_x,d_y\}$. On the other hand, we have proven in Theorem \ref{Theorem-Main5} that monotone and continuous activations cannot achieve uniform universal approximation of continuous functions with minimal width $\max\{d_x,d_y\}$, making it necessary to combine standard activations with discontinuities or non-monotone functions to reach the optimal minimal width.
        \medskip
        
        We have also extended our $L^p$ results to LU-decomposable networks and applied mollification to LReLU activations, showing that smoothed LReLU networks with fixed parameter $\alpha \in (0,1) \cup (1,\infty)$ are universal approximators of $L^p$ integrable functions on compact sets. Since these networks are diffeomorphisms, this directly implies Theorem 2.5(i) of Brenier and Gangbo \cite{Brenier_Gangbo}, demonstrating that neural networks can serve as theoretical tools for obtaining results in approximation theory. Furthermore, our generalization to LU-decomposable networks has established that LU-Nets with LReLU or bijective squashable activations possess the DUAP, i.e., they can approximate any probability distribution in the sense of weak convergence with appropriate depth. Since the class of bijective squashable functions is large, this raises the question of whether LU-Nets or INNs with such activations can achieve strong practical performance beyond their theoretical optimality. While most invertible squashable activations lack closed-form inverses, notable exceptions such as LReLU and its smooth approximations (investigated empirically in Section \ref{Section-Proof_main4}) have explicit inverses and could serve as practical bases for normalizing flows.
        \medskip
        
        Overall, the intensive theoretical study of narrow deep neural networks has yielded substantial insight into minimal width requirements for various activation functions. Squashable activations, ReLU-like variants, and specifically G-LReLUs have been shown to achieve the optimal minimal width for universal approximation of $L^p$ integrable functions in almost all dimensions, and with our results, we have extended this to uniform universal approximation of continuous functions by incorporating suitable discontinuous activations. In this sense, large classes of activations covering those most commonly used in practice \citep{Activation_functions,dubey2022activationfunctionsdeeplearning} are now known to be optimal in terms of minimal width (e.g., every non-affine analytic function is squashable, including non-monotone ones such as Swish). More broadly, however, universal approximation results remain predominantly theoretical in nature, and several open directions warrant further investigation. Whether networks with strictly bounded minimal widths can be made numerically scalable and achieve strong performance in practice remains an open question. Similarly, adapting activations with discontinuities to remain suitable for gradient-based training is challenging, as current constructions rely on FLOOR and STEP functions with high information loss or on trainable step sizes lacking gradient signals. Finally, most universal approximation proofs are either abstract or rely on constructions that, like the coding scheme of \cite{park2020minimum} used here, scale exponentially in depth as accuracy increases. Constructive proofs with more practical depth scaling would be valuable, though given the generality of the target function classes, such exponential scaling may be unavoidable.
        
\vskip 0.2in
\bibliography{JMLR_reviewed}

\end{document}